\documentclass[final]{siamltex}

\usepackage{amsmath}
\usepackage{amsfonts}
\usepackage{multirow}
\usepackage{lineno}
\usepackage{amssymb}
\usepackage{hyperref}
\usepackage{nameref}
\usepackage{cleveref}
\usepackage{mathrsfs}
\usepackage{graphicx}
\usepackage{epstopdf}
\usepackage{rotating}
\usepackage{color}
\usepackage{algorithm}
\usepackage{algorithmic}
\usepackage{mathrsfs}
\usepackage{cite}
\usepackage{subfigure}
\usepackage{longtable}
\usepackage{booktabs}
\usepackage{makecell}

\newcommand{\bv}{\operatorname{BV}}

\newcommand{\stvc}{\operatorname{\rm SVS-NLTV}}
\newcommand{\sbvc}{\operatorname{\rm SVS-NLBV}}
\newcommand{\stvct}{\operatorname{\rm SVS-NLTV-L2}}
\newcommand{\stvco}{\operatorname{\rm SVS-NLTV-L1}}

\graphicspath{{figs/}} 

\title{Color image restoration based on nonlocal saturation-value similarity}

\author{
	Wei Wang\thanks{The Corresponding Author. School of Mathematical Sciences, Key Laboratory of Intelligent Computing and Applications (Ministry of Education), Tongji University, Shanghai, China (wangw@tongji.edu.cn).
		W. Wang is supported by Natural Science Foundation of Shanghai (22ZR1465300).}
	\and
	Yakun Li\thanks{School of Mathematical Sciences, Key Laboratory of Intelligent Computing and Applications (Ministry of Education), Tongji University, Shanghai, China (2311740@tongji.edu.cn).}
}

\begin{document}
	\maketitle
	
	\begin{abstract} 
		In this paper, we propose and develop a novel nonlocal variational technique based on saturation-value similarity for color image restoration. In traditional nonlocal methods, image patches are extracted from red, green and blue channels of a color image directly, and the color information can not be described finely because the patch similarity is mainly based on the grayscale value of independent channel. The main aim of this paper is to propose and develop a novel nonlocal regularization method by considering the similarity of image patches in saturation-value channel of a color image. In particular, we first establish saturation-value similarity based nonlocal total variation by incorporating saturation-value similarity of color image patches into the proposed nonlocal gradients, which can describe the saturation and value similarity of two adjacent color image patches. The proposed nonlocal variational models are then formulated based on saturation-value similarity based nonlocal total variation. Moreover, we design an effective and efficient algorithm to solve the proposed optimization problem numerically by employing bregmanized operator splitting method, and we also study the convergence of the proposed algorithms.	 Numerical examples are presented to demonstrate that the performance of the proposed models is better than that of other testing methods in terms of visual quality and some quantitative metrics including peak signal-to-noise ratio (PSNR), structural similarity index (SSIM), quaternion structural similarity index (QSSIM) and S-CIELAB color error.  
		
	\end{abstract}
	
	
	
	\pagestyle{myheadings}
	\thispagestyle{plain}
	\markboth{}{Color image restoration based on $\stvc$}

 \section{Introduction}
	Image restoration can be formulated as an inverse problem.
The goal of image restoration is to find the unknown true image $\mathbf{u}$ from an observed degraded image $\mathbf{f}$. However, inverse problems are usually ill-posed, and it is standard to use a regularization technique to make them well-posed. In \cite{rudin1992nonlinear}, Rudin et al. proposed the classical total variation (TV) regularization which has become one of the most popular regularization methods in image processing, and has been developed into many other forms for handling corresponding image processing problems. For instance, anisotropic TV \cite{esedoḡlu2004decomposition} is originally designed for image decomposition problem, weighted TV \cite{coll2015half} improves the traditional TV method through a weighting mechanism to make it more adaptable and flexible in image recovery tasks, higher-order TV \cite{benning2013higher, bredies2010total, chan2000high, papafitsoros2014combined} uses higher order gradient information (like second order gradient) to improve the regularization effect and results in better preservation of image details, fourth-order PDE model \cite{lysaker2003noise, lysaker2004noise} introduces a fourth-order regularization term which helps to reduce artifacts. Liu et al. proposed a hybrid model combining the TV regularizer and the high-order TV regularizer with the L1 data fitting term in \cite{liu2014high}. TV regularization is also generalized for vector-valued (color or multichannel) image regularization. Blomgren and Chan proposed a synthetic measure of the image gradient for vector-valued images in \cite{blomgren1998color}. Bresson and Chan \cite{bresson2008fast} presented another color TV (CTV) regularization method based on local channel-coupling. Paul et al. \cite{rodriguez2009generalized} proposed the generalized vector-valued total variation (GVTV) by coupling different channels with different norms.

On the other hand, Buades et al. introduced an efficient technique for image restoration called nonlocal means (NL-means) filtering in \cite{buades2005review}. NL-means filtering is a nonlocal technique where the filtering weights are defined based on the similarity between the current image patch and the other patches in the image within a neighborhood. Based on the idea of NL-means filtering, Buades et al. considered patch similarity based on grayscale values and proposed a new regularization method for image denoising \cite{buades2006image}. Another nonlocal model for texture restoration is introduced in \cite{brox2007iterated}, where the similarity information is updated during each iteration. Inspired from the effectiveness of the graph Laplacian in \cite{chung1997spectral}, Gilboa and Osher proposed a nonlocal quadratic variational framework for image and signal regularization in \cite{gilboa2007nonlocal}. The nonlocal total variation (NLTV) regularization was then proposed in \cite{goldstein2009split}. The applications of NLTV include image restoration \cite{jidesh2018non, lou2010image, nie2016nonlocal, wang2016convex, Xiaoqun2010Bregmanized}, image inpainting \cite{li2018weighted, zhang2010wavelet}, image enhancement \cite{wang2014nonlocal}, etc. In \cite{wang2019structural}, Wang developed a nonlocal variational technique based on structural similarity, and established a nonlocal quadratic model and a nonlocal total variation model for image restoration. However, we remark here that all the above mentioned methods are proposed based on red-green-blue (RGB) color space for color image processing. The disadvantage of these approaches is that color information can not be described finely, and the patch similarity is mainly based on the grayscale similarity of independent channel. 

Different from the commonly used RGB color space, hue-saturation-value (HSV) color space is more closely related to the way humans perceive color, and is usually used for human visual perception \cite{gonzalez2009digital}. In \cite{jia2019color}, Jia, Ng and Wang utilized the representation of color images in quaternion framework and ultimately proposed a saturation-value total variation (SVTV) regularization model in HSV color space for color image restoration. The idea considers the coupling of different components and  makes use of neighborhood color pixel values in saturation and value components to control regularization in color image restoration. Therefore, color images are effectively processed by preserving the edges and the color information, and the unexpected chromatic intersection is significantly reduced. The applications of SVTV include color image restoration \cite{huang2021quaternion, wang2021color, jung2023saturation, jung2024group}, color image fusion \cite{wang2024color}, color image enhancement \cite{wang2022spatial} and color image segmentation \cite{wang2023two}, etc.

In this paper, we propose and develop a novel nonlocal variational technique based on saturation-value similarity for color image restoration. By considering the similarity of image patches in saturation-value channel of a color image, the novel saturation-value similarity based nonlocal total variation ($\stvc$) is proposed for color image regularization. Specifically, we first define nonlocal gradients in saturation-value space, which are able to describe the saturation and value similarity of two adjacent color image patches. We then define saturation-value similarity based nonlocal bounded variation function space ($\sbvc(\Omega)$) and study some properties of $\sbvc(\Omega)$.
The proposed nonlocal variational models are then formulated based on the novel nonlocal gradients and $\sbvc(\Omega)$ in saturation-value space.

The contribution of this paper is threefold. First, we incorporate saturation-value similarity of color image patches into the nonlocal weight of the proposed nonlocal gradients, and establish saturation-value similarity based nonlocal total variation. We then formulate the proposed color image restoration models by considering L2 fidelity ($\stvct$) and L1 fidelity ($\stvco$), which allows the proposed models to handle different types of noise, such as Gaussian noise, Poisson noise, etc.
Second, we design an effective and efficient algorithm to solve the proposed optimization problems numerically by employing bregmanized operator splitting method \cite{Xiaoqun2010Bregmanized, goldstein2009split}. Third, theoretically we define saturation-value similarity based nonlocal bounded variation function space ($\sbvc(\Omega)$) and study some properties of $\sbvc(\Omega)$,
meanwhile, we also study the convergence of the proposed algorithm. 
Numerical examples are presented to demonstrate that the performance of the proposed $\stvc$ regularization is better than that of other testing regularization methods in terms of visual quality and some criteria such as peak signal-to-noise ratio (PSNR), structural similarity index (SSIM) \cite{wang2004image}, quaternion structural similarity index (QSSIM) \cite{kolaman2011quaternion} 
 and S-CIELAB color error \cite{zhang1996spatial}.

The paper is organized as follows. In section 2, we first introduce saturation-value similarity and nonlocal gradients in saturation-value space, we then define $\sbvc(\Omega)$ and study some properties of $\sbvc(\Omega)$. In section 3, we present the proposed saturation-value similarity based nonlocal total variation and the proposed color image restoration models, meanwhile, we study the properties of the proposed regularization models. In section 4, we introduce the proposed algorithms to solve the proposed optimization problems. In section 5 we provide numerical results validating the effectiveness of the proposed methods.Finally, some concluding remarks are given in section 6.

\section{Saturation-value similarity and nonlocal saturation-value gradients}
\subsection{Saturation-value similarity}

HSV color space has been proven to be more compatible with human perception \cite{gonzalez2009digital}. By using operations on quaternions \cite{denis2007spatial}, the saturation and value components are given in the following formulas,
\begin{equation*}
c_s(x)=\frac{1}{2}|\textbf{u}(x)+\beta\textbf{u}(x)\beta|, \ \ 
c_v(x)=\frac{1}{2}|\textbf{u}(x)-\beta\textbf{u}(x)\beta|,
\end{equation*}
where $x$ is a pixel position in the image domain $\Omega$, $\beta = (i+j+k)/\sqrt{3}$ refers to the grey-value axis, and $\mathbf{u} = u_r i+u_g j+u_b k$ is a color image in the quaternion version. The saturation components $c_s$ is the distance between the color image $\mathbf{u}$ and the grey axis $\mu$. The value component $c_v$ represents the norm of the orthogonal projection of $\mathbf{u}$ on $\mu$. In \cite{jia2019color}, $c_s(x)$ and $c_v(x)$ are reformulated as follows,
\begin{gather*}
c_s(x)=\frac{1}{3}||C\textbf{u}(x)||_2, \quad c_v(x)=\frac{1}{\sqrt{3}}|u_r(x)+u_g(x)+u_b(x)|,
\end{gather*}
where
\begin{gather*}
\textbf{C}=\left[\begin{array}{ccc}
2 & -1  &  -1 \\
-1 &  2 &  -1\\
-1 & -1 &  2\\
\end{array}\right],\ \ \textbf{u}^T=\left[\begin{array}{c}
u_r \\
u_g \\
u_b \\
\end{array}\right].
\end{gather*}
Noting that matrix $\textbf{C}$ can be diagonalized as
\begin{equation}
\textbf{C}=\textbf{P}^T\left[\begin{array}{ccc}
3 & 0  &  0 \\
0 & 3 &  0\\
0 & 0 &  0\\
\end{array}\right] \textbf{P},\ \ 
\textbf{P}=\left[\begin{array}{ccc}
\frac{1}{\sqrt{2}}  & \frac{-1}{\sqrt{2}}  &  0 \\
\frac{1}{\sqrt{6}} &  \frac{1}{\sqrt{6}} &  \frac{-2}{\sqrt{6}}\\
\frac{1}{\sqrt{3}}  &  \frac{1}{\sqrt{3}} &  \frac{1}{\sqrt{3}}\\
\end{array}\right].
\label{eqn 2.1}
\end{equation}
Therefore, we define the saturation channel and the value channel of a color image as follows,
\begin{gather*}
\mathbf{u}_s = [\frac{1}{\sqrt{2}}u_r-\frac{1}{\sqrt{2}}u_g, \frac{1}{\sqrt{6}}u_r+\frac{1}{\sqrt{6}}u_g-\frac{2}{\sqrt{6}}u_b],\ \  
\mathbf{u}_v = \frac{1}{\sqrt{3}}u_r+\frac{1}{\sqrt{3}}u_g+\frac{1}{\sqrt{3}}u_b.
\end{gather*}
We then define $\omega_s$ and $\omega_v$ based on the proposed saturation channel and value channel 
to measure the saturation similarity and value similarity between two image patches $N_x$ and $N_y$,
\begin{gather*}
\begin{aligned}
\omega_s(x, y)=\exp \Big( -\frac{G_{\alpha}*||\mathbf{u}_s(x)-\mathbf{u}_s(y)||^2}{2h_0^2} \Big),\ \
\omega_v(x, y)=\exp \Big( -\frac{G_{\alpha}*|\mathbf{u}_v(x)-\mathbf{u}_v(y)|^2}{2h_0^2} \Big),  
\label{weight}
\end{aligned}
\end{gather*}
where $G_{\alpha}$ is a Gaussian kernel with standard deviation $\alpha$ and $h_0$ is a filtering parameter which corresponds to the noise level in general, $*$ is the convolution operator which is given as
\begin{eqnarray*}
G_{\alpha}*||\mathbf{u}_s(x)-\mathbf{u}_s(y)||^2 = \int_{U(0)} G_\alpha(t)||\mathbf{u}_s(x-t)-\mathbf{u}_s(y-t)||^2 dt.
\end{eqnarray*}

\subsection{Nonlocal saturation-value gradients}

Based on the definitions of saturation-value similarity measurements $\omega_s$ and $\omega_v$, we define nonlocal gradients in saturation-value space as follows,
\begin{gather*}
\nabla_{\omega}^s \mathbf{u}(x, y) = \Big( \big(\mathbf{u}_s^1(y)-\mathbf{u}_s^1(x) \big)\sqrt{\omega_s(x, y)},\ \big(\mathbf{u}_s^2(y)-\mathbf{u}_s^2(x)\big)\sqrt{\omega_s(x, y)}\Big),\\
\nabla_{\omega}^v \mathbf{u}(x, y) = \big(\mathbf{u}_v(y)-\mathbf{u}_v(x)\big)\sqrt{\omega_v(x, y)}.
\end{gather*}
We assume $\Omega$ to be a bounded open subset of $R^2$, and assume $p$, $p_1$, $p_2$: $\Omega \times \Omega \longrightarrow \mathbf{R}$ are functions, and $\mathbf{p}$, $\mathbf{p}_1$, $\mathbf{p}_2$: $\Omega \times \Omega \longrightarrow \mathbf{R}^2$ are vector valued functions defined in $\Omega \times \Omega$. In order to complete the calculation system related to nonlocal gradients in saturation-value space, we further define nonlocal inner product, nonlocal divergence and nonlocal Laplacian in saturation-value space as follows,
\begin{gather*}
\left \{
\begin{aligned}
<p_1, p_2>&=\int_{\Omega \times \Omega} p_1(x, y)p_2(x, y) dx dy,\\
(div_{\omega}^v p)(x)&=\int_{\Omega}\big(p(x, y)-p(y, x)\big)\sqrt{\omega_v }dy,\\
(\Delta_{\omega}^v \mathbf{u})(x)&= div_{\omega}^v(\nabla_{\omega}^v \mathbf{u})(x)=2 \int_{\Omega}\big(\mathbf{u}_v(y)-\mathbf{u}_v(x)\big) \omega_v  dy.\\
\end{aligned}
\right.
\end{gather*}
\begin{gather*}
\left \{
\begin{aligned}
<\textbf{p}_1, \textbf{p}_2>&=\big(\int_{\Omega \times \Omega} \textbf{p}^1_{1}(x, y)\textbf{p}^1_{2}(x, y) dx dy, \int_{\Omega \times \Omega} \textbf{p}^2_{1}(x, y)\textbf{p}^2_{2}(x, y) dx dy\big),\\
(div_{\omega}^{s} \textbf{p})(x)&=\Big(\int_{\Omega}\big(\textbf{p}^1(x, y)-\textbf{p}^1(y, x)\big)\sqrt{\omega_s}dy, \int_{\Omega}\big(\textbf{p}^2(x, y)-\textbf{p}^2(y, x)\big)\sqrt{\omega_s}dy\Big),\\
(\Delta_{\omega}^s \mathbf{u})(x)&= div_{\omega}^s(\nabla_{\omega}^s \mathbf{u})(x) = \Big(2 \int_{\Omega}\big(\mathbf{u}_s^1(y)-\mathbf{u}_s^1(x)\big) \omega_s dy, 2 \int_{\Omega}\big(\mathbf{u}_s^2(y)-\mathbf{u}_s^2(x)\big) \omega_s dy\Big).
\end{aligned}
\right.
\end{gather*}
Then we establish the following properties for the above operators.
\begin{proposition}
Assume $\mathbf{u}\in W^{1, 1}(\Omega)$,  $p\in C^1(\Omega \times \Omega)$, $\mathbf{p} = (p_1, p_2)$, $p_1, p_2 \in C^1(\Omega \times \Omega)$, $\omega_v, \omega_s\in C^1(\Omega \times \Omega)$, then the following formulas hold in saturation-value space,
\begin{gather*}
\left \{
\begin{aligned}
&<\nabla_{\omega}^v \mathbf{u}, p >=-<\mathbf{u}_v, div_{\omega}^v p>,\\
&<\nabla_{\omega}^v \mathbf{u}, \nabla_{\omega}^v \mathbf{u}>=-<\Delta_{\omega}^v \mathbf{u}, \mathbf{u}_v>,\\
&<\Delta_{\omega}^v \mathbf{u}, \mathbf{u}_v>=<\mathbf{u}_v, \Delta_{\omega}^v \mathbf{u}>,\\
&\int_{\Omega}(div_{\omega}^v p )dx=0,
\end{aligned}    
\right.\ \
\left \{
\begin{aligned}
&<\nabla_{\omega}^s \mathbf{u}, \mathbf{p}> = -<\mathbf{u}_v, div_{\omega}^s \mathbf{p}>,\\
&<\nabla_{\omega}^s \mathbf{u}, \nabla_{\omega}^s \mathbf{u}> = -<\Delta_{\omega}^s \mathbf{u}, \mathbf{u}_s>,\\
&<\Delta_{\omega}^s \mathbf{u}, \mathbf{u}_s> = <\mathbf{u}_s, \Delta_{\omega}^s \mathbf{u}>,\\
&\int_{\Omega}(div_{\omega}^s \mathbf{p})dx = 0.
\end{aligned}    
\right.
\end{gather*}
\label{pro2.1}
\end{proposition}

\begin{proof} Noting that the formulas corresponding to saturation and value have similar structure, thus we focus on the value part in the following proof.
\begin{gather*}
\begin{aligned}
     &<\nabla_{\omega}^v \mathbf{u}_v, p> = \int_{\Omega \times \Omega}\big(\mathbf{u}_v(y)-\mathbf{u}_v(x)\big)\sqrt{\omega_v} p (x, y)dy dx\\
     &=\int_{\Omega \times \Omega}\mathbf{u}_v(y)\sqrt{\omega_v} \cdot p(x, y)dy dx-\int_{\Omega \times \Omega}\mathbf{u}_v(x)\sqrt{\omega_v} \cdot p(x, y)dy dx\\
     &=\int_{\Omega \times \Omega}\mathbf{u}_v(x)\sqrt{\omega_v} \cdot p(y, x)dx dy-\int_{\Omega \times \Omega}\mathbf{u}_v(x)\sqrt{\omega_v} \cdot p(x, y)dy dx\\
     &=\int_{\Omega}\mathbf{u}_v(x)\int_{\Omega }\sqrt{\omega_v} \cdot p(y, x)-\sqrt{\omega_v} \cdot p(x, y)dydx\\
     &=-\int_{\Omega}\mathbf{u}_v(x) (div_{\omega}^v p)(x)dx\\
     &=-<\textbf{u}_v, div_{\omega}^v p>.
\end{aligned}
\end{gather*}
Therefore,
\begin{gather*}
\begin{aligned}
       & <\nabla_{\omega}^v \mathbf{u}, \nabla^v_{\omega} \mathbf{u}> = -<div_{\omega}^v (\nabla^v_{\omega} \mathbf{u}), \mathbf{u}_v> = -<\Delta_{\omega}^v \mathbf{u}, \mathbf{u}_v>,\\
       &\int_{\Omega}(div_{\omega}^v p)(x) dx = <1, div_{\omega}^v p> = -<\nabla^v_{\omega} \textbf{1}, p> = 0.
\end{aligned}
\end{gather*}

Finally, the following formula hold by using the commutativity of the inner product directly,
$$<\Delta_{\omega}^v \mathbf{u}, \mathbf{u}_v>=<\mathbf{u}_v, \Delta_{\omega}^v \mathbf{u}>.$$ 
\end{proof}

\subsection{Saturation-value similarity based nonlocal bounded variation space}

In this section, we will define saturation-value similarity based nonlocal bounded variation function space ($\sbvc(\Omega)$) and study some properties of $\sbvc(\Omega)$.
We first give the norm of the nonlocal saturation-value gradients,
\begin{gather*}
\begin{aligned}
\bigl|\nabla_{\omega}^s \mathbf{u}\bigr| &= \sqrt{\int_{\Omega}  \Big(\big(\mathbf{u}_s^1(x)-\mathbf{u}_s^1(y)\big)^2 +\big(\mathbf{u}_s^2(x)-\mathbf{u}_s^2(y)\big)^2 \Big) \omega_s \mathrm d y },\\
\bigl|\nabla_{\omega}^v \mathbf{u}\bigr|&= \sqrt{\int_{\Omega} \big(\mathbf{u}_v(x)-\mathbf{u}_v(y) \big)^2 \omega_v \mathrm d y}.
\end{aligned}
\end{gather*}
We define the semi-norm
\begin{gather*}
|\mathbf{u}|_{\sbvc}=\int_{\Omega} |\nabla_{\omega}^s \mathbf{u}| \mathrm d x\,+ \int_{\Omega} |\nabla_{\omega}^v \mathbf{u}| \mathrm d x,
\end{gather*}
and the norm of $\sbvc(\Omega)$ as
\begin{gather*}
||\mathbf{u}||_{\sbvc}=||\mathbf{u}||_{L_1}+|\mathbf{u}|_{\sbvc}. 
\end{gather*}
$\sbvc(\Omega)$ is then defined as

\begin{gather*}
\sbvc(\Omega)=\{ \mathbf{u} \in L^1(\Omega): \int_{\Omega} |\nabla_{\omega}^s \mathbf{u}|\mathrm d x +\int_{\Omega} |\nabla_{\omega}^v \mathbf{u} | \mathrm d x < \infty \},
\end{gather*}
and some properties such as lower semi-continuity, approximation, and compactness are given as follows,

\begin{proposition}[Lower semicontinuity] Let \(\mathbf{u}^n \in \sbvc (\Omega)\) and \( \mathbf{u}^n \xrightarrow{L^1(\Omega)} \mathbf{u}\).  Then
\begin{gather*}
\liminf_{n\to\infty}\int_\Omega |\nabla_{\omega}^s \mathbf{u}^n|\mathrm d x
\;\ge\;
\int_\Omega |\nabla_{\omega}^s \mathbf{u}|\mathrm d x, \ \ \liminf_{n\to\infty}\int_\Omega |\nabla_{\omega}^v \mathbf{u}^n|\mathrm d x
\;\ge\;
\int_\Omega |\nabla_{\omega}^v \mathbf{u}|\mathrm d x.
\end{gather*}
\label{prole}
\end{proposition}

\begin{proposition}[Approximation]\label{proapp}
For a bounded set $\Omega$, any \(\mathbf{u}\in \sbvc (\Omega)\), there exists a sequence \(\{\mathbf{u}^\varepsilon\}\subset W^{1, 1}(\Omega)\cap C^\infty(\Omega)\) such that
\begin{gather*}
\lim_{\varepsilon\to0}\int_\Omega|\mathbf{u}^\varepsilon - \mathbf{u}|\,\mathrm d x= 0,
\lim_{\varepsilon\to0}\int_\Omega |\nabla_{\omega}^s \mathbf{u}^\varepsilon|\mathrm d x
=\int_\Omega |\nabla_{\omega}^s \mathbf{u}|\mathrm d x,
\lim_{\varepsilon\to0}\int_\Omega |\nabla_{\omega}^v \mathbf{u}^\varepsilon|\mathrm d x
=\int_\Omega |\nabla_{\omega}^v \mathbf{u}|\mathrm d x.
\end{gather*}
\end{proposition}

\begin{proposition}[Compactness]\label{procom}
Let $\Omega$ be a bounded subset of $\mathbb R^2$. Assume $\{\mathbf{u}^n\}$ is uniformly bounded in $\sbvc (\Omega)$, and $\omega_s(x, y)$, $\omega_v(x, y)$ have a lower bound N, then there exists a subsequence (still denoted as $\{\mathbf{u}^n\}$) and a limit $\mathbf{u}\in \sbvc (\Omega)$ such that
\begin{gather*}
\mathbf{u}^n \;\longrightarrow\; \mathbf{u}
\quad\text{in }L^1(\Omega).
\end{gather*}
\end{proposition}
We give the proof of the above properties in the Appendix \ref{appendix}.

\section{SVS-NLTV and the proposed color image restoration models}

\subsection{SVS-NLTV}

Based on the formulation of the nonlocal gradients in saturation-value space, we propose saturation-value similarity based nonlocal total variation ($\stvc$) as follows,
\begin{gather*}
\begin{aligned}
&\stvc(\mathbf{u}) = \int_{\Omega}|\nabla_{\omega}^s \mathbf{u}|dx+\mu \int_{\Omega}|\nabla_{\omega}^v \mathbf{u}|dx,
\end{aligned}
\end{gather*}
where $\mu$ is a parameter which is designed to balance the regularization of the saturation part and the value part.
For the proposed $\stvc$ regularization, we have the following properties.
\begin{proposition} Assume $\mathbf{u} = [u_r, u_g, u_b]$ is differentiable, and let $K^m=C^1(\Omega, B^{2m})$ be the set of continuously differentiable and bounded functions from the compact support in $\Omega$ to $B^{2m}$, then $\stvc(\mathbf{u})$ is given by the following dual form,
\begin{gather*}
\begin{aligned}
&\sup_{(\epsilon_1, \epsilon_2) \in K^2, \epsilon_3 \in K^1}
     \left\{ \int_{\Omega }\frac{1}{\sqrt{2}}(u_r-u_g) div_{\omega}^s(\epsilon_1)+\frac{1}{\sqrt{6}}(u_r+u_g-2u_b)div_{\omega}^s(\epsilon_2)\right.\\
     &\hspace{70mm}\left.+\frac{\mu}{\sqrt{3}} \int_{\Omega}(u_r+u_g+u_b)div_{\omega}^v(\epsilon_3) dx
    \right\}.
\end{aligned}
\end{gather*}
\label{pro3.1}
\end{proposition}

\begin{proof} We first set
\begin{gather*}
\begin{aligned}
\left[\begin{array}{ccc}
\textbf{u}_s^1(x)  \\
\textbf{u}_s^2(x)  \\
\textbf{u}_v(x) \\
\end{array} \right]=\textbf{P}\left[\begin{array}{ccc}
u_r(x)  \\
u_g(x)  \\
u_b(x) \\
\end{array} \right],
\end{aligned}
\end{gather*}
where $\textbf{P}$ is given as in ($\ref{eqn 2.1}$). We set $\mathbf{u}_s(x) = [\mathbf{u}_s^1(x), \mathbf{u}_s^2(x)]$, $\epsilon_{s}(x) = [\epsilon_1(x), \epsilon_2(x)]$, then we have the following result by using Proposition \ref{pro2.1}, 
\begin{gather*}
\begin{aligned}
&\sup_{(\epsilon_1, \epsilon_2) \in K^2, \epsilon_3 \in K^1}
     \left\{ \int_{\Omega }\frac{1}{\sqrt{2}}(u_r-u_g) div_{\omega}^s(\epsilon_1)+\frac{1}{\sqrt{6}}(u_r+u_g-2u_b)div_{\omega}^s(\epsilon_2)\right.\\
     &\hspace{70mm}\left.+\frac{\mu}{\sqrt{3}} \int_{\Omega}(u_r+u_g+u_b)div_{\omega}^v(\epsilon_3) dx \right\}\\
=&\sup_{(\epsilon_1, \epsilon_2) \in K^2, \epsilon_3 \in K^1}\left\{ \int_{\Omega}\mathbf{u}_s^1 (div_{\omega}^s \epsilon_1)+\mathbf{u}_s^2 (div_{\omega}^s\epsilon_2) dx +\mu\int_{\Omega}\mathbf{u}_v div_{\omega}^v(\epsilon_3)dx\right\}\\
=&\sup_{(\epsilon_1, \epsilon_2) \in K^2, \epsilon_3 \in K^1}\left\{ \int_{\Omega}\mathbf{u}_s div_{\omega}^s (\epsilon_{s}) dx +\mu\int_{\Omega}\mathbf{u}_v div_{\omega}^v(\epsilon_3)dx\right\}\\
=&\sup_{(\epsilon_1, \epsilon_2) \in K^2, ||(\epsilon_1,\epsilon_2)||<1} \left\{ \int_{\Omega}\nabla^s_{\omega} \mathbf{u}(x) \epsilon_s dx\right\}  +\mu \sup_{(\epsilon_3) \in K^1, ||\epsilon_3||<1} \left\{\int_{\Omega} \nabla^v_{\omega} \mathbf{u}(x) \epsilon_3  dx\right\}\\
=&<\nabla^s_{\omega}\mathbf{u}, \frac{\nabla^s_{\omega}\mathbf{u}}{|\nabla^s_{\omega}\mathbf{u}|}>+\mu <\nabla^v_{\omega}\mathbf{u}, \frac{\nabla^v_{\omega}\mathbf{u}}{|\nabla^v_{\omega}\mathbf{u}|}> = \stvc(\mathbf{u}).
\end{aligned}
\end{gather*}
\end{proof}

\begin{proposition} Let $\textbf{q}=[q_1, q_2, q_3]^T=\textbf{P}[u_r, u_g, u_b]^T$, $\textbf{q}_{s}=[q_1, q_2]^T$, where $\textbf{P}$ is an orthogonal matrix defined in \ref{eqn 2.1}. Then $\stvc(\mathbf{u})$ can be written into the following equivalent form,
\begin{gather*}
\int_{\Omega}\sqrt{\int_{\Omega}   ||\textbf{q}_{s}(x)-\textbf{q}_{s}(y)||^2  \omega_{s} dy }dx+\mu \int_{\Omega}\sqrt{\int_{\Omega} \big(q_3(x)-q_3(y)\big)^2 \omega_v dy}dx
\end{gather*}
\label{pro3.2}
\end{proposition}
\begin{proof} 
First, we have the second term of $\stvc(\mathbf{u})$
\begin{gather*}
    \begin{aligned}
        &\mu \int_{\Omega}\sqrt{\int_{\Omega} \big(q_3(x)-q_3(y)\big)^2 \omega_v dy}dx\\
        =&\frac{\mu}{\sqrt{3}}\int_{\Omega}\sqrt{\int_{\Omega} |u_r(x)-u_r(y)+u_g(x)-u_g(y)+u_b(x)-u_b(y)|^2 \omega_v dy}dx\\
        =&\mu \int_{\Omega}\sqrt{\int_{\Omega}\big(\mathbf{u}_v(x)-\mathbf{u}_v(y)\big)^2  \omega_v dy}dx,\\
    \end{aligned}
\end{gather*}
Similarly, we have the first term of $\stvc(\mathbf{u})$
\begin{gather*}
    \begin{aligned}
        &\int_{\Omega}\sqrt{\int_{\Omega}   ||\textbf{q}_{s}(x)-\textbf{q}_{s}(y)||^2  \omega_{s} dy }dx\\
        =& \int_{\Omega}\sqrt{\int_{\Omega}   \big(q_1(x)-q_1(y)\big)^2 \omega_{s} +\big(q_2(x)-q_2(y)\big)^2  \omega_{s} dy }dx\\
    \end{aligned}
\end{gather*}
Noting that
   $\textbf{C}=\textbf{P}^T\left[\begin{array}{ccc}
3 & 0 & 0  \\
0 & 3 & 0  \\
0 & 0 & 0  \\
\end{array} \right]\textbf{P}$,
we have the following transformation, 
\begin{gather*}
\begin{aligned}
&\frac{1}{3} \int_{\Omega}\sqrt{\int_{\Omega}||\textbf{P}^T\left[\begin{array}{ccc}
3 & 0 & 0  \\
0 & 3 & 0  \\
0 & 0 & 0  \\
\end{array} \right]\textbf{P}\textbf{P}^T\left[\begin{array}{c}
q_1(x)-q_1(y) \\
q_2(x)-q_2(y) \\
q_3(x)-q_3(y) \\
\end{array} \right]  \sqrt{\omega_s }||^2 dy}dx\\
=&\frac{1}{3} \int_{\Omega}\sqrt{\int_{\Omega}||\textbf{C}\left[\begin{array}{c}
u_r(x)-u_r(y) \\
u_g(x)-u_g(y) \\
u_b(x)-u_b(y) \\
\end{array} \right]   \sqrt{\omega_s}||^2dy}dx\\
=&\int_{\Omega}\sqrt{\int_{\Omega}\Big( \big(\mathbf{u}_s^1(x)-\mathbf{u}_s^1(y)\big)^2 +\big(\mathbf{u}_s^2(x)-\mathbf{u}_s^2(y)\big)^2\Big)\omega_s(x, y) dy}dx.\\
    \end{aligned}
\end{gather*}
which completes the proof. \end{proof}

\subsection{The proposed color image restoration models}

In this section, we propose the nonlocal total variation models based on saturation-value similarity for color image restoration. In order to deal with the diversity of noise, we consider $L_2$ and $L_1$ fidelity. Then the proposed $\stvct$ model is given as 
\begin{equation}
\begin{aligned}
    \min_{\mathbf{u}} \left\{ \stvc(\mathbf{u}) + \frac{\lambda}{2} \int_{\Omega} \bigl|(K * \mathbf{u})(x)-\textbf{f}(x)\bigr|^2dx\right\},    
\end{aligned}
\label{eqn3.2.1}
\end{equation}
and the proposed $\stvco$ model is as
\begin{equation}
\begin{aligned}
    \min_{\mathbf{u}} \left\{ \stvc(\mathbf{u}) + \frac{\lambda}{2} \int_{\Omega} \bigl|(K * \mathbf{u})(x)-\textbf{f}(x)\bigr|dx\right\},   
\end{aligned}
\end{equation}
where $K$ is a given blurring operator, $*$ is the convolution operation, and $\lambda > 0$ is a positive regularization parameter. We take the $L_2$ fidelity model as example, and the next theorem states the existence and uniqueness of a solution of the above model, the $L_1$ fidelity model can be proved in the same way.
\begin{theorem}[Existence and Uniqueness]\label{thm:svs-nltv_exist_unique}
The above minimization problems have at least one solution.  If the mapping $\mathbf{u}(x)\;\mapsto\;(K*\mathbf{u})(x)$ is injective, then the solution is unique.
\end{theorem}

\begin{proof}
We choose $\mathbf{u}$ to be constant, so that the energy in \ref{eqn3.2.1} is finite, and the infimum of the energy is finite. Suppose $\{\mathbf{u}^{(n)}\}$ is a minimizing sequence for \ref{eqn3.2.1}.  Then there exists a constant $M>0$ such that
\[
\stvc\bigl(\mathbf{u}^{(n)}\bigr)\;\le\;M.
\]
By combining this with the boundedness of $u_r^{(n)}(x)$, $u_g^{(n)}(x)$, $u_b^{(n)}(x)$, we get that \\$\bigl\{ \stvc (\mathbf{u}^{(n)})+\sum_{i=r, g, b}\|u^{(n)}_i(x)\|_{L^1(\Omega)}\bigr\}$ is uniformly bounded.  Noting the compactness property of Proposition \ref{procom}, up to a subsequence (still denoted as $\{u^{(n)}_r(x), \\ \; u^{(n)}_g(x), \; u^{(n)}_b(x)\}$), there exist $u_r^*(x),\;u_g^*(x),\;u_b^* (x)\in \sbvc (\Omega)$
such that
\begin{gather*}
\begin{aligned}
&u_r^{(n)}(x)\xrightarrow[L^1(\Omega)]{}u_r^*(x),\quad u_r^{(n)}(x)\longrightarrow u_r^*(x), \quad\text{a.e.\ in }\Omega.\\
&u_g^{(n)}(x)\xrightarrow[L^1(\Omega)]{}u_g^*(x),\quad u_g^{(n)}(x)\longrightarrow u_g^*(x), \quad\text{a.e.\ in }\Omega.\\
&u_b^{(n)}(x)\xrightarrow[L^1(\Omega)]{}u_b^*(x),\quad u_b^{(n)}(x)\longrightarrow u_b^*(x), \quad\text{a.e.\ in }\Omega.
\end{aligned}
\end{gather*}
As a consequence of the lower semicontinuity,
\begin{equation}\label{eq:svtv_lsc}
\liminf_{n\to\infty}\stvc\bigl(\mathbf{u}^{(n)}\bigr)\;\ge\;
\stvc\bigl(\mathbf{u}^*\bigr).
\end{equation}
Meanwhile, the following convergence results hold
\begin{gather*}
\begin{aligned}
\bigl(K*u_r^{(n)}(x)-z_r(x)\bigr)^2 \longrightarrow\ \bigl(K*u_r^*(x)-z_r(x)\bigr)^2 \quad\text{a.e.\ in }\Omega, \\
\bigl(K*u_g^{(n)}(x)-z_g(x)\bigr)^2 \longrightarrow\ \bigl(K*u_g^*(x)-z_g(x)\bigr)^2 \quad\text{a.e.\ in }\Omega, \\
\bigl(K*u_b^{(n)}(x)-z_b(x)\bigr)^2 \longrightarrow\ \bigl(K*u_b^*(x)-z_b(x)\bigr)^2 \quad\text{a.e.\ in }\Omega, \\
\end{aligned}
\end{gather*}
By using Fatou's lemma, we have
\begin{gather*}
\begin{aligned}
\liminf \int_{\Omega} (K\star u_{r}^{(n)} (x)- z_{r}(x))^{2}  \mathrm d x &+ \int_{\Omega} (K\star u_{g}^{(n)}(x) - z_{g}(x))^{2}  \mathrm d x  \\ & + \int_{\Omega} (K\star u_{b}^{(n)}(x) - z_{b}(x))^{2}  \mathrm d x \nonumber \\
\geq \int_{\Omega} (K\star u_{r}^{*}(x) - z_{r}(x))^{2}  \mathrm d x &+ \int_{\Omega} (K\star u_{g}^{*}(x) - z_{g}(x))^{2}  \mathrm d x \\
& + \int_{\Omega} (K\star u_{b}^{*}(x) - z_{b}(x))^{2}  \mathrm d x. 
\end{aligned}
\end{gather*}
Combining the above inequalities with \ref{eq:svtv_lsc}, we obtain

\begin{equation}
\begin{aligned}
&\liminf \stvc({\mathbf{u}}^{(n)}) + \int_{\Omega} (K\star u_{r}^{(n)}(x) - z_{r}(x))^{2}  \mathrm d x \\
&+ \int_{\Omega} (K\star u_{g}^{(n)}(x) - z_{g}(x))^{2}  \mathrm d x + \int_{\Omega} (K\star u_{b}^{(n)}(x) - z_{b}(x))^{2}  \mathrm d x \\
&\geq \stvc({\mathbf{u}}^{*}) + \int_{\Omega} (K\star u_{r}^{*}(x) - z_{r}(x))^{2}  \mathrm d x \\
&+ \int_{\Omega} (K\star u_{g}^{*}(x) - z_{g} (x))^{2} \mathrm d x + \int_{\Omega} (K\star u_{b}^{*} (x)- z_{b}(x))^{2}  \mathrm d x.
\end{aligned}
\end{equation}
It leads to the existence of the solution of \ref{eqn3.2.1}. It is clear that if $\mathbf{u}\mapsto K \star \mathbf{u}$ is
injective, it follows the strict convexity of the functional which guarantees the uniqueness of the solution.
\end{proof}

We note that the Euler-Lagrange equation with respect to the red channel of the $\stvct$ model is as follows,
\begin{gather*}
    \begin{aligned}
       \int_{\Omega} div_{\omega}^s\Big(\big(2u_r(x)-u_g(x)-u_b(x)\big)\big(\frac{\sqrt{\omega_s}}{|\nabla_{\omega}^s u(x)|}+\frac{\sqrt{\omega_s}}{|\nabla_{\omega}^s u(y)|}\big)\Big) \mathrm d y\\
       +\mu \int_{\Omega} div_{\omega}^v\Big(\big(u_r(x)+u_g(x)+u_b(x)\big)\big(\frac{\sqrt{\omega_v}}{|\nabla_{\omega}^v u(x)|}+\frac{\sqrt{\omega_v}}{|\nabla_{\omega}^v u(y)|}\big)\Big) \mathrm d y\\
       -3\lambda\big((K^{*}*K*u_r)(x)-(K*f_r)(x)\big)=0,
    \end{aligned}
\end{gather*}
where $K^{*}$ is the conjugate transpose of $K$. As a comparison, we give the Euler-Lagrange equation with respect to the red channel of NLTV model \cite{Xiaoqun2010Bregmanized} as follows, 
\begin{gather*}
    \begin{aligned}
       \int_{\Omega} div_{\omega}\Big(u_r(x)\big(\frac{\sqrt{\omega}}{|\nabla_{\omega}u_r(x)|}+\frac{\sqrt{\omega}}{|\nabla_{\omega}u_r(y)|}\big)\Big)  \mathrm d y
       -\lambda\big((K^{*}*K*u_r)(x)-(K*f_r)(x)\big)=0.
    \end{aligned}
\end{gather*}
By comparing the above two equations, we can tell the difference between $\stvc$ model and NLTV model. For $\stvc$ regularization, we find that the nonlocal divergence operator acts on the coupling of RGB channels, $2u_r(x)-u_g(x)-u_b(x)$, meanwhile, $\stvc$ takes the form of coupling channel diffusion coefficients in the saturation component, 
$\frac{\sqrt{\omega_s}}{\nabla_{\omega}^s u(x)|}+\frac{\sqrt{\omega_s}}{|\nabla_{\omega}^s u(y)|}$, and in the value component, $\frac{\sqrt{\omega_v}}{|\nabla_{\omega}^v u(x)|}+\frac{\sqrt{\omega_v}}{|\nabla_{\omega}^v u(y)|}$. However, For NLTV regularization,  the divergence operator acts directly on red channel, and NLTV takes the form of individual channel diffusion coefficient, $\frac{\sqrt{\omega}}{|\nabla_{\omega}u_r(x)|}+\frac{\sqrt{\omega}}{|\nabla_{\omega}u_r(y)|}$, which is only related to red channel. Because of coupling among red, green, and blue channels in diffusion coefficients and equations, we expect that the color image restoration effect by using $\stvc$ model will be enhanced compared with that of NLTV model. Finally we remark here that the proposed $\stvc$ model will also outperform SVTV model \cite{jia2019color} in detail preservation and restoration due to the application of non-local technology. In section 5, numerical examples are given to demonstrate the effectiveness of the proposed $\stvc$ models.

\section{Numerical algorithm for $\stvc$ model}

In this section, we propose an efficient framework to solve $\stvc$ model.
We first give the discrete nonlocal operator $(\nabla_{\omega}^s \textbf{u})_{ij}$, $(\nabla_{\omega}^v \textbf{u})_{ij}$, and present the discrete version of the $\stvc$ as follows,
\begin{eqnarray}
&&\nonumber\hspace{10mm} (\nabla_{\omega}^s \textbf{u})_{ij} = \Big(\big(\textbf{u}_{s}^1(j)-\textbf{u}_{s}^1(i)\big) \sqrt{\omega_s },\ \big(\textbf{u}_{s}^2(j)-\textbf{u}_{s}^2(i)\big)\sqrt{\omega_s }\Big),\\
&&\nonumber\hspace{20mm} (\nabla_{\omega}^v \textbf{u})_{ij} = \big(\textbf{u}_v(j)-\textbf{u}_v(i)\big)\sqrt{\omega_v },\\
&&\hspace{5mm}\stvc(\textbf{u})=\sum_{i=1}^n \sum_{j=1}^m |(\nabla_{\omega}^s \textbf{u})_{ij}|_1+\mu \sum_{i=1}^n \sum_{j=1}^m |(\nabla_{\omega}^v \textbf{u})_{ij}|,
\end{eqnarray}
where $n$ is the pixel number of the discretized image and $m$ is the pixel number of the nonlocal neighborhood. In this section, we consider the following discrete models for color image restoration, 
\begin{equation}
\begin{aligned}
   \min_{l \leq \textbf{u} \leq L} \alpha \stvc(\textbf{u})+\frac{1}{2}||K\textbf{u}-\mathbf{f}||^2,\\
   \min_{l \leq \textbf{u} \leq L} \alpha \stvc(\textbf{u})+\frac{1}{2}||K\textbf{u}-\mathbf{f}||_1,
   \end{aligned}
   \label{eqn 6.1}
\end{equation}
where $l \leq L$ denote the lower bound and the upper bound of the RGB values. 
\subsection{The proposed algorithm for $\stvc$}\label{sec 4.1}

In this subsection, we consider the following $\stvct$ model, 
\begin{equation}
\min_{\textbf{u}} \alpha \stvc(\textbf{u})+\frac{1}{2}||K\textbf{u}-\mathbf{f}||^2.
\label{eqn L2}
\end{equation}
We introduce $\mathbf{p}$ by setting $\mathbf{p} = K\mathbf{u}-\mathbf{f}$, then we can transform ($\ref{eqn L2}$) into the following equivalent optimization problem,
\begin{gather*}
\min_{\textbf{u}} \alpha \stvc(\textbf{u})+\frac{1}{2}||\textbf{p}||^2 \quad \text{s.t.}\ \ \textbf{p}=K\textbf{u}-\mathbf{f},
\end{gather*}
by reorganizing the variables, we derive the following equivalent version,
\begin{equation}
\min_{\textbf{z}} H(\textbf{z}) \quad \text{s.t.}\ \ B\textbf{z} = \mathbf{f},
\label{eqn 4.3}
\end{equation}
here $\textbf{z}=\left[\begin{array}{c}\textbf{u}\\
\textbf{p}
\end{array} \right]$, $B =\left[K, -I\right]$ and $H(\textbf{z})=\alpha \stvc(\textbf{u})+\frac{1}{2}||\textbf{p}||^2$. Then by introducing $\textbf{w}=\left[\begin{array}{c}\textbf{v}\\
\textbf{q}
\end{array} \right]$ and considering a Moreau-Yosida regularization of ($\ref{eqn 4.3}$),
\begin{gather*}
\min_{\textbf{z}, \textbf{w}} \lambda H(\textbf{z})+\frac{1}{2\delta}||\textbf{z}-\textbf{w}||^2+\frac{1}{2}||B\textbf{w}-\mathbf{f}||^2 \quad \text{s.t.}\ B\textbf{z}=\mathbf{f},
\end{gather*}
we solve ($\ref{eqn 4.3}$) by using Bregman iteration scheme,
\begin{eqnarray}
\left\{
\begin{aligned}
\textbf{z}^{k+1}&=\arg\min_{\textbf{z}}\big(\lambda H(\textbf{z})+\frac{1}{2\delta}||\textbf{z}-\textbf{w}^k||^2\big),\\
\textbf{w}^{k+1}&=\arg\min_{\textbf{w}}\big(\frac{1}{2}||B \textbf{w}-\textbf{f}^k||^2+\frac{1}{2\delta}||\textbf{z}^{k+1}-\textbf{w}||^2\big),\\
\mathbf{f}^{k+1}&=\mathbf{f}^k+\mathbf{f}-\textbf{B}\textbf{z}^{k+1}.
\end{aligned}
\right.
\label{sequenceL2}
\end{eqnarray}
Noting that $\textbf{z}$-subproblem can be transformed into the following $\textbf{u}$-subproblem and $\textbf{p}$-subproblem,
\begin{gather*}
(\textbf{u}^{k+1}, \textbf{p}^{k+1})=\arg\min_{\textbf{u}, \textbf{p}} \big( \lambda \alpha \stvc(\textbf{u})+ \frac{1}{2\delta}||\textbf{u}-\textbf{v}^k||^2+\frac{1}{2}||\textbf{p}||^2+\frac{1}{2\delta}||\textbf{p}-\textbf{q}^k||^2\big),
\end{gather*}
which is equivalent to
\begin{gather*}
\left\{
\begin{aligned}
\textbf{u}^{k+1}&=\arg\min_{\textbf{u}} \big( \lambda \alpha \stvc(\textbf{u})+ \frac{1}{2\delta}||\textbf{u}-\textbf{v}^k||^2\big),\\
\textbf{p}^{k+1}&=\arg\min_{\textbf{p}} \big( \frac{1}{2}||\textbf{p}||^2+\frac{1}{2\delta}||\textbf{p}-\textbf{q}^k||^2\big).
\end{aligned}
\right.
\end{gather*}
For $\mathbf{u}$-subproblem, we show the detailed numerical algorithm in Section $\ref{sec 4.3}$. $\mathbf{p}$-subproblem has a closed form solution,
\begin{gather*}
\textbf{p}^{k+1} = \frac{1}{1+\delta \lambda} \textbf{q}^{k}.
\end{gather*}
$\textbf{w}$-subproblem is equivalent to the following equation,
\begin{gather*}
(B^TB+\frac{1}{\delta}I )\textbf{w} = \frac{1}{\delta}\textbf{z}^{k+1}+B^T\textbf{f}^k,
\end{gather*}
which is equivalent to
\begin{gather*}
\left[\begin{array}{cc}
K^T K+\frac{1}{\delta}I & -K^T\\
-K & \frac{\delta+1}{\delta}I\\
\end{array} \right] \left[\begin{array}{c}
\textbf{v}\\
\textbf{q}\\
\end{array} \right]=\left[\begin{array}{c}
\frac{1}{\delta}\textbf{u}^{k+1}+K^T \textbf{f}^k \\
\frac{1}{\delta}\textbf{p}^{k+1}-\textbf{f}^k\\
\end{array} \right],
\end{gather*}
which is
\begin{gather*}
\begin{aligned}
(K^TK+\frac{1}{\delta}I)\textbf{v}-K^T\textbf{q}&=\frac{1}{\delta}\textbf{u}^{k+1}+K^T \textbf{f}^k, \\
-K \textbf{v}+\frac{\delta+1}{\delta}\textbf{q}&=\frac{1}{\delta}\textbf{p}^{k+1}-\textbf{f}^{k}.
\end{aligned}
\end{gather*}
We then obtain the solution as follows, 
\begin{gather*}
\left\{
\begin{aligned}
\textbf{v}^{k+1}&=\big( \delta K^TK+(\delta+1)I\big)^{-1}\big( (\delta+1)\textbf{u}^{k+1} +\delta K^T(\textbf{p}^{k+1}+\textbf{f}^k)\big),\\
\textbf{q}^{k+1}&=\frac{\delta}{\delta+1}(\frac{1}{\delta}\textbf{p}^{k+1}-\textbf{f}^{k}+K \textbf{v}^{k+1}), 
\end{aligned}
\right.
\end{gather*}
where $\textbf{v}^{k+1}$ can be solved by using the fast Fourier transform and set a periodic boundary condition.

\subsection{The proposed algorithm for $\stvco$}

In this subsection, we consider the following $\stvco$ model, 
\begin{equation}
\min_{\textbf{u}} \alpha \stvc(\textbf{u})+\frac{1}{2}||K\textbf{u}-\mathbf{f}||_1.
\label{eqn L1}
\end{equation}
Again we introduce $\mathbf{p}$ by setting $\mathbf{p} = K\mathbf{u}-f$, and we transform ($\ref{eqn L1}$) into the following equivalent optimization problem,
\begin{gather*}
\min_{\textbf{u}} \alpha \stvc(\textbf{u})+\frac{1}{2}||\textbf{p}||_1 \quad \text{s.t.}\ \ \textbf{p} = K\textbf{u}-\mathbf{f},
\end{gather*}
we reorganize the variables and derive the following equivalent version,
\begin{equation}
\min_{\textbf{z}} H(\textbf{z}) \quad \text{s.t.}\ \ B\textbf{z} = \mathbf{f},
\label{eqn 4.9}
\end{equation}
with $\textbf{z}=\left[\begin{array}{c}\textbf{u}\\
\textbf{p}
\end{array} \right]$, $B=[K, -I]$ and $H(\textbf{z})=\alpha \stvc(\textbf{u})+\frac{1}{2}|\textbf{p}|_1$. By considering the same Moreau-Yosida regularization, we solve ($\ref{eqn 4.9}$) by using Bregman iteration scheme,
\begin{eqnarray}
\left\{
\begin{aligned}
\textbf{z}^{k+1}&=\arg\min_{\textbf{z}}\big(\lambda H(\textbf{z})+\frac{1}{2\delta}||\textbf{z}-\textbf{w}^k||^2\big),\\
\textbf{w}^{k+1}&=\arg\min_{\textbf{w}}\big(\frac{1}{2\delta}||\textbf{z}^{k+1}-\textbf{w}||^2+\frac{1}{2}||B\textbf{w}-\mathbf{f}^k||^2\big),\\
\mathbf{f}^{k+1}&=\mathbf{f}^k+\mathbf{f}-\textbf{B}\textbf{z}^{k+1}.
\end{aligned}
\right.
\label{sequenceL1}
\end{eqnarray}
Noting that $\textbf{z}$-subproblem can be transformed into the following $\textbf{u}$-subproblem and $\textbf{p}$-subproblem,
\begin{gather*}
(\textbf{u}^{k+1}, \textbf{p}^{k+1})=\arg\min_{\textbf{u}, \textbf{p}} \big( \lambda \alpha \stvc(\textbf{u})+ \frac{1}{2\delta}||\textbf{u}-\textbf{v}^k||^2+\frac{1}{2}||\textbf{p}||^2+\frac{1}{2\delta}||\textbf{p}-\textbf{q}^k||^2\big),
\end{gather*}
which is equivalent to
\begin{gather*}
\left\{
\begin{aligned}
\textbf{u}^{k+1}&=\arg\min_{\textbf{u}} \big( \lambda \alpha \stvc(\textbf{u})+ \frac{1}{2\delta}||\textbf{u}-\textbf{v}^k||^2\big),\\
\textbf{p}^{k+1}&=\arg\min_{\textbf{p}} \big( \frac{1}{2}||\textbf{p}||_1+\frac{1}{2\delta}||\textbf{p}-\textbf{q}^k||^2\big),
\end{aligned}
\right.
\end{gather*}
For the $\mathbf{u}$-subproblem, we show the detailed numerical algorithm in {Section} $\ref{sec 4.3}$. $\mathbf{p}$-subproblem can be solved by using shrinkage operator,
\begin{gather*}
\textbf{p}^{k+1}=\text{shrink}(\textbf{q}^k, \lambda \delta).
\end{gather*}
$\textbf{w}$-subproblem is also equivalent to the following equation,
\begin{gather*}
(B^TB+\frac{1}{\delta}I )\textbf{w} = \frac{1}{\delta}\textbf{z}^{k+1}+B^T\textbf{f}^k.
\end{gather*}
By using similar method as in Section \ref{sec 4.1}, we have 
\begin{gather*}
\left\{
\begin{aligned}
\textbf{v}^{k+1}&=\big( \delta K^TK+(\delta+1)I\big)^{-1}\big( (\delta+1)\textbf{u}^{k+1} +\delta K^T(\textbf{p}^{k+1}+\textbf{f}^k)\big),\\
\textbf{q}^{k+1}&=\frac{\delta}{\delta+1}(\frac{1}{\delta}\textbf{p}^{k+1}-\textbf{f}^{k}+K \textbf{v}^{k+1}).
\end{aligned}
\right.
\end{gather*}

\subsection{The proposed algorithm for $\mathbf{u}$-subproblem}\label{sec 4.3}

In this section, we show how to solve $\mathbf{u}$-subproblem which is
\begin{equation}
\mathbf{u}^{k+1} =\arg \min_{\textbf{u}}\big( \lambda\alpha \stvc(\textbf{u}) +\frac{1}{2\delta} ||\textbf{u}-\textbf{v}^k||^2\big).
\label{eqn 4.10}
\end{equation}
We consider the anisotropic discrete version of ($\ref{eqn 4.10}$), 
\begin{equation}
\min_{\textbf{u}} \sum_{i=1}^n \sum_{j=1}^m ||(\nabla_{\omega}^s \textbf{u})_{ij}||_1+\mu \sum_{i=1}^n \sum_{j=1}^m |(\nabla_{\omega}^v \textbf{u})_{ij}| +\frac{1}{2\delta}||\textbf{u}-\textbf{v}^k||^2
\label{eqn bregman u}.
\end{equation}
Noting that we set $\textbf{q}_s=[\textbf{q}_1, \textbf{q}_2]^T$ in Proposition \ref{pro3.2} before, let $\textbf{q}= [\textbf{q}_s, \textbf{q}_3]^T =\mathbf{P} \textbf{u}=\mathbf{P}[\textbf{u}_r, \textbf{u}_g, \textbf{u}_b]^T$ and $\tilde{\textbf{v}}^k=\mathbf{P} \textbf{v}^k$, we transform ($\ref{eqn bregman u}$) into the following equivalent minimization problem by using Proposition $\ref{pro3.2}$, 
\begin{gather*}
\min_{\textbf{u}} \sum_{i=1}^n\sum_{j=1}^m ||(\nabla_{\omega}  \textbf{q}_{s})_{ij}||_1+\mu \sum_{i=1}^n \sum_{j=1}^m |(\nabla_{\omega} \textbf{q}_3)_{ij}| +\frac{1}{2\delta}||\textbf{q}-\tilde{\textbf{v}}^k||^2.
\end{gather*}
We focus on $\mathbf{q}_3$ problem and remark that $ \mathbf{q}_{s}$ can be solved by using similar method.
\begin{equation}
\min_{\textbf{q}_3}  \mu \sum_{i=1}^n \sum_{j=1}^m |(\nabla_{\omega} \textbf{q}_3)_{ij}| +\frac{1}{2\delta}||\textbf{q}_3-\tilde{\textbf{v}}^k_3||^2.
\label{eqn 4.18}
\end{equation}
Let $d_{ij}=(\nabla_{\omega} (\textbf{q}_3)_{ij}$), we can reformulate ($\ref{eqn 4.18}$) as
\begin{gather*}
\min_{\textbf{q}_{3}, d} \sum_{j=1}^m\mu|d_{ij}|+ \frac{1}{2\delta}||\textbf{q}_{3, i}-\tilde{\textbf{v}}^k_{3, i}||^2   \quad \text{s.t.}\ \ d_{ij}=(\nabla_{\omega} \textbf{q}_3)_{ij}, 
\end{gather*}
and force the constraint with the Bregman iteration process as follows,
\begin{gather*}
\left\{
\begin{aligned}
(\textbf{q}_{3, i}^{k+1}, d_{ij}^{k+1})&=\arg \min_{\textbf{q}_{3}, d} \sum_{j=1}^m \big(\mu|d_{ij}|+\frac{\beta}{2}||d_{ij}-(\nabla_{\omega} \textbf{q}_3)_{ij}-b_{ij}^k||^2 \big) + \frac{1}{2\delta}||\textbf{q}_{3, i}-\tilde{\textbf{v}}^k_{3, i}||^2, \\
b_{ij}^{k+1}&=b_{ij}^k+(\nabla_{\omega} \textbf{q}_3)_{ij}^{k+1}-d_{ij}^{k+1},
\end{aligned}
\right.
\end{gather*} 
which is equivalent to
\begin{gather*}
\left\{
\begin{aligned}
\textbf{q}_{3, i}^{k+1}&=\arg \min_{\textbf{q}}\frac{\beta}{2}\sum_{j=1}^m||d_{ij}^k-(\nabla_{\omega} \textbf{q}_3)_{ij}-b_{ij}^k||^2 +\frac{1}{2\delta}||\textbf{q}_{3, i}-\tilde{\textbf{v}}^k_{3, i}||^2, \\
d_{ij}^{k+1}&=\arg \min_{d_{ij}}\mu|d_{ij}|+\frac{\beta}{2}||d_{ij}-(\nabla_{\omega} \textbf{q}_{3}^{k+1})_{ij}-b_{ij}^k||^2, \\
b_{ij}^{k+1}&=b_{ij}^k+(\nabla_{\omega} \textbf{q}_3^{k+1})_{ij}-d_{ij}^{k+1}.
\end{aligned}
\right.
\end{gather*}
The Euler-Lagrange equation for $\textbf{q}_3$-subproblem is given by
\begin{gather*}
(\textbf{q}_{3, i}^{k+1}-\tilde{\textbf{v}}_{3, i}^k)-\beta div_{\omega} \big((\nabla_{\omega} \textbf{q}_3)_i-d^k_i+b^k_i \big) = 0,
\end{gather*}
where $div_{\omega} (\nabla_{\omega} \textbf{q}_3)_{i, j} = (\Delta_{\omega} \textbf{q}_3)_{i, j} = 2 \big(\textbf{q}_3 (j)-\textbf{q}_3 (i)\big)\omega$, thus we have
\begin{gather*}
\textbf{q}_3^{k+1}=(1-\Delta_{\omega})^{-1}\big(\tilde{\textbf{v}}_{3}^k+\beta div_{\omega}(b^k_i-d^k_i) \big),
\end{gather*} 
numerically we solve $\textbf{q}^k_3$ by using Gauss-Seidel algorithm.
For $d$-subproblem, it can be solved efficiently by using shrinkage operator,
\begin{gather*}
d_{ij}^{k+1}=\text{shrink}(\nabla_{\omega} \textbf{q}_{3}^{k+1})_{ij} +b_{ij}^k, \frac{\alpha}{\beta}), 
\end{gather*}
Finally, we set $\mathbf{u}^{k+1} = \mathbf{P}^T [\mathbf{q}_s, \mathbf{q}_3]^T$. 
 
\subsection{Convergence analysis}

In this section, we give the following theorem about the convergence of the proposed algorithms.
\begin{theorem}
If $\delta$ and B satisfy $0 < \delta < \frac{1}{||B^TB||}$, then the sequences generated by $\ref{sequenceL2}$ and $\ref{sequenceL1}$ converge along subsequence to solutions of $\ref{eqn 4.3}$ and $\ref{eqn 4.9}$.
\end{theorem}

The proof of the above theorem can refer to the discussion in \cite{goldstein2009split, Xiaoqun2010Bregmanized}.

\section{Numerical experiments}

To demonstrate the effectiveness of the proposed $\stvc$ regularization and the proposed color image restoration models, we present the experimental results in this section. 
The quality of the recovered images is measured by

\begin{itemize}
\item the SSIM index \cite{wang2004image}, which has been proven to be consistent with human eye perception.
\item the QSSIM index \cite{kolaman2011quaternion}, which has been shown to be a better measure than SSIM for color
image quality.
\item the PSNR index, which measures the ratio between the maximum possible power of a signal and the power of corrupting noise that affects the fidelity of its representation.
\item the S-CIELAB color metric \cite{zhang1996spatial}, which includes a spatial processing step and is useful and effcient for measuring color reproduction errors of digital images.
\end{itemize}

In all experiments, we report the best result in terms of the best PSNR value corresponding to the optimal regularization parameter in some ranges. For the stopping criteria, we break the iteration when the relative error of the successive iterates is less than or equal to 1$\times 10^{-6}$. The proposed main algorithm is implemented in MATLAB. All the computations are performed on a PC with an Intel(R) Core(TM) i7-13700 2.10 GHz CPU.
\begin{figure}[htbp]
\centering
\begin{subfigure}
\centering
\includegraphics[width=0.23\textwidth]{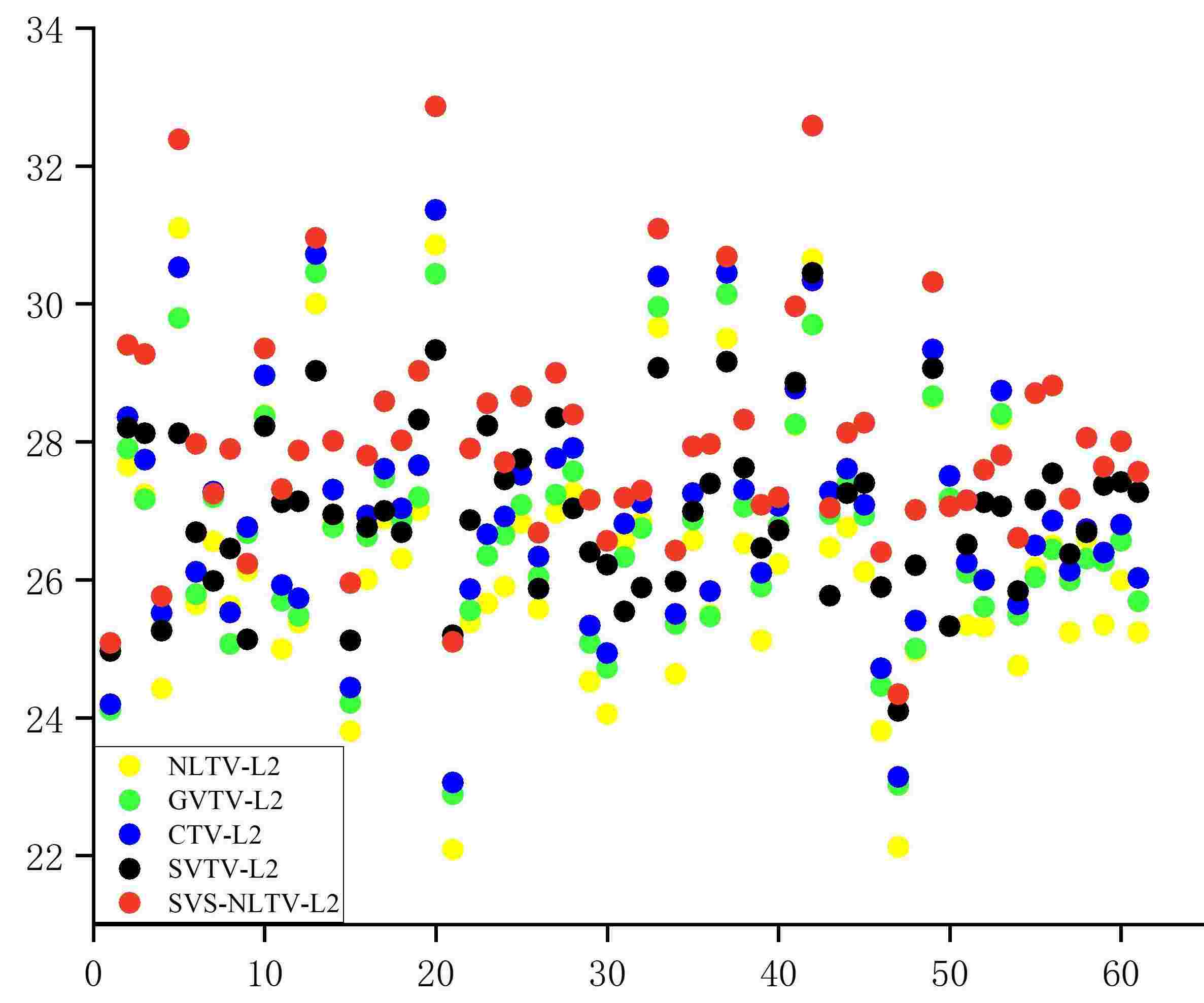}
\end{subfigure}
\begin{subfigure}
\centering
\includegraphics[width=0.23\textwidth]{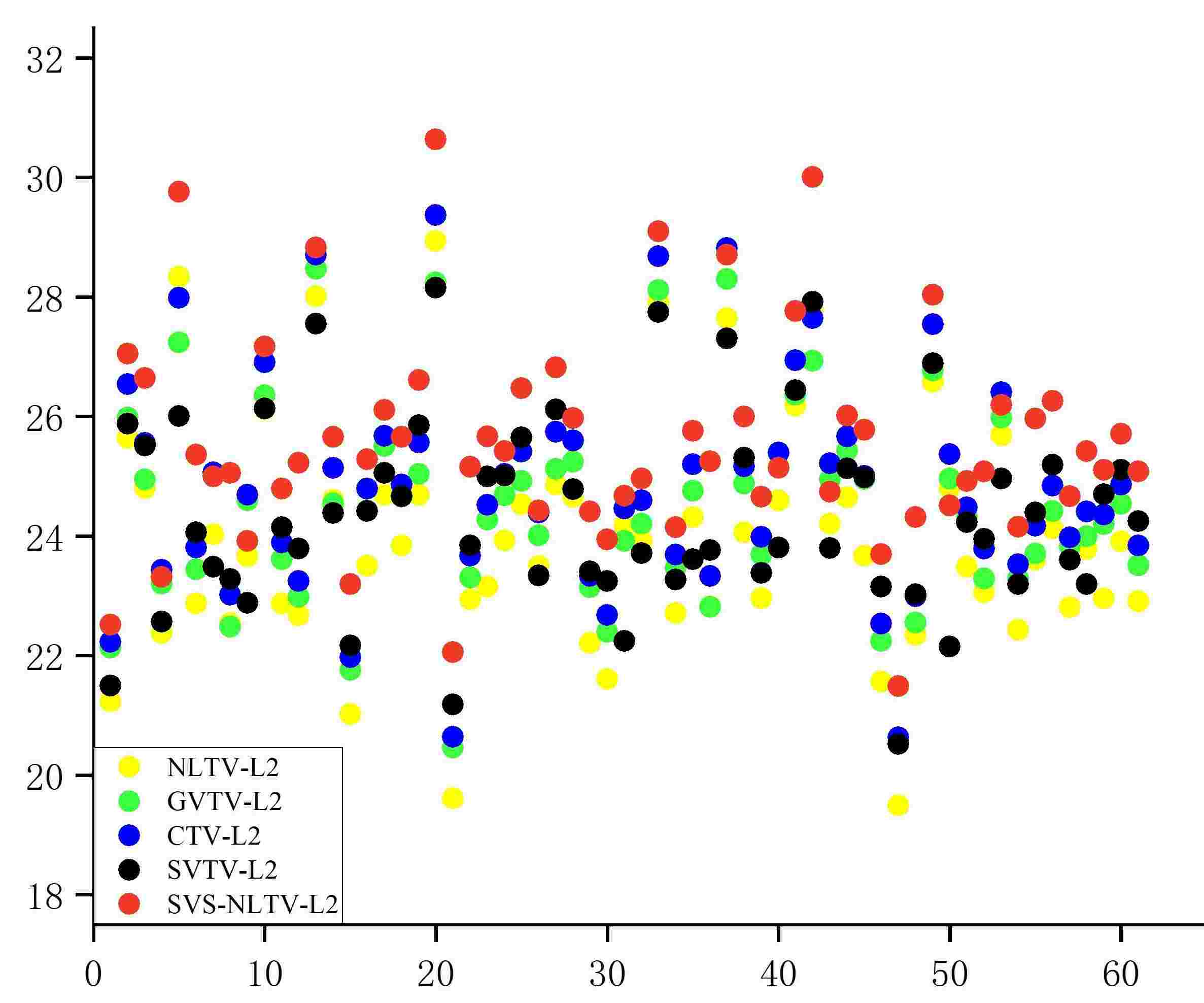}
\end{subfigure}
\begin{subfigure}
\centering
\includegraphics[width=0.23\textwidth]{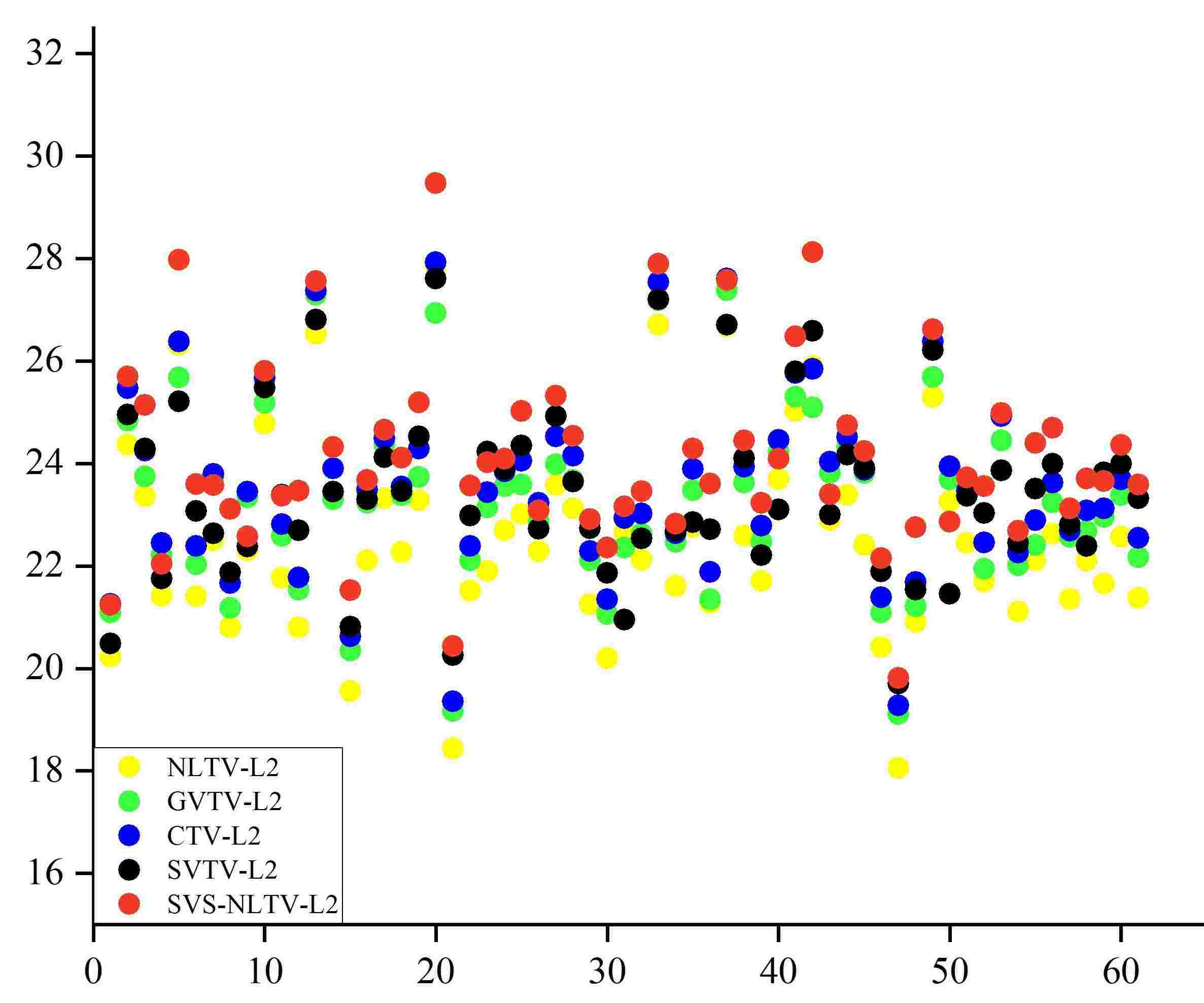}
\end{subfigure}
\begin{subfigure}
\centering
\includegraphics[width=0.23\textwidth]{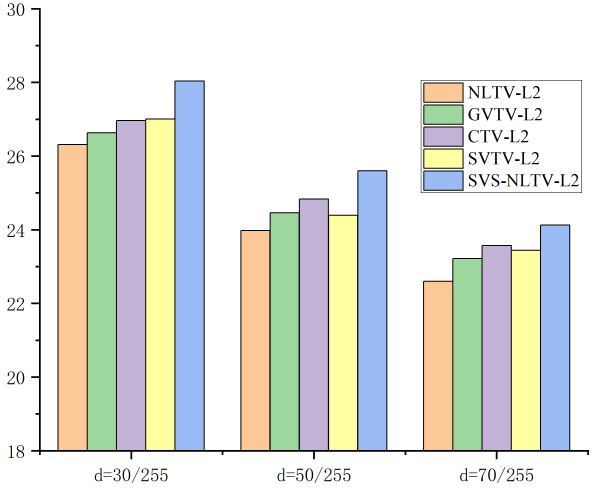}
\end{subfigure}
\caption{First to third: The spatial distributions of PSNR values of the restored results by using different methods corresponding to d = 30/255, 50/255, 70/255 respectively; Fourth: the histogram of the average PSNR values of the restored results by using different methods.}
\label{g-psnr}
\end{figure}

\begin{figure}[htbp]
\centering
\begin{subfigure}
\centering
\includegraphics[width=0.23\textwidth]{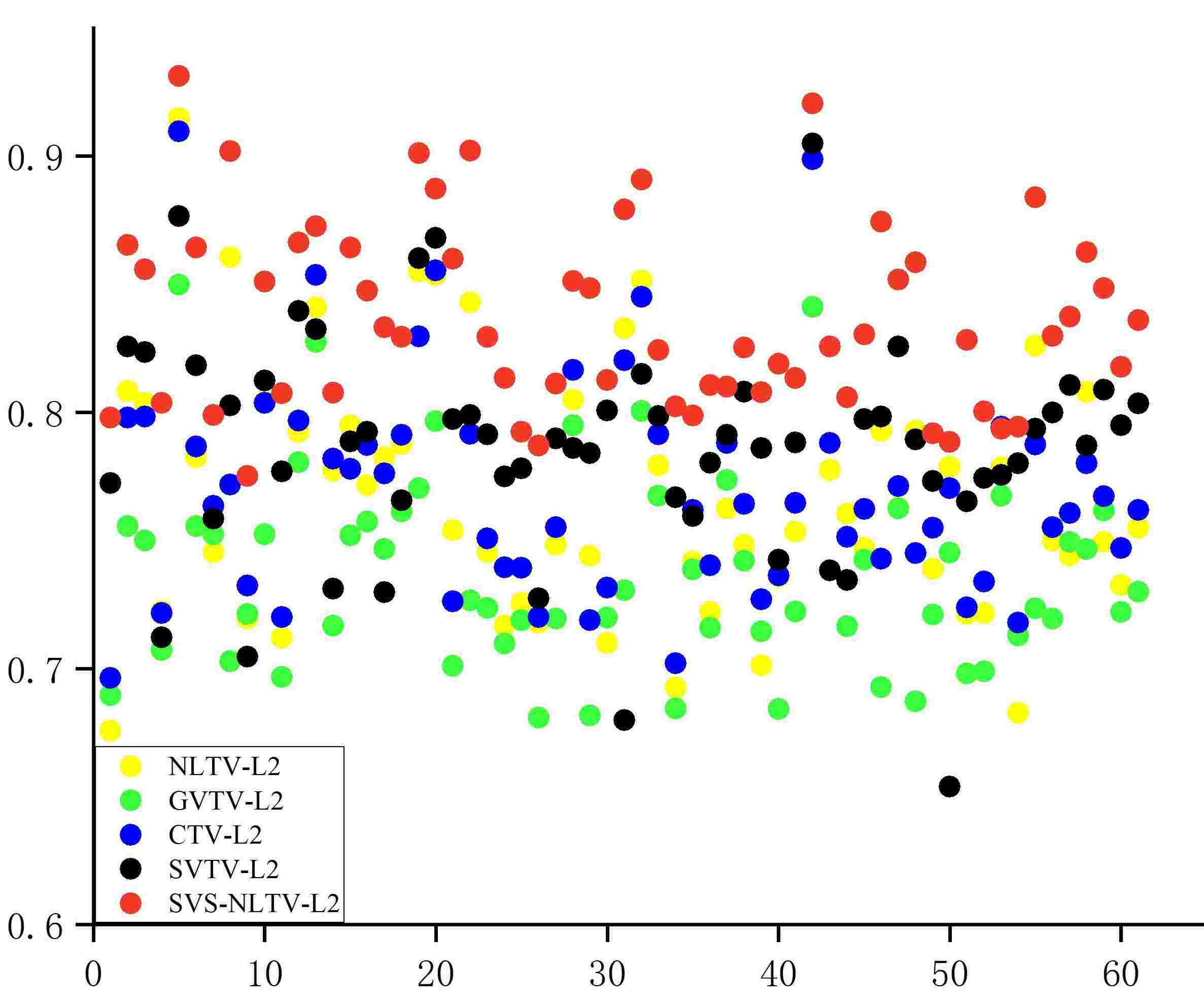}
\end{subfigure}
\begin{subfigure}
\centering
\includegraphics[width=0.23\textwidth]{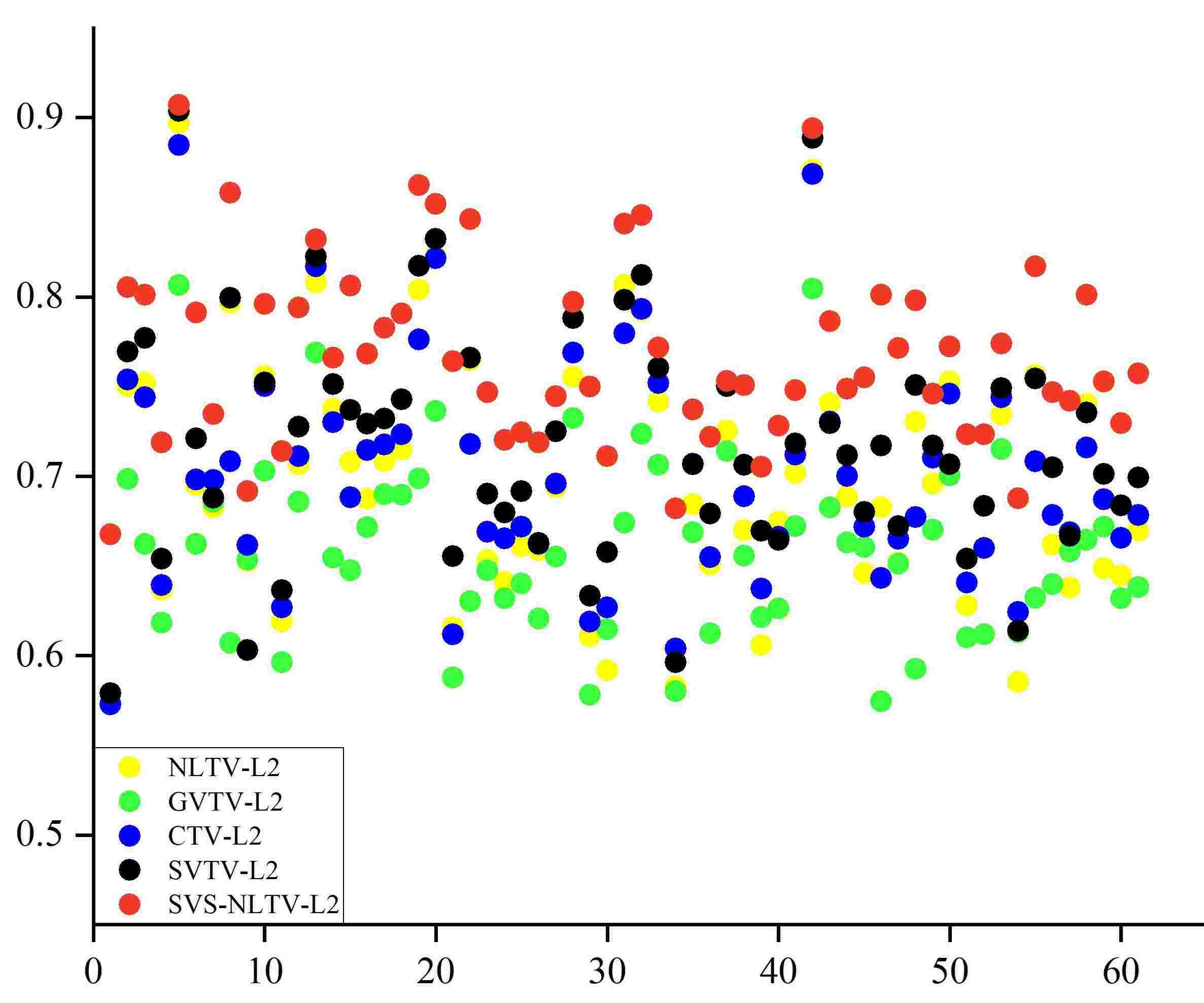}
\end{subfigure}
\begin{subfigure}
\centering
\includegraphics[width=0.23\textwidth]{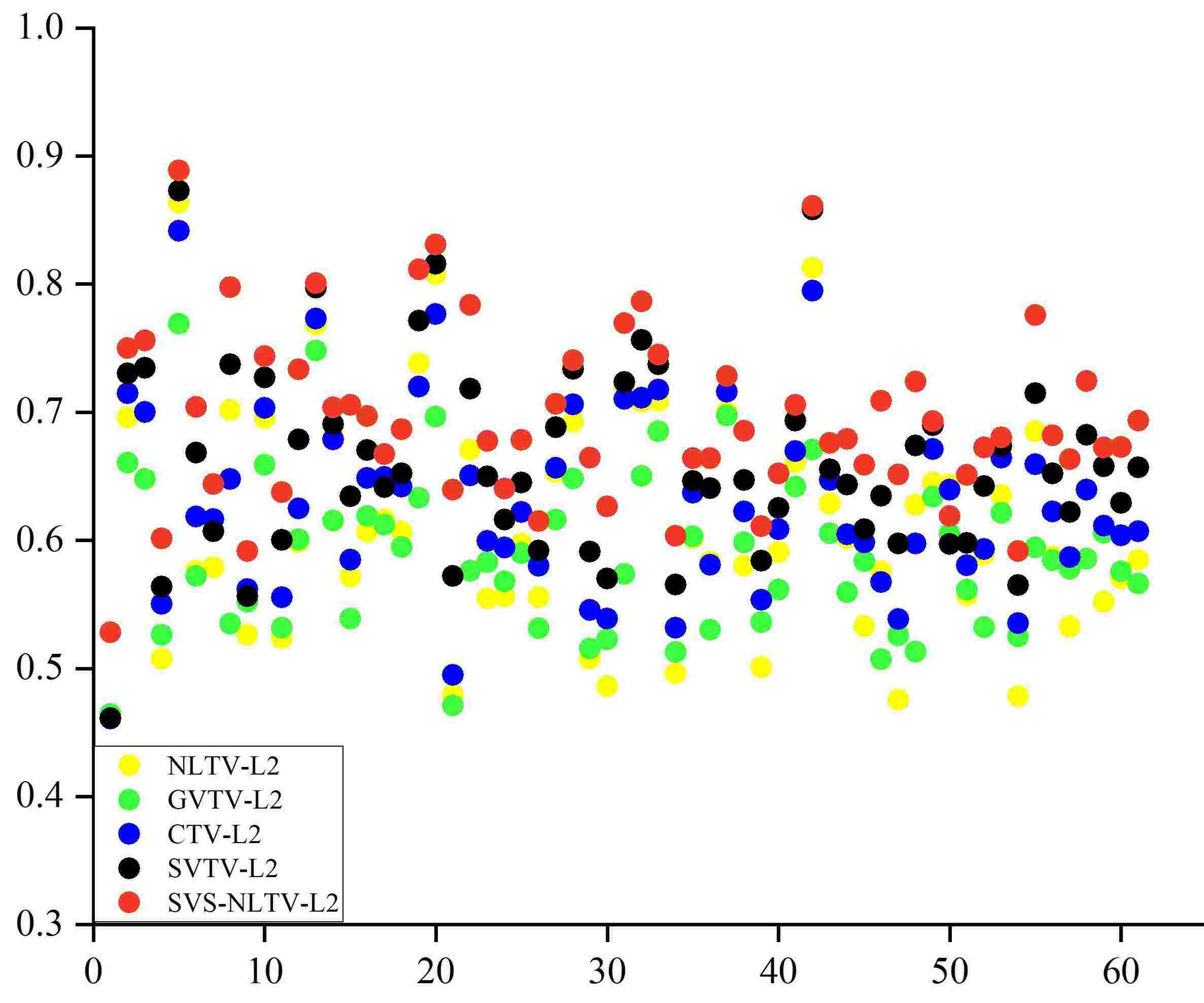}
\end{subfigure}
\begin{subfigure}
\centering
\includegraphics[width=0.23\textwidth]{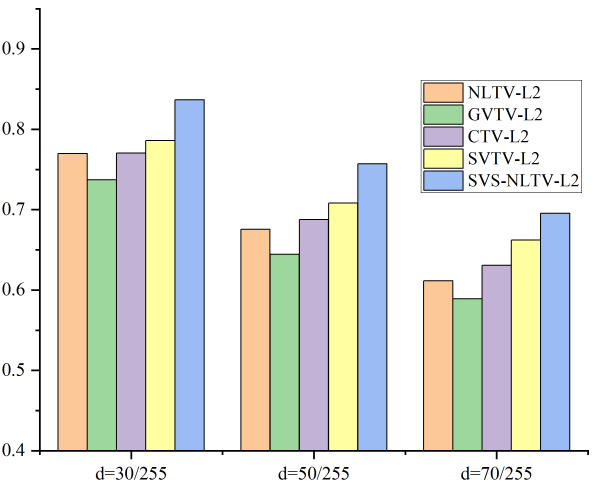}
\end{subfigure}
\caption{First to third: The spatial distributions of SSIM values of the restored results by using different methods corresponding to d = 30/255, 50/255, 70/255 respectively; Fourth: the histogram of the average SSIM values of the restored results by using different methods.}
\label{g-ssim}
\end{figure}

\begin{figure}[htbp]
\centering
\begin{subfigure}
\centering
\includegraphics[width=0.23\textwidth]{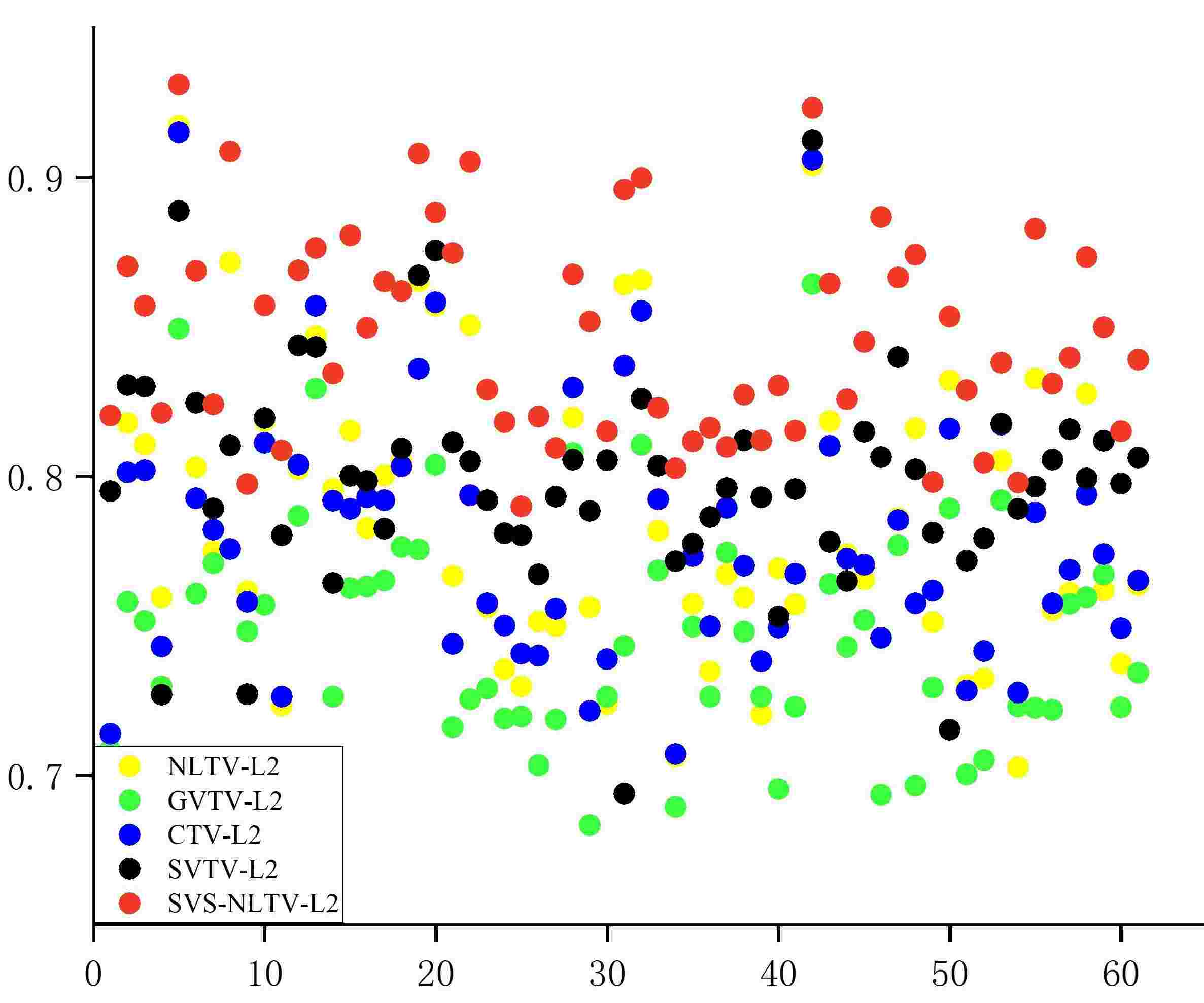}
\end{subfigure}
\begin{subfigure}
\centering
\includegraphics[width=0.23\textwidth]{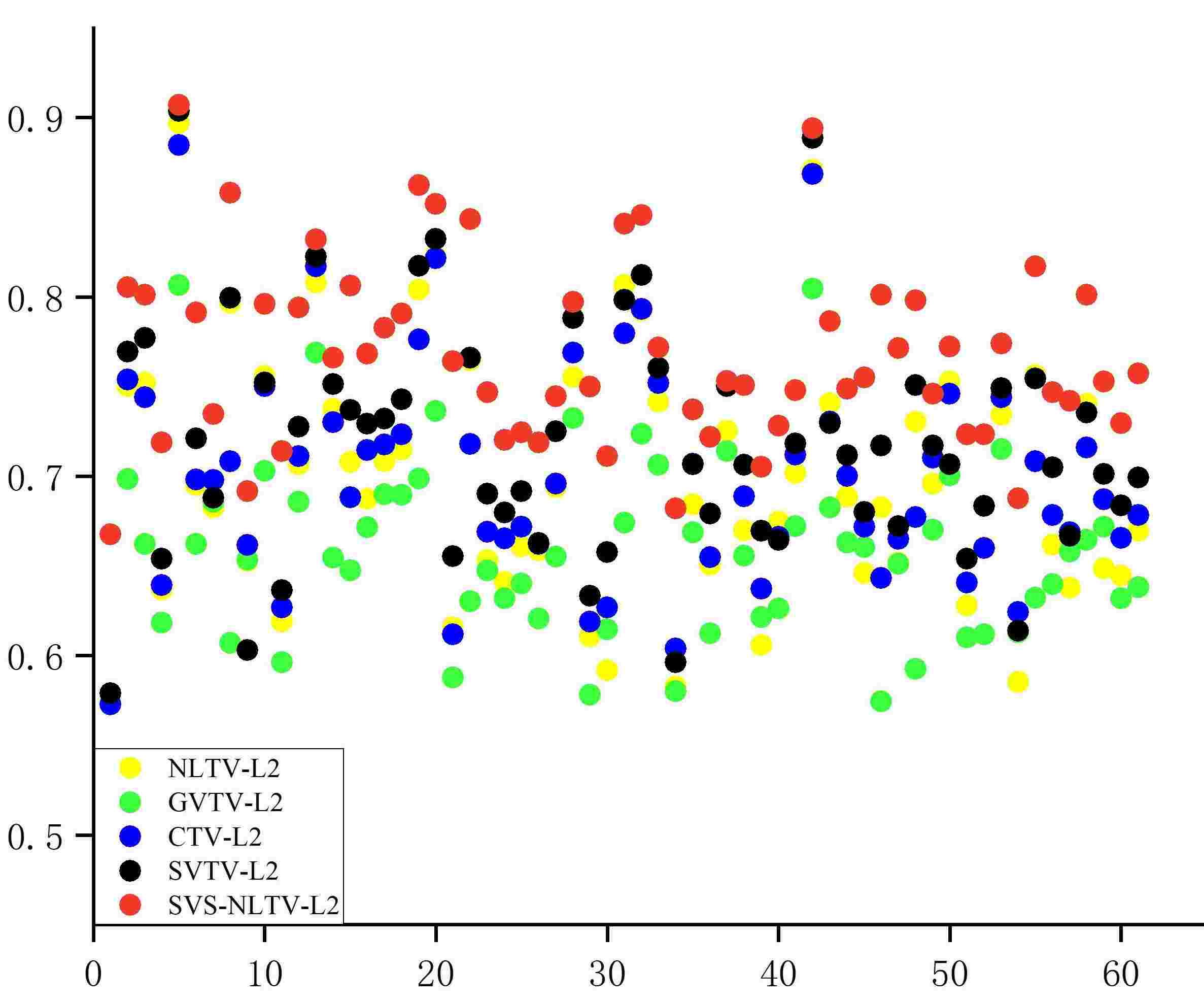}
\end{subfigure}
\begin{subfigure}
\centering
\includegraphics[width=0.23\textwidth]{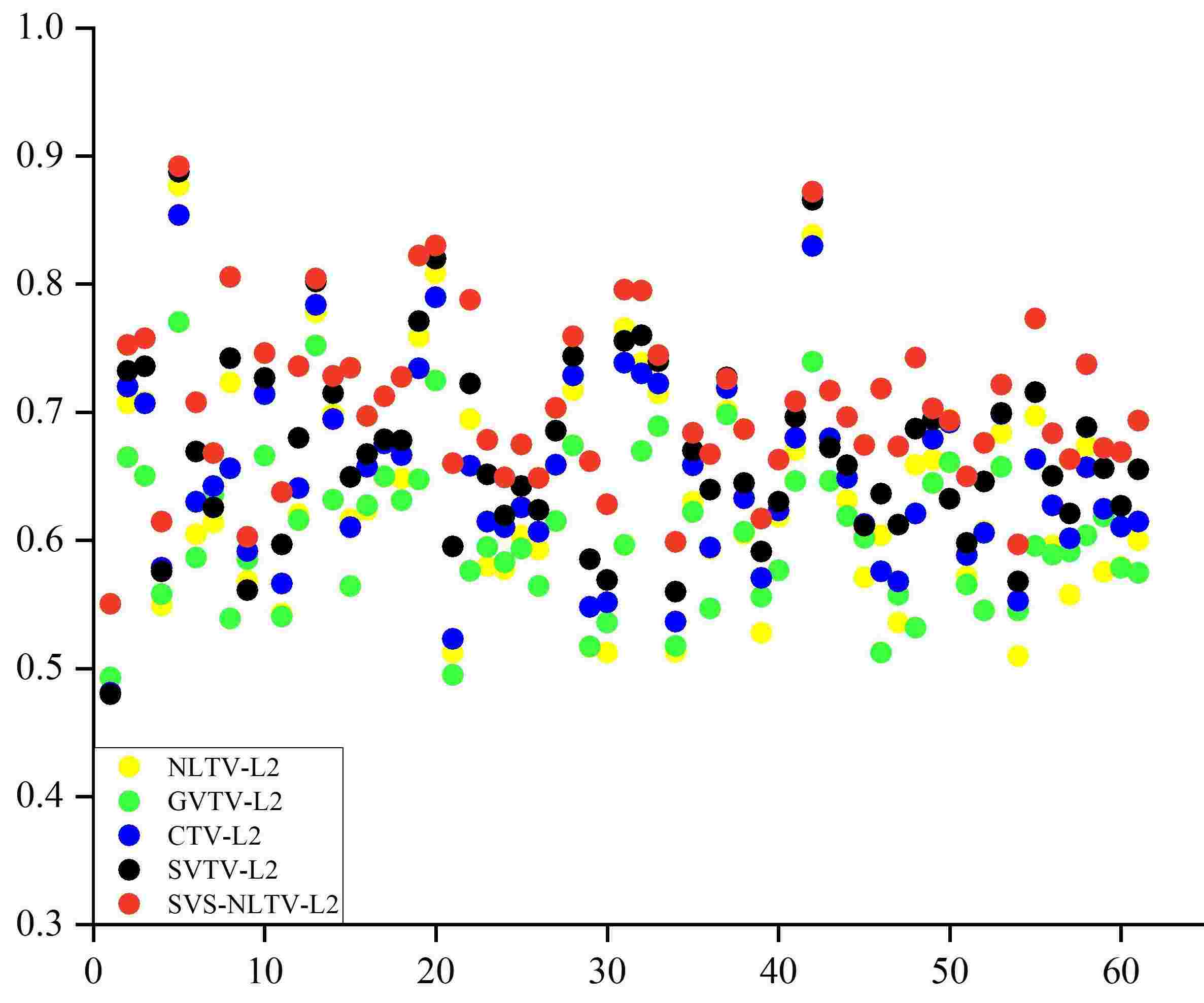}
\end{subfigure}
\begin{subfigure}
\centering
\includegraphics[width=0.23\textwidth]{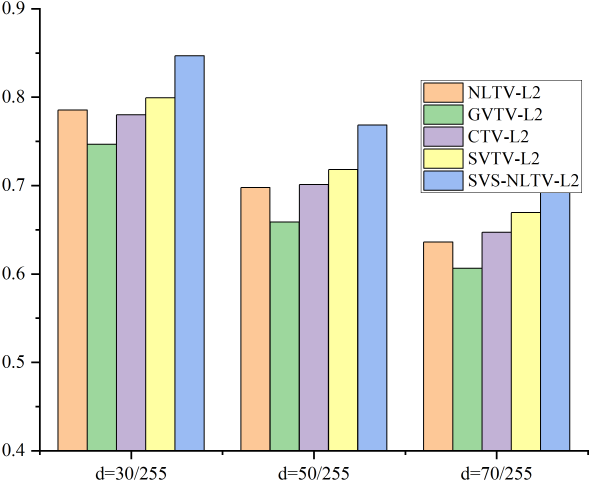}
\end{subfigure}
\caption{First to third: The spatial distributions of QSSIM values of the restored results by using different methods corresponding to d = 30/255, 50/255, 70/255 respectively; Fourth: the histogram of the average QSSIM values of the restored results by using different methods.}
\label{g-qssim}
\end{figure}

\subsection{Image denoising \uppercase\expandafter{\romannumeral1}: Gaussian noise}

In this section, we use 60 images taken from Berkeley Segmentation Database \cite{martin2001database} to test the proposed $\stvc$ plus L2 fidelity model for color image restoration with respect to different noise levels of degradation. We compare CTV-L2\cite{bresson2008fast}, GVTV-L2\cite{rodriguez2009generalized}, SVTV-L2\cite{jia2019color}, NLTV-L2 \cite{Xiaoqun2010Bregmanized} and the proposed $\stvct$ on the testing images. For the proposed $\stvct$ model, we set the parameter of the value channel to be $\mu$ = 0.05, the parameters $\lambda$, $\delta$ in Bregman iteration to be $\lambda =1$, $\delta = 1$. For the regularization parameter $\alpha$, we set a range of [$\frac{\sqrt{N}}{1000}$, $\frac{\sqrt{N}}{10}$] with a step size of 0.01 for both $\stvc$ model and NLTV model where $N$ is the total pixel numbers. The regularization parameter ($\lambda$) range for CTV model is set to be [$\frac{\sqrt{N}}{1000}$, $\frac{\sqrt{N}}{10}$]. For SVTV and GVTV model, we set a range of [$\frac{\sqrt{N}}{1000}$, $\frac{\sqrt{N}}{10}$] for the parameter of the regularization parameter($\lambda$).

We assume each pixel in the three channels takes value between 0 and 255, and artificially add Gaussian noises of standard deviation 30/255, 50/255, 70/255 in each channel to degrade the ground-truth color images. We compute the PSNR, SSIM, QSSIM values and the S-CIELAB error value (pixel number) for each restored result by comparing it with the ground-truth image. By choosing the optimal value of the regularization parameter in terms of PSNR value for each testing method, we get the optimal restored result and the corresponding values of the measures. In Figures \ref{g-psnr}-\ref{g-qssim}, we give the spatial distributions of PSNR, SSIM, and QSSIM values of the restored results corresponding to d = 30/255, 50/255, 70/255 respectively. We also show the histograms of the average PSNR, SSIM, and QSSIM values. We clearly observe from the figures that the proposed $\stvct$ model provides almost all the best PSNR, SSIM, and QSSIM values.

As examples, we display 6 sets of restored results in Figures \ref{242078L2}, \ref{159002L2}, \ref{56028L2}, \ref{167083L2}, \ref{101084L2}, \ref{24077L2}. The restored results and the corresponding histograms of PSNR, SSIM, QSSIM, and S-CIELAB error values by using different methods are also given in the figures. Again we see from the histograms that the proposed SVS-NLTV-L2 model always give the best PSNR, SSIM, QSSIM, and S-CIELAB error values compared to other testing methods. The zoom-in parts and the spatial distributions of the pixels with S-CIELAB error larger than 15 units are give in Figures \ref{242078L2zoom}, \ref{159002L2zoom}, \ref{56028L2zoom}, \ref{167083L2zoom}, \ref{101084L2zoom}, \ref{24077L2zoom}. We see from the restored results that CTV, GVTV, and NLTV methods can not handle color artifacts because of less coupling of RGB channels and lack of saturation-value information. SVS-NLTV-L2 and SVTV-L2 are more effective in handling color artifacts because of the application of saturation-value color space, however, SVTV-L2 produces unsatisfactory denoised results due to the staircase effect of total variation regularization. As expected, because of the combination of saturation-value similarity and nonlocal technique, SVS-NLTV-L2 always give visually better restored results, in which the noise and the color artifacts are removed more thoroughly and the fine edge and texture information are better preserved. See especially the denoising effect in the background regions of Figures \ref{56028L2} and the sky region of Figure \ref{167083L2}, the texture preserving effect in the fur region of Figure \ref{159002L2} and the face region of Figure \ref{101084L2}. In summary, we remark here that the proposed SVS-NLTV-L2 model is very effective, efficient, and competitive in terms of visual quality and the testing criteria such as PSNR, SSIM, QSSIM and S-CIELAB color error, especially QSSIM and S-CIELAB color error which are corresponding to the color restoration.

\begin{figure}[htbp]
\centering
\tabcolsep=1pt
\begin{tabular}{ccccc}
    \begin{minipage}[c]{0.184\textwidth} \centering
   \subfigure{\includegraphics[width =\textwidth]{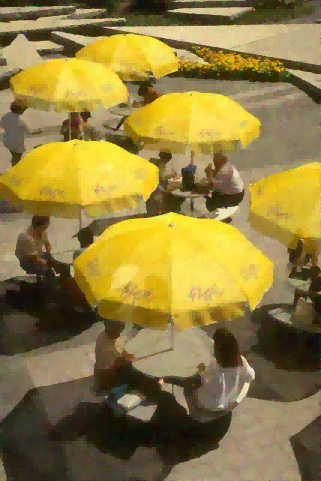}}
    \end{minipage} &
        \begin{minipage}[c]{0.184\textwidth} \centering
   \subfigure{\includegraphics[width =\textwidth]{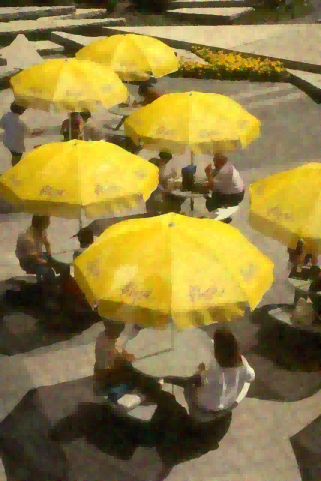}}
    \end{minipage} &
        \begin{minipage}[c]{0.184\textwidth} \centering
   \subfigure{\includegraphics[width =\textwidth]{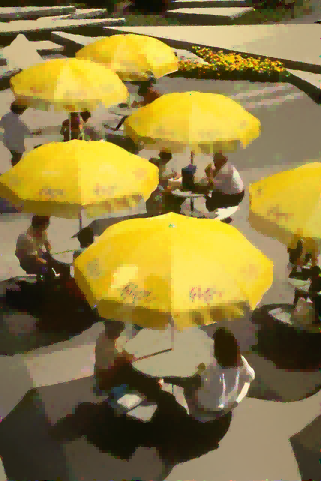}}
    \end{minipage} &
        \begin{minipage}[c]{0.184\textwidth} \centering
   \subfigure{\includegraphics[width = \textwidth]{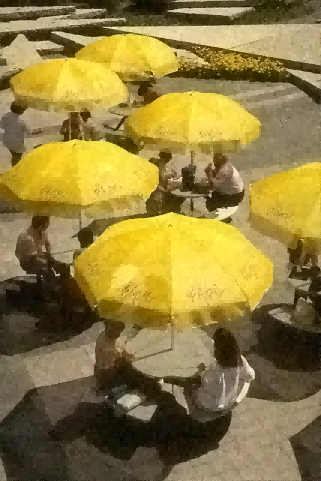}}
    \end{minipage} &
        \begin{minipage}[c]{0.184\textwidth} \centering
   \subfigure{\includegraphics[width = \textwidth]{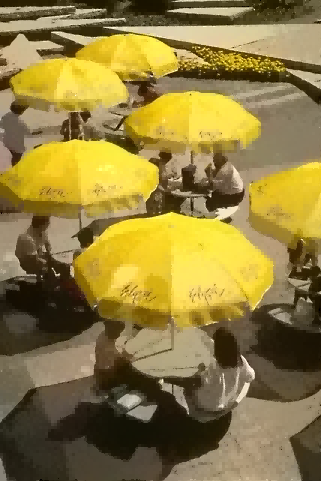}}
    \end{minipage} \\
    \begin{minipage}[c]{0.184\textwidth} \centering
   \subfigure{\includegraphics[width =\textwidth]{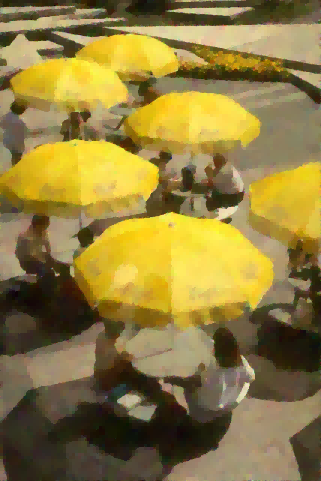}}
    \end{minipage} &
        \begin{minipage}[c]{0.184\textwidth} \centering
\subfigure{\includegraphics[width =\textwidth]{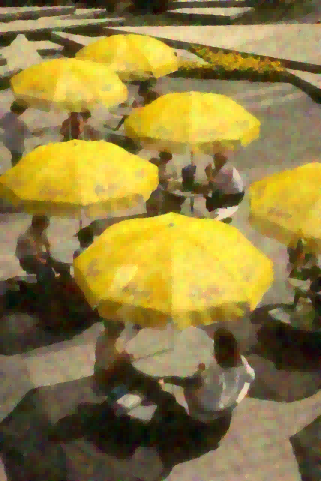}}
    \end{minipage} &
        \begin{minipage}[c]{0.184\textwidth} \centering
\subfigure{\includegraphics[width=\textwidth]{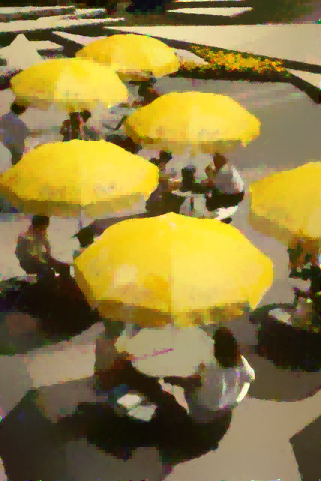}}
    \end{minipage} &
        \begin{minipage}[c]{0.184\textwidth} \centering
\subfigure{\includegraphics[width =\textwidth]{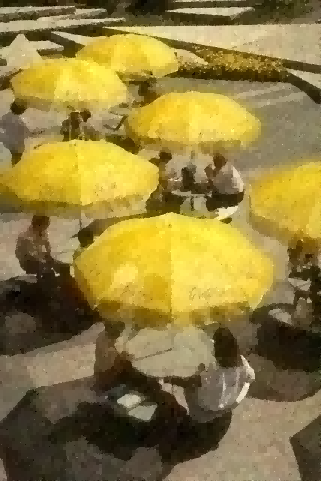}}
    \end{minipage} &
        \begin{minipage}[c]{0.184\textwidth} \centering
\subfigure{\includegraphics[width =\textwidth]{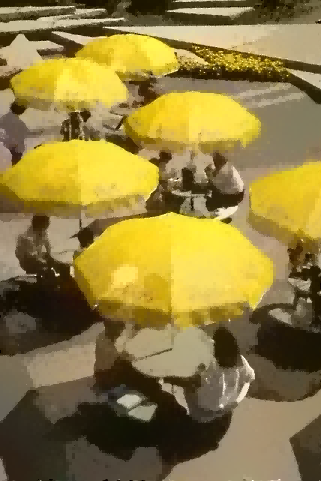}}
    \end{minipage} \\
    \begin{minipage}[c]{0.184\textwidth} \centering
\subfigure{\includegraphics[width =\textwidth]{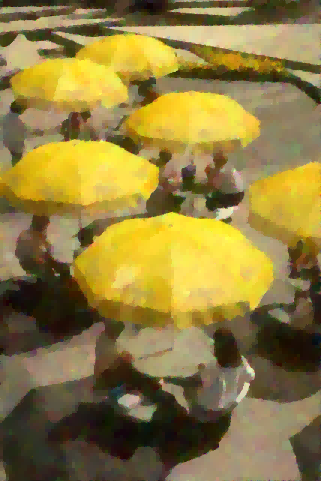}}
    \end{minipage} &
        \begin{minipage}[c]{0.184\textwidth} \centering
\subfigure{\includegraphics[width=\textwidth]{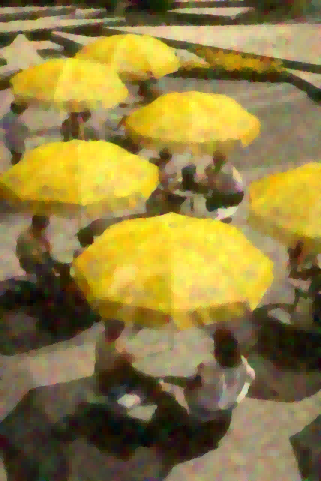}}
    \end{minipage} &
        \begin{minipage}[c]{0.184\textwidth} \centering
\subfigure{\includegraphics[width =\textwidth]{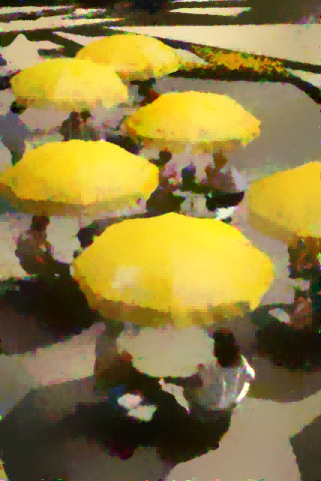}}
    \end{minipage} &
        \begin{minipage}[c]{0.184\textwidth} \centering
\subfigure{\includegraphics[width =\textwidth]{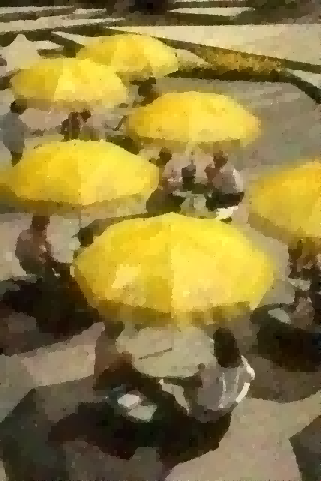}}
    \end{minipage} &
        \begin{minipage}[c]{0.184\textwidth} \centering
\subfigure{\includegraphics[width =\textwidth]{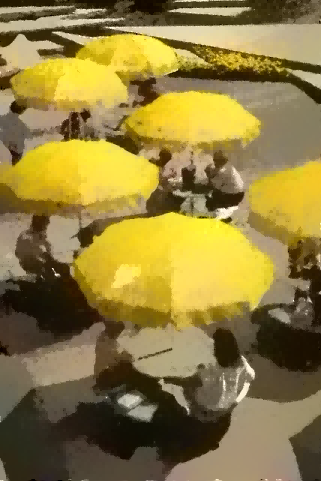}}
    \end{minipage} \\
\end{tabular}
\begin{subfigure}
\centering
\includegraphics[width=0.23\textwidth]{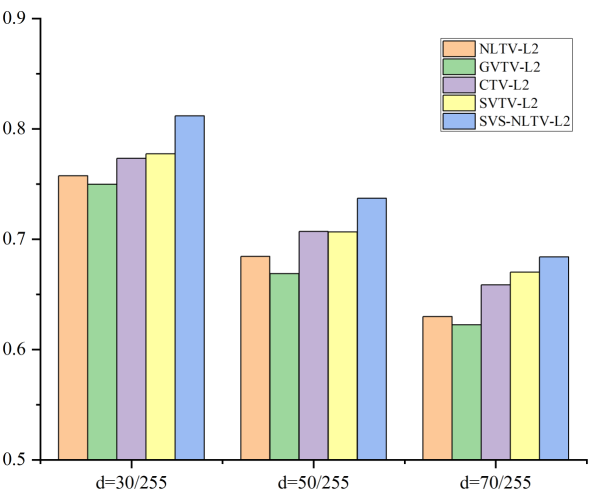}
\end{subfigure}
\begin{subfigure}
\centering
\includegraphics[width=0.23\textwidth]{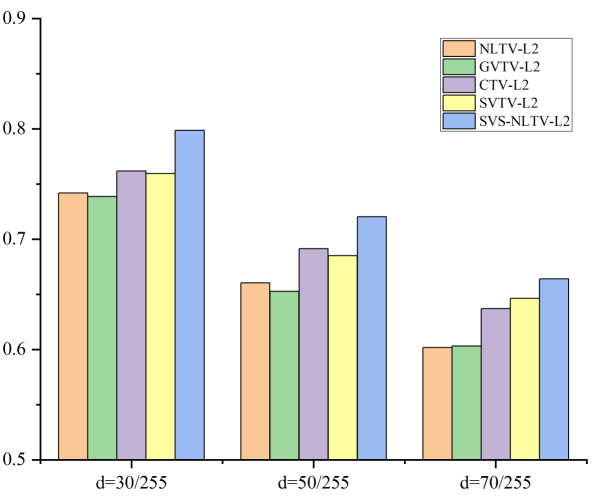}
\end{subfigure}
\begin{subfigure}
\centering
\includegraphics[width=0.23\textwidth]{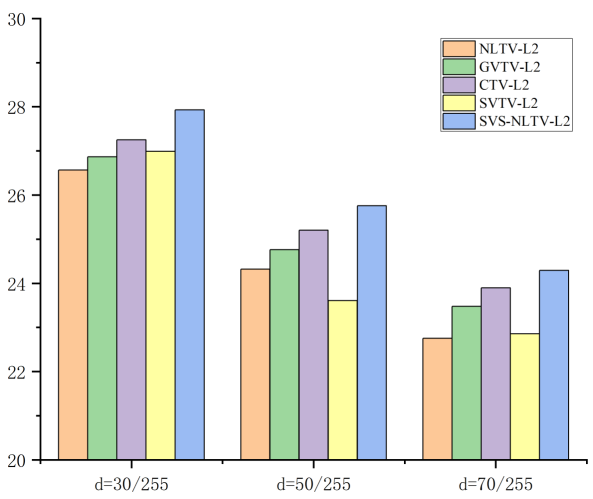}
\end{subfigure}
\begin{subfigure}
\centering
\includegraphics[width=0.23\textwidth]{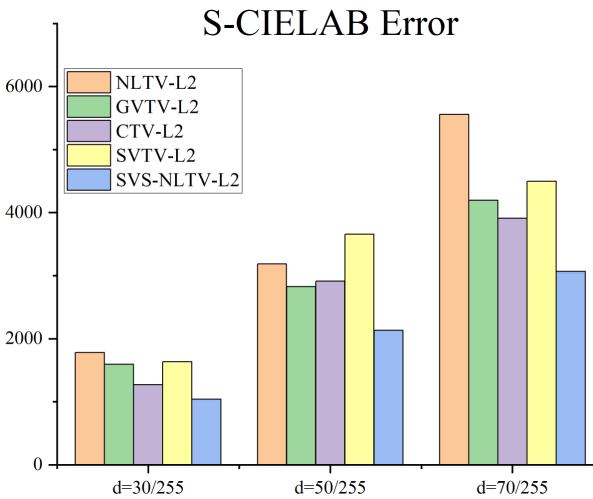}
\end{subfigure}
\caption{The first three rows: top to bottom: degraded and restored images with noise level d = 30/255, 50/255, 70/255 respectively; left to right: the restored results by using CTV, GVTV, NLTV, SVTV, and SVS-NLTV respectively. The fourth row: the histograms of measure values by using different methods.}
\label{242078L2}
\end{figure}

\begin{figure}[htbp]
\centering
\tabcolsep=1pt
\begin{minipage}[c]{0.21\textwidth} \centering
    {\includegraphics[height=\textwidth,width=\textwidth]{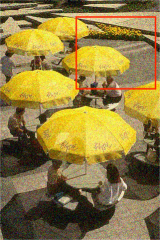}}
    \end{minipage}
\begin{tabular}{ccccccc}
    \begin{minipage}[c]{0.1\textwidth} \centering
    {\includegraphics[height=\textwidth,width=\textwidth]{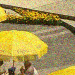}}
    \end{minipage} &
    \begin{minipage}[c]{0.1\textwidth} \centering
    {\includegraphics[height=\textwidth,width=\textwidth]{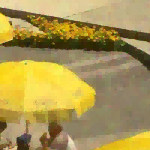}}
    \end{minipage} &
    \begin{minipage}[c]{0.1\textwidth} \centering
    {\includegraphics[height=\textwidth,width=\textwidth]{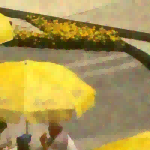}}
    \end{minipage} &
    \begin{minipage}[c]{0.1\textwidth} \centering
    {\includegraphics[height=\textwidth,width=\textwidth]{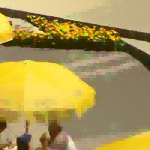}}
    \end{minipage} &
    \begin{minipage}[c]{0.1\textwidth} \centering
    {\includegraphics[height=\textwidth,width=\textwidth]{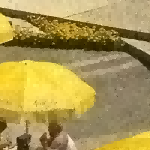}}
    \end{minipage} &
    \begin{minipage}[c]{0.1\textwidth} \centering
    {\includegraphics[height=\textwidth,width=\textwidth]{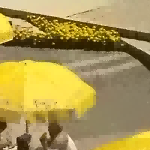}}
    \end{minipage} &
    \begin{minipage}[c]{0.1\textwidth} \centering
    {\includegraphics[height=\textwidth,width=\textwidth]{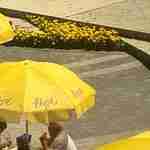}}
    \end{minipage}\\ \specialrule{0em}{1pt}{1pt}
     \begin{minipage}[c]{0.1\textwidth} \centering
    {\includegraphics[height=\textwidth,width=\textwidth]{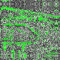}}
    \end{minipage} &
     \begin{minipage}[c]{0.1\textwidth} \centering
    {\includegraphics[height=\textwidth,width=\textwidth]{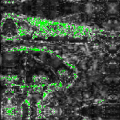}}
    \end{minipage} &
     \begin{minipage}[c]{0.1\textwidth} \centering
    {\includegraphics[height=\textwidth,width=\textwidth]{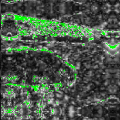}}
    \end{minipage} &
     \begin{minipage}[c]{0.1\textwidth} \centering
    {\includegraphics[height=\textwidth,width=\textwidth]{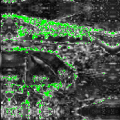}}
    \end{minipage} &
     \begin{minipage}[c]{0.1\textwidth} \centering
    {\includegraphics[height=\textwidth,width=\textwidth]{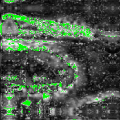}}
    \end{minipage} &
     \begin{minipage}[c]{0.1\textwidth} \centering
    {\includegraphics[height=\textwidth,width=\textwidth]{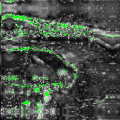}}
    \end{minipage}\\
    \end{tabular}\\
    \begin{minipage}[c]{0.21\textwidth} \centering
    {\includegraphics[height=\textwidth,width=\textwidth]{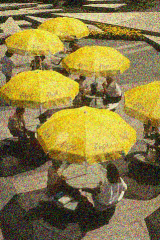}}
    \end{minipage}
\begin{tabular}{ccccccc}
    \begin{minipage}[c]{0.1\textwidth} \centering
    {\includegraphics[height=\textwidth,width=\textwidth]{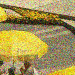}}
    \end{minipage} &
    \begin{minipage}[c]{0.1\textwidth} \centering
    {\includegraphics[height=\textwidth,width=\textwidth]{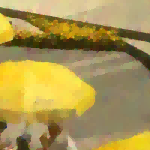}}
    \end{minipage} &
    \begin{minipage}[c]{0.1\textwidth} \centering
    {\includegraphics[height=\textwidth,width=\textwidth]{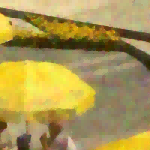}}
    \end{minipage} &
    \begin{minipage}[c]{0.1\textwidth} \centering
    {\includegraphics[height=\textwidth,width=\textwidth]{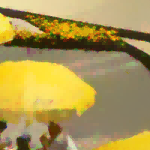}}
    \end{minipage} &
    \begin{minipage}[c]{0.1\textwidth} \centering
    {\includegraphics[height=\textwidth,width=\textwidth]{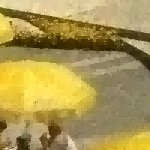}}
    \end{minipage} &
    \begin{minipage}[c]{0.1\textwidth} \centering
    {\includegraphics[height=\textwidth,width=\textwidth]{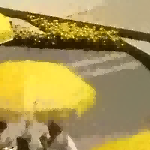}}
    \end{minipage} &
     \begin{minipage}[c]{0.1\textwidth} \centering
    {\includegraphics[height=\textwidth,width=\textwidth]{figs/242078,denoise/242078_zoom.jpg}}
    \end{minipage}\\ \specialrule{0em}{1pt}{1pt}
     \begin{minipage}[c]{0.1\textwidth} \centering
    {\includegraphics[height=\textwidth,width=\textwidth]{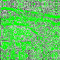}}
    \end{minipage} &
     \begin{minipage}[c]{0.1\textwidth} \centering
    {\includegraphics[height=\textwidth,width=\textwidth]{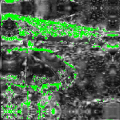}}
    \end{minipage} &
     \begin{minipage}[c]{0.1\textwidth} \centering
    {\includegraphics[height=\textwidth,width=\textwidth]{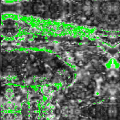}}
    \end{minipage} &
     \begin{minipage}[c]{0.1\textwidth} \centering
    {\includegraphics[height=\textwidth,width=\textwidth]{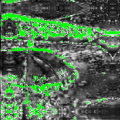}}
    \end{minipage} &
     \begin{minipage}[c]{0.1\textwidth} \centering
    {\includegraphics[height=\textwidth,width=\textwidth]{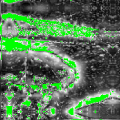}}
    \end{minipage} &
     \begin{minipage}[c]{0.1\textwidth} \centering
    {\includegraphics[height=\textwidth,width=\textwidth]{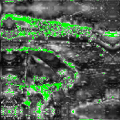}}
    \end{minipage}\\
    \end{tabular}\\
     \begin{minipage}[c]{0.21\textwidth} \centering
    {\includegraphics[height=\textwidth,width=\textwidth]{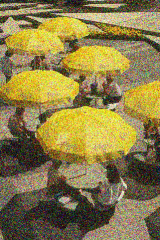}}
    \end{minipage}
\begin{tabular}{ccccccc}
    \begin{minipage}[c]{0.1\textwidth} \centering
    {\includegraphics[height=\textwidth,width=\textwidth]{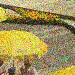}}
    \end{minipage} &
    \begin{minipage}[c]{0.1\textwidth} \centering
    {\includegraphics[height=\textwidth,width=\textwidth]{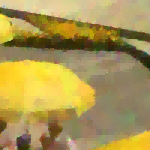}}
    \end{minipage} &
    \begin{minipage}[c]{0.1\textwidth} \centering
    {\includegraphics[height=\textwidth,width=\textwidth]{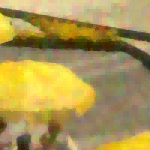}}
    \end{minipage} &
    \begin{minipage}[c]{0.1\textwidth} \centering
    {\includegraphics[height=\textwidth,width=\textwidth]{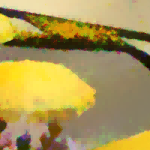}}
    \end{minipage} &
    \begin{minipage}[c]{0.1\textwidth} \centering
    {\includegraphics[height=\textwidth,width=\textwidth]{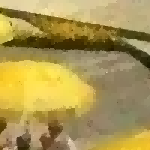}}
    \end{minipage} &
    \begin{minipage}[c]{0.1\textwidth} \centering
    {\includegraphics[height=\textwidth,width=\textwidth]{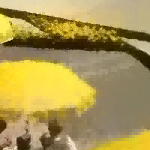}}
    \end{minipage} &
     \begin{minipage}[c]{0.1\textwidth} \centering
    {\includegraphics[height=\textwidth,width=\textwidth]{figs/242078,denoise/242078_zoom.jpg}}
    \end{minipage}\\ \specialrule{0em}{1pt}{1pt}
     \begin{minipage}[c]{0.1\textwidth} \centering
    {\includegraphics[height=\textwidth,width=\textwidth]{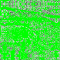}}
    \end{minipage} &
     \begin{minipage}[c]{0.1\textwidth} \centering
    {\includegraphics[height=\textwidth,width=\textwidth]{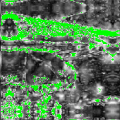}}
    \end{minipage} &
     \begin{minipage}[c]{0.1\textwidth} \centering
    {\includegraphics[height=\textwidth,width=\textwidth]{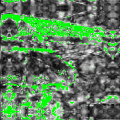}}
    \end{minipage} &
     \begin{minipage}[c]{0.1\textwidth} \centering
    {\includegraphics[height=\textwidth,width=\textwidth]{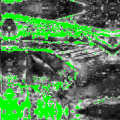}}
    \end{minipage} &
     \begin{minipage}[c]{0.1\textwidth} \centering
    {\includegraphics[height=\textwidth,width=\textwidth]{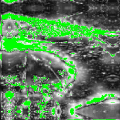}} 
    \end{minipage} &
     \begin{minipage}[c]{0.1\textwidth} \centering
    {\includegraphics[height=\textwidth,width=\textwidth]{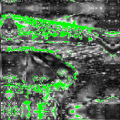}}
    \end{minipage}\\
    \end{tabular}\\
    \caption{Top to bottom: the corresponding results with noise level d = 30/255, 50/255, 70/255 respectively. The results include the noisy image (left large picture), the corresponding zoom-in parts of the noise image, the restored results by using CTV, HTV, NLTV, SVTV, SVS-NLTV, the ground-truth image respectively. The spatial distributions of S-CIELAB error (larger than 15 units) are also shown.}
    \label{242078L2zoom}
\end{figure}

\begin{figure}[htbp]
\centering
\tabcolsep=1pt
\begin{tabular}{ccccc}
    \begin{minipage}[c]{0.184\textwidth} \centering
   \subfigure{\includegraphics[width =\textwidth]{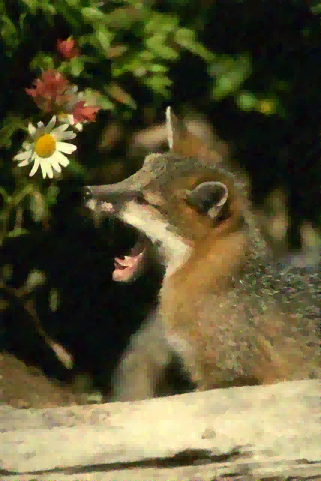}}
    \end{minipage} &
        \begin{minipage}[c]{0.184\textwidth} \centering
   \subfigure{\includegraphics[width =\textwidth]{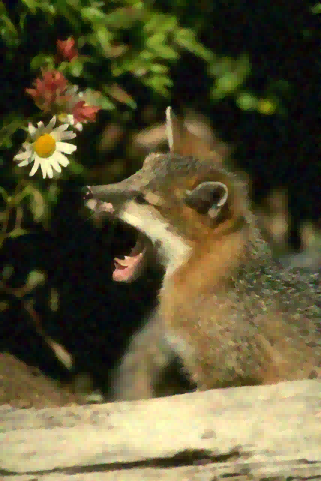}}
    \end{minipage} &
        \begin{minipage}[c]{0.184\textwidth} \centering
   \subfigure{\includegraphics[width =\textwidth]{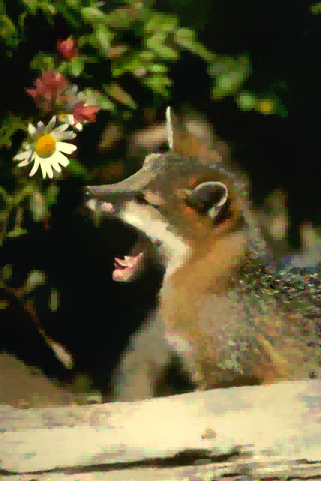}}
    \end{minipage} &
        \begin{minipage}[c]{0.184\textwidth} \centering
   \subfigure{\includegraphics[width = \textwidth]{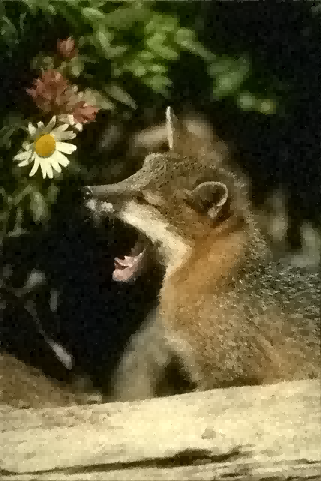}}
    \end{minipage} &
        \begin{minipage}[c]{0.184\textwidth} \centering
   \subfigure{\includegraphics[width = \textwidth]{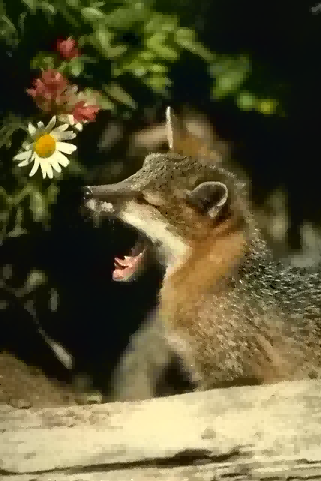}}
    \end{minipage} \\
    \begin{minipage}[c]{0.184\textwidth} \centering
   \subfigure{\includegraphics[width =\textwidth]{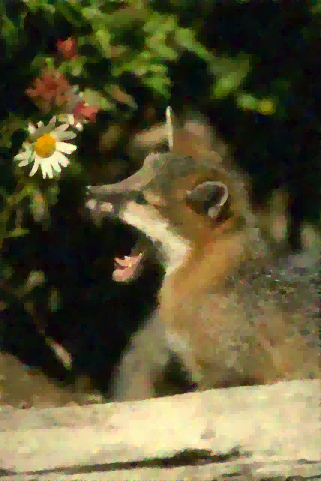}}
    \end{minipage} &
        \begin{minipage}[c]{0.184\textwidth} \centering
\subfigure{\includegraphics[width =\textwidth]{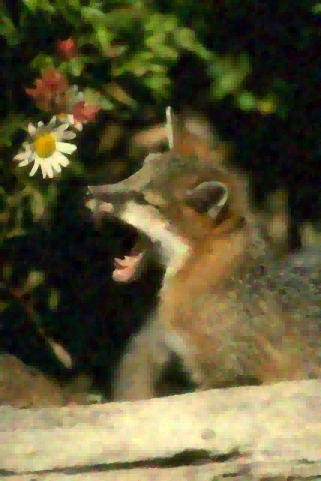}}
    \end{minipage} &
        \begin{minipage}[c]{0.184\textwidth} \centering
\subfigure{\includegraphics[width=\textwidth]{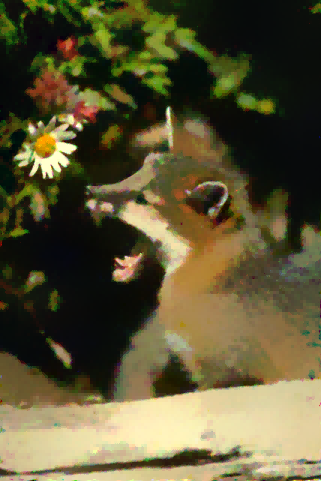}}
    \end{minipage} &
        \begin{minipage}[c]{0.184\textwidth} \centering
\subfigure{\includegraphics[width =\textwidth]{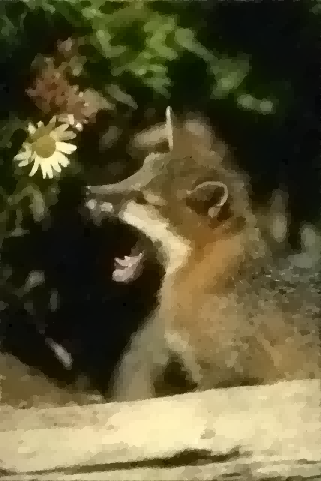}}
    \end{minipage} &
        \begin{minipage}[c]{0.184\textwidth} \centering
\subfigure{\includegraphics[width =\textwidth]{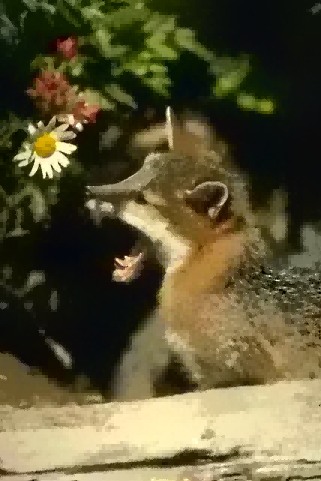}}
    \end{minipage} \\
    \begin{minipage}[c]{0.184\textwidth} \centering
\subfigure{\includegraphics[width =\textwidth]{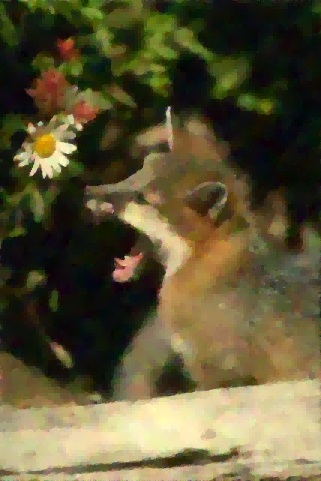}}
    \end{minipage} &
        \begin{minipage}[c]{0.184\textwidth} \centering
\subfigure{\includegraphics[width=\textwidth]{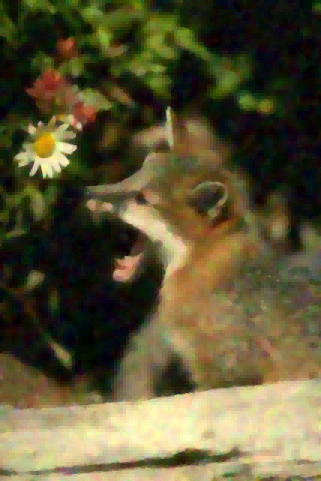}}
    \end{minipage} &
        \begin{minipage}[c]{0.184\textwidth} \centering
\subfigure{\includegraphics[width =\textwidth]{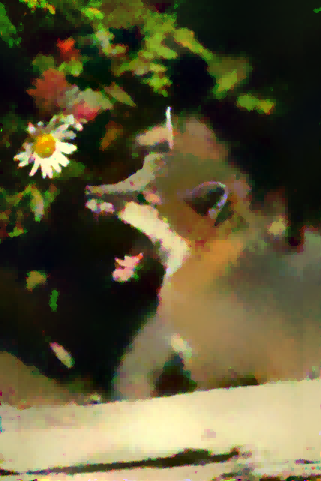}}
    \end{minipage} &
        \begin{minipage}[c]{0.184\textwidth} \centering
\subfigure{\includegraphics[width =\textwidth]{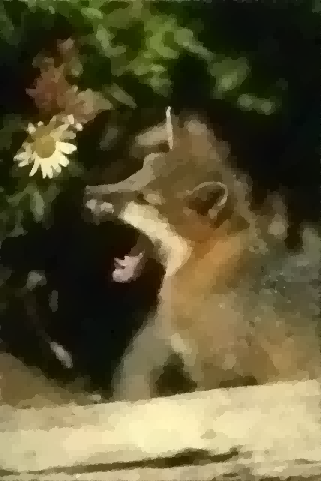}}
    \end{minipage} &
        \begin{minipage}[c]{0.184\textwidth} \centering
\subfigure{\includegraphics[width =\textwidth]{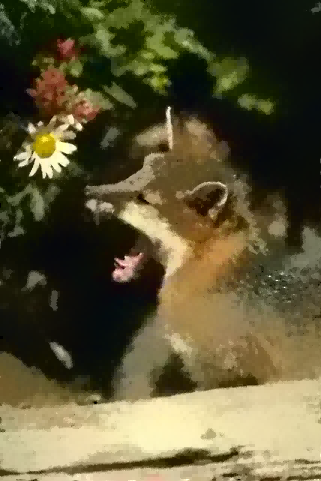}}
    \end{minipage} \\
\end{tabular}
\begin{subfigure}
\centering
\includegraphics[width=0.23\textwidth]{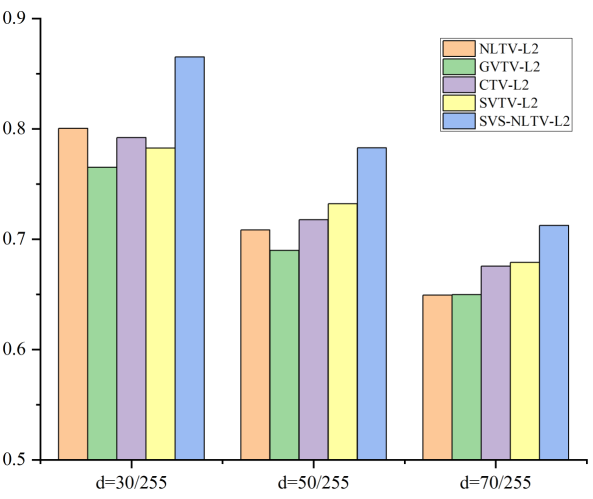}
\end{subfigure}
\begin{subfigure}
\centering
\includegraphics[width=0.23\textwidth]{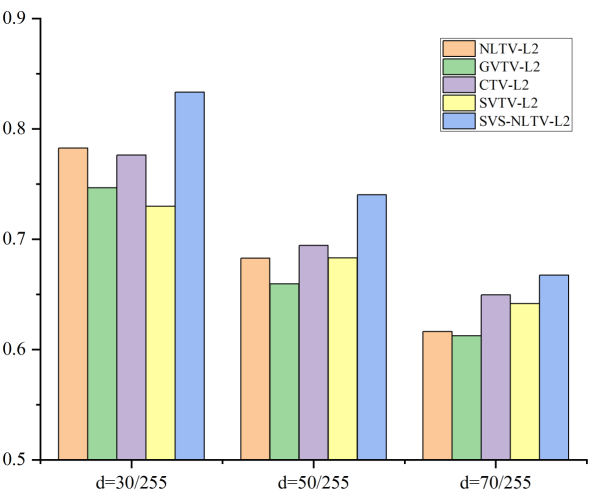}
\end{subfigure}
\begin{subfigure}
\centering
\includegraphics[width=0.23\textwidth]{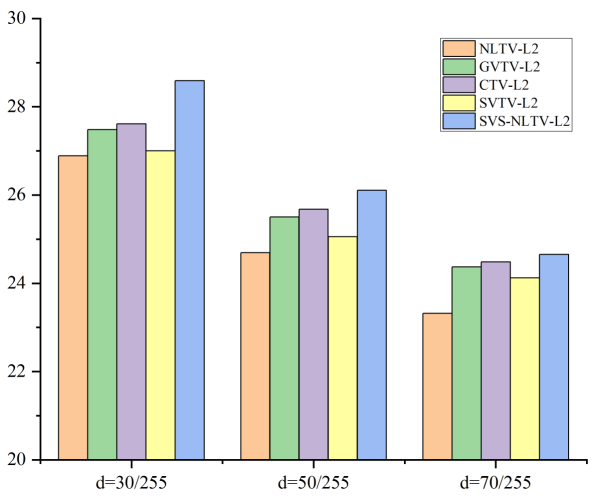}
\end{subfigure}
\begin{subfigure}
\centering
\includegraphics[width=0.23\textwidth]{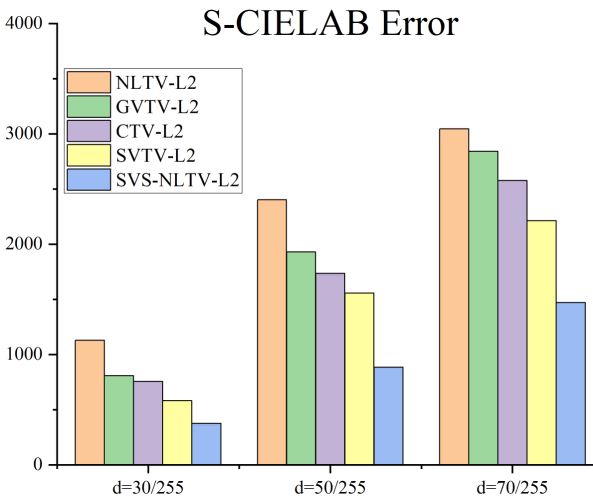}
\end{subfigure}
\caption{The first three rows: top to bottom: degraded and restored images with noise level d = 30/255, 50/255, 70/255 respectively; left to right: the restored results by using CTV, GVTV, NLTV, SVTV, and SVS-NLTV respectively. The fourth row: the histograms of measure values by using different methods.}
\label{159002L2}
\end{figure}

\begin{figure}[htbp]
\centering
\tabcolsep=1pt
\begin{minipage}[c]{0.21\textwidth} \centering
    {\includegraphics[height=\textwidth,width=\textwidth]{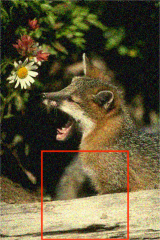}}
    \end{minipage}
\begin{tabular}{ccccccc}
    \begin{minipage}[c]{0.1\textwidth} \centering
    {\includegraphics[height=\textwidth,width=\textwidth]{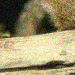}}
    \end{minipage} &
    \begin{minipage}[c]{0.1\textwidth} \centering
    {\includegraphics[height=\textwidth,width=\textwidth]{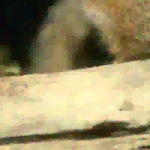}}
    \end{minipage} &
    \begin{minipage}[c]{0.1\textwidth} \centering
    {\includegraphics[height=\textwidth,width=\textwidth]{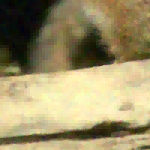}}
    \end{minipage} &
    \begin{minipage}[c]{0.1\textwidth} \centering
    {\includegraphics[height=\textwidth,width=\textwidth]{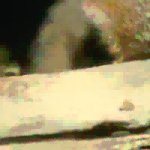}}
    \end{minipage} &
    \begin{minipage}[c]{0.1\textwidth} \centering
    {\includegraphics[height=\textwidth,width=\textwidth]{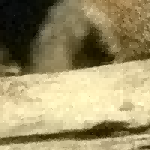}}
    \end{minipage} &
    \begin{minipage}[c]{0.1\textwidth} \centering
    {\includegraphics[height=\textwidth,width=\textwidth]{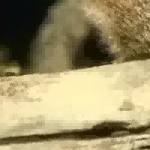}}
    \end{minipage} &
    \begin{minipage}[c]{0.1\textwidth} \centering
    {\includegraphics[height=\textwidth,width=\textwidth]{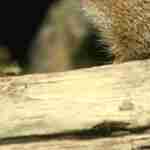}}
    \end{minipage}\\ \specialrule{0em}{1pt}{1pt}
     \begin{minipage}[c]{0.1\textwidth} \centering
    {\includegraphics[height=\textwidth,width=\textwidth]{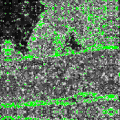}}
    \end{minipage} &
     \begin{minipage}[c]{0.1\textwidth} \centering
    {\includegraphics[height=\textwidth,width=\textwidth]{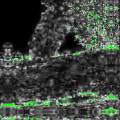}}
    \end{minipage} &
     \begin{minipage}[c]{0.1\textwidth} \centering
    {\includegraphics[height=\textwidth,width=\textwidth]{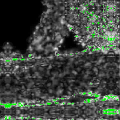}}
    \end{minipage} &
     \begin{minipage}[c]{0.1\textwidth} \centering
    {\includegraphics[height=\textwidth,width=\textwidth]{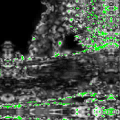}}
    \end{minipage} &
     \begin{minipage}[c]{0.1\textwidth} \centering
    {\includegraphics[height=\textwidth,width=\textwidth]{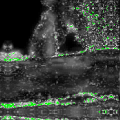}}
    \end{minipage} &
     \begin{minipage}[c]{0.1\textwidth} \centering
    {\includegraphics[height=\textwidth,width=\textwidth]{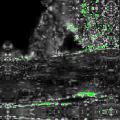}}
    \end{minipage}\\
    \end{tabular}
    \begin{minipage}[c]{0.21\textwidth} \centering
    {\includegraphics[height=\textwidth,width=\textwidth]{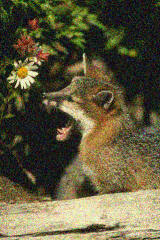}}
    \end{minipage}
\begin{tabular}{ccccccc}
    \begin{minipage}[c]{0.1\textwidth} \centering
    {\includegraphics[height=\textwidth,width=\textwidth]{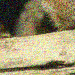}}
    \end{minipage} &
    \begin{minipage}[c]{0.1\textwidth} \centering
    {\includegraphics[height=\textwidth,width=\textwidth]{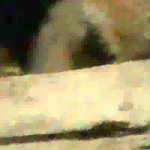}}
    \end{minipage} &
    \begin{minipage}[c]{0.1\textwidth} \centering
    {\includegraphics[height=\textwidth,width=\textwidth]{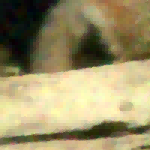}}
    \end{minipage} &
    \begin{minipage}[c]{0.1\textwidth} \centering
    {\includegraphics[height=\textwidth,width=\textwidth]{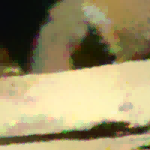}}
    \end{minipage} &
    \begin{minipage}[c]{0.1\textwidth} \centering
    {\includegraphics[height=\textwidth,width=\textwidth]{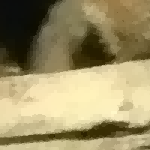}}
    \end{minipage} &
    \begin{minipage}[c]{0.1\textwidth} \centering
    {\includegraphics[height=\textwidth,width=\textwidth]{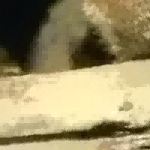}}
    \end{minipage} &
    \begin{minipage}[c]{0.1\textwidth} \centering
    {\includegraphics[height=\textwidth,width=\textwidth]{figs/159002,denoise/159002_zoom.jpg}}
    \end{minipage}\\ \specialrule{0em}{1pt}{1pt}
     \begin{minipage}[c]{0.1\textwidth} \centering
    {\includegraphics[height=\textwidth,width=\textwidth]{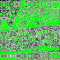}}
    \end{minipage} &
     \begin{minipage}[c]{0.1\textwidth} \centering
    {\includegraphics[height=\textwidth,width=\textwidth]{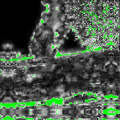}}
    \end{minipage} &
     \begin{minipage}[c]{0.1\textwidth} \centering
    {\includegraphics[height=\textwidth,width=\textwidth]{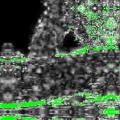}}
    \end{minipage} &
     \begin{minipage}[c]{0.1\textwidth} \centering
    {\includegraphics[height=\textwidth,width=\textwidth]{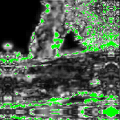}}
    \end{minipage} &
     \begin{minipage}[c]{0.1\textwidth} \centering
    {\includegraphics[height=\textwidth,width=\textwidth]{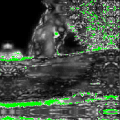}}
    \end{minipage} &
     \begin{minipage}[c]{0.1\textwidth} \centering
    {\includegraphics[height=\textwidth,width=\textwidth]{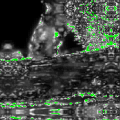}}
    \end{minipage}\\
    \end{tabular}
     \begin{minipage}[c]{0.21\textwidth} \centering
    {\includegraphics[height=\textwidth,width=\textwidth]{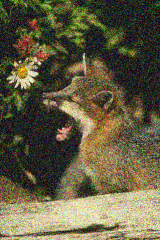}}
    \end{minipage}
\begin{tabular}{ccccccc}
    \begin{minipage}[c]{0.1\textwidth} \centering
    {\includegraphics[height=\textwidth,width=\textwidth]{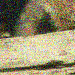}}
    \end{minipage} &
    \begin{minipage}[c]{0.1\textwidth} \centering
    {\includegraphics[height=\textwidth,width=\textwidth]{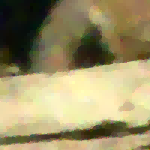}}
    \end{minipage} &
    \begin{minipage}[c]{0.1\textwidth} \centering
    {\includegraphics[height=\textwidth,width=\textwidth]{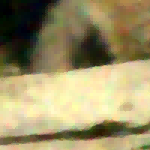}}
    \end{minipage} &
    \begin{minipage}[c]{0.1\textwidth} \centering
    {\includegraphics[height=\textwidth,width=\textwidth]{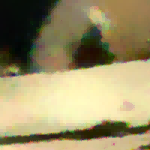}}
    \end{minipage} &
    \begin{minipage}[c]{0.1\textwidth} \centering
    {\includegraphics[height=\textwidth,width=\textwidth]{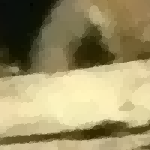}}
    \end{minipage} &
    \begin{minipage}[c]{0.1\textwidth} \centering
    {\includegraphics[height=\textwidth,width=\textwidth]{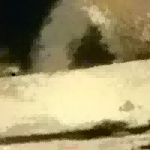}}
    \end{minipage} &
    \begin{minipage}[c]{0.1\textwidth} \centering
    {\includegraphics[height=\textwidth,width=\textwidth]{figs/159002,denoise/159002_zoom.jpg}}
    \end{minipage}\\ \specialrule{0em}{1pt}{1pt}
     \begin{minipage}[c]{0.1\textwidth} \centering
    {\includegraphics[height=\textwidth,width=\textwidth]{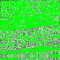}}
    \end{minipage} &
     \begin{minipage}[c]{0.1\textwidth} \centering
    {\includegraphics[height=\textwidth,width=\textwidth]{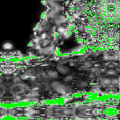}}
    \end{minipage} &
     \begin{minipage}[c]{0.1\textwidth} \centering
    {\includegraphics[height=\textwidth,width=\textwidth]{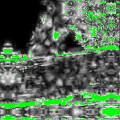}}
    \end{minipage} &
     \begin{minipage}[c]{0.1\textwidth} \centering
    {\includegraphics[height=\textwidth,width=\textwidth]{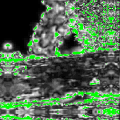}}
    \end{minipage} &
     \begin{minipage}[c]{0.1\textwidth} \centering
    {\includegraphics[height=\textwidth,width=\textwidth]{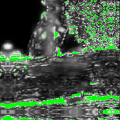}}
    \end{minipage} &
     \begin{minipage}[c]{0.1\textwidth} \centering
    {\includegraphics[height=\textwidth,width=\textwidth]{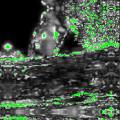}}
    \end{minipage}\\
    \end{tabular}\\
    \caption{Top to bottom: the corresponding results with noise level d = 30/255, 50/255, 70/255 respectively. The results include the noisy image (left large picture), the corresponding zoom-in parts of the noise image, the restored results by using CTV, HTV, NLTV, SVTV, SVS-NLTV, the ground-truth image respectively. The spatial distributions of S-CIELAB error (larger than 15 units) are also shown.}
    \label{159002L2zoom}
\end{figure}

\begin{figure}[htbp]
\centering
\tabcolsep=1pt
\begin{tabular}{ccccc}
    \begin{minipage}[c]{0.184\textwidth} \centering
   \subfigure{\includegraphics[width =\textwidth]{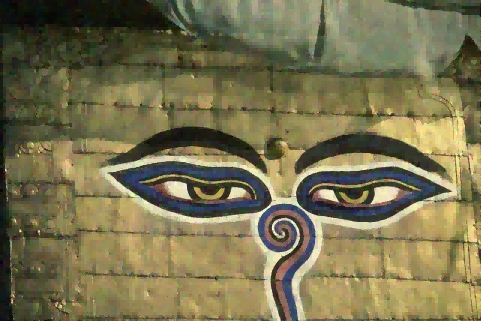}}
    \end{minipage} &
        \begin{minipage}[c]{0.184\textwidth} \centering
   \subfigure{\includegraphics[width =\textwidth]{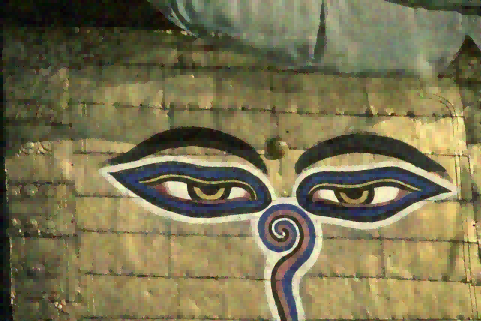}}
    \end{minipage} &
        \begin{minipage}[c]{0.184\textwidth} \centering
   \subfigure{\includegraphics[width =\textwidth]{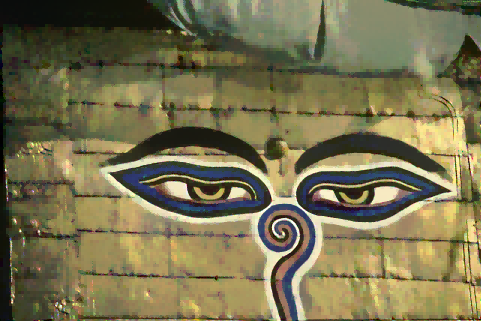}}
    \end{minipage} &
        \begin{minipage}[c]{0.184\textwidth} \centering
   \subfigure{\includegraphics[width = \textwidth]{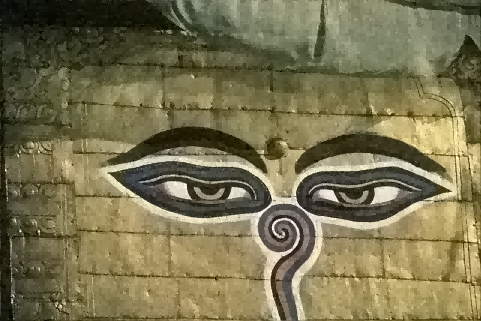}}
    \end{minipage} &
        \begin{minipage}[c]{0.184\textwidth} \centering
   \subfigure{\includegraphics[width = \textwidth]{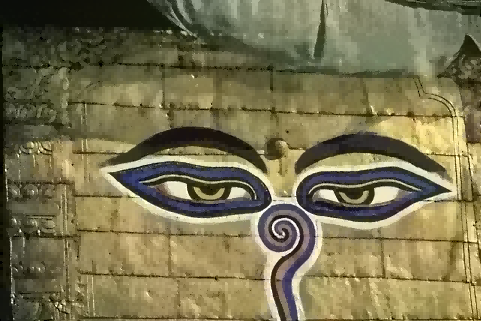}}
    \end{minipage} \\
    \begin{minipage}[c]{0.184\textwidth} \centering
   \subfigure{\includegraphics[width =\textwidth]{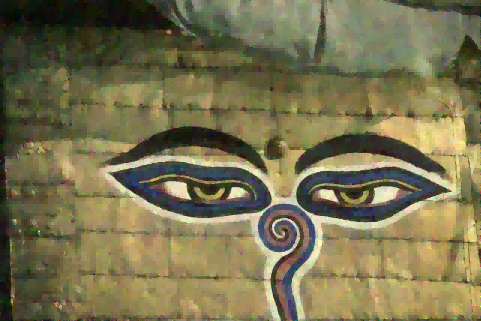}}
    \end{minipage} &
        \begin{minipage}[c]{0.184\textwidth} \centering
\subfigure{\includegraphics[width =\textwidth]{figs/56028,denoise/56028_GVTV-L2_sigma=30_a=008_qssim=072292_ssim=071285_psnr=254837.png}}
    \end{minipage} &
        \begin{minipage}[c]{0.184\textwidth} \centering
\subfigure{\includegraphics[width=\textwidth]{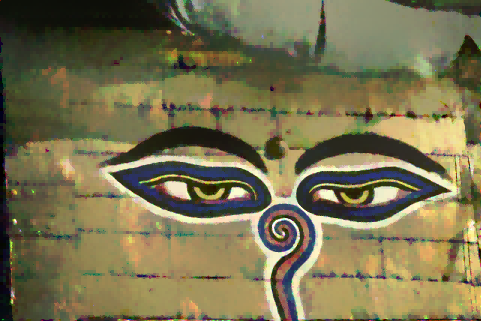}}
    \end{minipage} &
        \begin{minipage}[c]{0.184\textwidth} \centering
\subfigure{\includegraphics[width =\textwidth]{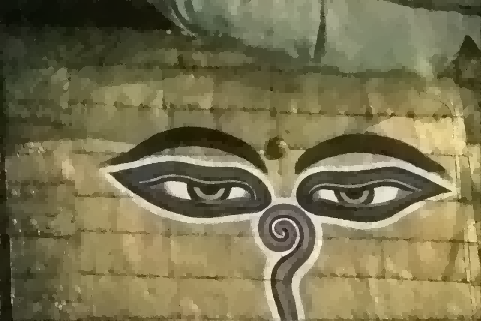}}
    \end{minipage} &
        \begin{minipage}[c]{0.184\textwidth} \centering
\subfigure{\includegraphics[width =\textwidth]{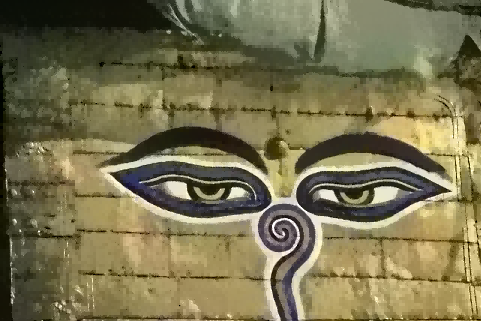}}
    \end{minipage} \\
    \begin{minipage}[c]{0.184\textwidth} \centering
\subfigure{\includegraphics[width =\textwidth]{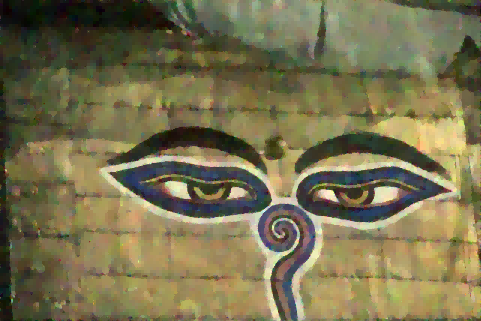}}
    \end{minipage} &
        \begin{minipage}[c]{0.184\textwidth} \centering
\subfigure{\includegraphics[width=\textwidth]{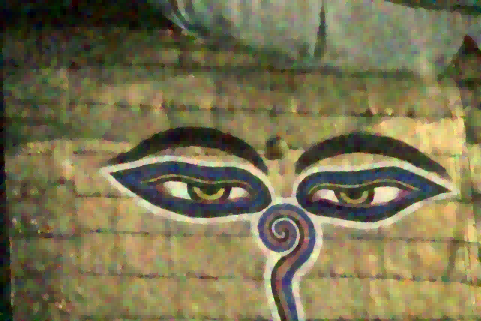}}
    \end{minipage} &
        \begin{minipage}[c]{0.184\textwidth} \centering
\subfigure{\includegraphics[width =\textwidth]{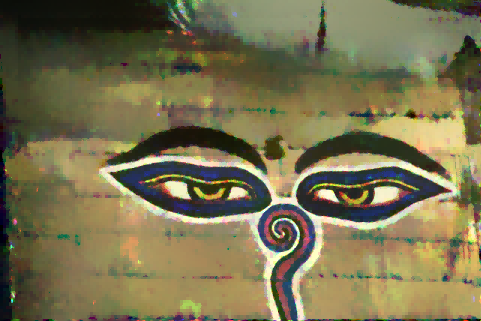}}
    \end{minipage} &
        \begin{minipage}[c]{0.184\textwidth} \centering
\subfigure{\includegraphics[width =\textwidth]{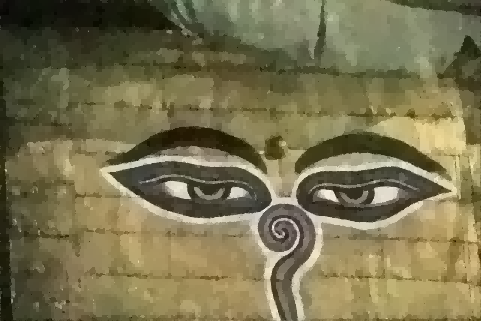}}
    \end{minipage} &
        \begin{minipage}[c]{0.184\textwidth} \centering
\subfigure{\includegraphics[width =\textwidth]{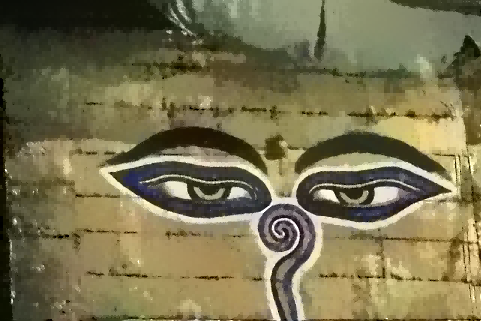}}
\end{minipage}\\
\end{tabular}
\begin{subfigure}
\centering
\includegraphics[width=0.23\textwidth]{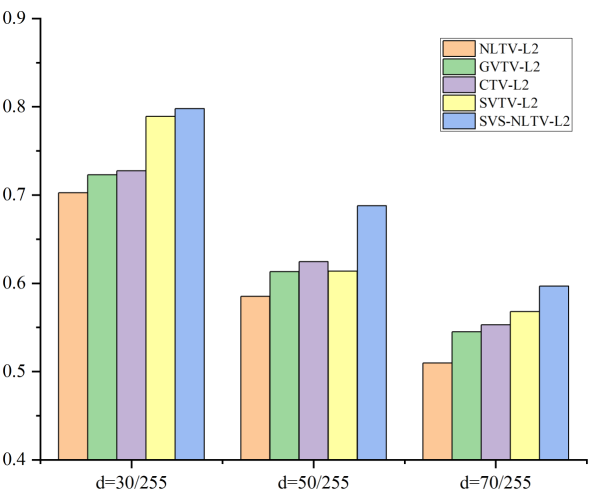}
\end{subfigure}
\begin{subfigure}
\centering
\includegraphics[width=0.23\textwidth]{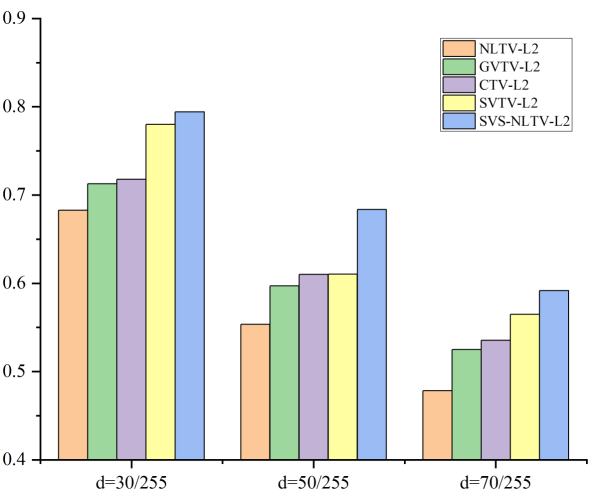}
\end{subfigure}
\begin{subfigure}
\centering
\includegraphics[width=0.23\textwidth]{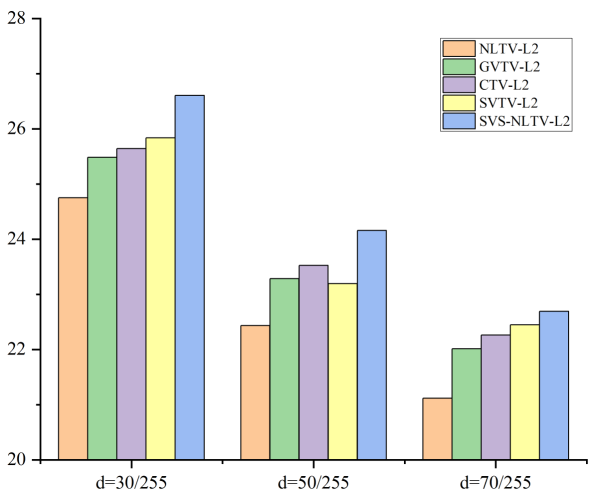}
\end{subfigure}
\begin{subfigure}
\centering
\includegraphics[width=0.23\textwidth]{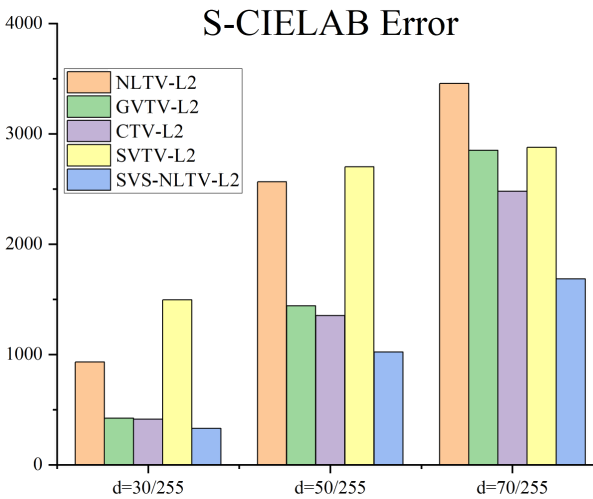}
\end{subfigure}
\caption{The first three rows: top to bottom: degraded and restored images with noise level d = 30/255, 50/255, 70/255 respectively; left to right: the restored results by using CTV, GVTV, NLTV, SVTV, and SVS-NLTV respectively. The fourth row: the histograms of measure values by using different methods.}
\label{56028L2}
\end{figure}

\begin{figure}[htbp]
\centering
\tabcolsep=1pt
\begin{minipage}[c]{0.21\textwidth} \centering
    {\includegraphics[height=\textwidth,width=\textwidth]{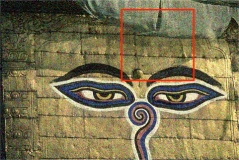}}
    \end{minipage}
\begin{tabular}{ccccccc}
    \begin{minipage}[c]{0.1\textwidth} \centering
    {\includegraphics[height=\textwidth,width=\textwidth]{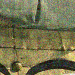}}
    \end{minipage} &
    \begin{minipage}[c]{0.1\textwidth} \centering
    {\includegraphics[height=\textwidth,width=\textwidth]{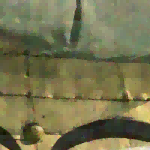}}
    \end{minipage} &
    \begin{minipage}[c]{0.1\textwidth} \centering
    {\includegraphics[height=\textwidth,width=\textwidth]{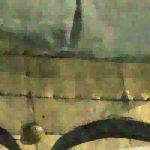}}
    \end{minipage} &
    \begin{minipage}[c]{0.1\textwidth} \centering
    {\includegraphics[height=\textwidth,width=\textwidth]{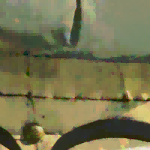}}
    \end{minipage} &
    \begin{minipage}[c]{0.1\textwidth} \centering
    {\includegraphics[height=\textwidth,width=\textwidth]{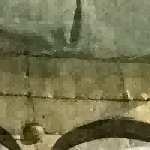}}
    \end{minipage} &
    \begin{minipage}[c]{0.1\textwidth} \centering
    {\includegraphics[height=\textwidth,width=\textwidth]{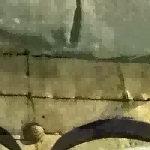}}
    \end{minipage} &
    \begin{minipage}[c]{0.1\textwidth} \centering
    {\includegraphics[height=\textwidth,width=\textwidth]{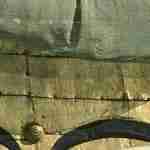}}
    \end{minipage}\\ \specialrule{0em}{1pt}{1pt}
     \begin{minipage}[c]{0.1\textwidth} \centering
    {\includegraphics[height=\textwidth,width=\textwidth]{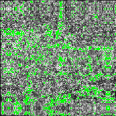}}
    \end{minipage} &
     \begin{minipage}[c]{0.1\textwidth} \centering
    {\includegraphics[height=\textwidth,width=\textwidth]{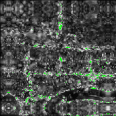}}
    \end{minipage} &
     \begin{minipage}[c]{0.1\textwidth} \centering
    {\includegraphics[height=\textwidth,width=\textwidth]{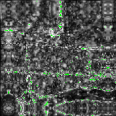}}
    \end{minipage} &
     \begin{minipage}[c]{0.1\textwidth} \centering
    {\includegraphics[height=\textwidth,width=\textwidth]{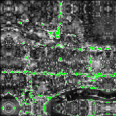}}
    \end{minipage} &
     \begin{minipage}[c]{0.1\textwidth} \centering
    {\includegraphics[height=\textwidth,width=\textwidth]{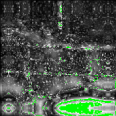}}
    \end{minipage} &
     \begin{minipage}[c]{0.1\textwidth} \centering
    {\includegraphics[height=\textwidth,width=\textwidth]{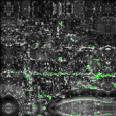}}
    \end{minipage}\\
    \end{tabular}
    \begin{minipage}[c]{0.21\textwidth} \centering
    {\includegraphics[height=\textwidth,width=\textwidth]{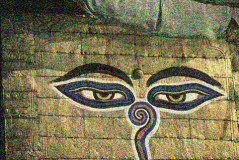}}
    \end{minipage}
\begin{tabular}{ccccccc}
    \begin{minipage}[c]{0.1\textwidth} \centering
    {\includegraphics[height=\textwidth,width=\textwidth]{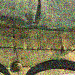}}
    \end{minipage} &
    \begin{minipage}[c]{0.1\textwidth} \centering
    {\includegraphics[height=\textwidth,width=\textwidth]{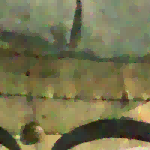}}
    \end{minipage} &
    \begin{minipage}[c]{0.1\textwidth} \centering
    {\includegraphics[height=\textwidth,width=\textwidth]{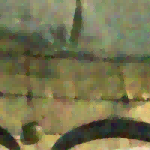}}
    \end{minipage} &
    \begin{minipage}[c]{0.1\textwidth} \centering
    {\includegraphics[height=\textwidth,width=\textwidth]{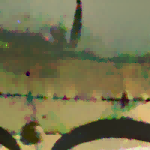}}
    \end{minipage} &
    \begin{minipage}[c]{0.1\textwidth} \centering
    {\includegraphics[height=\textwidth,width=\textwidth]{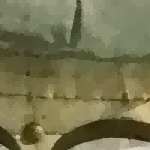}}
    \end{minipage} &
    \begin{minipage}[c]{0.1\textwidth} \centering
    {\includegraphics[height=\textwidth,width=\textwidth]{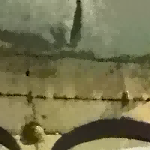}}
    \end{minipage} &
    \begin{minipage}[c]{0.1\textwidth} \centering
    {\includegraphics[height=\textwidth,width=\textwidth]{figs/56028,denoise/56028_zoom.jpg}}
    \end{minipage}\\ \specialrule{0em}{1pt}{1pt}
     \begin{minipage}[c]{0.1\textwidth} \centering
    {\includegraphics[height=\textwidth,width=\textwidth]{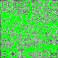}}
    \end{minipage} &
     \begin{minipage}[c]{0.1\textwidth} \centering
    {\includegraphics[height=\textwidth,width=\textwidth]{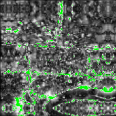}}
    \end{minipage} &
     \begin{minipage}[c]{0.1\textwidth} \centering
    {\includegraphics[height=\textwidth,width=\textwidth]{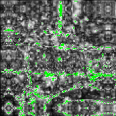}}
    \end{minipage} &
     \begin{minipage}[c]{0.1\textwidth} \centering
    {\includegraphics[height=\textwidth,width=\textwidth]{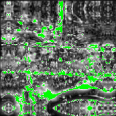}}
    \end{minipage} &
     \begin{minipage}[c]{0.1\textwidth} \centering
    {\includegraphics[height=\textwidth,width=\textwidth]{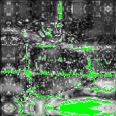}}
    \end{minipage} &
     \begin{minipage}[c]{0.1\textwidth} \centering
    {\includegraphics[height=\textwidth,width=\textwidth]{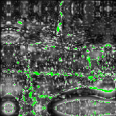}}
    \end{minipage}\\
    \end{tabular}
     \begin{minipage}[c]{0.21\textwidth} \centering
    {\includegraphics[height=\textwidth,width=\textwidth]{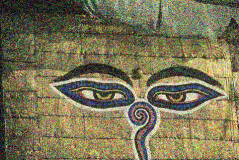}}
    \end{minipage}
\begin{tabular}{ccccccc}
    \begin{minipage}[c]{0.1\textwidth} \centering
    {\includegraphics[height=\textwidth,width=\textwidth]{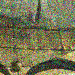}}
    \end{minipage} &
    \begin{minipage}[c]{0.1\textwidth} \centering
    {\includegraphics[height=\textwidth,width=\textwidth]{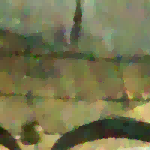}}
    \end{minipage} &
    \begin{minipage}[c]{0.1\textwidth} \centering
    {\includegraphics[height=\textwidth,width=\textwidth]{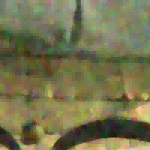}}
    \end{minipage} &
    \begin{minipage}[c]{0.1\textwidth} \centering
    {\includegraphics[height=\textwidth,width=\textwidth]{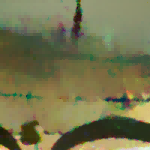}}
    \end{minipage} &
    \begin{minipage}[c]{0.1\textwidth} \centering
    {\includegraphics[height=\textwidth,width=\textwidth]{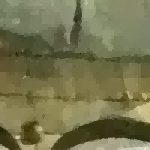}}
    \end{minipage} &
    \begin{minipage}[c]{0.1\textwidth} \centering
    {\includegraphics[height=\textwidth,width=\textwidth]{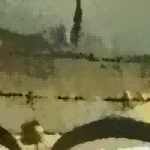}}
    \end{minipage} &
    \begin{minipage}[c]{0.1\textwidth} \centering
    {\includegraphics[height=\textwidth,width=\textwidth]{figs/56028,denoise/56028_zoom.jpg}}
    \end{minipage}\\ \specialrule{0em}{1pt}{1pt}
     \begin{minipage}[c]{0.1\textwidth} \centering
    {\includegraphics[height=\textwidth,width=\textwidth]{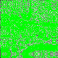}}
    \end{minipage} &
     \begin{minipage}[c]{0.1\textwidth} \centering
    {\includegraphics[height=\textwidth,width=\textwidth]{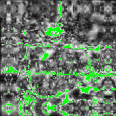}}
    \end{minipage} &
     \begin{minipage}[c]{0.1\textwidth} \centering
    {\includegraphics[height=\textwidth,width=\textwidth]{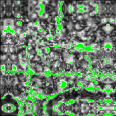}}
    \end{minipage}&
     \begin{minipage}[c]{0.1\textwidth} \centering
    {\includegraphics[height=\textwidth,width=\textwidth]{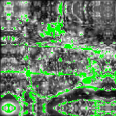}}
    \end{minipage} &
     \begin{minipage}[c]{0.1\textwidth} \centering
    {\includegraphics[height=\textwidth,width=\textwidth]{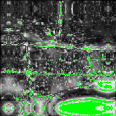}}
    \end{minipage} &
     \begin{minipage}[c]{0.1\textwidth} \centering
    {\includegraphics[height=\textwidth,width=\textwidth]{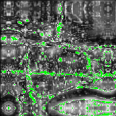}}
    \end{minipage} \\
    \end{tabular}\\
    \caption{Top to bottom: the corresponding results with noise level d = 30/255, 50/255, 70/255 respectively. The results include the noisy image (left large picture), the corresponding zoom-in parts of the noise image, the restored results by using CTV, HTV, NLTV, SVTV, SVS-NLTV, the ground-truth image respectively. The spatial distributions of S-CIELAB error (larger than 15 units) are also shown.}
    \label{56028L2zoom}
\end{figure}

\begin{figure}[htbp]
\centering
\tabcolsep=1pt
\begin{tabular}{ccccc}
    \begin{minipage}[c]{0.184\textwidth} \centering
   \subfigure{\includegraphics[width =\textwidth]{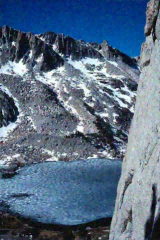}}
    \end{minipage} &
        \begin{minipage}[c]{0.184\textwidth} \centering
   \subfigure{\includegraphics[width =\textwidth]{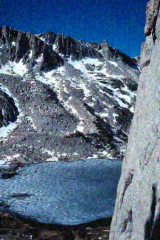}}
    \end{minipage} &
        \begin{minipage}[c]{0.184\textwidth} \centering
   \subfigure{\includegraphics[width =\textwidth]{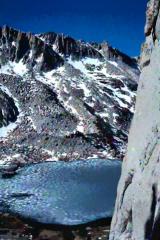}}
    \end{minipage} &
    \begin{minipage}[c]{0.184\textwidth} \centering
   \subfigure{\includegraphics[width = \textwidth]{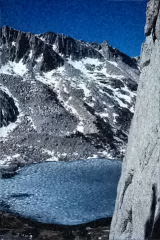}}
    \end{minipage} &
        \begin{minipage}[c]{0.184\textwidth} \centering
   \subfigure{\includegraphics[width = \textwidth]{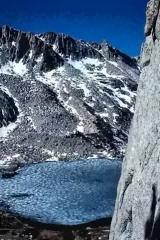}}
    \end{minipage} \\
    \begin{minipage}[c]{0.184\textwidth} \centering
   \subfigure{\includegraphics[width =\textwidth]{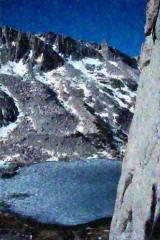}}
    \end{minipage} &
        \begin{minipage}[c]{0.184\textwidth} \centering
\subfigure{\includegraphics[width =\textwidth]{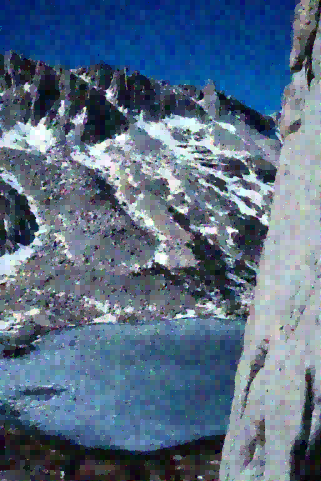}}
    \end{minipage} &
        \begin{minipage}[c]{0.184\textwidth} \centering
\subfigure{\includegraphics[width=\textwidth]{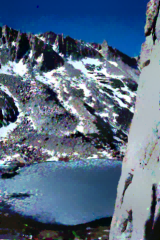}}
    \end{minipage} &
        \begin{minipage}[c]{0.184\textwidth} \centering
\subfigure{\includegraphics[width =\textwidth]{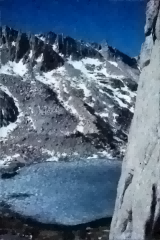}}
    \end{minipage} &
        \begin{minipage}[c]{0.184\textwidth} \centering
\subfigure{\includegraphics[width =\textwidth]{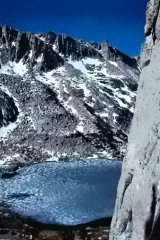}}
    \end{minipage} \\
    \begin{minipage}[c]{0.184\textwidth} \centering
\subfigure{\includegraphics[width =\textwidth]{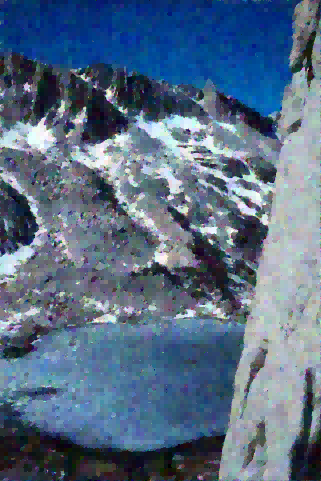}}
    \end{minipage} &
        \begin{minipage}[c]{0.184\textwidth} \centering
\subfigure{\includegraphics[width=\textwidth]{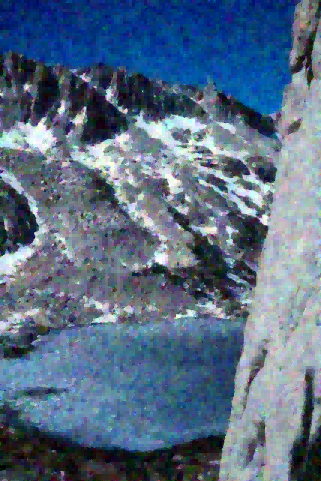}}
    \end{minipage} &
        \begin{minipage}[c]{0.184\textwidth} \centering
\subfigure{\includegraphics[width =\textwidth]{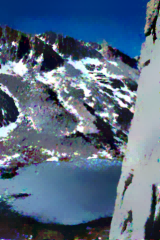}}
   \end{minipage} &
        \begin{minipage}[c]{0.184\textwidth} \centering
\subfigure{\includegraphics[width =\textwidth]{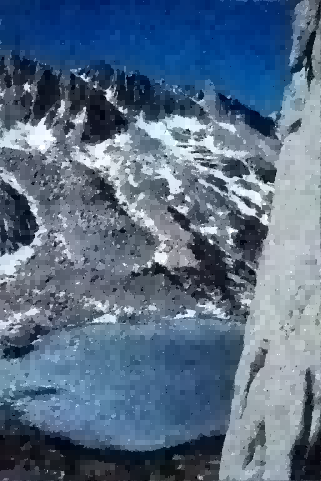}}
    \end{minipage} &
        \begin{minipage}[c]{0.184\textwidth} \centering
\subfigure{\includegraphics[width =\textwidth]{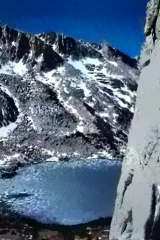}}
\end{minipage}\\
\end{tabular}
\begin{subfigure}
\centering
\includegraphics[width=0.23\textwidth]{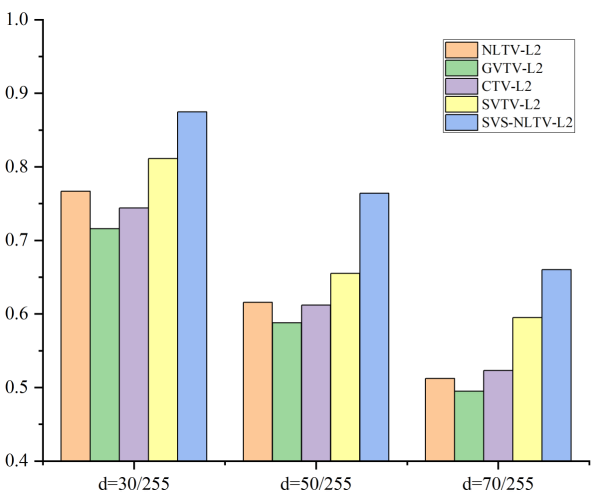}
\end{subfigure}
\begin{subfigure}
\centering
\includegraphics[width=0.23\textwidth]{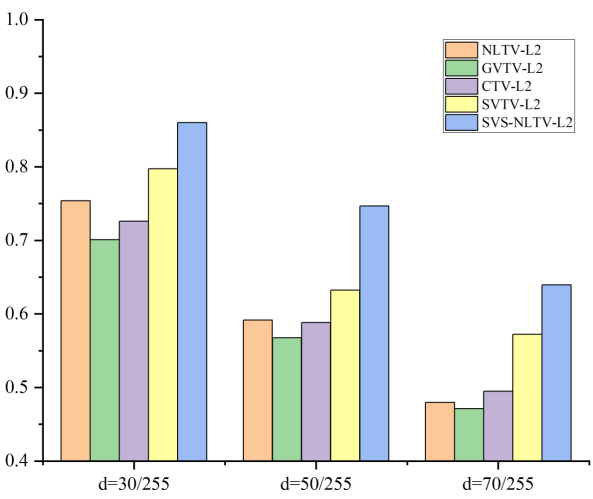}
\end{subfigure}
\begin{subfigure}
\centering
\includegraphics[width=0.23\textwidth]{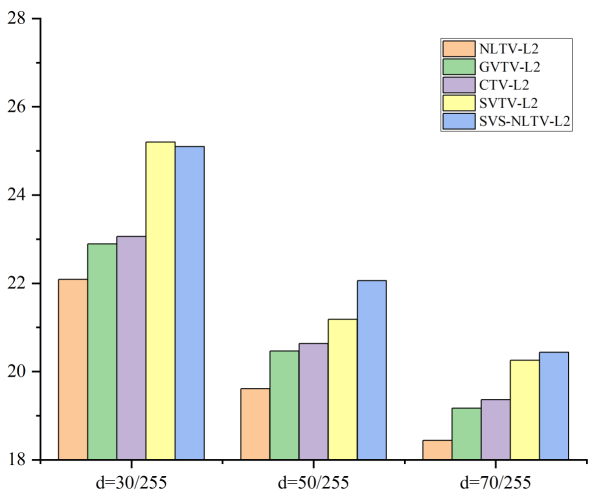}
\end{subfigure}
\begin{subfigure}
\centering
\includegraphics[width=0.23\textwidth]{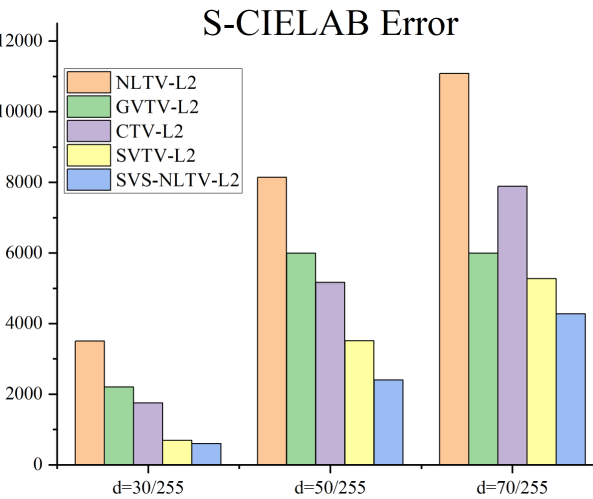}
\end{subfigure}
\caption{The first three rows: top to bottom: degraded and restored images with noise level d = 30/255, 50/255, 70/255 respectively; left to right: the restored results by using CTV, GVTV, NLTV, SVTV, and SVS-NLTV respectively. The fourth row: the histograms of measure values by using different methods.}
\label{167083L2}
\end{figure}

\begin{figure}[htbp]
\centering
\tabcolsep=1pt
\begin{minipage}[c]{0.21\textwidth} \centering
    {\includegraphics[height=\textwidth,width=\textwidth]{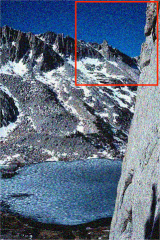}}
    \end{minipage}
\begin{tabular}{ccccccc}
    \begin{minipage}[c]{0.1\textwidth} \centering
    {\includegraphics[height=\textwidth,width=\textwidth]{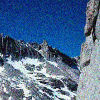}}
    \end{minipage} &
    \begin{minipage}[c]{0.1\textwidth} \centering
    {\includegraphics[height=\textwidth,width=\textwidth]{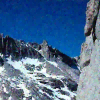}}
    \end{minipage} &
    \begin{minipage}[c]{0.1\textwidth} \centering
    {\includegraphics[height=\textwidth,width=\textwidth]{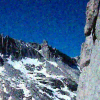}}
    \end{minipage} &
    \begin{minipage}[c]{0.1\textwidth} \centering
    {\includegraphics[height=\textwidth,width=\textwidth]{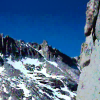}}
    \end{minipage} &
    \begin{minipage}[c]{0.1\textwidth} \centering
    {\includegraphics[height=\textwidth,width=\textwidth]{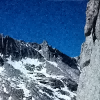}}
    \end{minipage} &
    \begin{minipage}[c]{0.1\textwidth} \centering
    {\includegraphics[height=\textwidth,width=\textwidth]{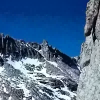}}
    \end{minipage} &
    \begin{minipage}[c]{0.1\textwidth} \centering
    {\includegraphics[height=\textwidth,width=\textwidth]{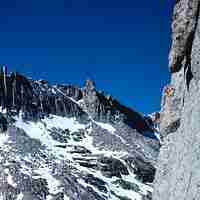}}
    \end{minipage}\\ \specialrule{0em}{1pt}{1pt}
    \begin{minipage}[c]{0.1\textwidth} \centering
    {\includegraphics[height=\textwidth,width=\textwidth]{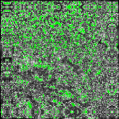}}
    \end{minipage} &
     \begin{minipage}[c]{0.1\textwidth} \centering
    {\includegraphics[height=\textwidth,width=\textwidth]{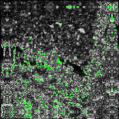}}
    \end{minipage} &
     \begin{minipage}[c]{0.1\textwidth} \centering
    {\includegraphics[height=\textwidth,width=\textwidth]{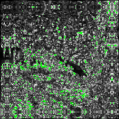}}
    \end{minipage} &
     \begin{minipage}[c]{0.1\textwidth} \centering
    {\includegraphics[height=\textwidth,width=\textwidth]{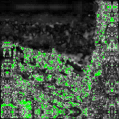}}
    \end{minipage} &
     \begin{minipage}[c]{0.1\textwidth} \centering
    {\includegraphics[height=\textwidth,width=\textwidth]{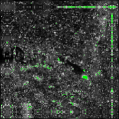}}
    \end{minipage} &
     \begin{minipage}[c]{0.1\textwidth} \centering
    {\includegraphics[height=\textwidth,width=\textwidth]{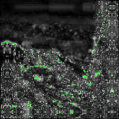}}
    \end{minipage} \\
    \end{tabular}
    \begin{minipage}[c]{0.21\textwidth} \centering
    {\includegraphics[height=\textwidth,width=\textwidth]{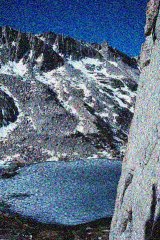}}
    \end{minipage}
\begin{tabular}{ccccccc}
    \begin{minipage}[c]{0.1\textwidth} \centering
    {\includegraphics[height=\textwidth,width=\textwidth]{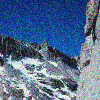}}
    \end{minipage} &
    \begin{minipage}[c]{0.1\textwidth} \centering
    {\includegraphics[height=\textwidth,width=\textwidth]{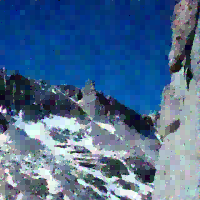}}
    \end{minipage} &
    \begin{minipage}[c]{0.1\textwidth} \centering
    {\includegraphics[height=\textwidth,width=\textwidth]{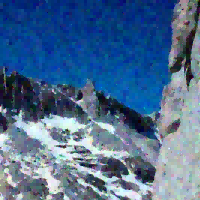}}
    \end{minipage} &
    \begin{minipage}[c]{0.1\textwidth} \centering
    {\includegraphics[height=\textwidth,width=\textwidth]{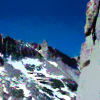}}
    \end{minipage} &
    \begin{minipage}[c]{0.1\textwidth} \centering
    {\includegraphics[height=\textwidth,width=\textwidth]{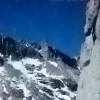}}
    \end{minipage} &
    \begin{minipage}[c]{0.1\textwidth} \centering
    {\includegraphics[height=\textwidth,width=\textwidth]{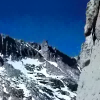}}
    \end{minipage} &
    \begin{minipage}[c]{0.1\textwidth} \centering
    {\includegraphics[height=\textwidth,width=\textwidth]{figs/167083,denoise/167083_zoom.jpg}}
    \end{minipage}\\ \specialrule{0em}{1pt}{1pt}
     \begin{minipage}[c]{0.1\textwidth} \centering
    {\includegraphics[height=\textwidth,width=\textwidth]{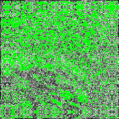}}
    \end{minipage} &
     \begin{minipage}[c]{0.1\textwidth} \centering
    {\includegraphics[height=\textwidth,width=\textwidth]{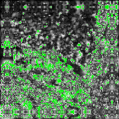}}
    \end{minipage} &
     \begin{minipage}[c]{0.1\textwidth} \centering
    {\includegraphics[height=\textwidth,width=\textwidth]{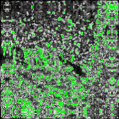}}
    \end{minipage} &
     \begin{minipage}[c]{0.1\textwidth} \centering
    {\includegraphics[height=\textwidth,width=\textwidth]{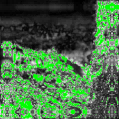}}
    \end{minipage} &
     \begin{minipage}[c]{0.1\textwidth} \centering
    {\includegraphics[height=\textwidth,width=\textwidth]{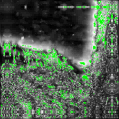}}
    \end{minipage} &
     \begin{minipage}[c]{0.1\textwidth} \centering
    {\includegraphics[height=\textwidth,width=\textwidth]{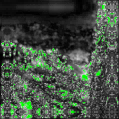}}
    \end{minipage}\\
    \end{tabular}
     \begin{minipage}[c]{0.21\textwidth} \centering
    {\includegraphics[height=\textwidth,width=\textwidth]{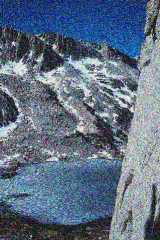}}
    \end{minipage}
\begin{tabular}{ccccccc}
    \begin{minipage}[c]{0.1\textwidth} \centering
    {\includegraphics[height=\textwidth,width=\textwidth]{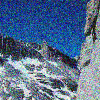}}
    \end{minipage} &
    \begin{minipage}[c]{0.1\textwidth} \centering
    {\includegraphics[height=\textwidth,width=\textwidth]{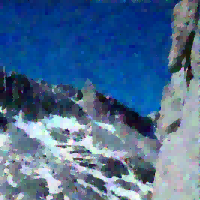}}
    \end{minipage} &
    \begin{minipage}[c]{0.1\textwidth} \centering
    {\includegraphics[height=\textwidth,width=\textwidth]{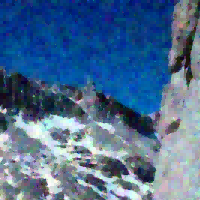}}
    \end{minipage} &
    \begin{minipage}[c]{0.1\textwidth} \centering
    {\includegraphics[height=\textwidth,width=\textwidth]{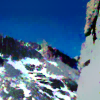}}
    \end{minipage} &
    \begin{minipage}[c]{0.1\textwidth} \centering
    {\includegraphics[height=\textwidth,width=\textwidth]{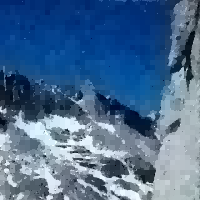}}
    \end{minipage} &
    \begin{minipage}[c]{0.1\textwidth} \centering
    {\includegraphics[height=\textwidth,width=\textwidth]{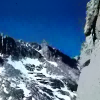}}
    \end{minipage} &
    \begin{minipage}[c]{0.1\textwidth} \centering
    {\includegraphics[height=\textwidth,width=\textwidth]{figs/167083,denoise/167083_zoom.jpg}}
    \end{minipage}\\ \specialrule{0em}{1pt}{1pt}
     \begin{minipage}[c]{0.1\textwidth} \centering
    {\includegraphics[height=\textwidth,width=\textwidth]{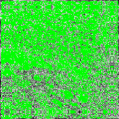}}
    \end{minipage} &
     \begin{minipage}[c]{0.1\textwidth} \centering
    {\includegraphics[height=\textwidth,width=\textwidth]{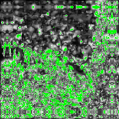}}
    \end{minipage} &
     \begin{minipage}[c]{0.1\textwidth} \centering
    {\includegraphics[height=\textwidth,width=\textwidth]{figs/167083,denoise/167083_GVTV_5993_zoom.png}}
    \end{minipage}&
     \begin{minipage}[c]{0.1\textwidth} \centering
    {\includegraphics[height=\textwidth,width=\textwidth]{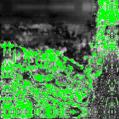}}
    \end{minipage} &
     \begin{minipage}[c]{0.1\textwidth} \centering
    {\includegraphics[height=\textwidth,width=\textwidth]{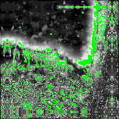}}
    \end{minipage} &
     \begin{minipage}[c]{0.1\textwidth} \centering
    {\includegraphics[height=\textwidth,width=\textwidth]{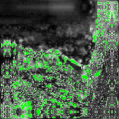}}
    \end{minipage} \\
    \end{tabular}\\
    \caption{Top to bottom: the corresponding results with noise level d = 30/255, 50/255, 70/255 respectively. The results include the noisy image (left large picture), the corresponding zoom-in parts of the noise image, the restored results by using CTV, HTV, NLTV, SVTV, SVS-NLTV, the ground-truth image respectively. The spatial distributions of S-CIELAB error (larger than 15 units) are also shown.}
    \label{167083L2zoom}
\end{figure}

\begin{figure}[htbp]
\centering
\tabcolsep=1pt
\begin{tabular}{ccccc}
    \begin{minipage}[c]{0.184\textwidth} \centering
   \subfigure{\includegraphics[width =\textwidth]{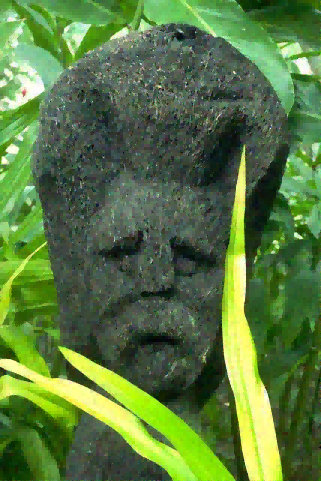}}
    \end{minipage} &
        \begin{minipage}[c]{0.184\textwidth} \centering
   \subfigure{\includegraphics[width =\textwidth]{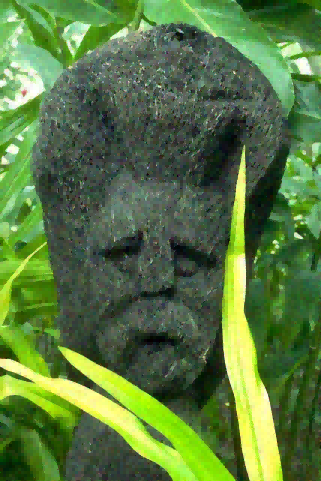}}
    \end{minipage} &
        \begin{minipage}[c]{0.184\textwidth} \centering
   \subfigure{\includegraphics[width =\textwidth]{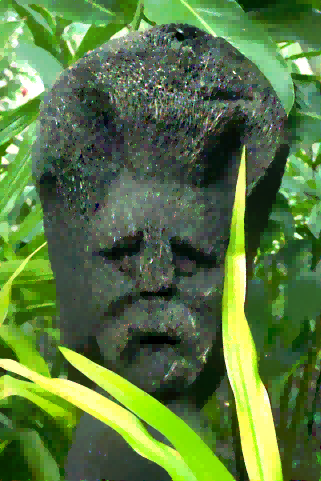}}
    \end{minipage} &
    \begin{minipage}[c]{0.184\textwidth} \centering
   \subfigure{\includegraphics[width = \textwidth]{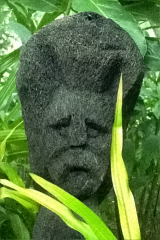}}
    \end{minipage} &
        \begin{minipage}[c]{0.184\textwidth} \centering
   \subfigure{\includegraphics[width = \textwidth]{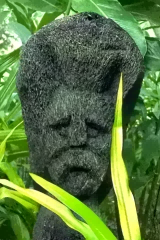}}
    \end{minipage} \\
    \begin{minipage}[c]{0.184\textwidth} \centering
   \subfigure{\includegraphics[width =\textwidth]{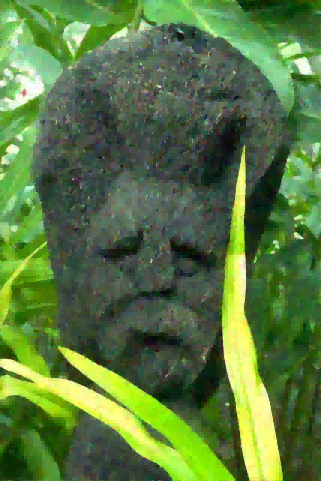}}
    \end{minipage} &
        \begin{minipage}[c]{0.184\textwidth} \centering
\subfigure{\includegraphics[width =\textwidth]{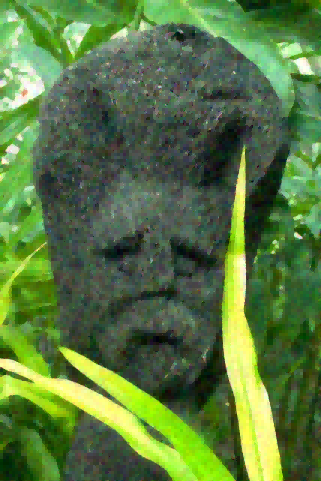}}
    \end{minipage} &
        \begin{minipage}[c]{0.184\textwidth} \centering
\subfigure{\includegraphics[width=\textwidth]{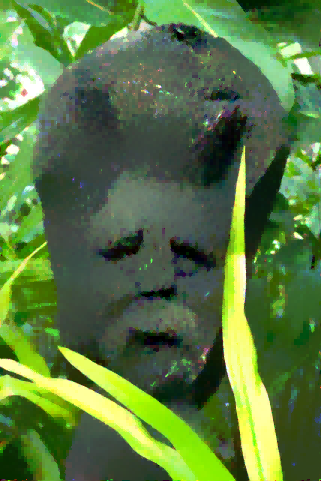}}
    \end{minipage} &
        \begin{minipage}[c]{0.184\textwidth} \centering
\subfigure{\includegraphics[width =\textwidth]{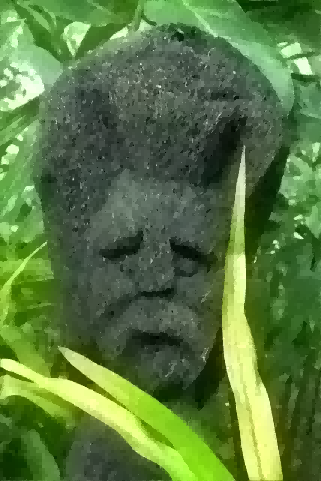}}
    \end{minipage} &
        \begin{minipage}[c]{0.184\textwidth} \centering
\subfigure{\includegraphics[width =\textwidth]{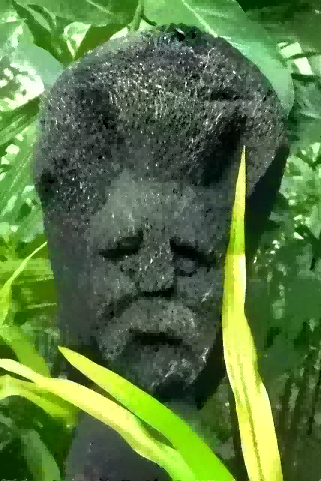}}
    \end{minipage} \\
    \begin{minipage}[c]{0.184\textwidth} \centering
\subfigure{\includegraphics[width =\textwidth]{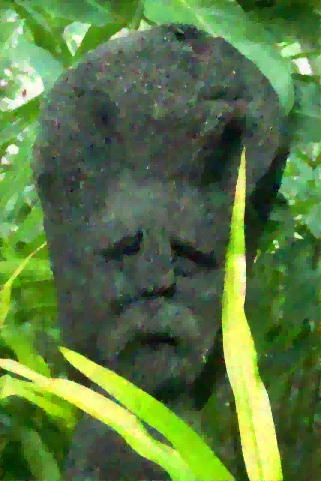}}
    \end{minipage} &
        \begin{minipage}[c]{0.184\textwidth} \centering
\subfigure{\includegraphics[width=\textwidth]{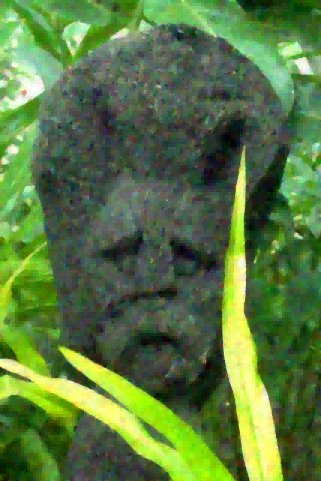}}
    \end{minipage} &
        \begin{minipage}[c]{0.184\textwidth} \centering
\subfigure{\includegraphics[width =\textwidth]{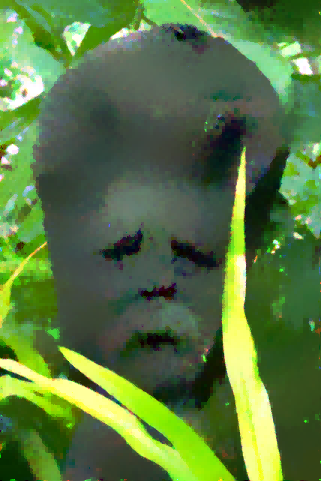}}
    \end{minipage} &
        \begin{minipage}[c]{0.184\textwidth} \centering
\subfigure{\includegraphics[width =\textwidth]{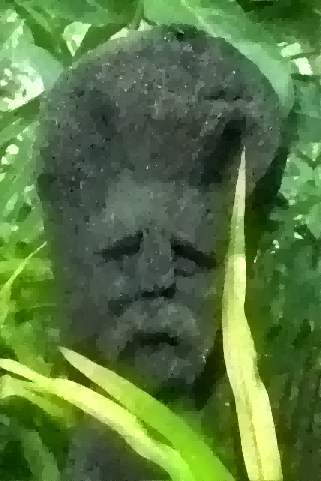}}
    \end{minipage} &
        \begin{minipage}[c]{0.184\textwidth} \centering
\subfigure{\includegraphics[width =\textwidth]{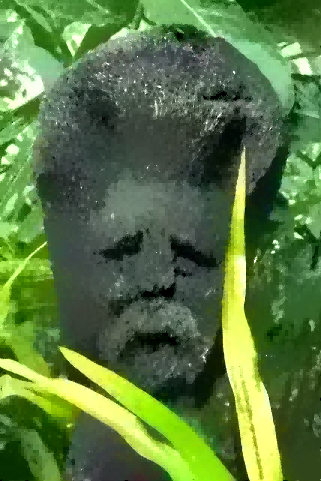}}
\end{minipage}\\
\end{tabular}
\begin{subfigure}
\centering
\includegraphics[width=0.23\textwidth]{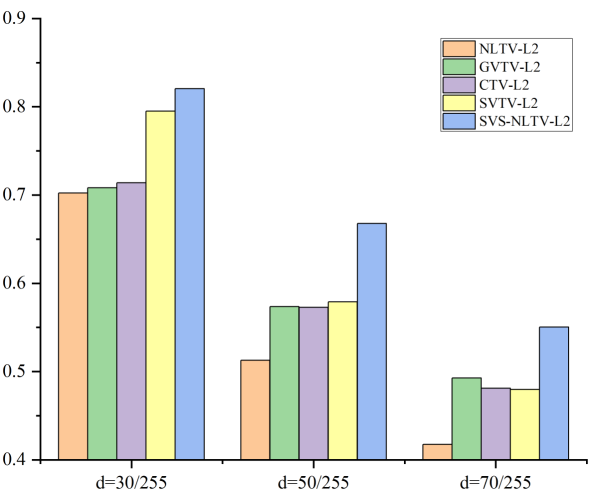}
\end{subfigure}
\begin{subfigure}
\centering
\includegraphics[width=0.23\textwidth]{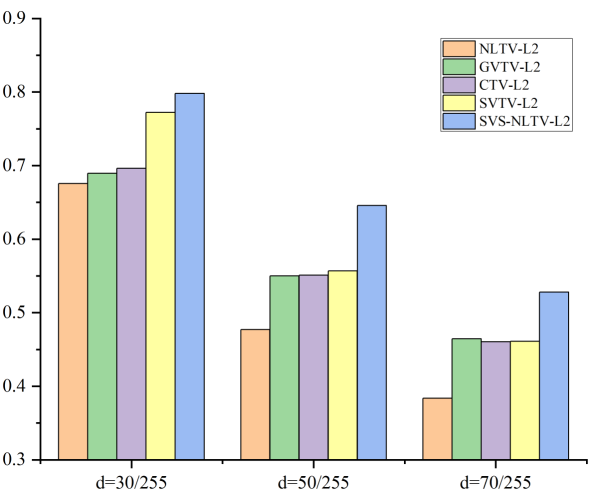}
\end{subfigure}
\begin{subfigure}
\centering
\includegraphics[width=0.23\textwidth]{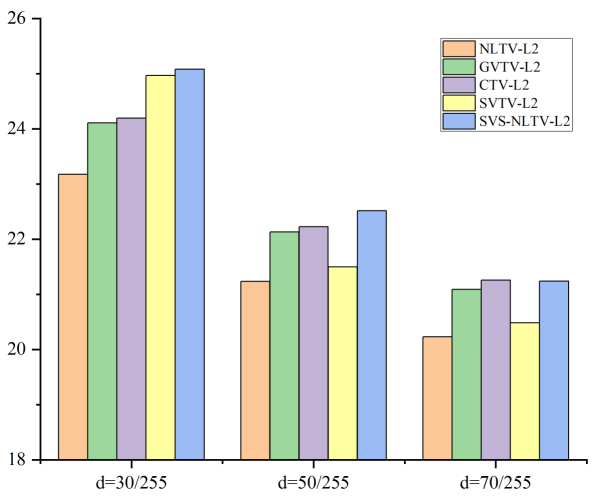}
\end{subfigure}
\begin{subfigure}
\centering
\includegraphics[width=0.23\textwidth]{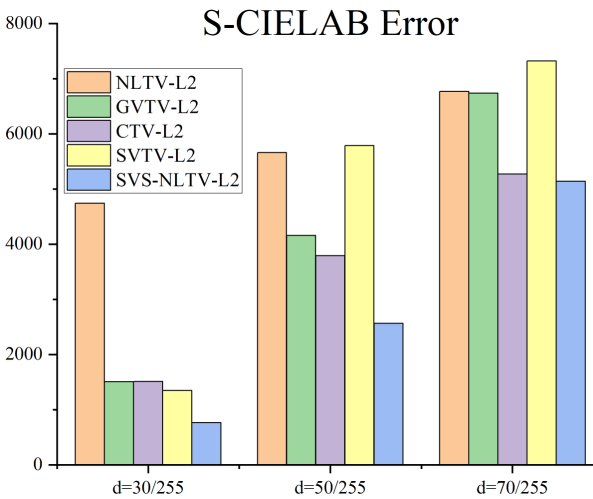}
\end{subfigure}
\caption{The first three rows: top to bottom: degraded and restored images with noise level d = 30/255, 50/255, 70/255 respectively; left to right: the restored results by using CTV, GVTV, NLTV, SVTV, and SVS-NLTV respectively. The fourth row: the histograms of measure values by using different methods.}
\label{101084L2}
\end{figure}

\begin{figure}[htbp]
\centering
\tabcolsep=1pt
\begin{minipage}[c]{0.21\textwidth} \centering
    {\includegraphics[height=\textwidth,width=\textwidth]{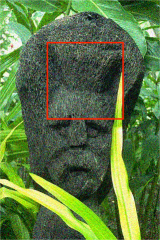}}
    \end{minipage}
\begin{tabular}{ccccccc}
    \begin{minipage}[c]{0.1\textwidth} \centering
    {\includegraphics[height=\textwidth,width=\textwidth]{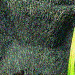}}
    \end{minipage} &
    \begin{minipage}[c]{0.1\textwidth} \centering
    {\includegraphics[height=\textwidth,width=\textwidth]{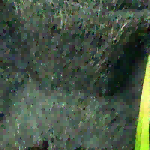}}
    \end{minipage} &
    \begin{minipage}[c]{0.1\textwidth} \centering
    {\includegraphics[height=\textwidth,width=\textwidth]{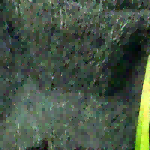}}
    \end{minipage} &
    \begin{minipage}[c]{0.1\textwidth} \centering
    {\includegraphics[height=\textwidth,width=\textwidth]{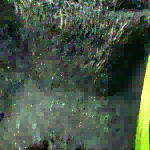}}
    \end{minipage} &
    \begin{minipage}[c]{0.1\textwidth} \centering
    {\includegraphics[height=\textwidth,width=\textwidth]{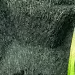}}
    \end{minipage} &
    \begin{minipage}[c]{0.1\textwidth} \centering
    {\includegraphics[height=\textwidth,width=\textwidth]{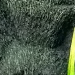}}
    \end{minipage} &
     \begin{minipage}[c]{0.1\textwidth} \centering
    {\includegraphics[height=\textwidth,width=\textwidth]{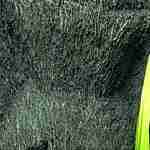}}
    \end{minipage}    \\ \specialrule{0em}{1pt}{1pt}
    \begin{minipage}[c]{0.1\textwidth} \centering
    {\includegraphics[height=\textwidth,width=\textwidth]{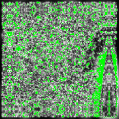}}
    \end{minipage} &
     \begin{minipage}[c]{0.1\textwidth} \centering
    {\includegraphics[height=\textwidth,width=\textwidth]{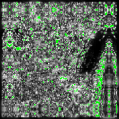}}
    \end{minipage} &
     \begin{minipage}[c]{0.1\textwidth} \centering
    {\includegraphics[height=\textwidth,width=\textwidth]{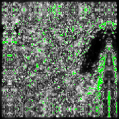}}
    \end{minipage} &
     \begin{minipage}[c]{0.1\textwidth} \centering
    {\includegraphics[height=\textwidth,width=\textwidth]{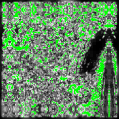}}
    \end{minipage} &
     \begin{minipage}[c]{0.1\textwidth} \centering
    {\includegraphics[height=\textwidth,width=\textwidth]{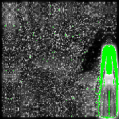}}
    \end{minipage} &
     \begin{minipage}[c]{0.1\textwidth} \centering
    {\includegraphics[height=\textwidth,width=\textwidth]{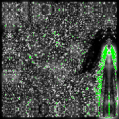}}
    \end{minipage} \\
    \end{tabular}
    \begin{minipage}[c]{0.21\textwidth} \centering
    {\includegraphics[height=\textwidth,width=\textwidth]{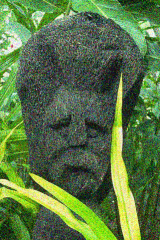}}
    \end{minipage}
\begin{tabular}{ccccccc}
    \begin{minipage}[c]{0.1\textwidth} \centering
    {\includegraphics[height=\textwidth,width=\textwidth]{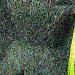}}
    \end{minipage} &
    \begin{minipage}[c]{0.1\textwidth} \centering
    {\includegraphics[height=\textwidth,width=\textwidth]{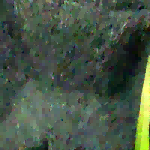}}
    \end{minipage} &
    \begin{minipage}[c]{0.1\textwidth} \centering
    {\includegraphics[height=\textwidth,width=\textwidth]{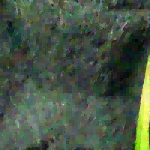}}
    \end{minipage} &
    \begin{minipage}[c]{0.1\textwidth} \centering
    {\includegraphics[height=\textwidth,width=\textwidth]{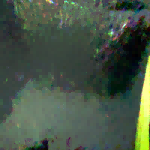}}
    \end{minipage} &
    \begin{minipage}[c]{0.1\textwidth} \centering
    {\includegraphics[height=\textwidth,width=\textwidth]{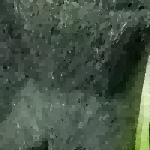}}
    \end{minipage} &
    \begin{minipage}[c]{0.1\textwidth} \centering
    {\includegraphics[height=\textwidth,width=\textwidth]{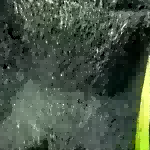}}
    \end{minipage} &
    \begin{minipage}[c]{0.1\textwidth} \centering
    {\includegraphics[height=\textwidth,width=\textwidth]{figs/101084,denoise/101084_zoom.jpg}}
    \end{minipage} \\ \specialrule{0em}{1pt}{1pt}
     \begin{minipage}[c]{0.1\textwidth} \centering
    {\includegraphics[height=\textwidth,width=\textwidth]{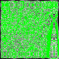}}
    \end{minipage} &
     \begin{minipage}[c]{0.1\textwidth} \centering
    {\includegraphics[height=\textwidth,width=\textwidth]{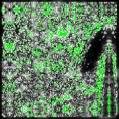}}
    \end{minipage} &
     \begin{minipage}[c]{0.1\textwidth} \centering
    {\includegraphics[height=\textwidth,width=\textwidth]{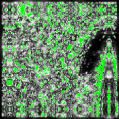}}
    \end{minipage} &
     \begin{minipage}[c]{0.1\textwidth} \centering
    {\includegraphics[height=\textwidth,width=\textwidth]{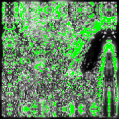}}
    \end{minipage} &
     \begin{minipage}[c]{0.1\textwidth} \centering
    {\includegraphics[height=\textwidth,width=\textwidth]{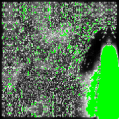}}
    \end{minipage} &
     \begin{minipage}[c]{0.1\textwidth} \centering
    {\includegraphics[height=\textwidth,width=\textwidth]{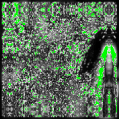}}
    \end{minipage}\\
    \end{tabular}
     \begin{minipage}[c]{0.21\textwidth} \centering
    {\includegraphics[height=\textwidth,width=\textwidth]{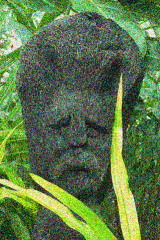}}
    \end{minipage}
\begin{tabular}{ccccccc}
    \begin{minipage}[c]{0.1\textwidth} \centering
    {\includegraphics[height=\textwidth,width=\textwidth]{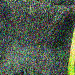}}
    \end{minipage} &
    \begin{minipage}[c]{0.1\textwidth} \centering
    {\includegraphics[height=\textwidth,width=\textwidth]{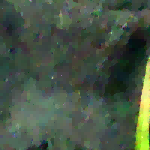}}
    \end{minipage} &
    \begin{minipage}[c]{0.1\textwidth} \centering
    {\includegraphics[height=\textwidth,width=\textwidth]{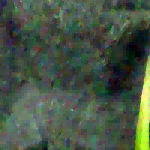}}
    \end{minipage} &
    \begin{minipage}[c]{0.1\textwidth} \centering
    {\includegraphics[height=\textwidth,width=\textwidth]{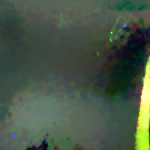}}
    \end{minipage} &
    \begin{minipage}[c]{0.1\textwidth} \centering
    {\includegraphics[height=\textwidth,width=\textwidth]{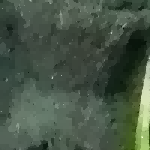}}
    \end{minipage} &
    \begin{minipage}[c]{0.1\textwidth} \centering
    {\includegraphics[height=\textwidth,width=\textwidth]{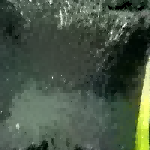}}
    \end{minipage} &
    \begin{minipage}[c]{0.1\textwidth} \centering
    {\includegraphics[height=\textwidth,width=\textwidth]{figs/101084,denoise/101084_zoom.jpg}}
    \end{minipage} \\ \specialrule{0em}{1pt}{1pt}
     \begin{minipage}[c]{0.1\textwidth} \centering
    {\includegraphics[height=\textwidth,width=\textwidth]{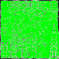}}
    \end{minipage} &
     \begin{minipage}[c]{0.1\textwidth} \centering
    {\includegraphics[height=\textwidth,width=\textwidth]{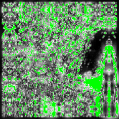}}
    \end{minipage} &
     \begin{minipage}[c]{0.1\textwidth} \centering
    {\includegraphics[height=\textwidth,width=\textwidth]{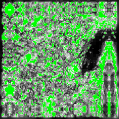}}
    \end{minipage}&
     \begin{minipage}[c]{0.1\textwidth} \centering
    {\includegraphics[height=\textwidth,width=\textwidth]{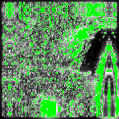}}
    \end{minipage} &
     \begin{minipage}[c]{0.1\textwidth} \centering
    {\includegraphics[height=\textwidth,width=\textwidth]{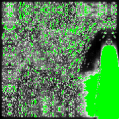}}
    \end{minipage} &
     \begin{minipage}[c]{0.1\textwidth} \centering
    {\includegraphics[height=\textwidth,width=\textwidth]{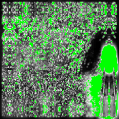}}
    \end{minipage} \\
    \end{tabular}\\
    \caption{Top to bottom: the corresponding results with noise level d = 30/255, 50/255, 70/255 respectively. The results include the noisy image (left large picture), the corresponding zoom-in parts of the noise image, the restored results by using CTV, HTV, NLTV, SVTV, SVS-NLTV, the ground-truth image respectively. The spatial distributions of S-CIELAB error (larger than 15 units) are also shown.}
    \label{101084L2zoom}
\end{figure}

\begin{figure}[htbp]
\centering
\tabcolsep=1pt
\begin{tabular}{ccccc}
    \begin{minipage}[c]{0.184\textwidth} \centering
   \subfigure{\includegraphics[width =\textwidth]{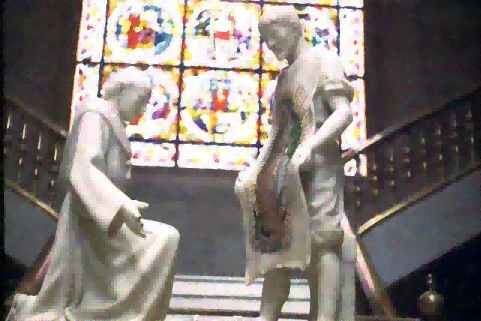}}
    \end{minipage} &
        \begin{minipage}[c]{0.184\textwidth} \centering
   \subfigure{\includegraphics[width =\textwidth]{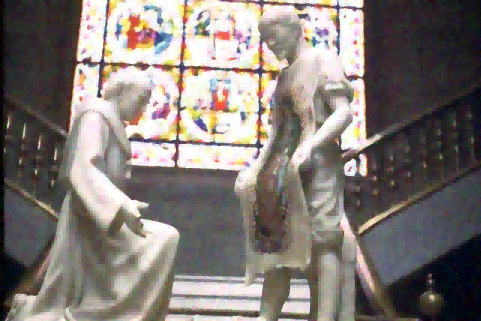}}
    \end{minipage} &
        \begin{minipage}[c]{0.184\textwidth} \centering
   \subfigure{\includegraphics[width =\textwidth]{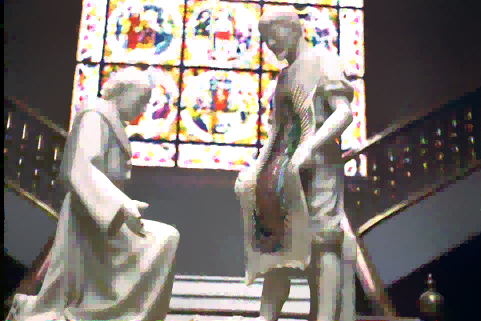}}
    \end{minipage} &
    \begin{minipage}[c]{0.184\textwidth} \centering
   \subfigure{\includegraphics[width = \textwidth]{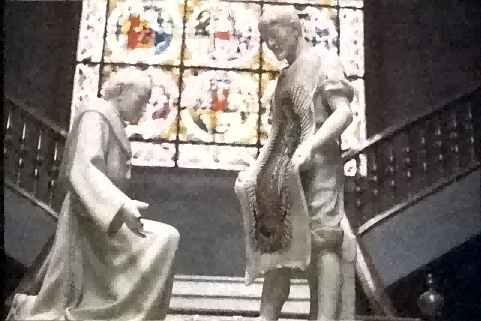}}
    \end{minipage} &
        \begin{minipage}[c]{0.184\textwidth} \centering
   \subfigure{\includegraphics[width = \textwidth]{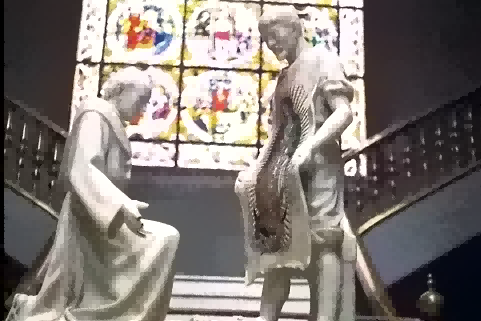}}
    \end{minipage} \\
    \begin{minipage}[c]{0.184\textwidth} \centering
   \subfigure{\includegraphics[width =\textwidth]{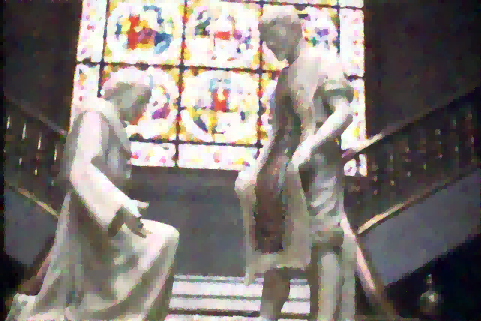}}
    \end{minipage} &
        \begin{minipage}[c]{0.184\textwidth} \centering
\subfigure{\includegraphics[width =\textwidth]{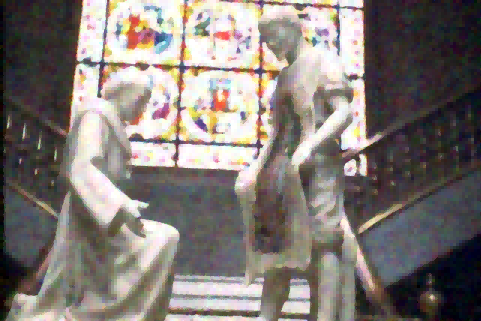}}
    \end{minipage} &
        \begin{minipage}[c]{0.184\textwidth} \centering
\subfigure{\includegraphics[width=\textwidth]{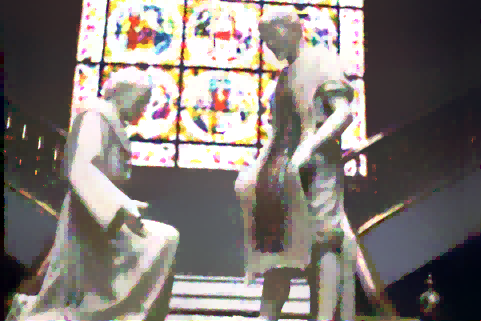}}
    \end{minipage} &
        \begin{minipage}[c]{0.184\textwidth} \centering
\subfigure{\includegraphics[width =\textwidth]{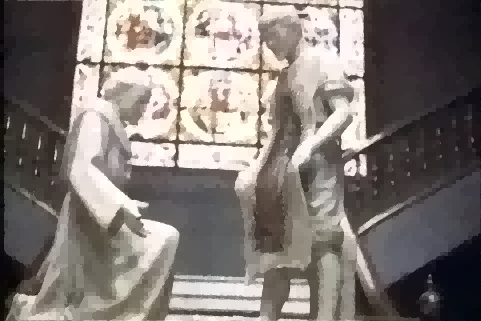}}
    \end{minipage} &
        \begin{minipage}[c]{0.184\textwidth} \centering
\subfigure{\includegraphics[width=\textwidth]{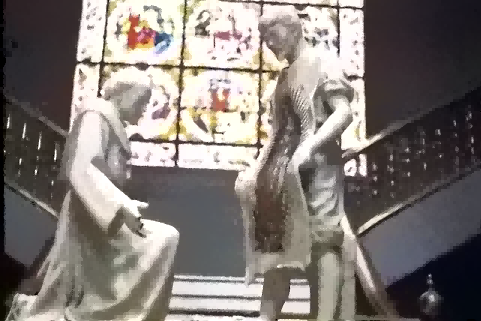}}
    \end{minipage} \\
    \begin{minipage}[c]{0.184\textwidth} \centering
\subfigure{\includegraphics[width =\textwidth]{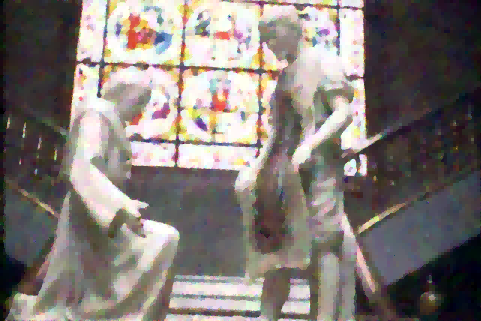}}
    \end{minipage} &
        \begin{minipage}[c]{0.184\textwidth} \centering
\subfigure{\includegraphics[width=\textwidth]{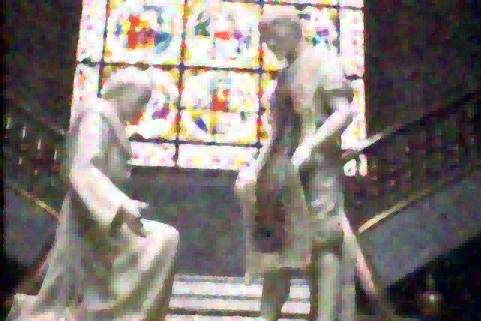}}
    \end{minipage} &
        \begin{minipage}[c]{0.184\textwidth} \centering
\subfigure{\includegraphics[width =\textwidth]{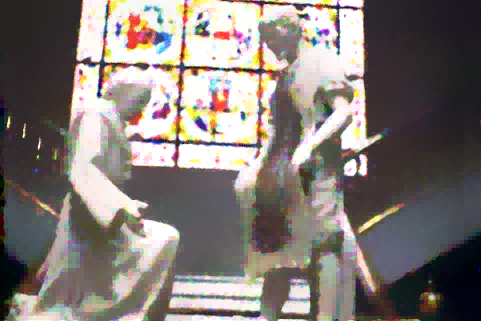}}
    \end{minipage} &
        \begin{minipage}[c]{0.184\textwidth} \centering
\subfigure{\includegraphics[width =\textwidth]{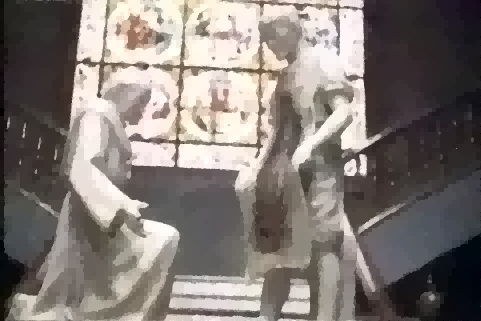}}
    \end{minipage} &
        \begin{minipage}[c]{0.184\textwidth} \centering
\subfigure{\includegraphics[width =\textwidth]{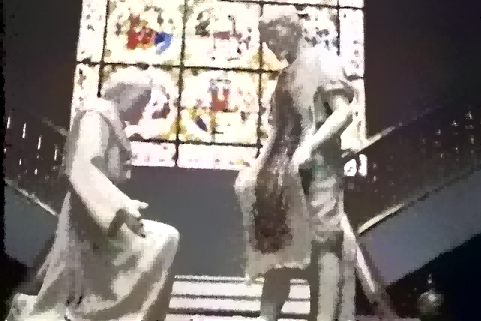}}
\end{minipage}\\
\end{tabular}
\begin{subfigure}
\centering
\includegraphics[width=0.23\textwidth]{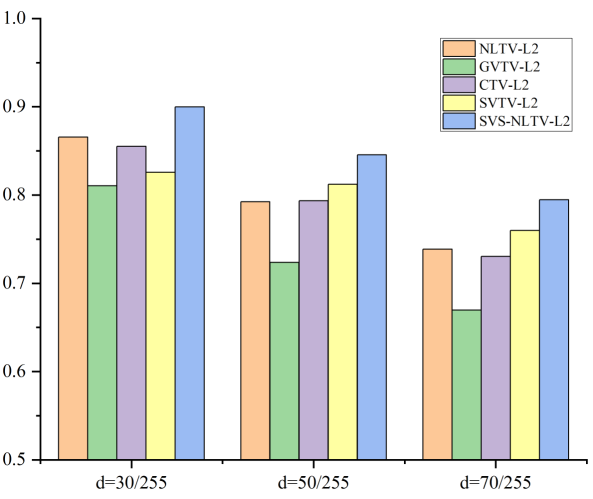}
\end{subfigure}
\begin{subfigure}
\centering
\includegraphics[width=0.22\textwidth]{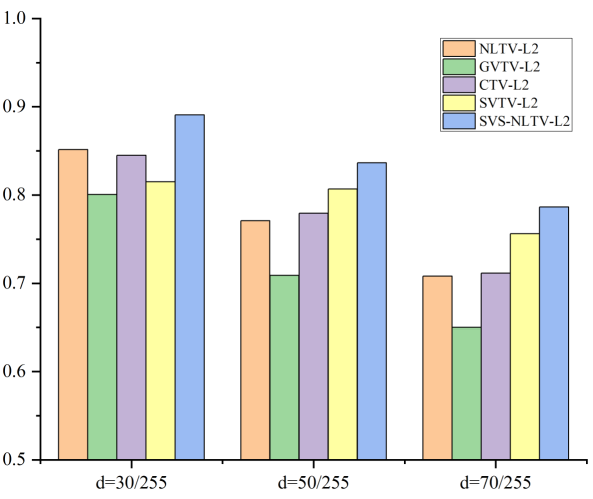}
\end{subfigure}
\begin{subfigure}
\centering
\includegraphics[width=0.23\textwidth]{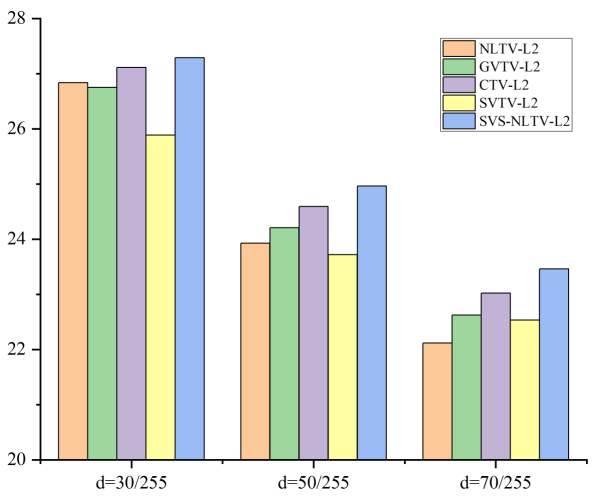}
\end{subfigure}
\begin{subfigure}
\centering
\includegraphics[width=0.23\textwidth]{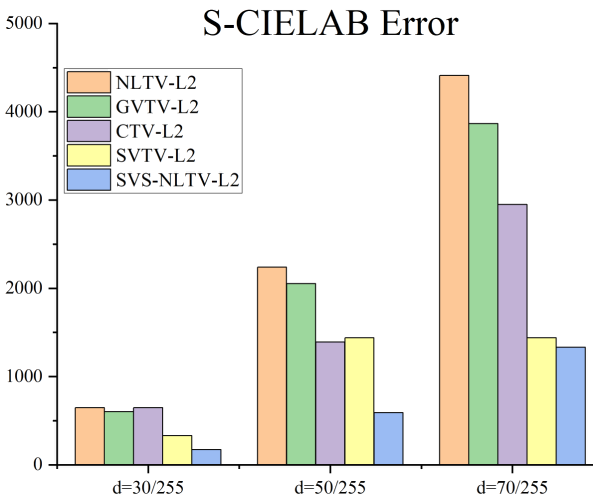}
\end{subfigure}
\caption{The first three rows: top to bottom: degraded and restored images with noise level d = 30/255, 50/255, 70/255 respectively; left to right: the restored results by using CTV, GVTV, NLTV, SVTV, and SVS-NLTV respectively. The fourth row: the histograms of measure values by using different methods.}
\label{24077L2}
\end{figure}

\begin{figure}[htbp]
\centering
\tabcolsep=1pt
\begin{minipage}[c]{0.21\textwidth} \centering
    {\includegraphics[height=\textwidth,width=\textwidth]{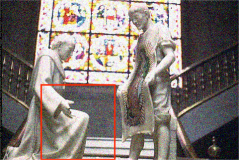}}
    \end{minipage}
\begin{tabular}{ccccccc}
    \begin{minipage}[c]{0.1\textwidth} \centering
    {\includegraphics[height=\textwidth,width=\textwidth]{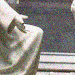}}
    \end{minipage} &
    \begin{minipage}[c]{0.1\textwidth} \centering
    {\includegraphics[height=\textwidth,width=\textwidth]{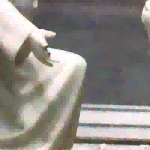}}
    \end{minipage} &
    \begin{minipage}[c]{0.1\textwidth} \centering
    {\includegraphics[height=\textwidth,width=\textwidth]{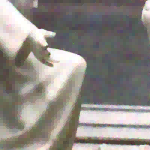}}
    \end{minipage} &
    \begin{minipage}[c]{0.1\textwidth} \centering
    {\includegraphics[height=\textwidth,width=\textwidth]{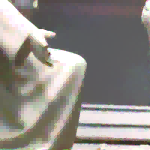}}
    \end{minipage} &
    \begin{minipage}[c]{0.1\textwidth} \centering
    {\includegraphics[height=\textwidth,width=\textwidth]{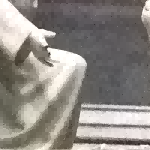}}
    \end{minipage} &
    \begin{minipage}[c]{0.1\textwidth} \centering
    {\includegraphics[height=\textwidth,width=\textwidth]{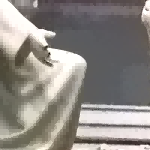}}
    \end{minipage} &
    \begin{minipage}[c]{0.1\textwidth} \centering
    {\includegraphics[height=\textwidth,width=\textwidth]{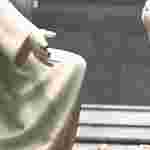}}
    \end{minipage}\\ \specialrule{0em}{1pt}{1pt}
    \begin{minipage}[c]{0.1\textwidth} \centering
    {\includegraphics[height=\textwidth,width=\textwidth]{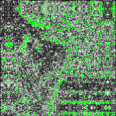}}
    \end{minipage} &
     \begin{minipage}[c]{0.1\textwidth} \centering
    {\includegraphics[height=\textwidth,width=\textwidth]{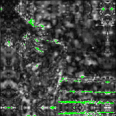}}
    \end{minipage} &
     \begin{minipage}[c]{0.1\textwidth} \centering
    {\includegraphics[height=\textwidth,width=\textwidth]{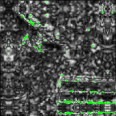}}
    \end{minipage} &
     \begin{minipage}[c]{0.1\textwidth} \centering
    {\includegraphics[height=\textwidth,width=\textwidth]{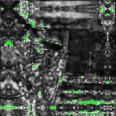}}
    \end{minipage} &
     \begin{minipage}[c]{0.1\textwidth} \centering
    {\includegraphics[height=\textwidth,width=\textwidth]{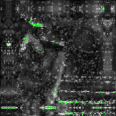}}
    \end{minipage} &
     \begin{minipage}[c]{0.1\textwidth} \centering
    {\includegraphics[height=\textwidth,width=\textwidth]{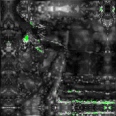}}
    \end{minipage} \\
    \end{tabular}
    \begin{minipage}[c]{0.21\textwidth} \centering
    {\includegraphics[height=\textwidth,width=\textwidth]{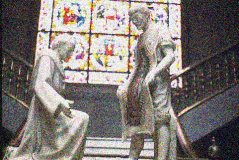}}
    \end{minipage}
\begin{tabular}{ccccccc}
    \begin{minipage}[c]{0.1\textwidth} \centering
    {\includegraphics[height=\textwidth,width=\textwidth]{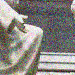}}
    \end{minipage} &
    \begin{minipage}[c]{0.1\textwidth} \centering
    {\includegraphics[height=\textwidth,width=\textwidth]{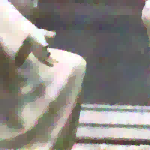}}
    \end{minipage} &
    \begin{minipage}[c]{0.1\textwidth} \centering
    {\includegraphics[height=\textwidth,width=\textwidth]{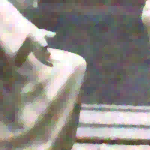}}
    \end{minipage} &
    \begin{minipage}[c]{0.1\textwidth} \centering
    {\includegraphics[height=\textwidth,width=\textwidth]{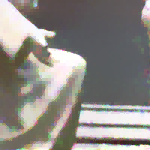}}
    \end{minipage} &
    \begin{minipage}[c]{0.1\textwidth} \centering
    {\includegraphics[height=\textwidth,width=\textwidth]{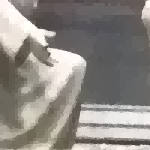}}
    \end{minipage} &
    \begin{minipage}[c]{0.1\textwidth} \centering
    {\includegraphics[height=\textwidth,width=\textwidth]{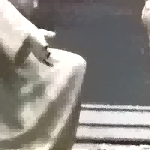}}
    \end{minipage} &
    \begin{minipage}[c]{0.1\textwidth} \centering
    {\includegraphics[height=\textwidth,width=\textwidth]{figs/24077,denoise/24077_zoom.jpg}}
    \end{minipage}\\ \specialrule{0em}{1pt}{1pt}
     \begin{minipage}[c]{0.1\textwidth} \centering
    {\includegraphics[height=\textwidth,width=\textwidth]{figs/24077,denoise/24077_noisy_3707_zoom.png}}
    \end{minipage} &
     \begin{minipage}[c]{0.1\textwidth} \centering
    {\includegraphics[height=\textwidth,width=\textwidth]{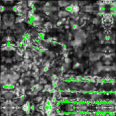}}
    \end{minipage} &
     \begin{minipage}[c]{0.1\textwidth} \centering
    {\includegraphics[height=\textwidth,width=\textwidth]{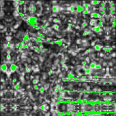}}
    \end{minipage} &
     \begin{minipage}[c]{0.1\textwidth} \centering
    {\includegraphics[height=\textwidth,width=\textwidth]{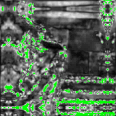}}
    \end{minipage} &
     \begin{minipage}[c]{0.1\textwidth} \centering
    {\includegraphics[height=\textwidth,width=\textwidth]{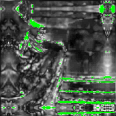}}
    \end{minipage} &
     \begin{minipage}[c]{0.1\textwidth} \centering
    {\includegraphics[height=\textwidth,width=\textwidth]{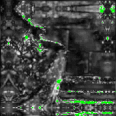}}
    \end{minipage}\\
    \end{tabular}
     \begin{minipage}[c]{0.21\textwidth} \centering
    {\includegraphics[height=\textwidth,width=\textwidth]{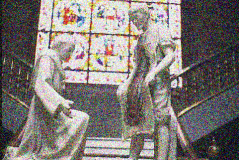}}
    \end{minipage}
\begin{tabular}{ccccccc}
     \begin{minipage}[c]{0.1\textwidth} \centering
    {\includegraphics[height=\textwidth,width=\textwidth]{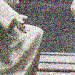}}
    \end{minipage} &
    \begin{minipage}[c]{0.1\textwidth} \centering
    {\includegraphics[height=\textwidth,width=\textwidth]{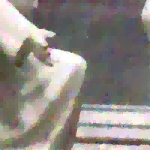}}
    \end{minipage} &
    \begin{minipage}[c]{0.1\textwidth} \centering
    {\includegraphics[height=\textwidth,width=\textwidth]{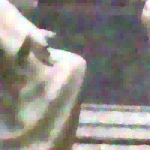}}
    \end{minipage} &
    \begin{minipage}[c]{0.1\textwidth} \centering
    {\includegraphics[height=\textwidth,width=\textwidth]{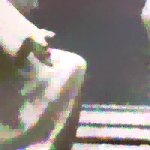}}
    \end{minipage} &
    \begin{minipage}[c]{0.1\textwidth} \centering
    {\includegraphics[height=\textwidth,width=\textwidth]{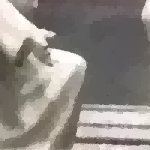}}
    \end{minipage} &
    \begin{minipage}[c]{0.1\textwidth} \centering
    {\includegraphics[height=\textwidth,width=\textwidth]{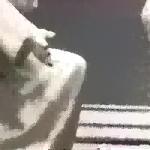}}
    \end{minipage}&
    \begin{minipage}[c]{0.1\textwidth} \centering
    {\includegraphics[height=\textwidth,width=\textwidth]{figs/24077,denoise/24077_zoom.jpg}}
    \end{minipage}\\ \specialrule{0em}{1pt}{1pt}
       \begin{minipage}[c]{0.1\textwidth} \centering
    {\includegraphics[height=\textwidth,width=\textwidth]{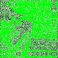}}
    \end{minipage} &
    \begin{minipage}[c]{0.1\textwidth} \centering
    {\includegraphics[height=\textwidth,width=\textwidth]{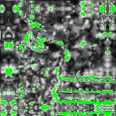}}
    \end{minipage} &
    \begin{minipage}[c]{0.1\textwidth} \centering
    {\includegraphics[height=\textwidth,width=\textwidth]{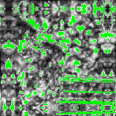}}
    \end{minipage} &
    \begin{minipage}[c]{0.1\textwidth} \centering
    {\includegraphics[height=\textwidth,width=\textwidth]{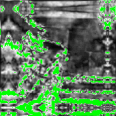}}
    \end{minipage} &
    \begin{minipage}[c]{0.1\textwidth} \centering
    {\includegraphics[height=\textwidth,width=\textwidth]{figs/24077,denoise/24077_SVTV_1438_zoom.png}}
    \end{minipage} &
    \begin{minipage}[c]{0.1\textwidth} \centering
    {\includegraphics[height=\textwidth,width=\textwidth]{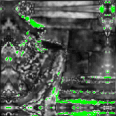}}
    \end{minipage} \\
    \end{tabular}\\
    \caption{Top to bottom: the corresponding results with noise level d = 30/255, 50/255, 70/255 respectively. The results include the noisy image (left large picture), the corresponding zoom-in parts of the noise image, the restored results by using CTV, HTV, NLTV, SVTV, SVS-NLTV, the ground-truth image respectively. The spatial distributions of S-CIELAB error (larger than 15 units) are also shown.}
    \label{24077L2zoom}
\end{figure}

\begin{table}[htbp]
\centering
\caption{Measure values of the restored results (Gaussian blur) by using different methods}
\begin{tabular}{|c|c|c|c|c|c|c|}
\hline
\multicolumn{1}{|l|}{} & Measure & CTV & GVTV  & NLTV & SV-TV & Proposed\\ \hline
\multirow{4}{*}{Fig \ref{figgaussianblur} (a)}  & QSSIM  & 0.54932 & 0.44037 & 0.43908 & 0.58703 & \textbf{0.73872}\\ \cline{2-7} 
& SSIM     & 0.52587 & 0.42153 & 0.50911 & 0.42634 & \textbf{0.57291}\\ \cline{2-7} 
& PSNR  & 19.0002 & 18.1206  & 18.8619 & 18.2246 & \textbf{19.3935} \\ \cline{2-7} 
& S-CIELAB   & 9259  & 13567  & 10447  & 15106  & \textbf{7365}  \\ \hline
\multirow{4}{*}{Fig \ref{figgaussianblur} (b)}  & QSSIM    & 0.63192 & 0.52831  & 0.61665 & 0.5013 & \textbf{0.66251} \\ \cline{2-7} 
& SSIM    & 0.61424 & 0.51153  & 0.59669 & 0.49134 & \textbf{0.65377}  \\ \cline{2-7} 
& PSNR   & 20.9294 & 19.6248  & 20.7522 & 19.7442 & \textbf{21.5578}       \\ \cline{2-7} 
& S-CIELAB    & 19711  & 29256 & 21599  & 28666   & \textbf{16015}  \\ \hline
\multirow{4}{*}{Fig \ref{figgaussianblur} (c)}  & QSSIM  & 0.58846 & 0.49069  & 0.57283 & 0.50512 & \textbf{0.62673}   \\ \cline{2-7} 
& SSIM  & 0.58054 & 0.48477  & 0.56231 & 50154 & \textbf{0.62733}   \\ \cline{2-7} 
& PSNR   & 21.7031 & 20.399 & 21.4503 &  20.5671 & \textbf{22.335} \\ \cline{2-7} 
& S-CIELAB   & 10903  & 18967 & 13083  & 17727   & \textbf{7690}  \\ \hline
\multirow{4}{*}{\thead{\scriptsize Average of\\ 60 testing images}} & QSSIM    & 0.73279 & 0.67031   & 0.73297 & 0.6715         & \textbf{0.76944} \\ \cline{2-7} 
& SSIM    & 0.72219 & 0.6598 & 0.7214 & 0.6588  & \textbf{0.7634} \\ \cline{2-7} 
& PSNR    & 25.3757 & 24.0101 & 25.1935 & 23.8565 & \textbf{26.1201}       \\ \cline{2-7} 
& S-CIELAB & 10903  & 18967 & 13083  & 17727   & \textbf{7690}  \\ \hline
\end{tabular}
\label{table5.1}
\end{table}

\subsection{Image denoising \uppercase\expandafter{\romannumeral2}: Poisson noise}

In this section, we still make use of 60 images taken from Berkeley Segmentation Database \cite{martin2001database} to test the proposed $\stvc$ plus L1 fidelity model for color image restoration with respect to Poisson noise. We compare CTV-L1\cite{blomgren1998total}, SVTV-L1\cite{wang2022color}, HTV-L1\cite{liu2014high}, NLTV-L1 and the proposed $\stvco$ on the testing images. For the proposed $\stvco$ model, we set the parameter of the value channel to be $\mu$ = 0.05, the parameters $\lambda$, $\delta$ in Bregman iteration to be $\lambda =1$, $\delta = 1$. For the regularization parameter $\alpha$, we set a range of [$\frac{\sqrt{N}}{1000}$,$\frac{\sqrt{N}}{10}$] with a step size of 0.01 where N is the total pixel numbers. The regularization parameter($\lambda$) range for CTV model is set to be [$\frac{\sqrt{N}}{1000}$, $\frac{\sqrt{N}}{10}$]. For HTV model, we set c $\in$ $\{0.1, 0.2\}$, $\beta_1$ =15, $\beta_2$ =40, $\beta_3$ $\in$ $\{100, 200\}$ for all the experiments. For SVTV model, we set a range of [$\frac{\sqrt{N}}{1000}$, $\frac{\sqrt{N}}{10}$] with a step size of 0.5 for the parameter of the regularization parameter($\lambda$), and set the penalty parameter $\beta \in \{0.01, 0.1 \}$ for both SVTV and CTV model. 

In order to test the proposed model with respect to different noise levels, we make use of the following Matlab command to generate the degraded image
contaminated by Poisson noise,
\begin{gather*}
\textbf{Z} = \text{poissrnd} \big(\max(0, \textbf{I}/d^2)\big)*d^2, 
\end{gather*}
where $\textbf{I}$ is the ground-truth image, \textbf{Z} is the noisy image, $d$ is the scale factor. The ground-truth images are degraded artificially by Poisson noise with different scales. We compute the PSNR, SSIM, QSSIM values and the S-CIELAB error value (pixel number) for each restored result by comparing it with the ground-truth image. Again we obtain the optimal restored result by choosing the optimal value of the regularization parameter in terms of PSNR value for each testing method. In Figures \ref{p-psnr}-\ref{p-qssim}, we give the spatial distributions of PSNR, SSIM, and QSSIM values of the restored results corresponding to $d = 0.2, 0.3, 0.4$ respectively. We also show the histograms of the average PSNR, SSIM, and QSSIM values. Again we observe from the figures that the proposed $\stvco$ model provides very competitive PSNR, SSIM, and QSSIM values.

As examples of this experiment, we display 4 sets of restored results in Figures \ref{253027L1}, \ref{101084L1}, \ref{167083L1}, \ref{326084L1}. The restored results and the corresponding histograms of PSNR, SSIM, QSSIM, and S-CIELAB error values by using different methods are also given in the figures. We see from the histograms that the proposed SVS-NLTV-L1 model consistently achieves the best PSNR, SSIM, QSSIM, and S-CIELAB error values. Additionally, the zoom-in parts and the spatial distribution of pixels with S-CIELAB errors greater than 15 units are provided in Figures \ref{253027L1zoom}, \ref{101084L1zoom}, \ref{167083L1zoom}, and \ref{326085L1zoom}. Since RGB channels are not coupled in CTV-L1, HTV, and NLTV-L1 models, these three methods do not work effectively for color artifacts removal. SVTV-L1 considers the coupling of RGB channels, but produces unsatisfactory results due to the staircase effect of total variation. $\stvco$ model combines the saturation-value similarity information and nonlocal method, thus it always provides visually better restored results, and the noise and color artifacts
are removed more thoroughly. As expected, the restored results and the corresponding zoom-in parts show that the proposed $\stvco$ model are much better than those by using SVTV-L1, CTV-L1, HTV-L1, and NLTV-L1 visually, see especially the extraordinary effectiveness of edge preserving effect in \ref{253027L1} and \ref{253027L1zoom}, the denoising effect of the sky part in \ref{101084L1} and \ref{101084L1zoom}, the texture preserving effect in the head region of \ref{326085L1zoom}, the face region of \ref{101084L1} and \ref{101084L1zoom}. In summary, we emphasize that the proposed $\stvco$ model is highly competitive in terms of visual quality and performance on testing criteria for Poisson noise removal.

\begin{figure}[htbp]
\centering
\tabcolsep=1pt
\begin{tabular}{ccccc}
    \begin{minipage}[c]{0.184\textwidth} \centering
   \subfigure{\includegraphics[width =\textwidth]{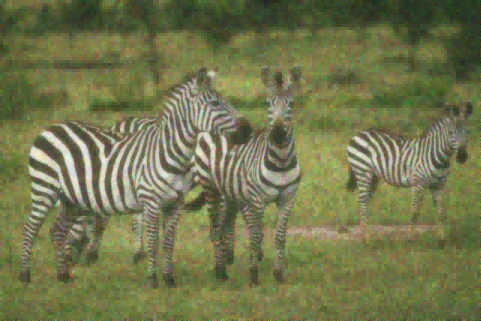}}
    \end{minipage} &
        \begin{minipage}[c]{0.184\textwidth} \centering
   \subfigure{\includegraphics[width =\textwidth]{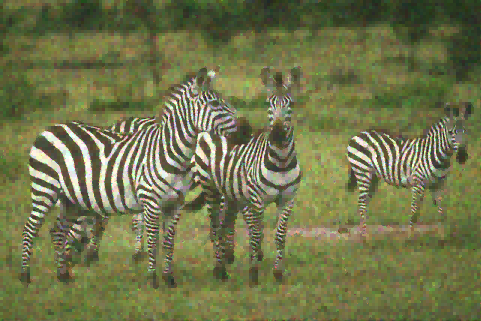}}
    \end{minipage} &
        \begin{minipage}[c]{0.184\textwidth} \centering
   \subfigure{\includegraphics[width =\textwidth]{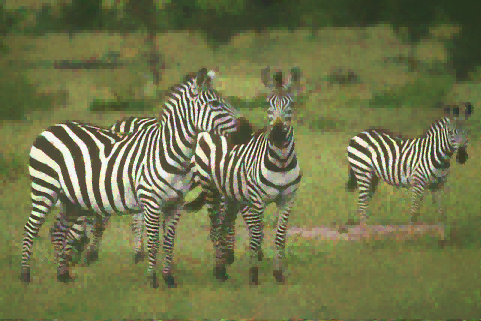}}
    \end{minipage} &
    \begin{minipage}[c]{0.184\textwidth} \centering
   \subfigure{\includegraphics[width = \textwidth]{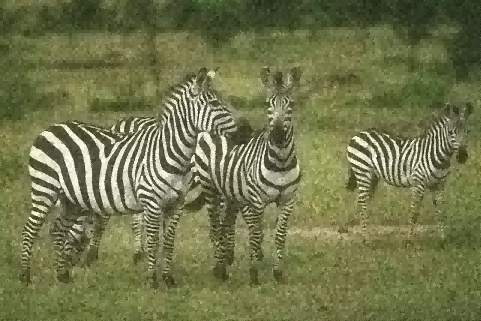}}
    \end{minipage} &
        \begin{minipage}[c]{0.184\textwidth} \centering
   \subfigure{\includegraphics[width = \textwidth]{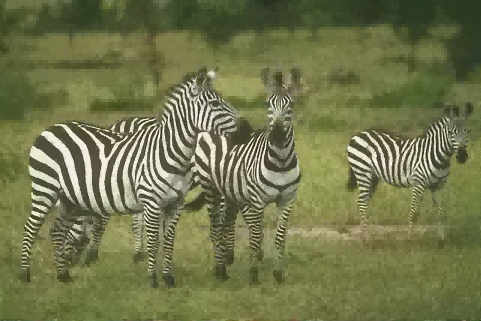}}
    \end{minipage} \\
    \begin{minipage}[c]{0.184\textwidth} \centering
   \subfigure{\includegraphics[width =\textwidth]{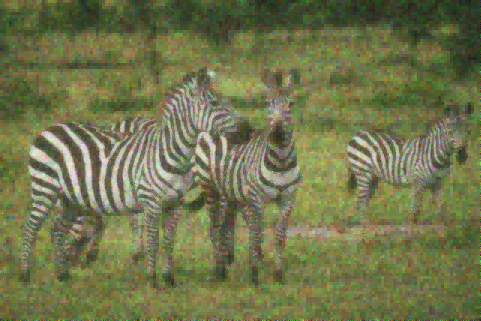}}
    \end{minipage} &
        \begin{minipage}[c]{0.184\textwidth} \centering
\subfigure{\includegraphics[width =\textwidth]{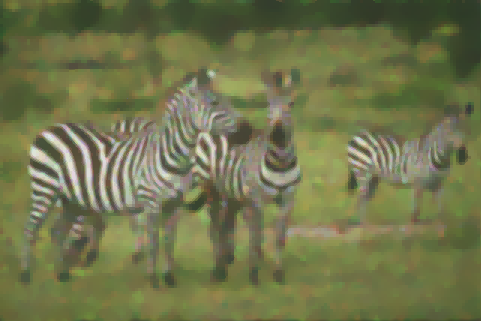}}
    \end{minipage} &
        \begin{minipage}[c]{0.184\textwidth} \centering
\subfigure{\includegraphics[width=\textwidth]{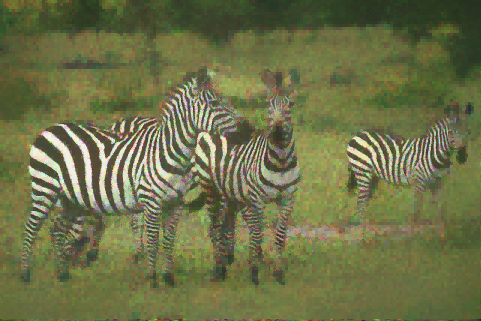}}
    \end{minipage} &
        \begin{minipage}[c]{0.184\textwidth} \centering
\subfigure{\includegraphics[width =\textwidth]{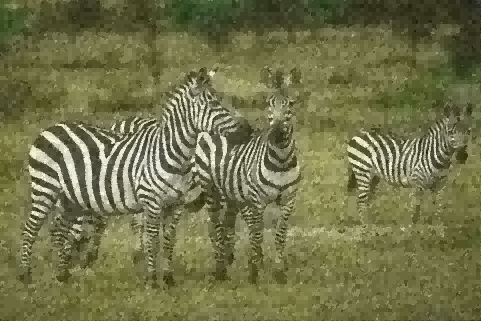}}
    \end{minipage} &
        \begin{minipage}[c]{0.184\textwidth} \centering
\subfigure{\includegraphics[width=\textwidth]{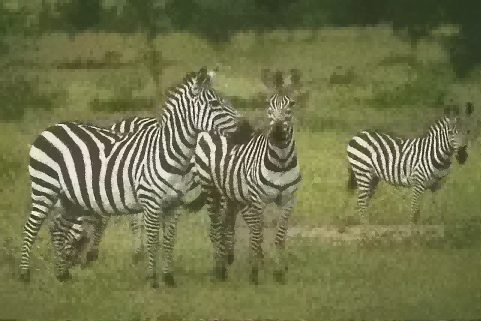}}
    \end{minipage} \\
     \begin{minipage}[c]{0.184\textwidth} \centering
   \subfigure{\includegraphics[width =\textwidth]{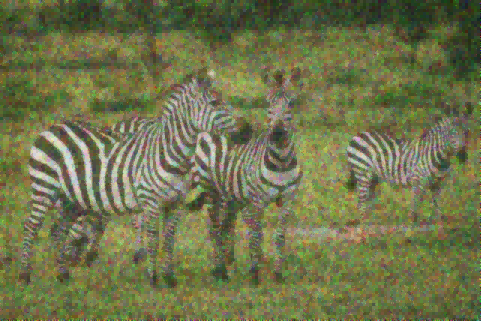}}
    \end{minipage} &
        \begin{minipage}[c]{0.184\textwidth} \centering
\subfigure{\includegraphics[width =\textwidth]{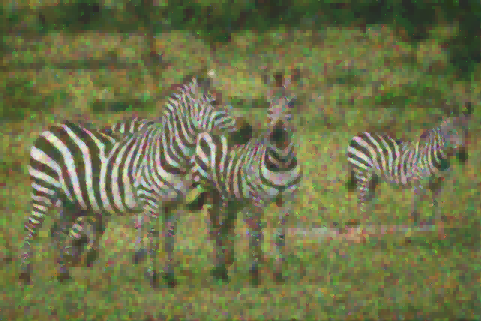}}
    \end{minipage} &
        \begin{minipage}[c]{0.184\textwidth} \centering
\subfigure{\includegraphics[width=\textwidth]{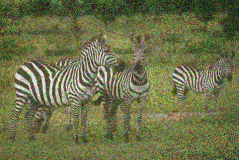}}
    \end{minipage} &
        \begin{minipage}[c]{0.184\textwidth} \centering
\subfigure{\includegraphics[width =\textwidth]{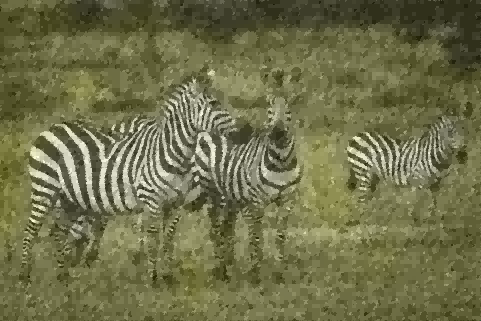}}
    \end{minipage} &
        \begin{minipage}[c]{0.184\textwidth} \centering
\subfigure{\includegraphics[width=\textwidth]{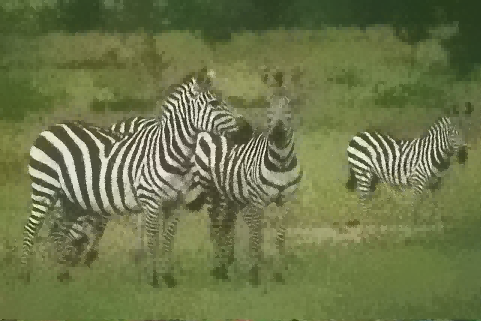}}
    \end{minipage} \\
\end{tabular}
\begin{subfigure}
\centering
\includegraphics[width=0.23\textwidth]{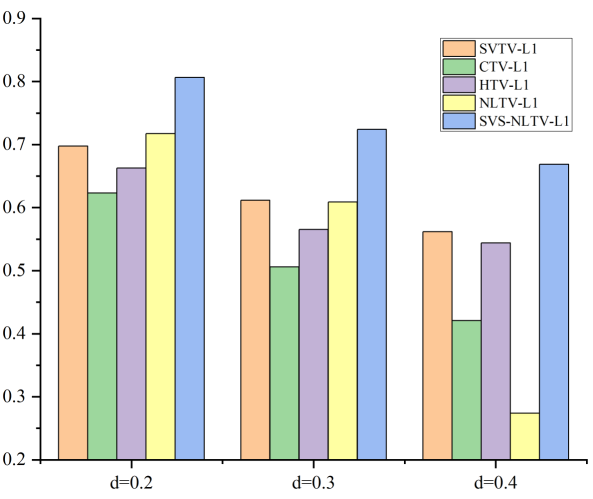}
\end{subfigure}
\begin{subfigure}
\centering
\includegraphics[width=0.23\textwidth]{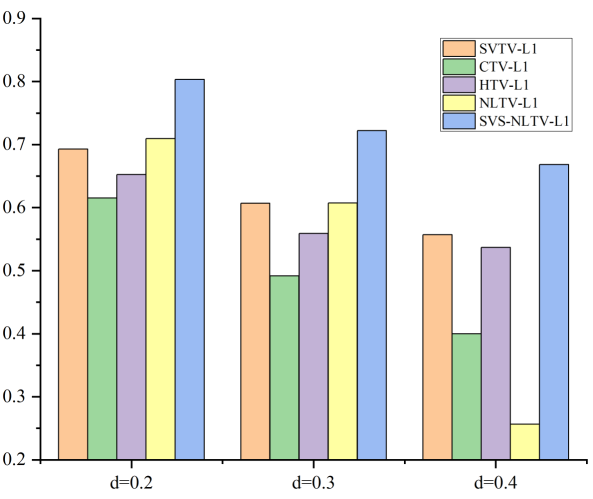}
\end{subfigure}
\begin{subfigure}
\centering
\includegraphics[width=0.23\textwidth]{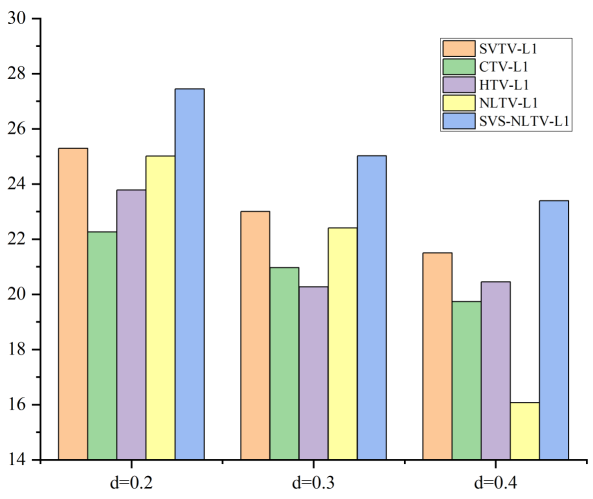}
\end{subfigure}
\begin{subfigure}
\centering
\includegraphics[width=0.23\textwidth]{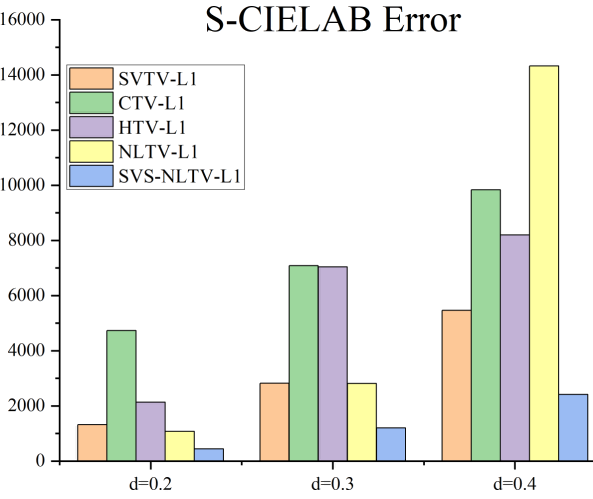}
\end{subfigure}
\caption{The first three rows: top to bottom: degraded and restored images with noise level d = 0.2, 0.3, 0.4 respectively; left to right: the restored results by using CTV, HTV, NLTV, SVTV, and SVS-NLTV respectively. The fourth row: the histograms of measure values by using different methods.}
\label{253027L1}
\end{figure}

\begin{figure}[htbp]
\centering
\tabcolsep=1pt
\begin{minipage}[c]{0.21\textwidth} \centering
    {\includegraphics[height=\textwidth,width=\textwidth]{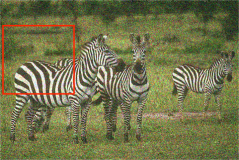}}
    \end{minipage}
\begin{tabular}{ccccccc}
    \begin{minipage}[c]{0.1\textwidth} \centering
    {\includegraphics[height=\textwidth,width=\textwidth]{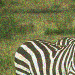}}
    \end{minipage} &
    \begin{minipage}[c]{0.1\textwidth} \centering
    {\includegraphics[height=\textwidth,width=\textwidth]{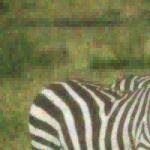}}
    \end{minipage} &
    \begin{minipage}[c]{0.1\textwidth} \centering
    {\includegraphics[height=\textwidth,width=\textwidth]{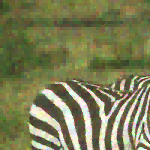}}
    \end{minipage} &
    \begin{minipage}[c]{0.1\textwidth} \centering
    {\includegraphics[height=\textwidth,width=\textwidth]{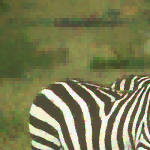}}
    \end{minipage} &
    \begin{minipage}[c]{0.1\textwidth} \centering
    {\includegraphics[height=\textwidth,width=\textwidth]{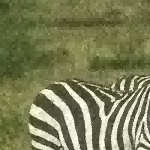}}
    \end{minipage} &
    \begin{minipage}[c]{0.1\textwidth} \centering
    {\includegraphics[height=\textwidth,width=\textwidth]{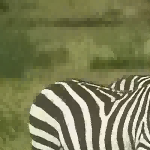}}
    \end{minipage} &
    \begin{minipage}[c]{0.1\textwidth} \centering
    {\includegraphics[height=\textwidth,width=\textwidth]{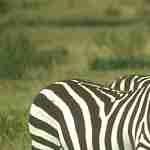}}
    \end{minipage}\\ \specialrule{0em}{1pt}{1pt}
    \begin{minipage}[c]{0.1\textwidth} \centering
    {\includegraphics[height=\textwidth,width=\textwidth]{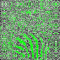}}
    \end{minipage} &
     \begin{minipage}[c]{0.1\textwidth} \centering
    {\includegraphics[height=\textwidth,width=\textwidth]{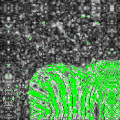}}
    \end{minipage} &
     \begin{minipage}[c]{0.1\textwidth} \centering
    {\includegraphics[height=\textwidth,width=\textwidth]{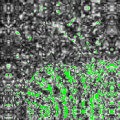}}
    \end{minipage} &
     \begin{minipage}[c]{0.1\textwidth} \centering
    {\includegraphics[height=\textwidth,width=\textwidth]{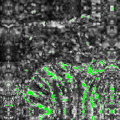}}
    \end{minipage} &
     \begin{minipage}[c]{0.1\textwidth} \centering
    {\includegraphics[height=\textwidth,width=\textwidth]{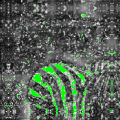}}
    \end{minipage} &
     \begin{minipage}[c]{0.1\textwidth} \centering
    {\includegraphics[height=\textwidth,width=\textwidth]{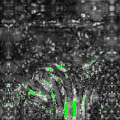}}
    \end{minipage} \\
    \end{tabular}
    \begin{minipage}[c]{0.21\textwidth} \centering
    {\includegraphics[height=\textwidth,width=\textwidth]{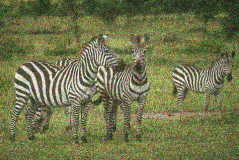}}
    \end{minipage}
\begin{tabular}{ccccccc}
     \begin{minipage}[c]{0.1\textwidth} \centering
    {\includegraphics[height=\textwidth,width=\textwidth]{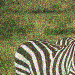}}
    \end{minipage} &
    \begin{minipage}[c]{0.1\textwidth} \centering
    {\includegraphics[height=\textwidth,width=\textwidth]{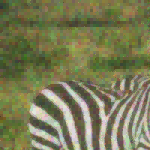}}
    \end{minipage} &
    \begin{minipage}[c]{0.1\textwidth} \centering
    {\includegraphics[height=\textwidth,width=\textwidth]{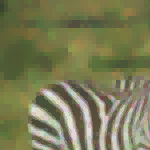}}
    \end{minipage} &
    \begin{minipage}[c]{0.1\textwidth} \centering
    {\includegraphics[height=\textwidth,width=\textwidth]{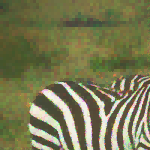}}
    \end{minipage} &
    \begin{minipage}[c]{0.1\textwidth} \centering
    {\includegraphics[height=\textwidth,width=\textwidth]{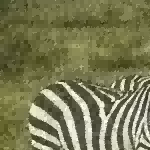}}
    \end{minipage} &
    \begin{minipage}[c]{0.1\textwidth} \centering
    {\includegraphics[height=\textwidth,width=\textwidth]{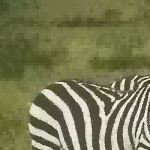}}
    \end{minipage} &
    \begin{minipage}[c]{0.1\textwidth} \centering
    {\includegraphics[height=\textwidth,width=\textwidth]{figs/253027,L1,denoise/253027_zoom.jpg}}
    \end{minipage}\\ \specialrule{0em}{1pt}{1pt}
     \begin{minipage}[c]{0.1\textwidth} \centering
    {\includegraphics[height=\textwidth,width=\textwidth]{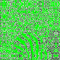}}
    \end{minipage} &
     \begin{minipage}[c]{0.1\textwidth} \centering
    {\includegraphics[height=\textwidth,width=\textwidth]{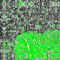}}
    \end{minipage} &
     \begin{minipage}[c]{0.1\textwidth} \centering
    {\includegraphics[height=\textwidth,width=\textwidth]{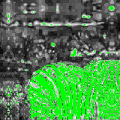}}
    \end{minipage} &
     \begin{minipage}[c]{0.1\textwidth} \centering
    {\includegraphics[height=\textwidth,width=\textwidth]{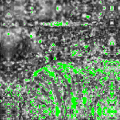}}
    \end{minipage} &
     \begin{minipage}[c]{0.1\textwidth} \centering
    {\includegraphics[height=\textwidth,width=\textwidth]{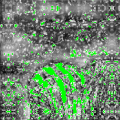}}
    \end{minipage} &
     \begin{minipage}[c]{0.1\textwidth} \centering
    {\includegraphics[height=\textwidth,width=\textwidth]{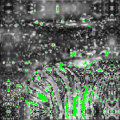}}
    \end{minipage}\\
    \end{tabular}
     \begin{minipage}[c]{0.21\textwidth} \centering
    {\includegraphics[height=\textwidth,width=\textwidth]{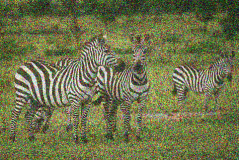}}
    \end{minipage}
\begin{tabular}{ccccccc}
     \begin{minipage}[c]{0.1\textwidth} \centering
    {\includegraphics[height=\textwidth,width=\textwidth]{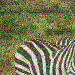}}
    \end{minipage} &
    \begin{minipage}[c]{0.1\textwidth} \centering
    {\includegraphics[height=\textwidth,width=\textwidth]{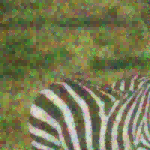}}
    \end{minipage} &
    \begin{minipage}[c]{0.1\textwidth} \centering
    {\includegraphics[height=\textwidth,width=\textwidth]{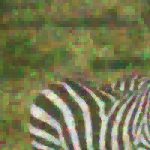}}
    \end{minipage} &
    \begin{minipage}[c]{0.1\textwidth} \centering
    {\includegraphics[height=\textwidth,width=\textwidth]{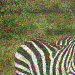}}
    \end{minipage} &
    \begin{minipage}[c]{0.1\textwidth} \centering
    {\includegraphics[height=\textwidth,width=\textwidth]{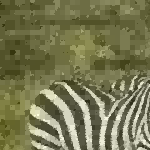}}
    \end{minipage} &
    \begin{minipage}[c]{0.1\textwidth} \centering
    {\includegraphics[height=\textwidth,width=\textwidth]{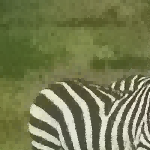}}
    \end{minipage} &
    \begin{minipage}[c]{0.1\textwidth} \centering
    {\includegraphics[height=\textwidth,width=\textwidth]{figs/253027,L1,denoise/253027_zoom.jpg}}
    \end{minipage}\\ \specialrule{0em}{1pt}{1pt}
     \begin{minipage}[c]{0.1\textwidth} \centering
    {\includegraphics[height=\textwidth,width=\textwidth]{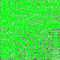}}
    \end{minipage} &
     \begin{minipage}[c]{0.1\textwidth} \centering
    {\includegraphics[height=\textwidth,width=\textwidth]{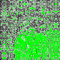}}
    \end{minipage} &
     \begin{minipage}[c]{0.1\textwidth} \centering
    {\includegraphics[height=\textwidth,width=\textwidth]{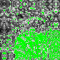}}
    \end{minipage} &
     \begin{minipage}[c]{0.1\textwidth} \centering
    {\includegraphics[height=\textwidth,width=\textwidth]{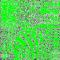}}
    \end{minipage} &
     \begin{minipage}[c]{0.1\textwidth} \centering
    {\includegraphics[height=\textwidth,width=\textwidth]{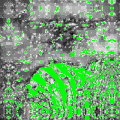}}
    \end{minipage} &
     \begin{minipage}[c]{0.1\textwidth} \centering
    {\includegraphics[height=\textwidth,width=\textwidth]{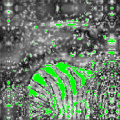}}
    \end{minipage}\\
    \end{tabular}\\
    \caption{Top to bottom: the corresponding results with noise level d = 0.2, 0.3, 0.4 respectively. The results include the noisy image (left large picture), the corresponding zoom-in parts of the noise image, the restored results by using CTV, HTV, NLTV, SVTV, SVS-NLTV, the ground-truth image respectively. The spatial distributions of S-CIELAB error (larger than 15 units) are also shown.}
    \label{253027L1zoom}
\end{figure}

\begin{figure}[htbp]
\centering
\tabcolsep=1pt
\begin{tabular}{ccccc}
    \begin{minipage}[c]{0.184\textwidth} \centering
   \subfigure{\includegraphics[width =\textwidth]{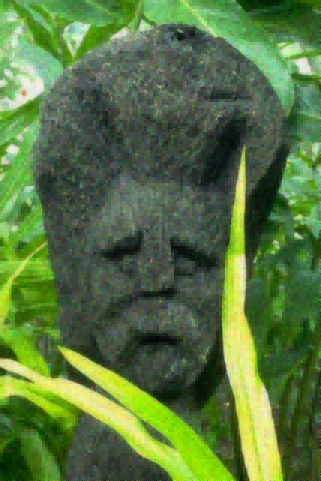}}
    \end{minipage} &
        \begin{minipage}[c]{0.184\textwidth} \centering
   \subfigure{\includegraphics[width =\textwidth]{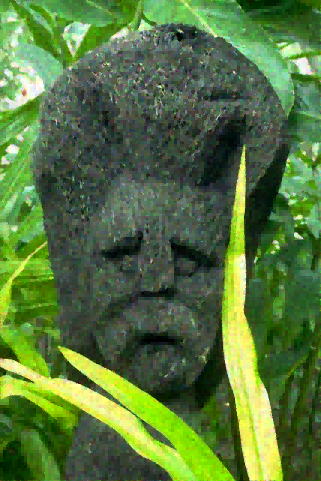}}
    \end{minipage} &
        \begin{minipage}[c]{0.184\textwidth} \centering
   \subfigure{\includegraphics[width =\textwidth]{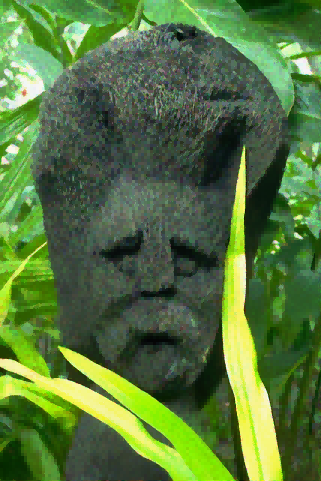}}
    \end{minipage} &
    \begin{minipage}[c]{0.184\textwidth} \centering
   \subfigure{\includegraphics[width = \textwidth]{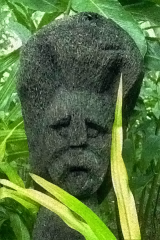}}
    \end{minipage} &
        \begin{minipage}[c]{0.184\textwidth} \centering
   \subfigure{\includegraphics[width = \textwidth]{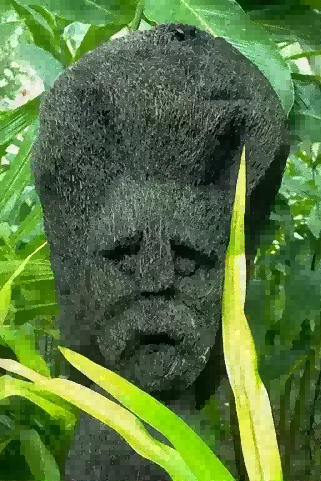}}
    \end{minipage} \\
     \begin{minipage}[c]{0.184\textwidth} \centering
   \subfigure{\includegraphics[width =\textwidth]{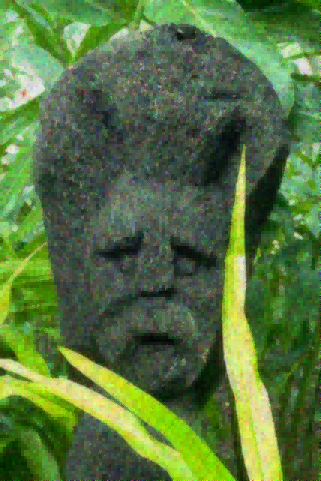}}
    \end{minipage} &
        \begin{minipage}[c]{0.184\textwidth} \centering
   \subfigure{\includegraphics[width =\textwidth]{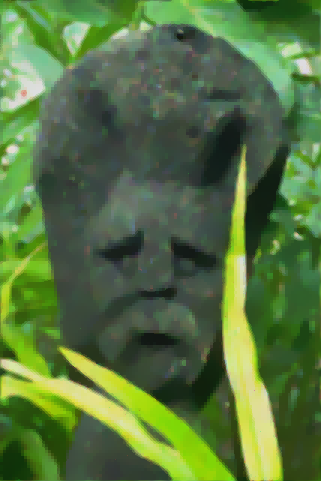}}
   \end{minipage} &
        \begin{minipage}[c]{0.184\textwidth} \centering
   \subfigure{\includegraphics[width =\textwidth]{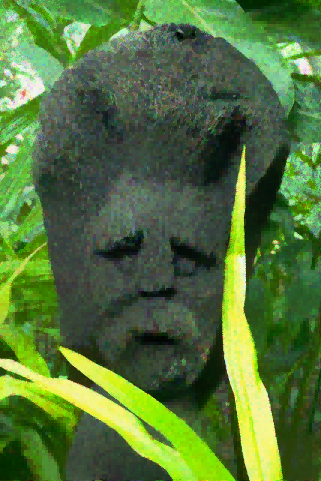}}
    \end{minipage} &
    \begin{minipage}[c]{0.184\textwidth} \centering
   \subfigure{\includegraphics[width = \textwidth]{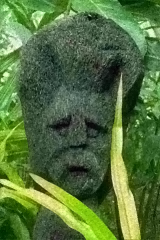}}
    \end{minipage} &
        \begin{minipage}[c]{0.184\textwidth} \centering
   \subfigure{\includegraphics[width = \textwidth]{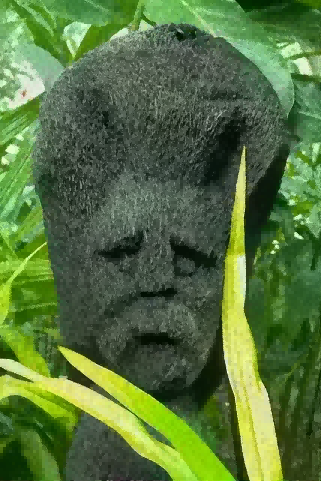}}
    \end{minipage} \\
    \begin{minipage}[c]{0.184\textwidth} \centering
   \subfigure{\includegraphics[width =\textwidth]{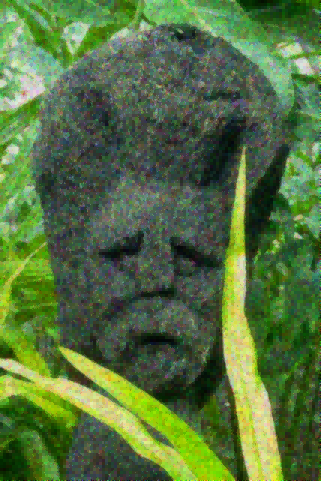}}
    \end{minipage} &
        \begin{minipage}[c]{0.184\textwidth} \centering
\subfigure{\includegraphics[width =\textwidth]{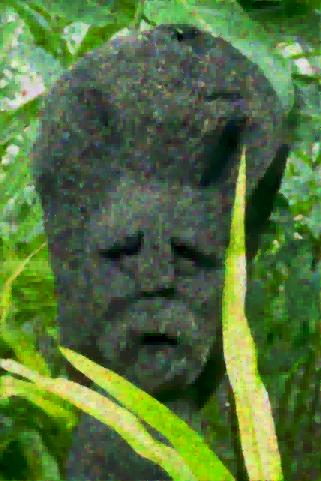}}
    \end{minipage} &
        \begin{minipage}[c]{0.184\textwidth} \centering
\subfigure{\includegraphics[width=\textwidth]{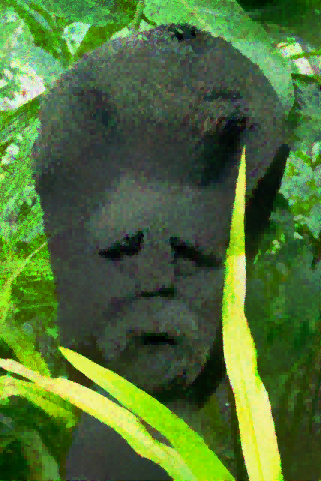}}
    \end{minipage} &
        \begin{minipage}[c]{0.184\textwidth} \centering
\subfigure{\includegraphics[width =\textwidth]{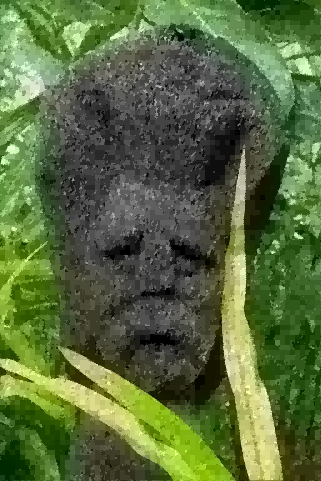}}
    \end{minipage} &
        \begin{minipage}[c]{0.184\textwidth} \centering
\subfigure{\includegraphics[width=\textwidth]{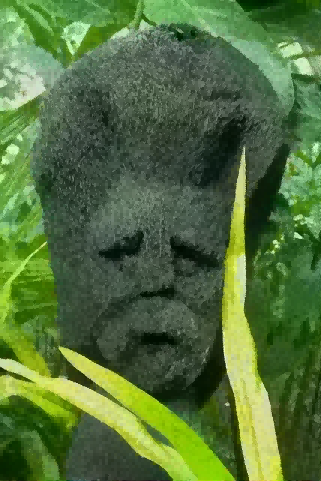}}
    \end{minipage} \\
\end{tabular}
\begin{subfigure}
\centering
\includegraphics[width=0.23\textwidth]{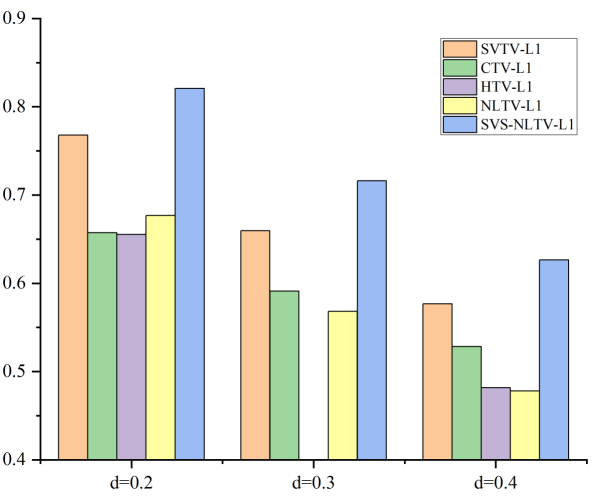}
\end{subfigure}
\begin{subfigure}
\centering
\includegraphics[width=0.23\textwidth]{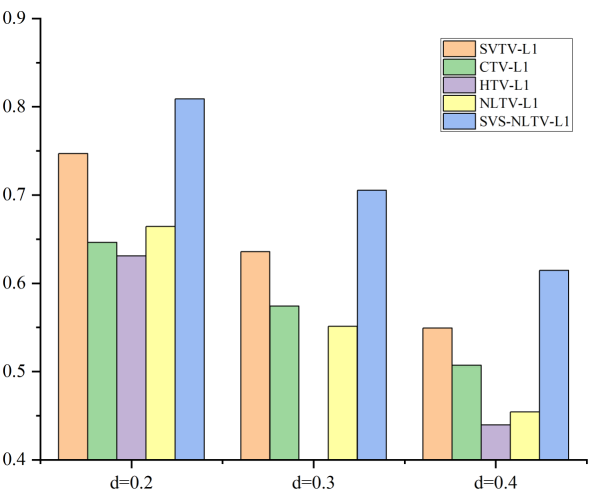}
\end{subfigure}
\begin{subfigure}
\centering
\includegraphics[width=0.23\textwidth]{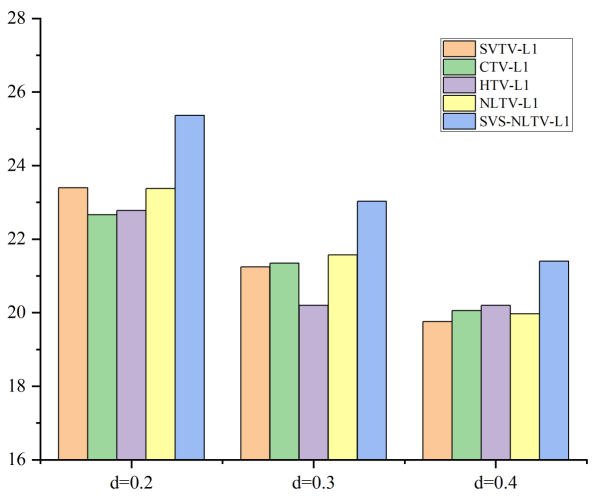}
\end{subfigure}
\begin{subfigure}
\centering
\includegraphics[width=0.23\textwidth]{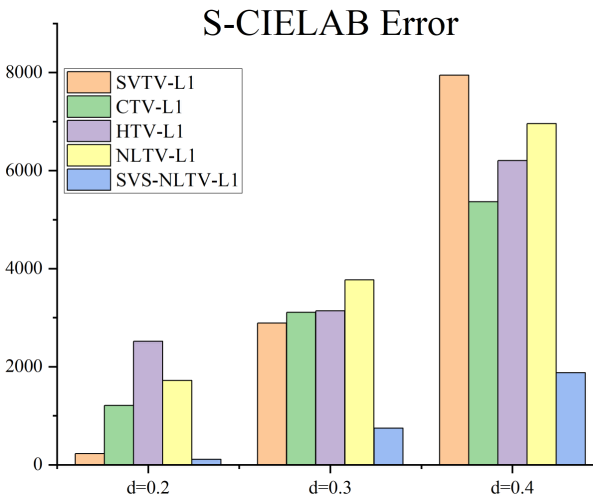}
\end{subfigure}
\caption{The first three rows: top to bottom: degraded and restored images with noise level d = 0.2, 0.3, 0.4 respectively; left to right: the restored results by using CTV, HTV, NLTV, SVTV, and SVS-NLTV respectively. The fourth row: the histograms of measure values by using different methods.}
\label{101084L1}
\end{figure}

\begin{figure}[htbp]
\centering
\begin{subfigure}
\centering
\includegraphics[width=0.23\textwidth]{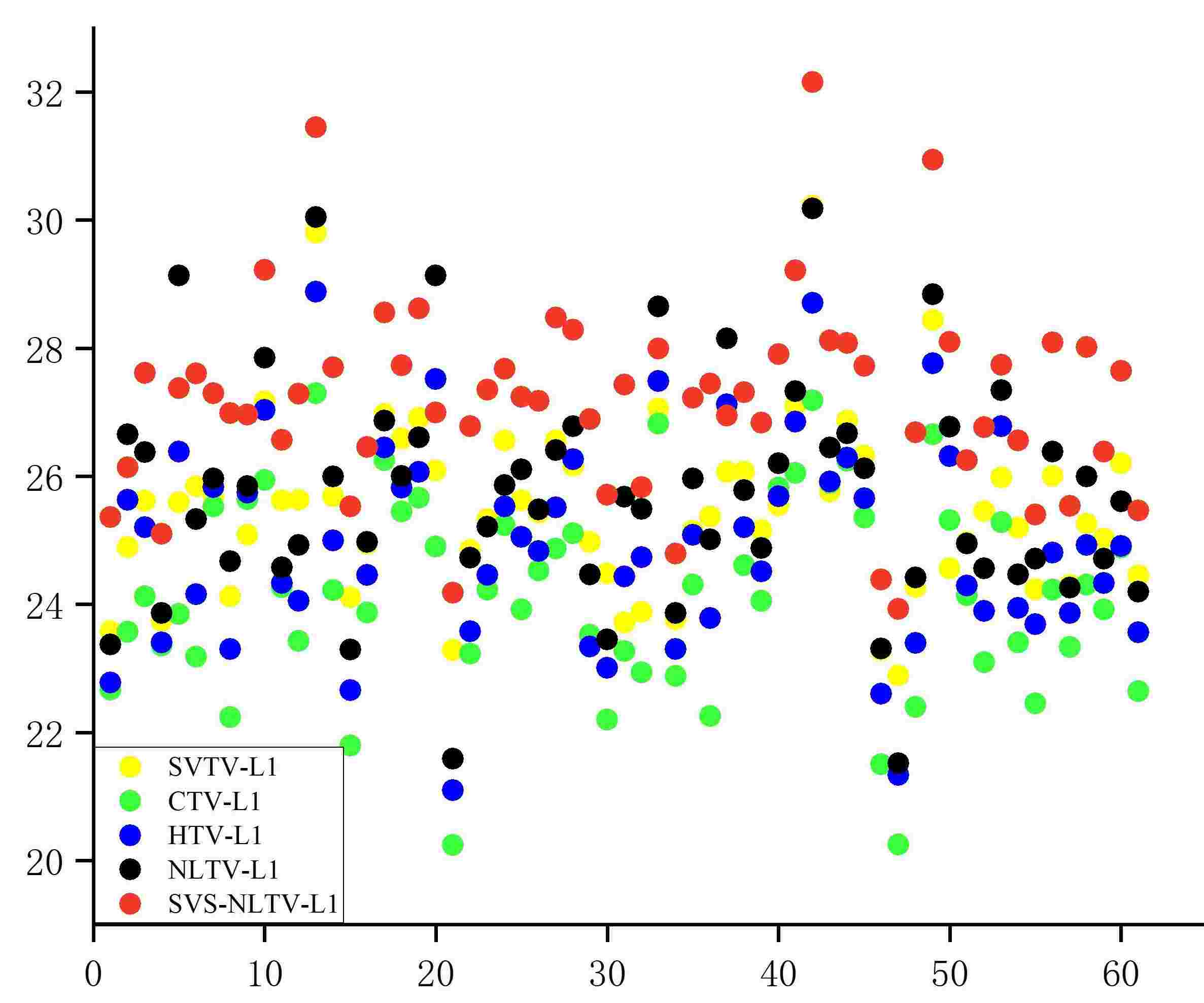}
\end{subfigure}
\begin{subfigure}
\centering
\includegraphics[width=0.23\textwidth]{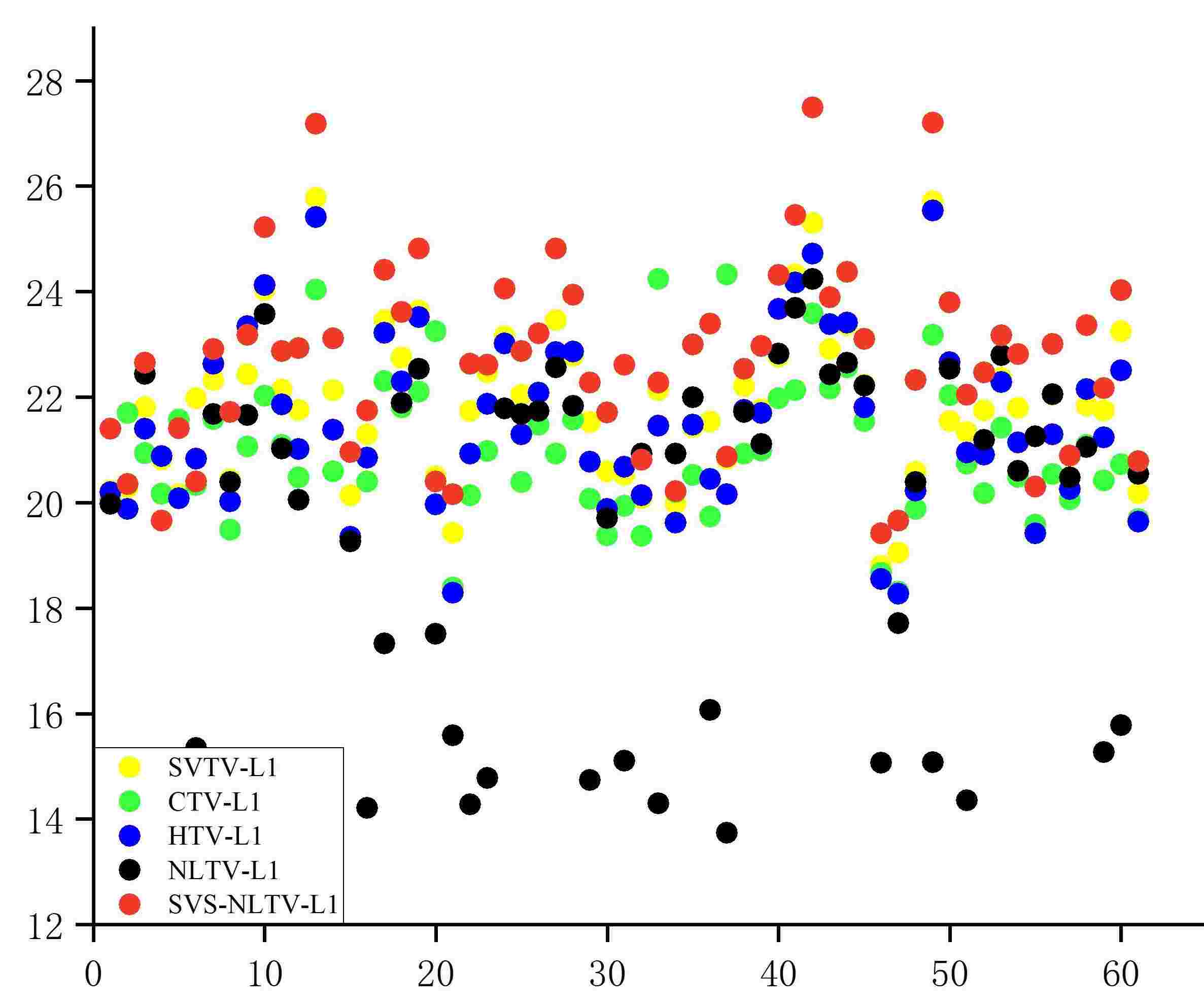}
\end{subfigure}
\begin{subfigure}
\centering
\includegraphics[width=0.23\textwidth]{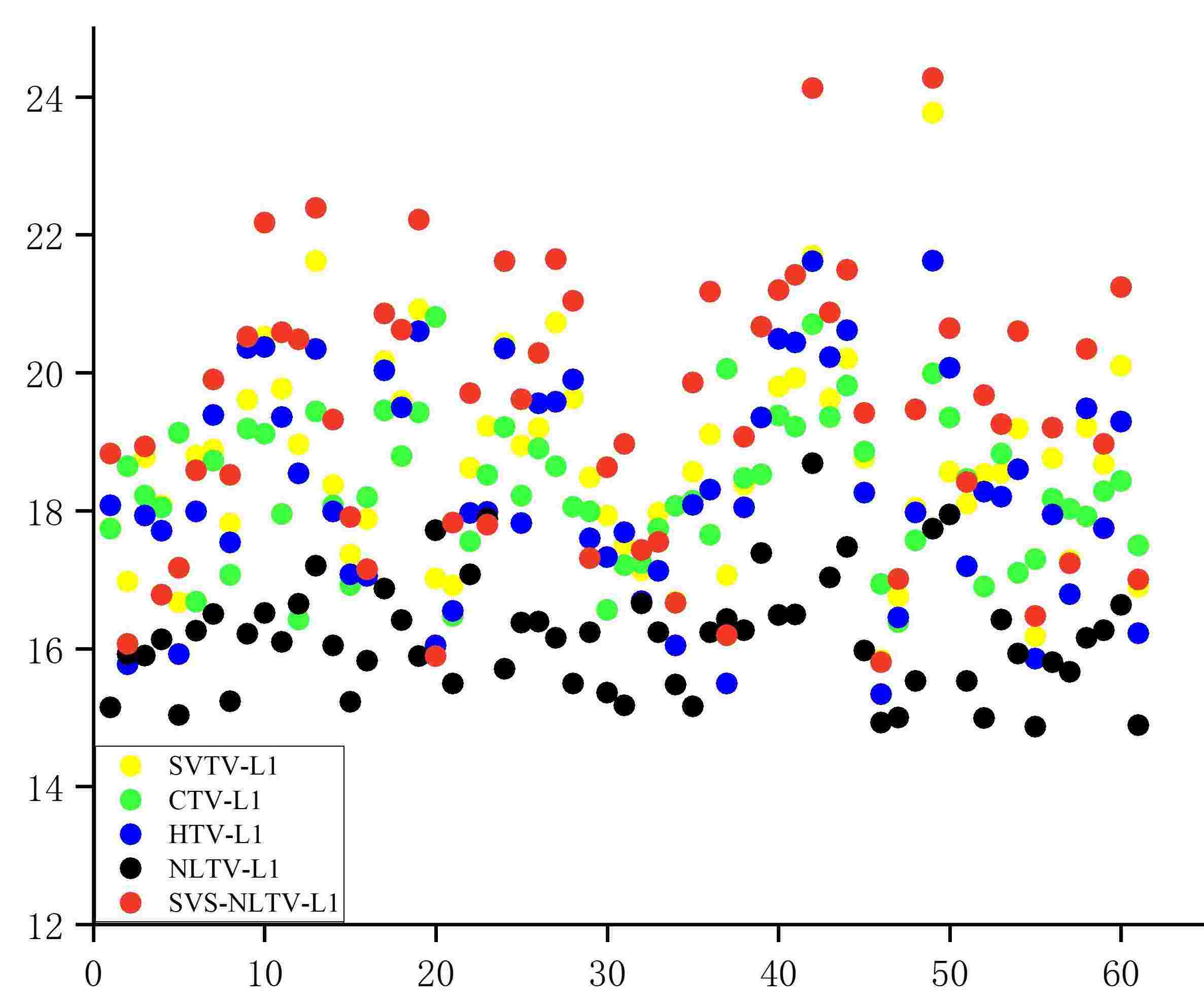}
\end{subfigure}
\begin{subfigure}
\centering
\includegraphics[width=0.23\textwidth]{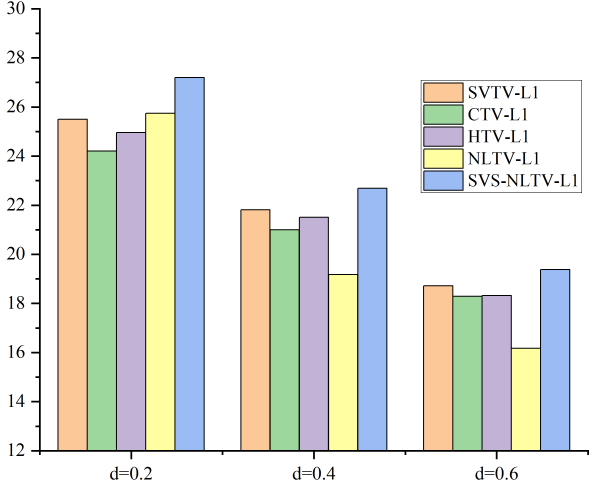}
\end{subfigure}
\caption{First to third: The spatial distributions of PSNR values of the restored results by using different methods corresponding to d = 0.2, 0.4, 0.6 respectively; Fourth: the histogram of the average PSNR values of the restored results by using different methods.}
\label{p-psnr}
\end{figure}

\begin{figure}[htbp]
\centering
\tabcolsep=1pt
\begin{minipage}[c]{0.21\textwidth} \centering
    {\includegraphics[height=\textwidth,width=\textwidth]{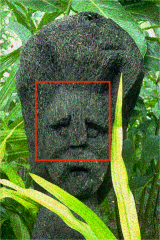}}
    \end{minipage}
\begin{tabular}{ccccccc}
    \begin{minipage}[c]{0.1\textwidth} \centering
    {\includegraphics[height=\textwidth,width=\textwidth]{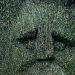}}
    \end{minipage} &
    \begin{minipage}[c]{0.1\textwidth} \centering
    {\includegraphics[height=\textwidth,width=\textwidth]{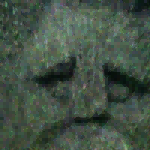}}
    \end{minipage} &
    \begin{minipage}[c]{0.1\textwidth} \centering
    {\includegraphics[height=\textwidth,width=\textwidth]{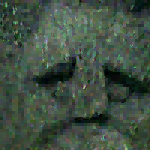}}
    \end{minipage} &
    \begin{minipage}[c]{0.1\textwidth} \centering
    {\includegraphics[height=\textwidth,width=\textwidth]{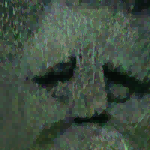}}
    \end{minipage} &
    \begin{minipage}[c]{0.1\textwidth} \centering
    {\includegraphics[height=\textwidth,width=\textwidth]{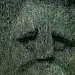}}
    \end{minipage} &
    \begin{minipage}[c]{0.1\textwidth} \centering
    {\includegraphics[height=\textwidth,width=\textwidth]{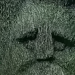}}
    \end{minipage} &
    \begin{minipage}[c]{0.1\textwidth} \centering
    {\includegraphics[height=\textwidth,width=\textwidth]{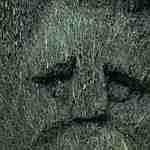}}
    \end{minipage}\\ \specialrule{0em}{1pt}{1pt}
    \begin{minipage}[c]{0.1\textwidth} \centering
    {\includegraphics[height=\textwidth,width=\textwidth]{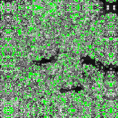}}
    \end{minipage} &
     \begin{minipage}[c]{0.1\textwidth} \centering
    {\includegraphics[height=\textwidth,width=\textwidth]{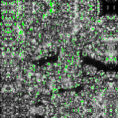}}
    \end{minipage} &
     \begin{minipage}[c]{0.1\textwidth} \centering
    {\includegraphics[height=\textwidth,width=\textwidth]{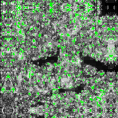}}
    \end{minipage} &
     \begin{minipage}[c]{0.1\textwidth} \centering
    {\includegraphics[height=\textwidth,width=\textwidth]{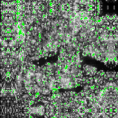}}
    \end{minipage} &
     \begin{minipage}[c]{0.1\textwidth} \centering
    {\includegraphics[height=\textwidth,width=\textwidth]{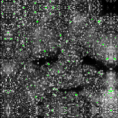}}
    \end{minipage} &
     \begin{minipage}[c]{0.1\textwidth} \centering
    {\includegraphics[height=\textwidth,width=\textwidth]{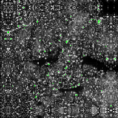}}
    \end{minipage} \\
    \end{tabular}
    \begin{minipage}[c]{0.21\textwidth} \centering
    {\includegraphics[height=\textwidth,width=\textwidth]{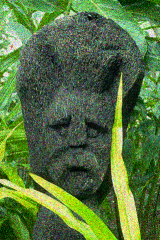}}
    \end{minipage}
\begin{tabular}{ccccccc}
     \begin{minipage}[c]{0.1\textwidth} \centering
    {\includegraphics[height=\textwidth,width=\textwidth]{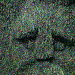}}
    \end{minipage} &
    \begin{minipage}[c]{0.1\textwidth} \centering
    {\includegraphics[height=\textwidth,width=\textwidth]{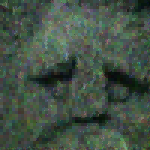}}
    \end{minipage} &
    \begin{minipage}[c]{0.1\textwidth} \centering
    {\includegraphics[height=\textwidth,width=\textwidth]{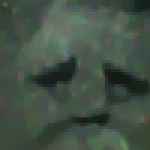}}
    \end{minipage} &
    \begin{minipage}[c]{0.1\textwidth} \centering
    {\includegraphics[height=\textwidth,width=\textwidth]{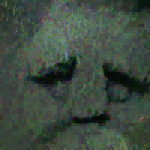}}
    \end{minipage} &
    \begin{minipage}[c]{0.1\textwidth} \centering
    {\includegraphics[height=\textwidth,width=\textwidth]{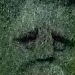}}
    \end{minipage} &
    \begin{minipage}[c]{0.1\textwidth} \centering
    {\includegraphics[height=\textwidth,width=\textwidth]{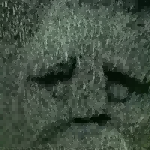}}
    \end{minipage} &
    \begin{minipage}[c]{0.1\textwidth} \centering
    {\includegraphics[height=\textwidth,width=\textwidth]{figs/101084,L1,denoise/101084_zoom.jpg}}
    \end{minipage}\\ \specialrule{0em}{1pt}{1pt}
     \begin{minipage}[c]{0.1\textwidth} \centering
    {\includegraphics[height=\textwidth,width=\textwidth]{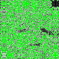}}
    \end{minipage} &
     \begin{minipage}[c]{0.1\textwidth} \centering
    {\includegraphics[height=\textwidth,width=\textwidth]{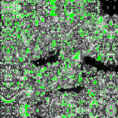}}
    \end{minipage} &
     \begin{minipage}[c]{0.1\textwidth} \centering
    {\includegraphics[height=\textwidth,width=\textwidth]{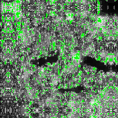}}
    \end{minipage} &
     \begin{minipage}[c]{0.1\textwidth} \centering
    {\includegraphics[height=\textwidth,width=\textwidth]{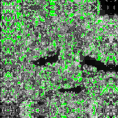}}
    \end{minipage} &
     \begin{minipage}[c]{0.1\textwidth} \centering
    {\includegraphics[height=\textwidth,width=\textwidth]{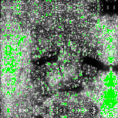}}
    \end{minipage} &
     \begin{minipage}[c]{0.1\textwidth} \centering
    {\includegraphics[height=\textwidth,width=\textwidth]{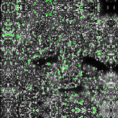}}
    \end{minipage}\\
    \end{tabular}
 \begin{minipage}[c]{0.21\textwidth} \centering
    {\includegraphics[height=\textwidth,width=\textwidth]{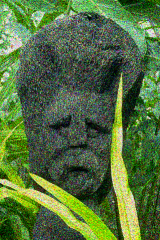}}
    \end{minipage}
\begin{tabular}{ccccccc}
     \begin{minipage}[c]{0.1\textwidth} \centering
    {\includegraphics[height=\textwidth,width=\textwidth]{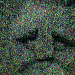}}
    \end{minipage} &
    \begin{minipage}[c]{0.1\textwidth} \centering
    {\includegraphics[height=\textwidth,width=\textwidth]{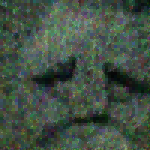}}
    \end{minipage} &
    \begin{minipage}[c]{0.1\textwidth} \centering
    {\includegraphics[height=\textwidth,width=\textwidth]{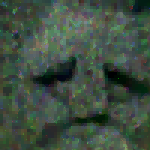}}
    \end{minipage} &
    \begin{minipage}[c]{0.1\textwidth} \centering
    {\includegraphics[height=\textwidth,width=\textwidth]{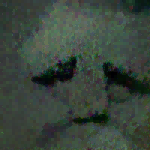}}
    \end{minipage} &
    \begin{minipage}[c]{0.1\textwidth} \centering
    {\includegraphics[height=\textwidth,width=\textwidth]{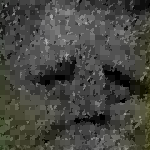}}
    \end{minipage} &
    \begin{minipage}[c]{0.1\textwidth} \centering
    {\includegraphics[height=\textwidth,width=\textwidth]{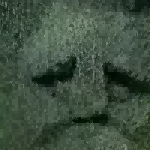}}
    \end{minipage} &
    \begin{minipage}[c]{0.1\textwidth} \centering
    {\includegraphics[height=\textwidth,width=\textwidth]{figs/101084,L1,denoise/101084_zoom.jpg}}
    \end{minipage}\\ \specialrule{0em}{1pt}{1pt}
     \begin{minipage}[c]{0.1\textwidth} \centering
    {\includegraphics[height=\textwidth,width=\textwidth]{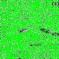}}
    \end{minipage} &
     \begin{minipage}[c]{0.1\textwidth} \centering
    {\includegraphics[height=\textwidth,width=\textwidth]{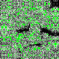}}
    \end{minipage} &
     \begin{minipage}[c]{0.1\textwidth} \centering
    {\includegraphics[height=\textwidth,width=\textwidth]{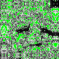}}
    \end{minipage} &
     \begin{minipage}[c]{0.1\textwidth} \centering
    {\includegraphics[height=\textwidth,width=\textwidth]{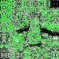}}
    \end{minipage} &
     \begin{minipage}[c]{0.1\textwidth} \centering
    {\includegraphics[height=\textwidth,width=\textwidth]{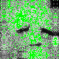}}
    \end{minipage} &
     \begin{minipage}[c]{0.1\textwidth} \centering
    {\includegraphics[height=\textwidth,width=\textwidth]{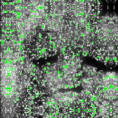}}
    \end{minipage}\\
    \end{tabular}\\
    \caption{Top to bottom: the corresponding results with noise level d = 0.2, 0.3, 0.4 respectively. The results include the noisy image (left large picture), the corresponding zoom-in parts of the noise image, the restored results by using CTV, HTV, NLTV, SVTV, SVS-NLTV, the ground-truth image respectively. The spatial distributions of S-CIELAB error (larger than 15 units) are also shown.}
    \label{101084L1zoom}
\end{figure}

\begin{figure}[htbp]
\centering
\begin{subfigure}
\centering
\includegraphics[width=0.23\textwidth]{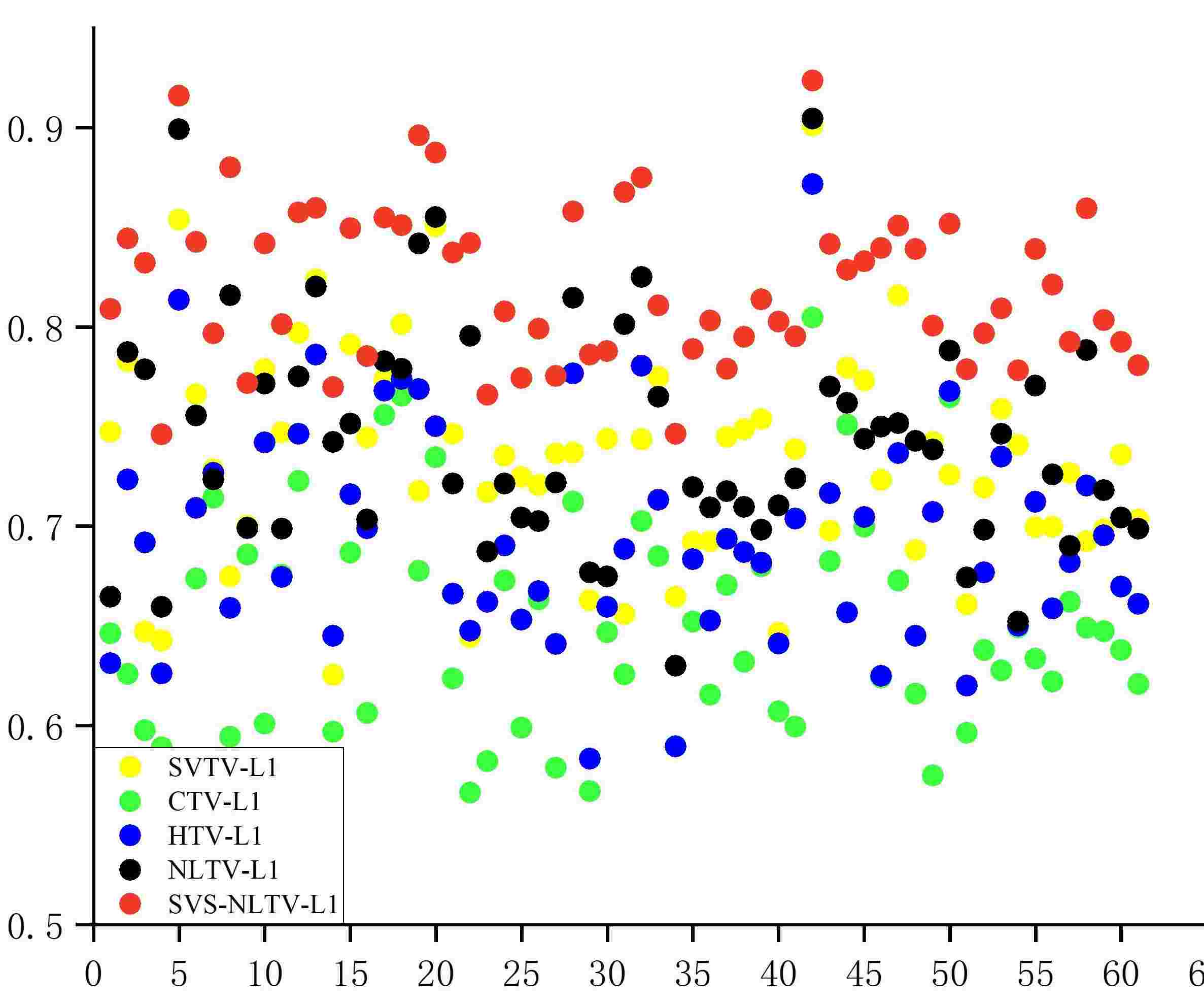}
\end{subfigure}
\begin{subfigure}
\centering
\includegraphics[width=0.23\textwidth]{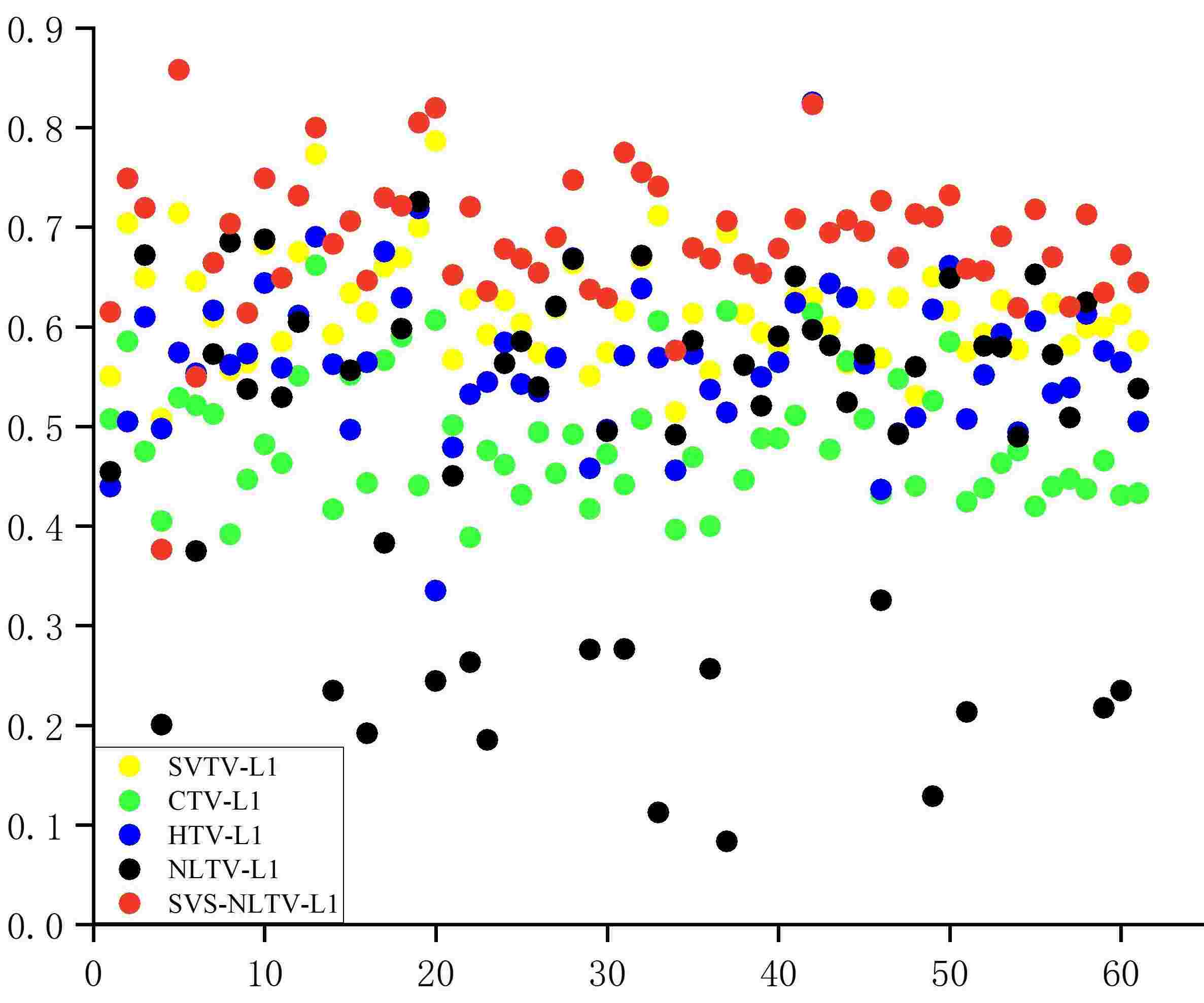}
\end{subfigure}
\begin{subfigure}
\centering
\includegraphics[width=0.23\textwidth]{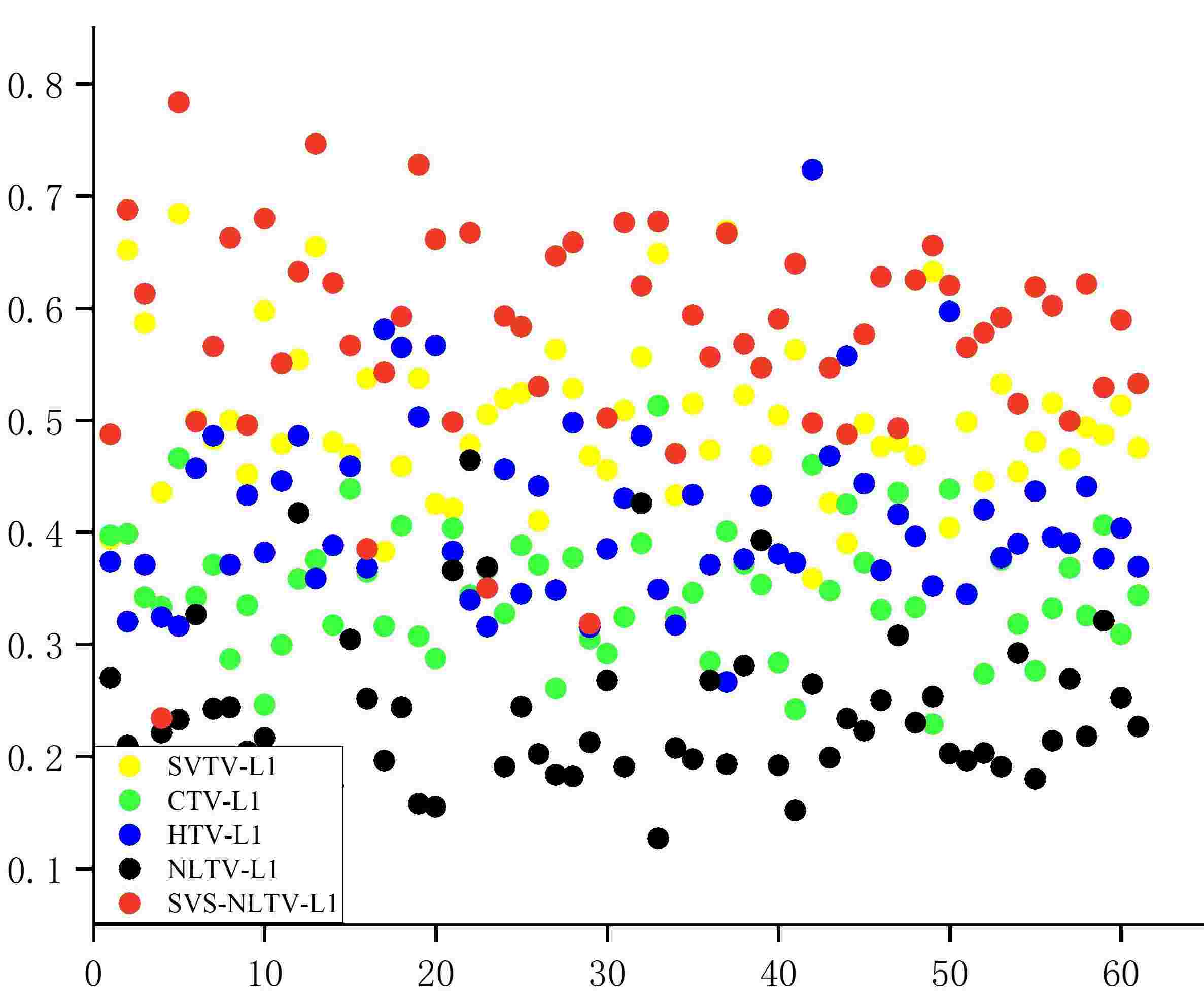}
\end{subfigure}
\begin{subfigure}
\centering
\includegraphics[width=0.23\textwidth]{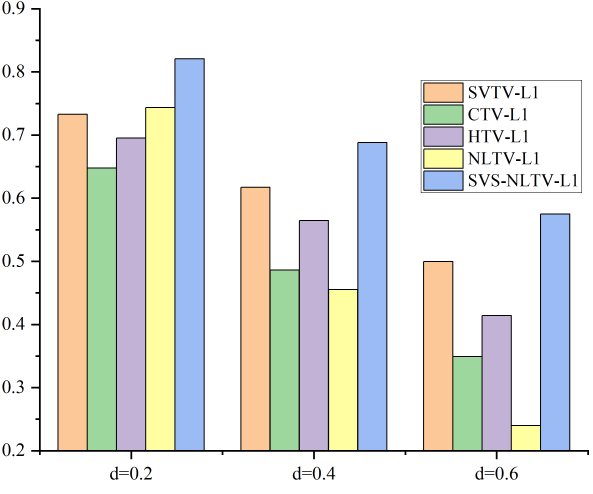}
\end{subfigure}
\caption{First to third: The spatial distributions of SSIM values of the restored results by using different methods corresponding to d = 0.2, 0.4, 0.6 respectively; Fourth: the histogram of the average SSIM values of the restored results by using different methods.}
\label{p-ssim1}
\end{figure}

\begin{figure}[htbp]
\centering
\tabcolsep=1pt
\begin{tabular}{ccccc}
    \begin{minipage}[c]{0.184\textwidth} \centering
   \subfigure{\includegraphics[width =\textwidth]{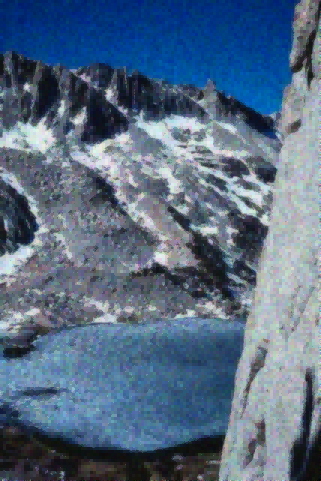}}
    \end{minipage} &
        \begin{minipage}[c]{0.184\textwidth} \centering
   \subfigure{\includegraphics[width =\textwidth]{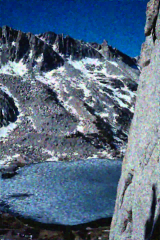}}
    \end{minipage} &
        \begin{minipage}[c]{0.184\textwidth} \centering
   \subfigure{\includegraphics[width =\textwidth]{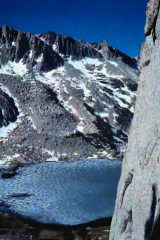}}
    \end{minipage} &
    \begin{minipage}[c]{0.184\textwidth} \centering
   \subfigure{\includegraphics[width = \textwidth]{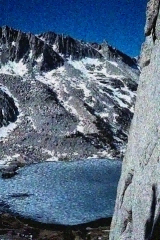}}
    \end{minipage} &
        \begin{minipage}[c]{0.184\textwidth} \centering
   \subfigure{\includegraphics[width = \textwidth]{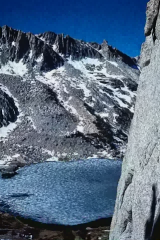}}
    \end{minipage} \\
       \begin{minipage}[c]{0.184\textwidth} \centering
   \subfigure{\includegraphics[width =\textwidth]{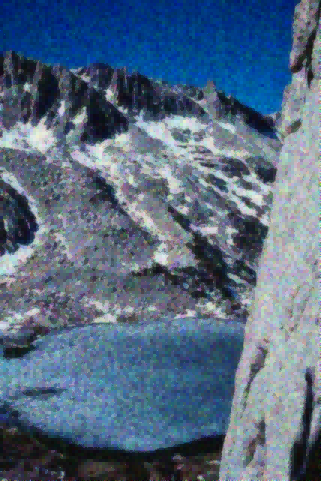}}
    \end{minipage} &
        \begin{minipage}[c]{0.184\textwidth} \centering
   \subfigure{\includegraphics[width =\textwidth]{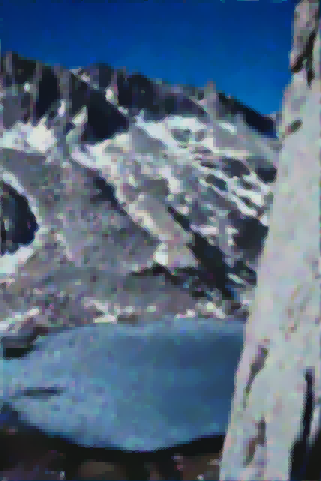}}
    \end{minipage} &
        \begin{minipage}[c]{0.184\textwidth} \centering
   \subfigure{\includegraphics[width =\textwidth]{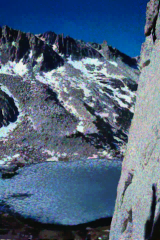}}
    \end{minipage} &
    \begin{minipage}[c]{0.184\textwidth} \centering
   \subfigure{\includegraphics[width = \textwidth]{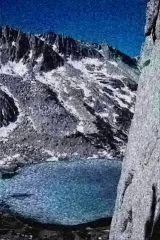}}
    \end{minipage} &
        \begin{minipage}[c]{0.184\textwidth} \centering
   \subfigure{\includegraphics[width = \textwidth]{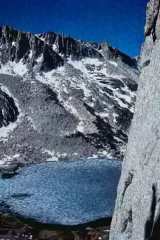}}
    \end{minipage} \\
        \begin{minipage}[c]{0.184\textwidth} \centering
\subfigure{\includegraphics[width =\textwidth]{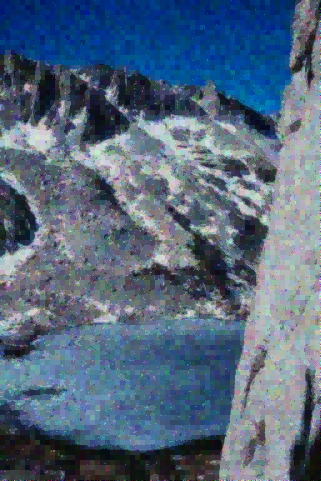}}
    \end{minipage} &
        \begin{minipage}[c]{0.184\textwidth} \centering
\subfigure{\includegraphics[width=\textwidth]{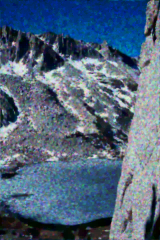}}
    \end{minipage} &
        \begin{minipage}[c]{0.184\textwidth} \centering
\subfigure{\includegraphics[width =\textwidth]{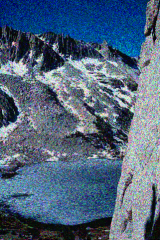}}
    \end{minipage} &
        \begin{minipage}[c]{0.184\textwidth} \centering
\subfigure{\includegraphics[width=\textwidth]{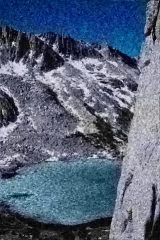}}
    \end{minipage} &
    \begin{minipage}[c]{0.184\textwidth} \centering
\subfigure{\includegraphics[width =\textwidth]{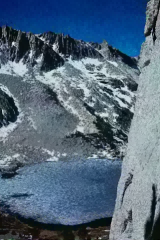}}
    \end{minipage} \\
\end{tabular}
\begin{subfigure}
\centering
\includegraphics[width=0.23\textwidth]{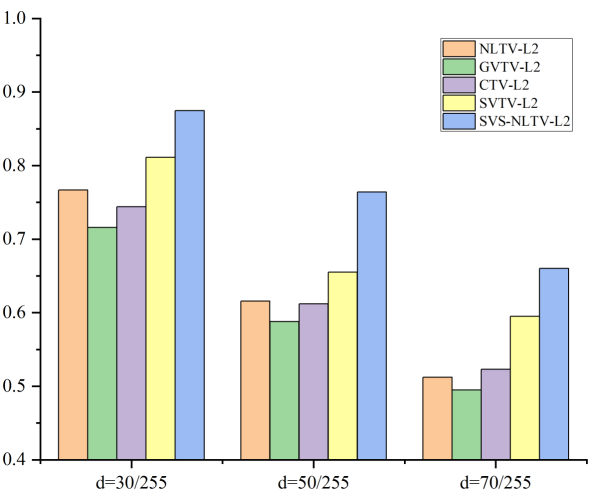}
\end{subfigure}
\begin{subfigure}
\centering
\includegraphics[width=0.23\textwidth]{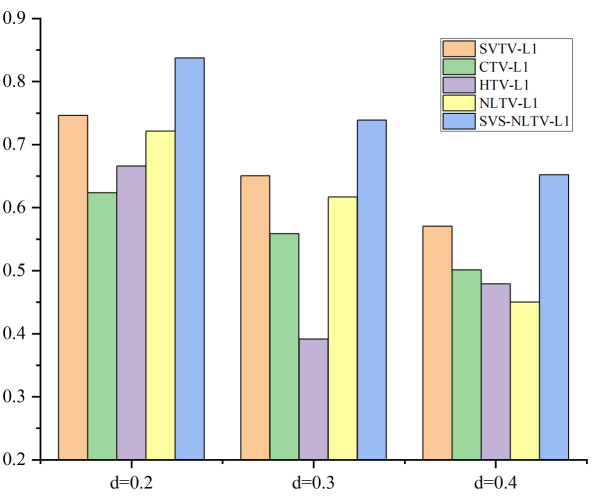}
\end{subfigure}
\begin{subfigure}
\centering
\includegraphics[width=0.23\textwidth]{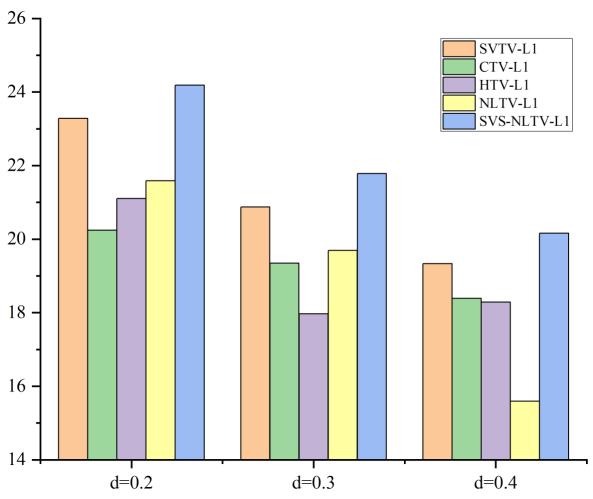}
\end{subfigure}
\begin{subfigure}
\centering
\includegraphics[width=0.23\textwidth]{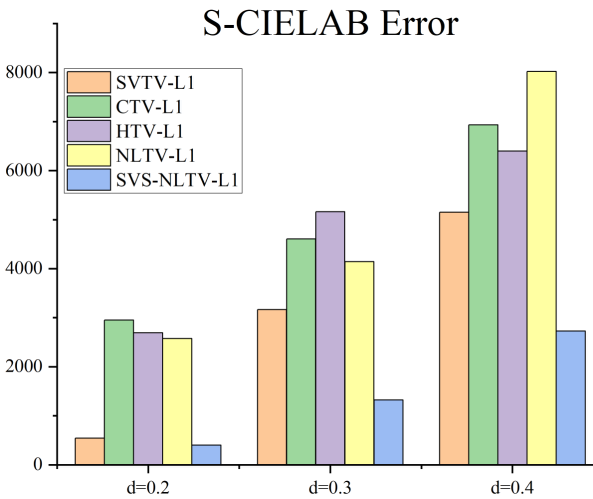}
\end{subfigure}
\caption{The first three rows: top to bottom: degraded and restored images with noise level d = 0.2, 0.3, 0.4 respectively; left to right: the restored results by using CTV, HTV, NLTV, SVTV, and SVS-NLTV respectively. The fourth row: the histograms of measure values by using different methods.}
\label{167083L1}
\end{figure}

\begin{figure}[htbp]
\centering
\begin{subfigure}
\centering
\includegraphics[width=0.23\textwidth]{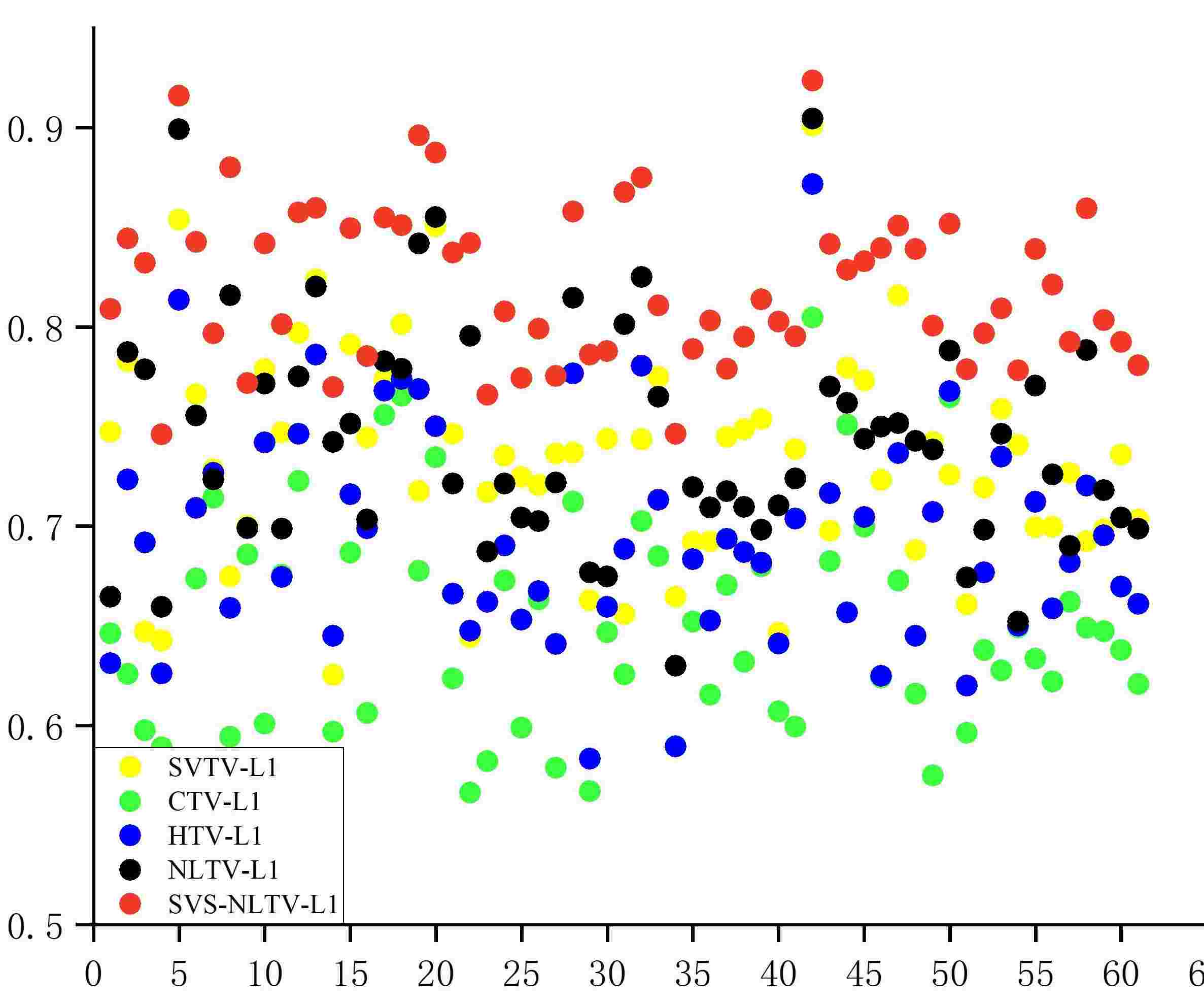}
\end{subfigure}
\begin{subfigure}
\centering
\includegraphics[width=0.23\textwidth]{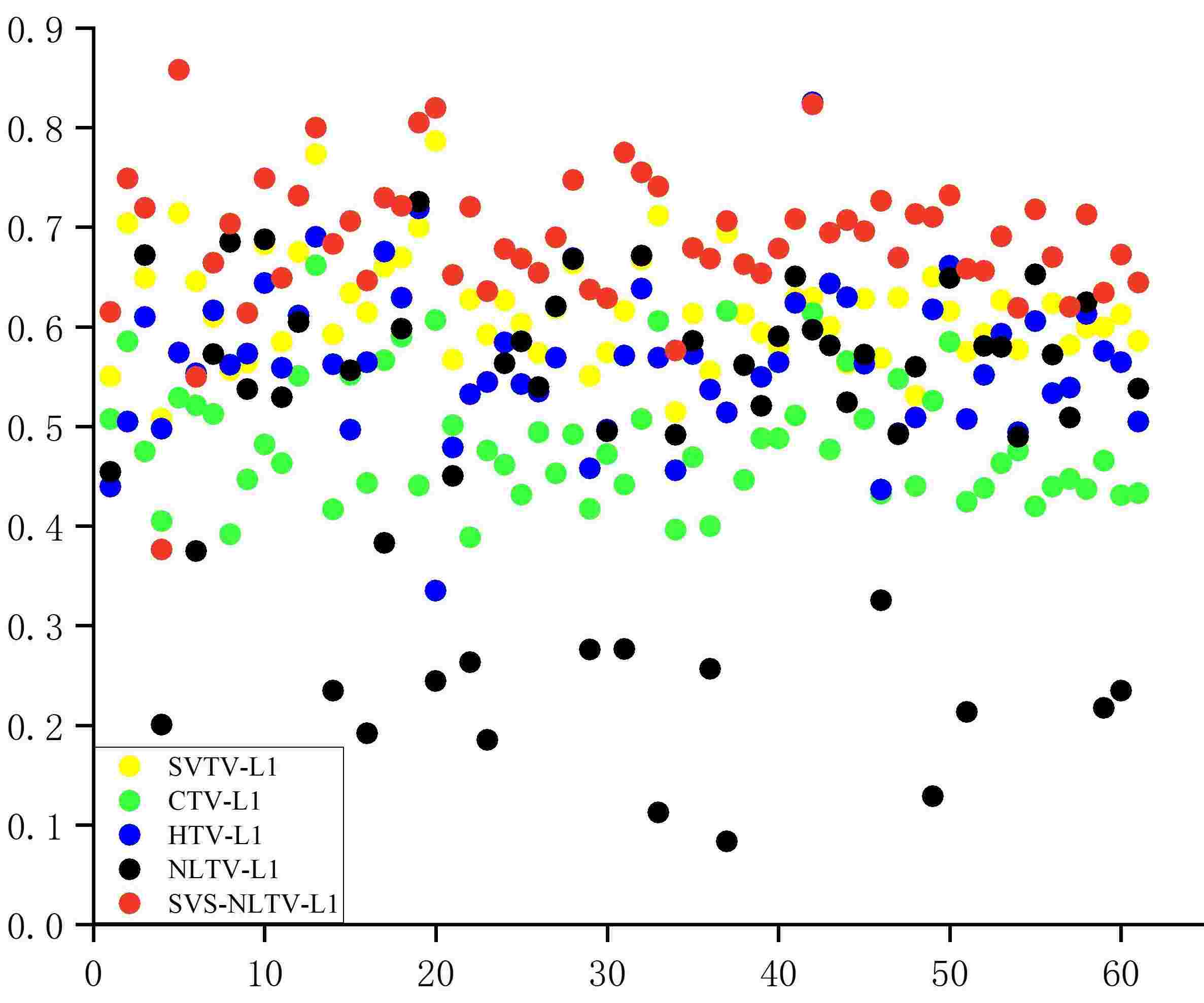}
\end{subfigure}
\begin{subfigure}
\centering
\includegraphics[width=0.23\textwidth]{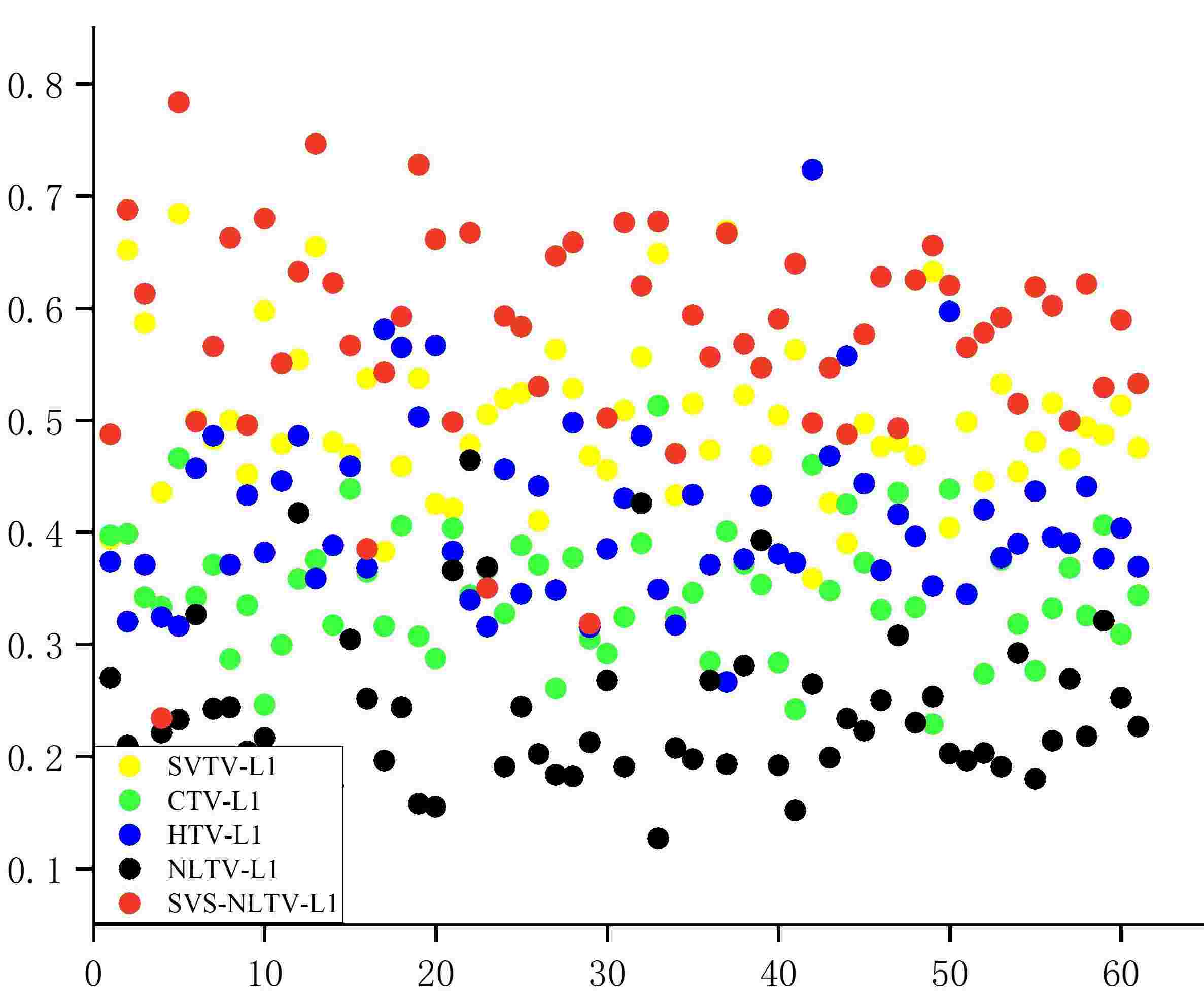}
\end{subfigure}
\begin{subfigure}
\centering
\includegraphics[width=0.23\textwidth]{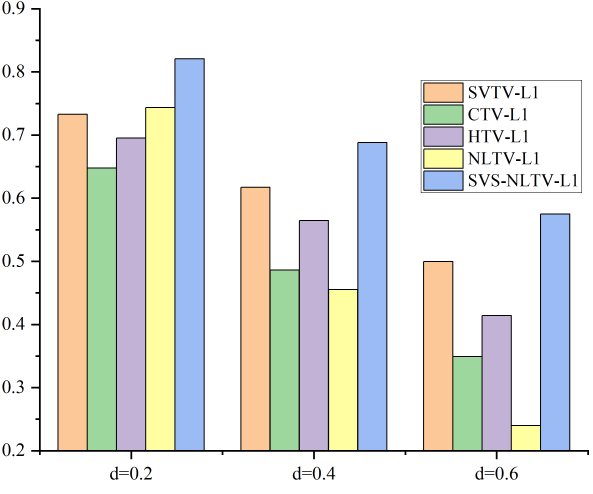}
\end{subfigure}
\caption{First to third: The spatial distributions of SSIM values of the restored results by using different methods corresponding to d = 0.2, 0.4, 0.6 respectively; Fourth: the histogram of the average SSIM values of the restored results by using different methods.}
\label{p-ssim}
\end{figure}

\begin{figure}[htbp]
\centering
\tabcolsep=1pt
\begin{minipage}[c]{0.21\textwidth} \centering
    {\includegraphics[height=\textwidth,width=\textwidth]{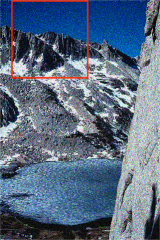}}
    \end{minipage}
\begin{tabular}{ccccccc}
    \begin{minipage}[c]{0.1\textwidth} \centering
    {\includegraphics[height=\textwidth,width=\textwidth]{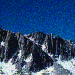}}
    \end{minipage} &
    \begin{minipage}[c]{0.1\textwidth} \centering
    {\includegraphics[height=\textwidth,width=\textwidth]{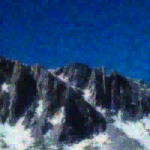}}
    \end{minipage} &
    \begin{minipage}[c]{0.1\textwidth} \centering
    {\includegraphics[height=\textwidth,width=\textwidth]{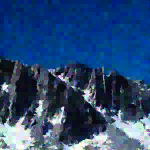}}
    \end{minipage} &
    \begin{minipage}[c]{0.1\textwidth} \centering
    {\includegraphics[height=\textwidth,width=\textwidth]{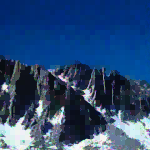}}
    \end{minipage} &
    \begin{minipage}[c]{0.1\textwidth} \centering
    {\includegraphics[height=\textwidth,width=\textwidth]{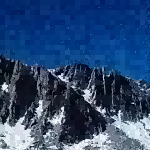}}
    \end{minipage} &
    \begin{minipage}[c]{0.1\textwidth} \centering
    {\includegraphics[height=\textwidth,width=\textwidth]{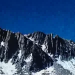}}
    \end{minipage} &
    \begin{minipage}[c]{0.1\textwidth} \centering
    {\includegraphics[height=\textwidth,width=\textwidth]{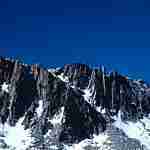}}
    \end{minipage}\\ \specialrule{0em}{1pt}{1pt}
    \begin{minipage}[c]{0.1\textwidth} \centering
    {\includegraphics[height=\textwidth,width=\textwidth]{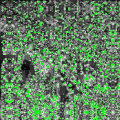}}
    \end{minipage} &
     \begin{minipage}[c]{0.1\textwidth} \centering
    {\includegraphics[height=\textwidth,width=\textwidth]{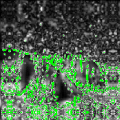}}
    \end{minipage} &
     \begin{minipage}[c]{0.1\textwidth} \centering
    {\includegraphics[height=\textwidth,width=\textwidth]{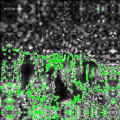}}
    \end{minipage} &
     \begin{minipage}[c]{0.1\textwidth} \centering
    {\includegraphics[height=\textwidth,width=\textwidth]{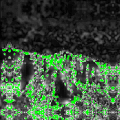}}
    \end{minipage} &
     \begin{minipage}[c]{0.1\textwidth} \centering
    {\includegraphics[height=\textwidth,width=\textwidth]{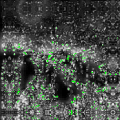}}
    \end{minipage} &
     \begin{minipage}[c]{0.1\textwidth} \centering
    {\includegraphics[height=\textwidth,width=\textwidth]{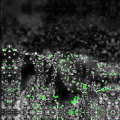}}
    \end{minipage} \\
    \end{tabular}
\begin{minipage}[c]{0.21\textwidth} \centering
    {\includegraphics[height=\textwidth,width=\textwidth]{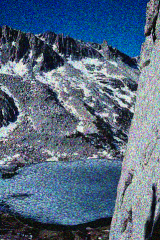}}
    \end{minipage}
\begin{tabular}{ccccccc}
    \begin{minipage}[c]{0.1\textwidth} \centering
    {\includegraphics[height=\textwidth,width=\textwidth]{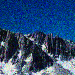}}
    \end{minipage} &
    \begin{minipage}[c]{0.1\textwidth} \centering
    {\includegraphics[height=\textwidth,width=\textwidth]{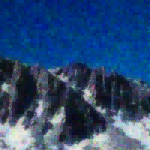}}
    \end{minipage} &
    \begin{minipage}[c]{0.1\textwidth} \centering
    {\includegraphics[height=\textwidth,width=\textwidth]{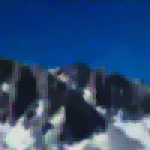}}
    \end{minipage} &
    \begin{minipage}[c]{0.1\textwidth} \centering
    {\includegraphics[height=\textwidth,width=\textwidth]{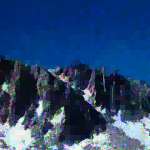}}
    \end{minipage} &
    \begin{minipage}[c]{0.1\textwidth} \centering
    {\includegraphics[height=\textwidth,width=\textwidth]{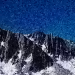}}
    \end{minipage} &
    \begin{minipage}[c]{0.1\textwidth} \centering
    {\includegraphics[height=\textwidth,width=\textwidth]{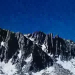}}
    \end{minipage} &
    \begin{minipage}[c]{0.1\textwidth} \centering
    {\includegraphics[height=\textwidth,width=\textwidth]{figs/167083,L1,denoise/167083_zoom.jpg}}
    \end{minipage}\\ \specialrule{0em}{1pt}{1pt}
    \begin{minipage}[c]{0.1\textwidth} \centering
    {\includegraphics[height=\textwidth,width=\textwidth]{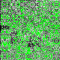}}
    \end{minipage} &
     \begin{minipage}[c]{0.1\textwidth} \centering
    {\includegraphics[height=\textwidth,width=\textwidth]{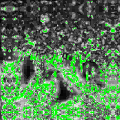}}
    \end{minipage} &
     \begin{minipage}[c]{0.1\textwidth} \centering
    {\includegraphics[height=\textwidth,width=\textwidth]{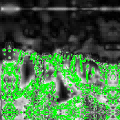}}
    \end{minipage} &
     \begin{minipage}[c]{0.1\textwidth} \centering
    {\includegraphics[height=\textwidth,width=\textwidth]{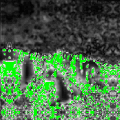}}
    \end{minipage} &
     \begin{minipage}[c]{0.1\textwidth} \centering
    {\includegraphics[height=\textwidth,width=\textwidth]{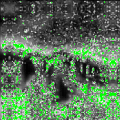}}
    \end{minipage} &
     \begin{minipage}[c]{0.1\textwidth} \centering
    {\includegraphics[height=\textwidth,width=\textwidth]{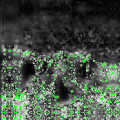}}
    \end{minipage} \\
    \end{tabular}
    \begin{minipage}[c]{0.21\textwidth} \centering
    {\includegraphics[height=\textwidth,width=\textwidth]{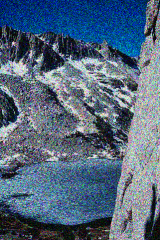}}
    \end{minipage}
\begin{tabular}{ccccccc}
     \begin{minipage}[c]{0.1\textwidth} \centering
    {\includegraphics[height=\textwidth,width=\textwidth]{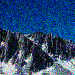}}
    \end{minipage} &
    \begin{minipage}[c]{0.1\textwidth} \centering
    {\includegraphics[height=\textwidth,width=\textwidth]{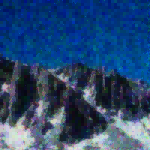}}
    \end{minipage} &
    \begin{minipage}[c]{0.1\textwidth} \centering
    {\includegraphics[height=\textwidth,width=\textwidth]{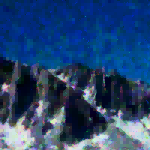}}
    \end{minipage} &
    \begin{minipage}[c]{0.1\textwidth} \centering
    {\includegraphics[height=\textwidth,width=\textwidth]{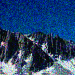}}
    \end{minipage} &
    \begin{minipage}[c]{0.1\textwidth} \centering
    {\includegraphics[height=\textwidth,width=\textwidth]{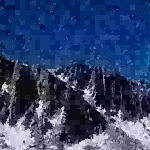}}
    \end{minipage} &
    \begin{minipage}[c]{0.1\textwidth} \centering
    {\includegraphics[height=\textwidth,width=\textwidth]{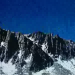}}
    \end{minipage} &
    \begin{minipage}[c]{0.1\textwidth} \centering
    {\includegraphics[height=\textwidth,width=\textwidth]{figs/167083,L1,denoise/167083_zoom.jpg}}
    \end{minipage}\\ \specialrule{0em}{1pt}{1pt}
     \begin{minipage}[c]{0.1\textwidth} \centering
    {\includegraphics[height=\textwidth,width=\textwidth]{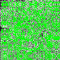}}
    \end{minipage} &
     \begin{minipage}[c]{0.1\textwidth} \centering
    {\includegraphics[height=\textwidth,width=\textwidth]{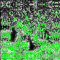}}
    \end{minipage} &
     \begin{minipage}[c]{0.1\textwidth} \centering
    {\includegraphics[height=\textwidth,width=\textwidth]{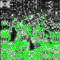}}
    \end{minipage} &
     \begin{minipage}[c]{0.1\textwidth} \centering
    {\includegraphics[height=\textwidth,width=\textwidth]{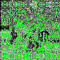}}
    \end{minipage} &
     \begin{minipage}[c]{0.1\textwidth} \centering
    {\includegraphics[height=\textwidth,width=\textwidth]{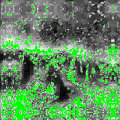}}
    \end{minipage} &
     \begin{minipage}[c]{0.1\textwidth} \centering
    {\includegraphics[height=\textwidth,width=\textwidth]{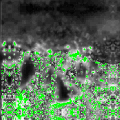}}
    \end{minipage}\\
    \end{tabular}\\
    \caption{Top to bottom: the corresponding results with noise level d = 0.2, 0.3, 0.4 respectively. The results include the noisy image (left large picture), the corresponding zoom-in parts of the noise image, the restored results by using CTV, HTV, NLTV, SVTV, SVS-NLTV, the ground-truth image respectively. The spatial distributions of S-CIELAB error (larger than 15 units) are also shown.}
    \label{167083L1zoom}
\end{figure}

\begin{figure}[htbp]
\centering
\begin{subfigure}
\centering
\includegraphics[width=0.23\textwidth]{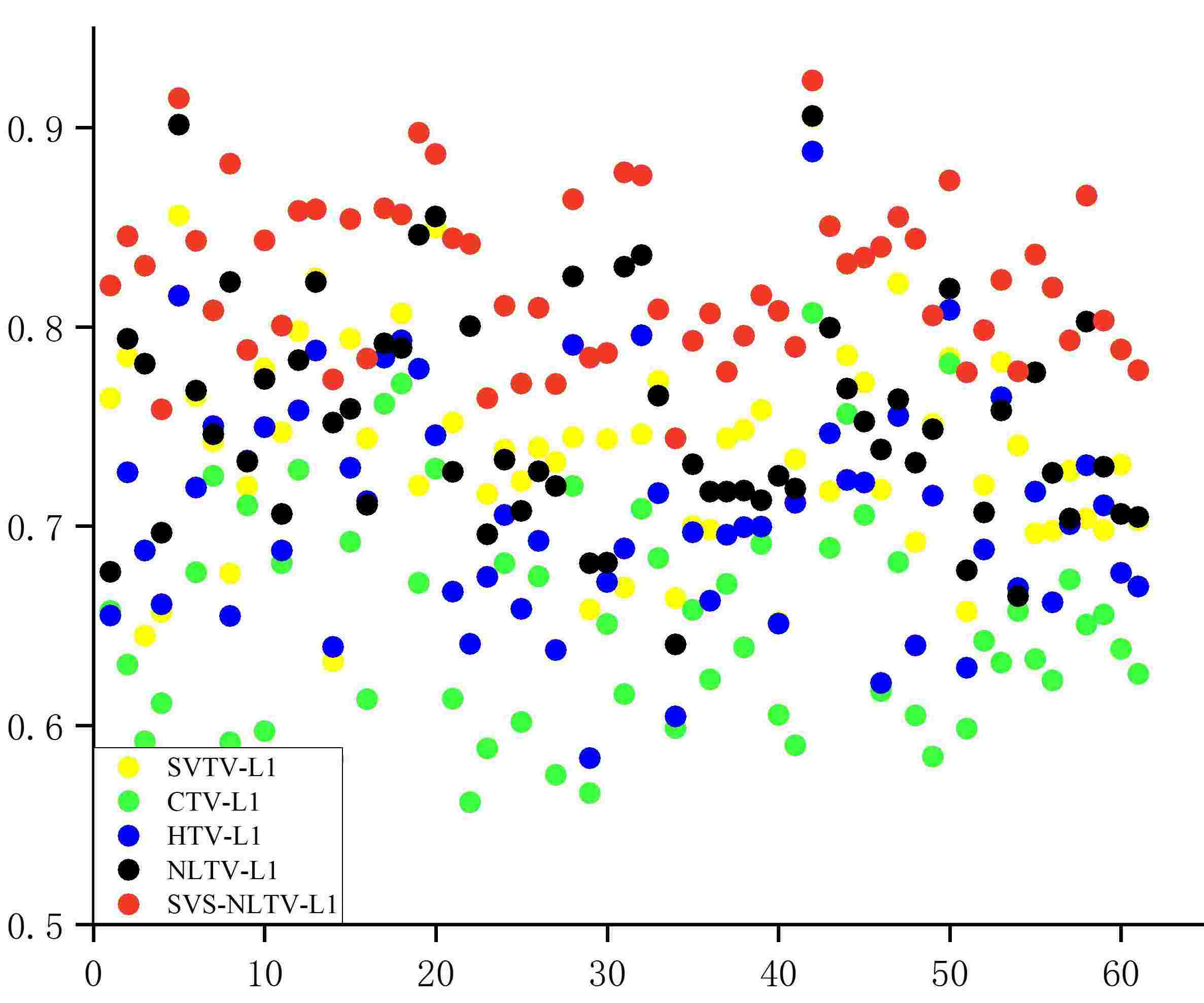}
\end{subfigure}
\begin{subfigure}
\centering
\includegraphics[width=0.23\textwidth]{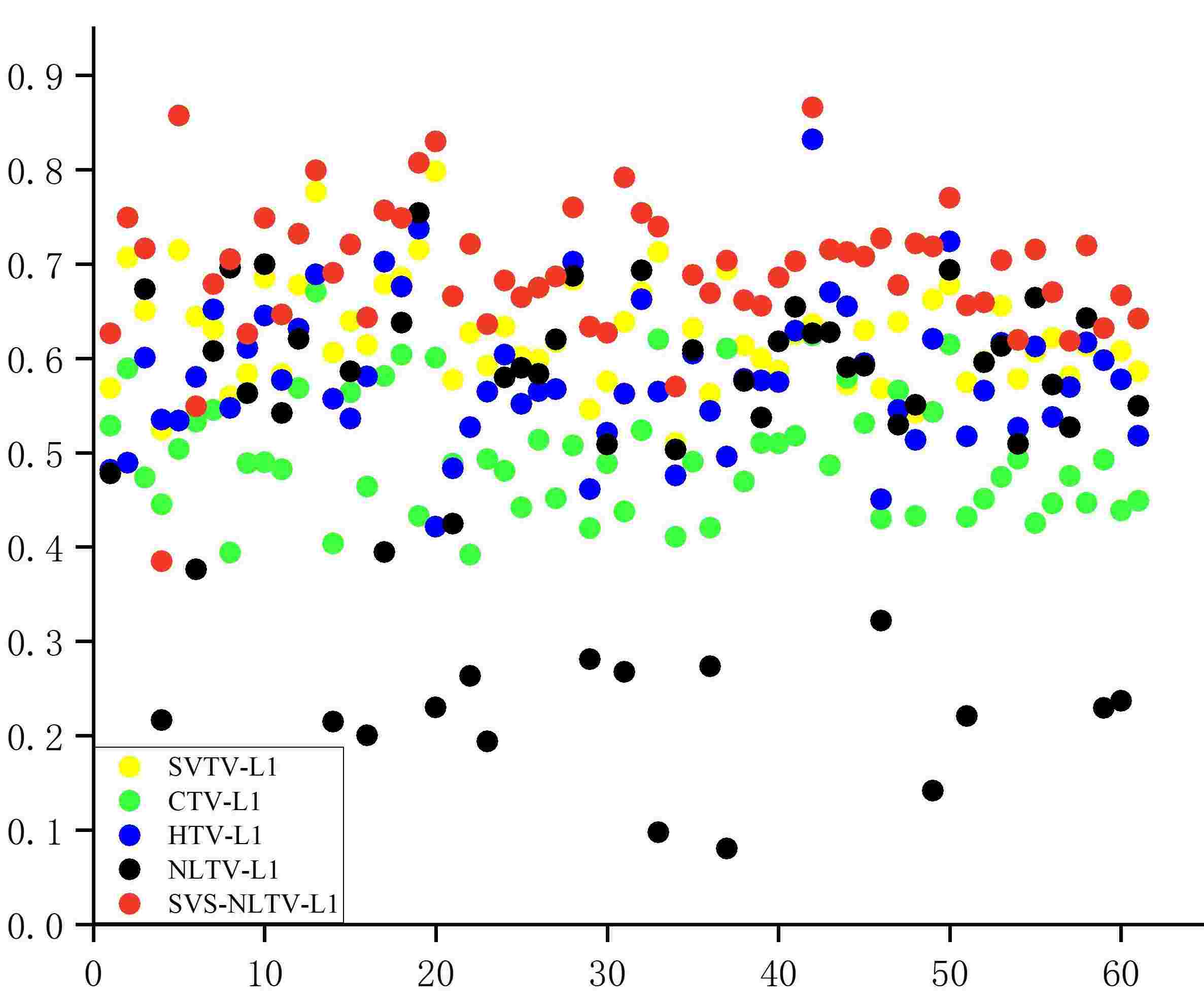}
\end{subfigure}
\begin{subfigure}
\centering
\includegraphics[width=0.23\textwidth]{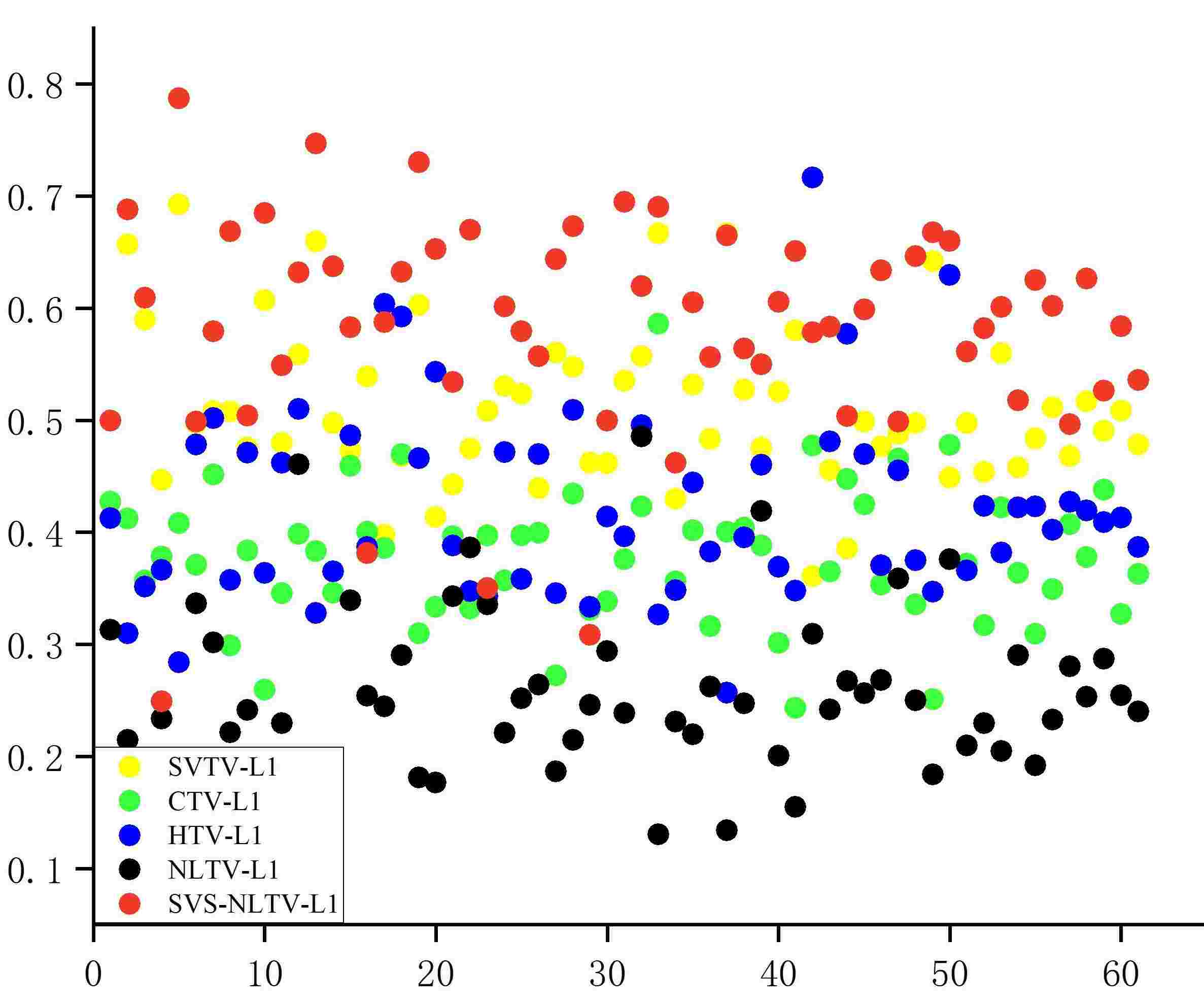}
\end{subfigure}
\begin{subfigure}
\centering
\includegraphics[width=0.23\textwidth]{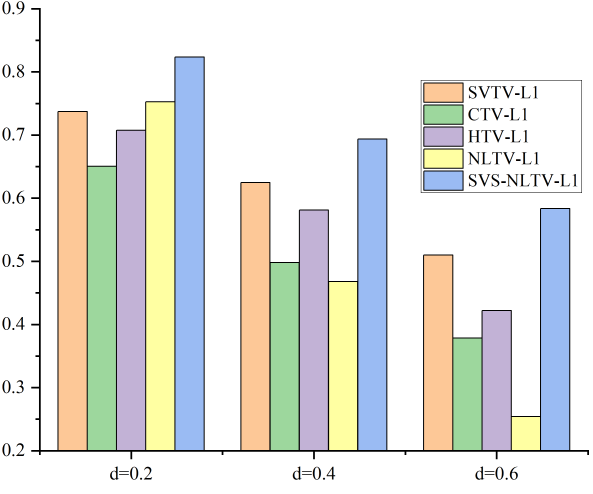}
\end{subfigure}
\caption{First to third: The spatial distributions of QSSIM values of the restored results by using different methods corresponding to d = 0.2, 0.4, 0.6 respectively; Fourth: the histogram of the average QSSIM values of the restored results by using different methods.}
\label{p-qssim}
\end{figure}

\begin{figure}[htbp]
\centering
\begin{subfigure}
\centering
\includegraphics[width=0.32\textwidth]{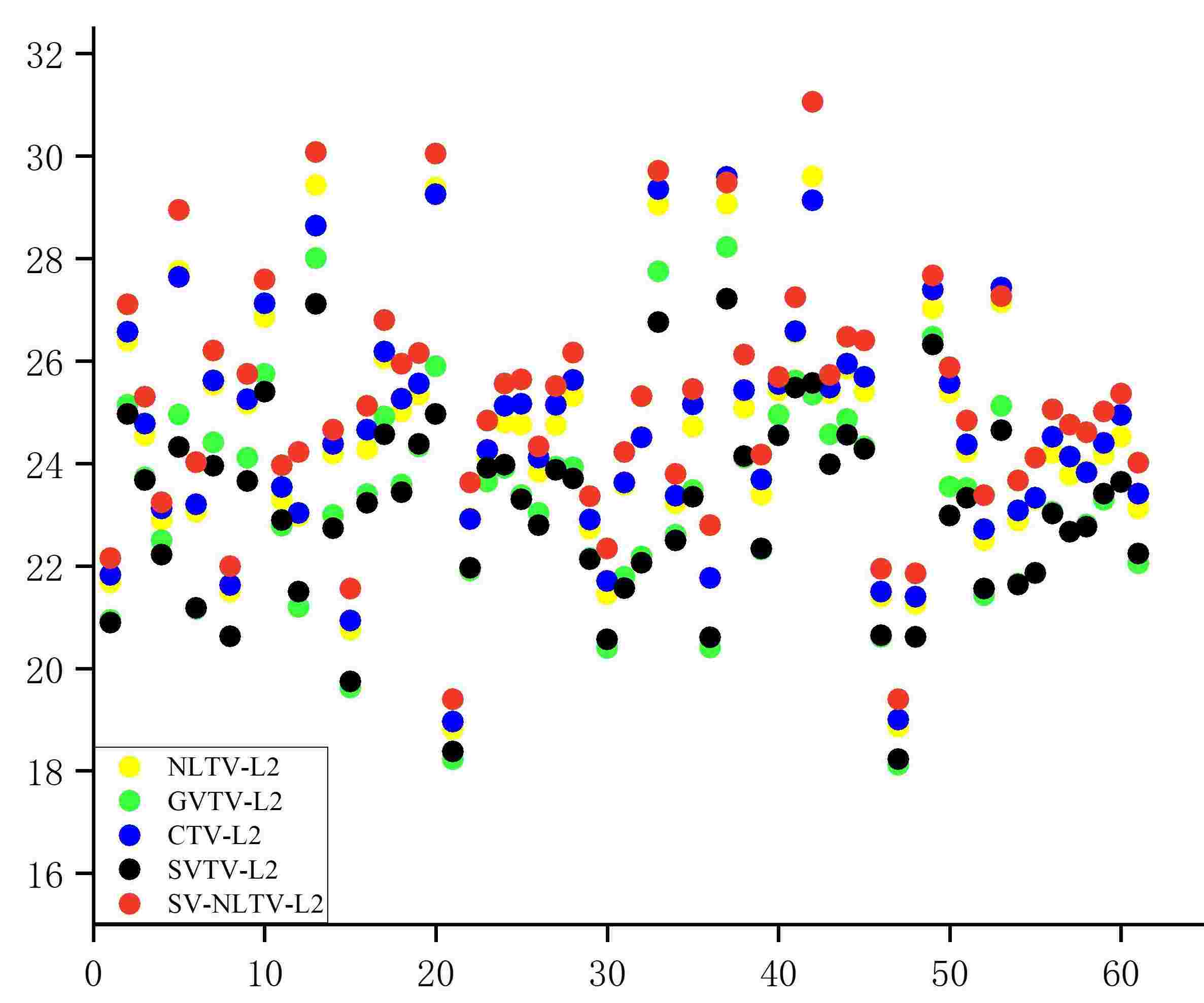}
\end{subfigure}
\begin{subfigure}
\centering
\includegraphics[width=0.32\textwidth]{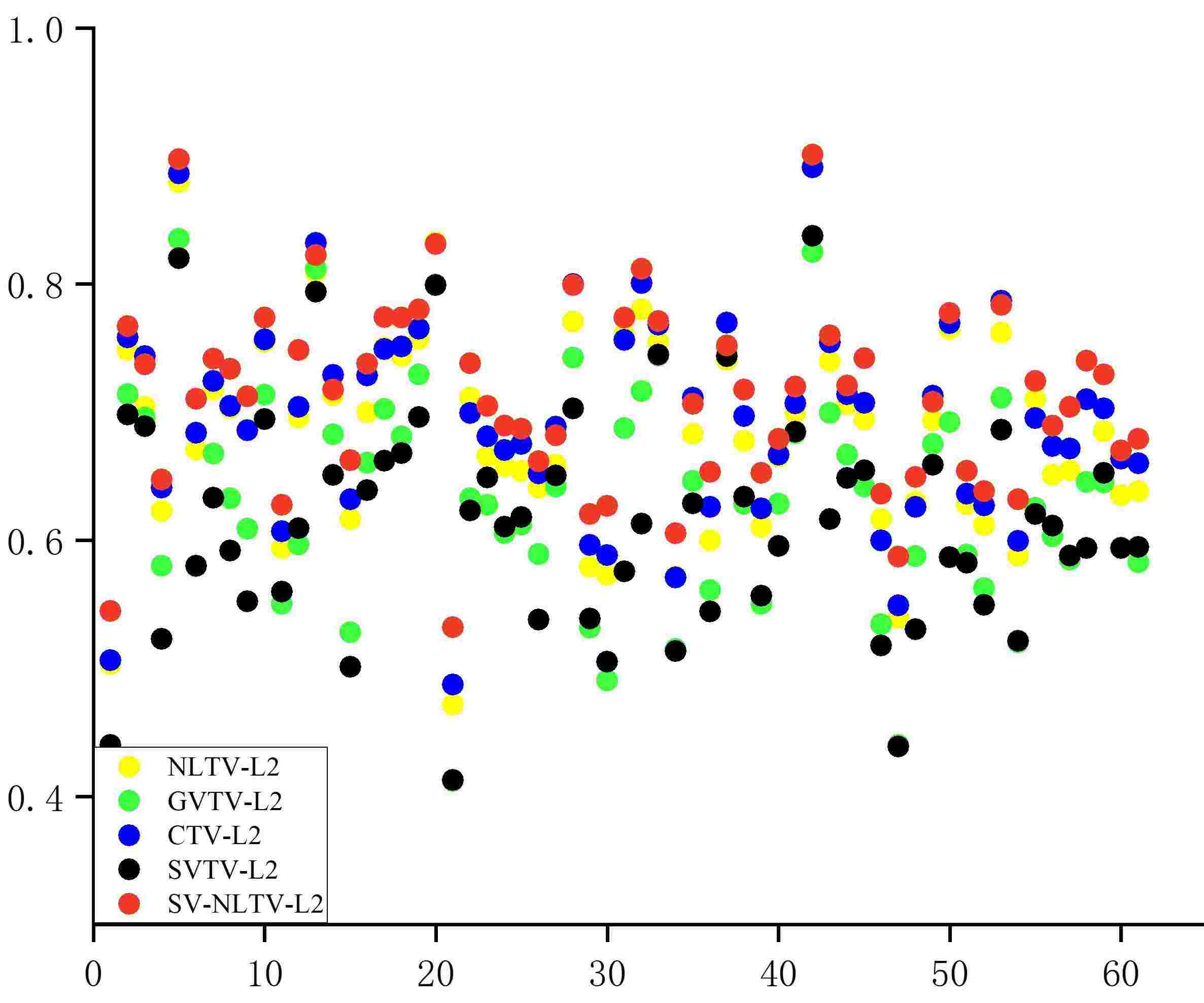}
\end{subfigure}
\begin{subfigure}
\centering
\includegraphics[width=0.32\textwidth]{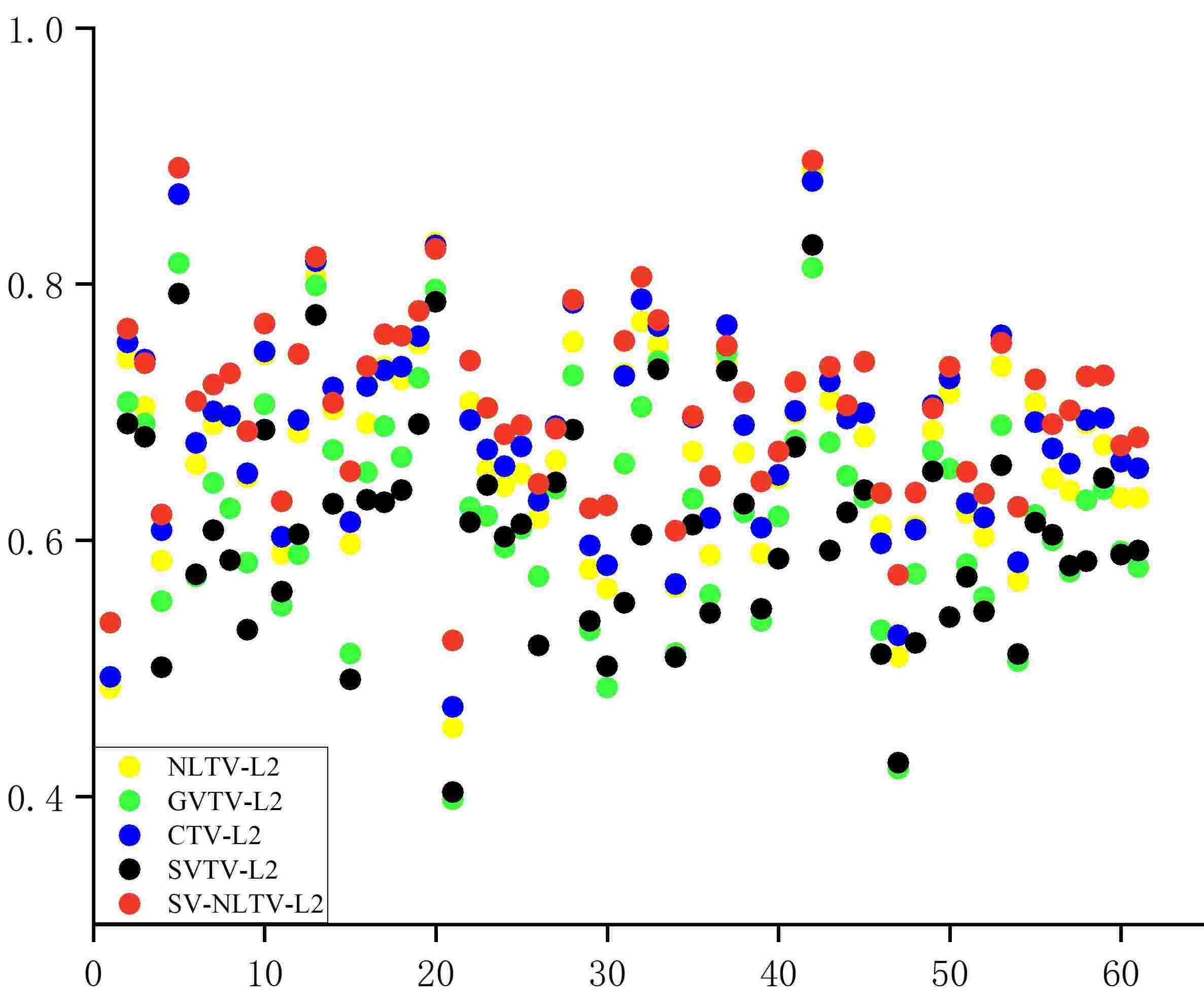}
\end{subfigure}
\caption{The spatial distributions of PSNR, QSSIM and SSIM values of 60 images with Gaussian blur and Gaussian noise with d = 20/255.}
\label{figspatialgaussianblur}
\end{figure}

\begin{figure}[htbp]
\centering
\tabcolsep=1pt
\begin{tabular}{ccccc}
    \begin{minipage}[c]{0.184\textwidth} \centering
   \subfigure{\includegraphics[width =\textwidth]{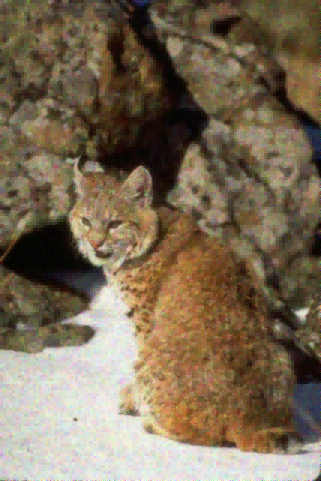}}
    \end{minipage} &
        \begin{minipage}[c]{0.184\textwidth} \centering
   \subfigure{\includegraphics[width =\textwidth]{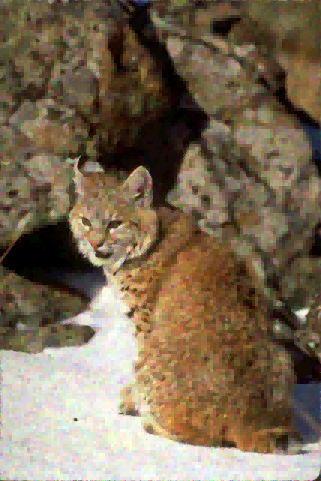}}
    \end{minipage} &
        \begin{minipage}[c]{0.184\textwidth} \centering
   \subfigure{\includegraphics[width =\textwidth]{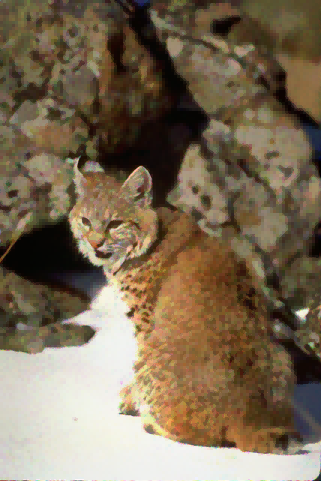}}
    \end{minipage} &
    \begin{minipage}[c]{0.184\textwidth} \centering
   \subfigure{\includegraphics[width = \textwidth]{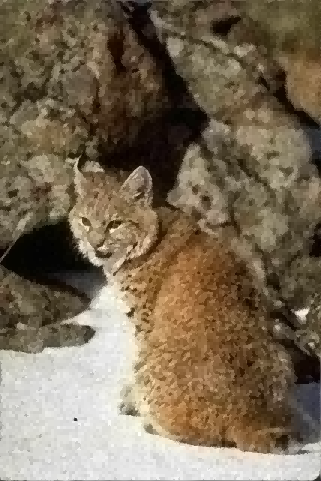}}
    \end{minipage} &
        \begin{minipage}[c]{0.184\textwidth} \centering
   \subfigure{\includegraphics[width = \textwidth]{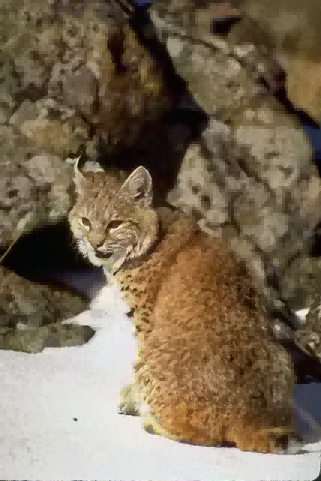}}
    \end{minipage} \\
       \begin{minipage}[c]{0.184\textwidth} \centering
   \subfigure{\includegraphics[width =\textwidth]{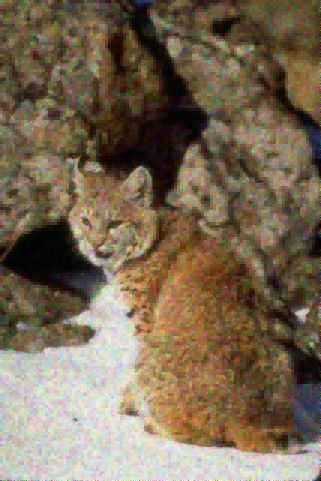}}
    \end{minipage} &
        \begin{minipage}[c]{0.184\textwidth} \centering
   \subfigure{\includegraphics[width =\textwidth]{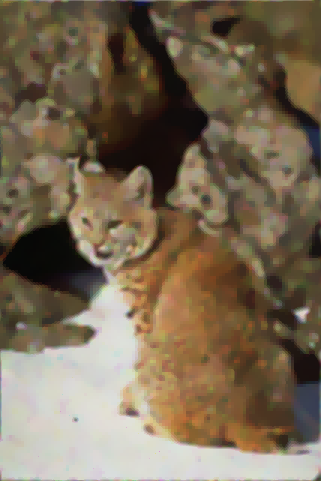}}
    \end{minipage} &
        \begin{minipage}[c]{0.184\textwidth} \centering
   \subfigure{\includegraphics[width =\textwidth]{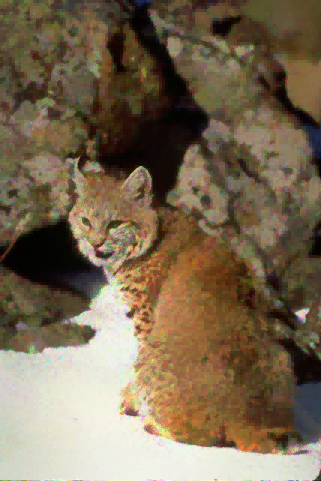}}
    \end{minipage} &
    \begin{minipage}[c]{0.184\textwidth} \centering
   \subfigure{\includegraphics[width = \textwidth]{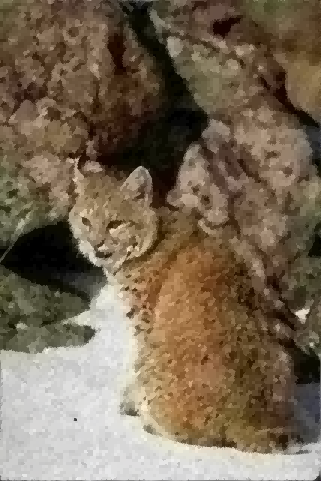}}
    \end{minipage} &
        \begin{minipage}[c]{0.184\textwidth} \centering
   \subfigure{\includegraphics[width = \textwidth]{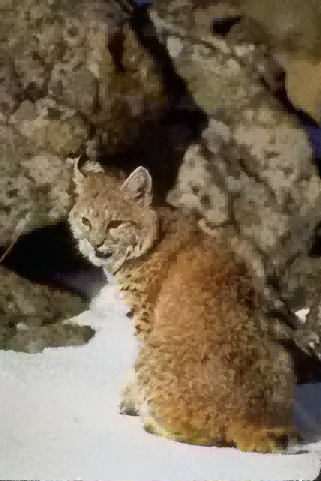}}
    \end{minipage} \\
    \begin{minipage}[c]{0.184\textwidth} \centering
   \subfigure{\includegraphics[width =\textwidth]{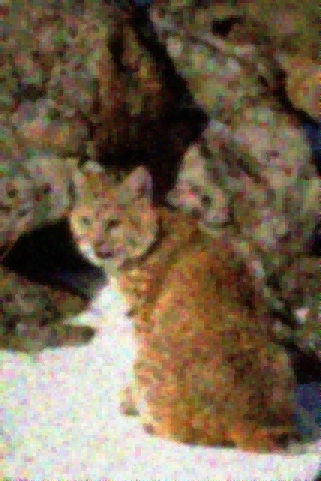}}
    \end{minipage} &
        \begin{minipage}[c]{0.18\textwidth} \centering
\subfigure{\includegraphics[width =\textwidth]{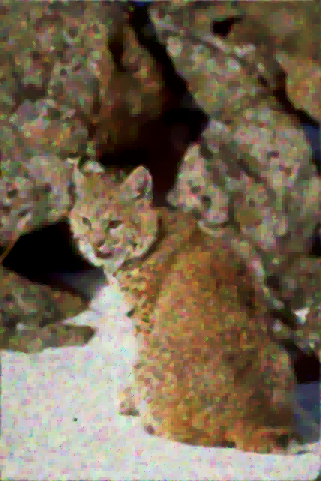}}
    \end{minipage} &
        \begin{minipage}[c]{0.184\textwidth} \centering
\subfigure{\includegraphics[width=\textwidth]{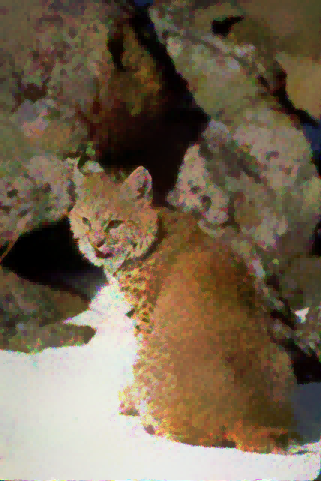}}
    \end{minipage} &
        \begin{minipage}[c]{0.184\textwidth} \centering
\subfigure{\includegraphics[width =\textwidth]{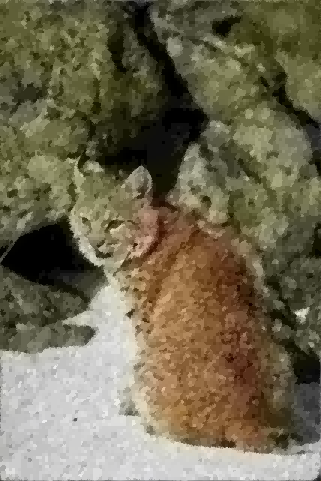}}
    \end{minipage} &
        \begin{minipage}[c]{0.184\textwidth} \centering
\subfigure{\includegraphics[width=\textwidth]{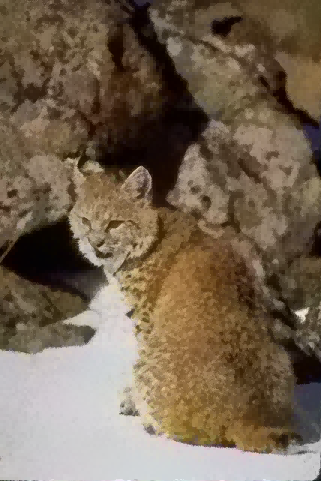}}
    \end{minipage} \\
\end{tabular}
\begin{subfigure}
\centering
\includegraphics[width=0.23\textwidth]{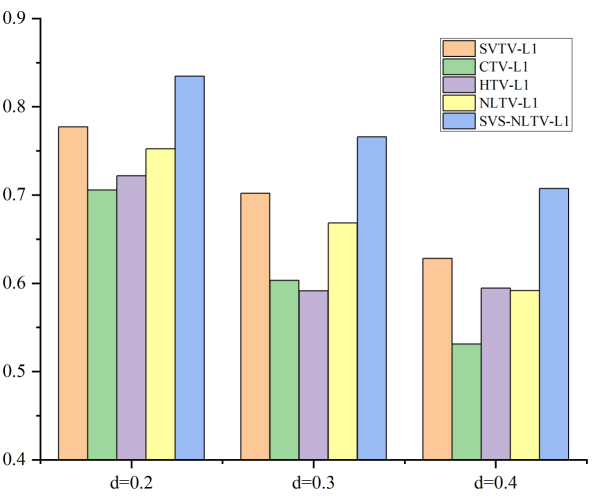}
\end{subfigure}
\begin{subfigure}
\centering
\includegraphics[width=0.23\textwidth]{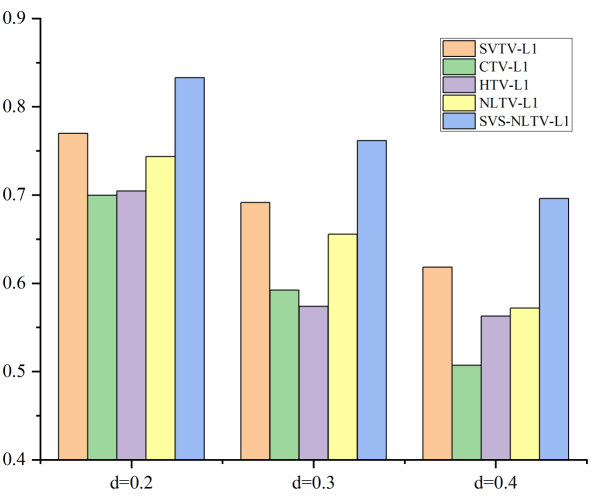}
\end{subfigure}
\begin{subfigure}
\centering
\includegraphics[width=0.23\textwidth]{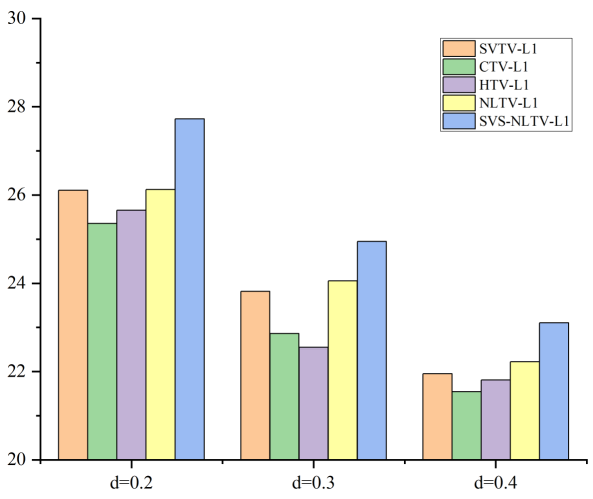}
\end{subfigure}
\begin{subfigure}
\centering
\includegraphics[width=0.23\textwidth]{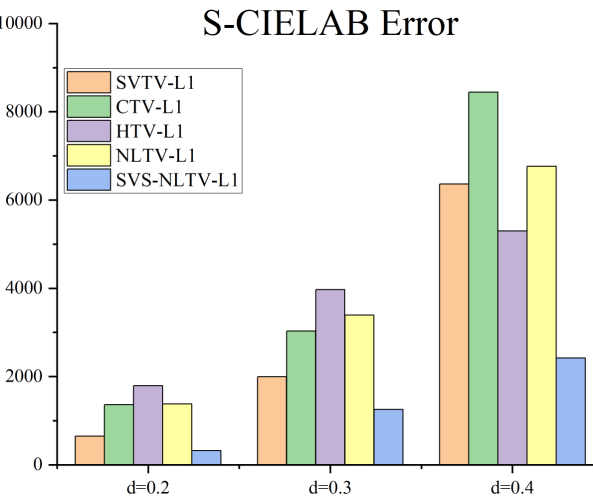}
\end{subfigure}
\caption{The first three rows: top to bottom: degraded and restored images with noise level d = 0.2, 0.3, 0.4 respectively; left to right: the restored results by using CTV, HTV, NLTV, SVTV, and SVS-NLTV respectively. The fourth row: the histograms of measure values by using different methods.}
\label{326084L1}
\end{figure}

\begin{figure}[htbp]
\centering
\begin{subfigure}
\centering
\includegraphics[width=0.32\textwidth]{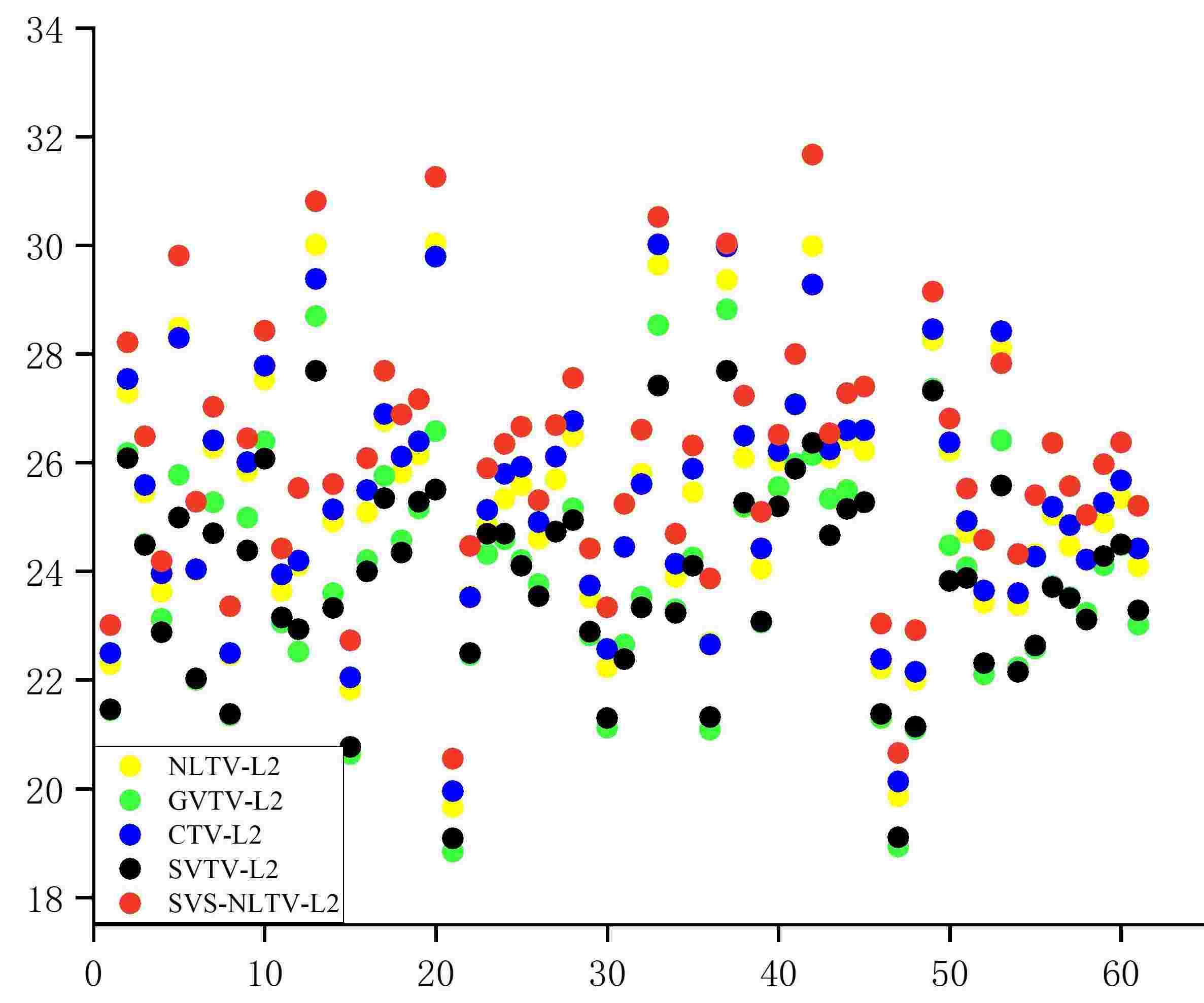}
\end{subfigure}
\begin{subfigure}
\centering
\includegraphics[width=0.32\textwidth]{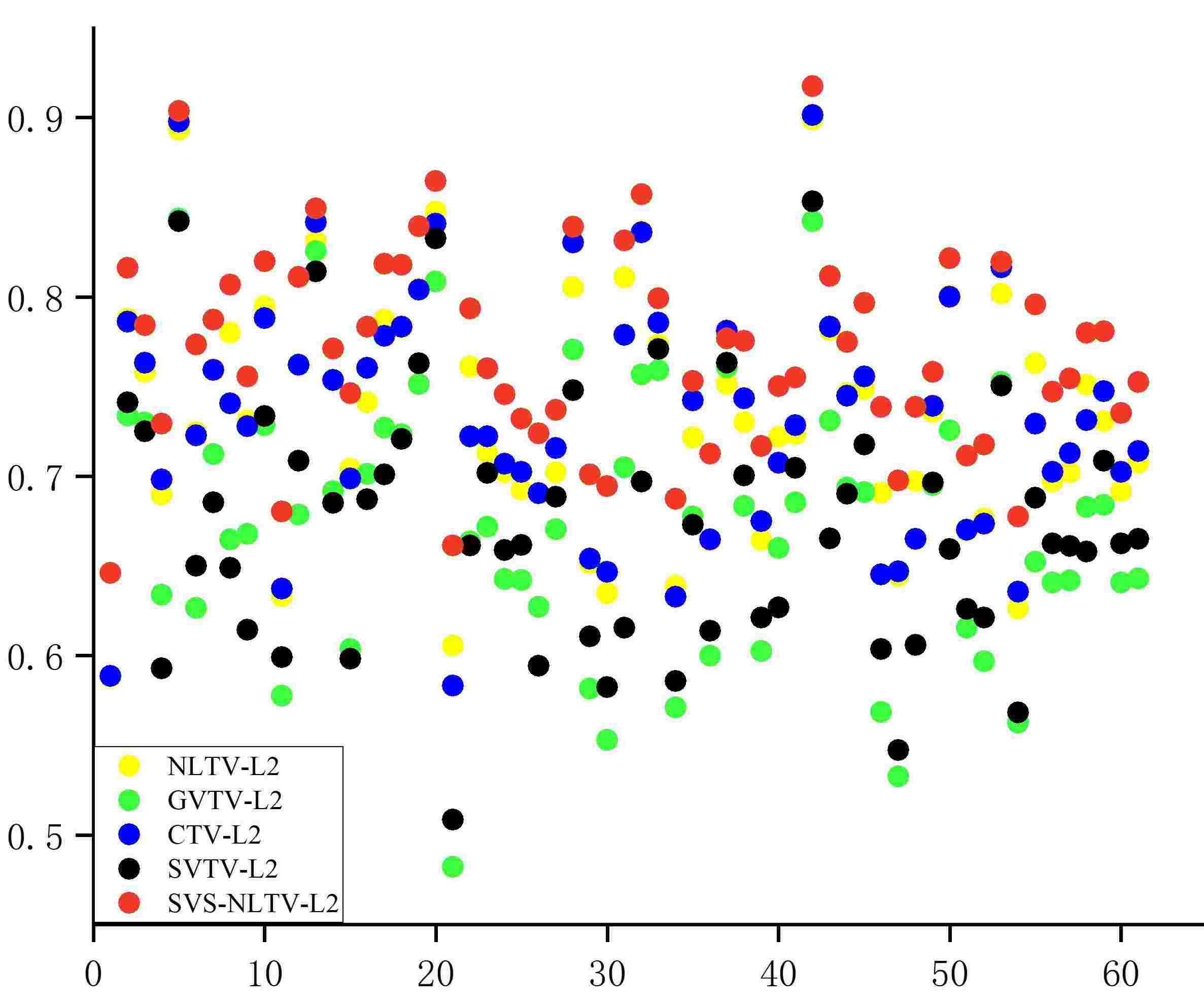}
\end{subfigure}
\begin{subfigure}
\centering
\includegraphics[width=0.32\textwidth]{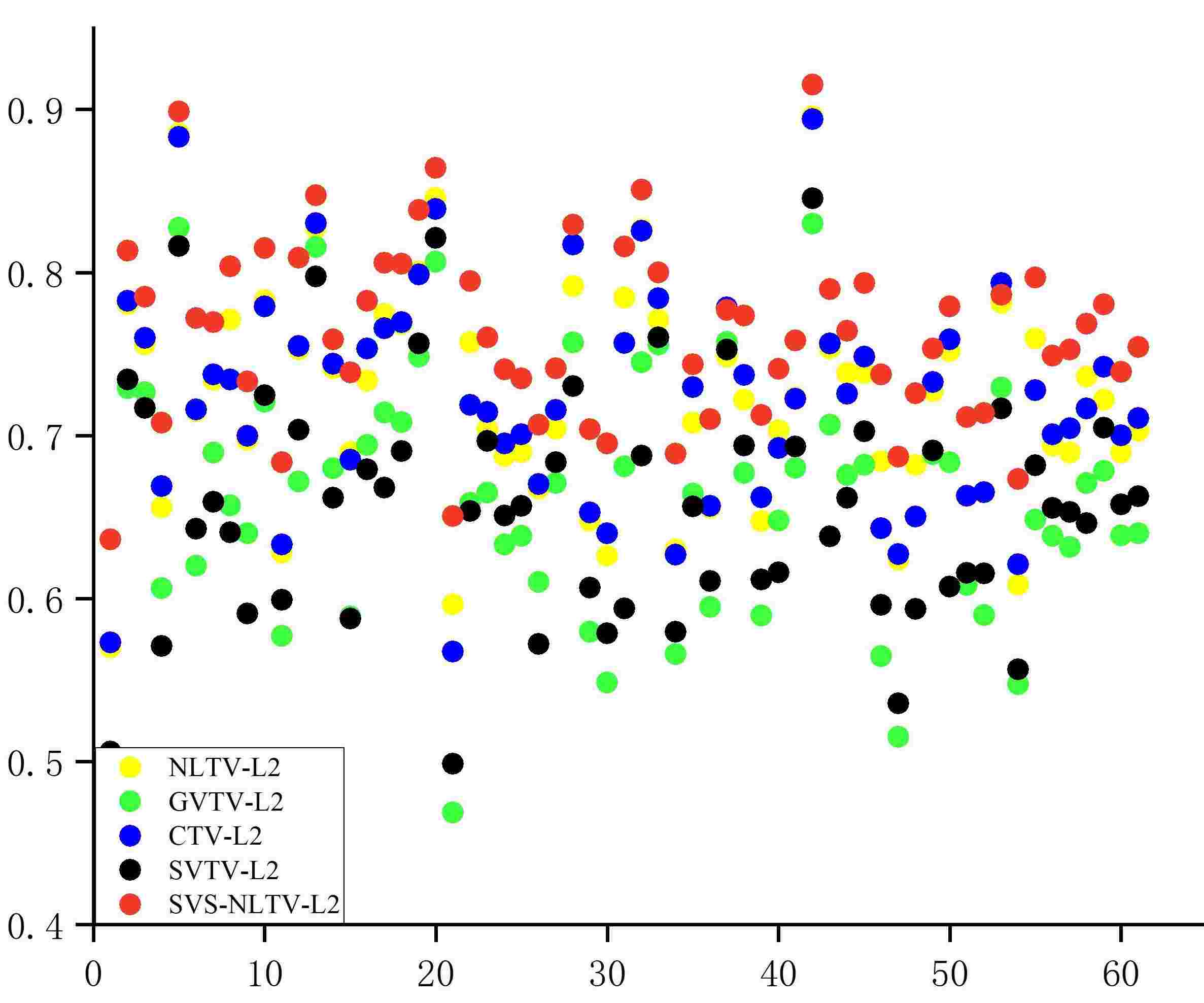}
\end{subfigure}
\caption{The spatial distributions of PSNR, QSSIM and SSIM values of 60 images with Motion blur and Gaussian noise with d = 20/255.}
\label{figspatialmotionblur}
\end{figure}

\begin{figure}[htbp]
\centering
\tabcolsep=1pt
\begin{minipage}[c]{0.21\textwidth} \centering
    {\includegraphics[height=\textwidth,width=\textwidth]{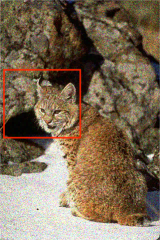}}
    \end{minipage}
\begin{tabular}{ccccccc}
    \begin{minipage}[c]{0.1\textwidth} \centering
    {\includegraphics[height=\textwidth,width=\textwidth]{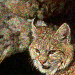}}
    \end{minipage} &
    \begin{minipage}[c]{0.1\textwidth} \centering
    {\includegraphics[height=\textwidth,width=\textwidth]{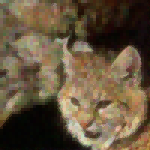}}
    \end{minipage} &
    \begin{minipage}[c]{0.1\textwidth} \centering
    {\includegraphics[height=\textwidth,width=\textwidth]{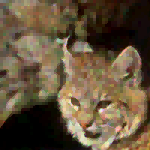}}
    \end{minipage} &
    \begin{minipage}[c]{0.1\textwidth} \centering
    {\includegraphics[height=\textwidth,width=\textwidth]{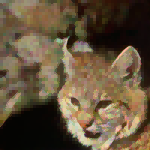}}
    \end{minipage} &
    \begin{minipage}[c]{0.1\textwidth} \centering
    {\includegraphics[height=\textwidth,width=\textwidth]{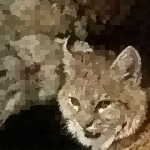}}
    \end{minipage} &
    \begin{minipage}[c]{0.1\textwidth} \centering
    {\includegraphics[height=\textwidth,width=\textwidth]{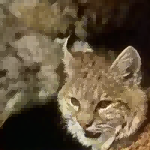}}
    \end{minipage} &
    \begin{minipage}[c]{0.1\textwidth} \centering
    {\includegraphics[height=\textwidth,width=\textwidth]{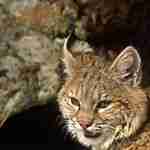}}
    \end{minipage}\\ \specialrule{0em}{1pt}{1pt}
    \begin{minipage}[c]{0.1\textwidth} \centering
    {\includegraphics[height=\textwidth,width=\textwidth]{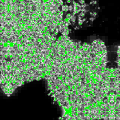}}
    \end{minipage} &
     \begin{minipage}[c]{0.1\textwidth} \centering
    {\includegraphics[height=\textwidth,width=\textwidth]{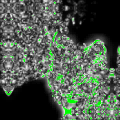}}
    \end{minipage} &
     \begin{minipage}[c]{0.1\textwidth} \centering
    {\includegraphics[height=\textwidth,width=\textwidth]{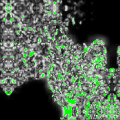}}
    \end{minipage} &
     \begin{minipage}[c]{0.1\textwidth} \centering
    {\includegraphics[height=\textwidth,width=\textwidth]{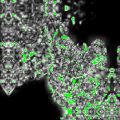}}
    \end{minipage} &
     \begin{minipage}[c]{0.1\textwidth} \centering
    {\includegraphics[height=\textwidth,width=\textwidth]{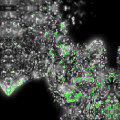}}
    \end{minipage} &
     \begin{minipage}[c]{0.1\textwidth} \centering
    {\includegraphics[height=\textwidth,width=\textwidth]{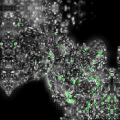}}
    \end{minipage} \\
    \end{tabular}
\begin{minipage}[c]{0.21\textwidth} \centering
    {\includegraphics[height=\textwidth,width=\textwidth]{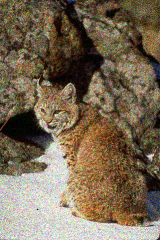}}
    \end{minipage}
\begin{tabular}{ccccccc}
    \begin{minipage}[c]{0.1\textwidth} \centering
    {\includegraphics[height=\textwidth,width=\textwidth]{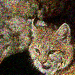}}
    \end{minipage} &
    \begin{minipage}[c]{0.1\textwidth} \centering
    {\includegraphics[height=\textwidth,width=\textwidth]{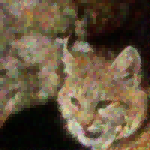}}
    \end{minipage} &
    \begin{minipage}[c]{0.1\textwidth} \centering
    {\includegraphics[height=\textwidth,width=\textwidth]{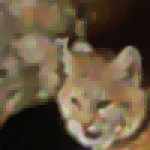}}
    \end{minipage} &
    \begin{minipage}[c]{0.1\textwidth} \centering
    {\includegraphics[height=\textwidth,width=\textwidth]{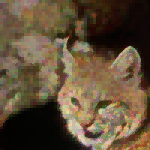}}
    \end{minipage} &
    \begin{minipage}[c]{0.1\textwidth} \centering
    {\includegraphics[height=\textwidth,width=\textwidth]{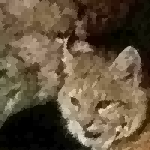}}
    \end{minipage} &
    \begin{minipage}[c]{0.1\textwidth} \centering
    {\includegraphics[height=\textwidth,width=\textwidth]{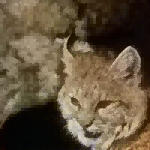}}
    \end{minipage} &
    \begin{minipage}[c]{0.1\textwidth} \centering
    {\includegraphics[height=\textwidth,width=\textwidth]{figs/326085,L1,denoise/326085_zoom.jpg}}
    \end{minipage}\\ \specialrule{0em}{1pt}{1pt}
    \begin{minipage}[c]{0.1\textwidth} \centering
    {\includegraphics[height=\textwidth,width=\textwidth]{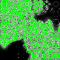}}
    \end{minipage} &
     \begin{minipage}[c]{0.1\textwidth} \centering
    {\includegraphics[height=\textwidth,width=\textwidth]{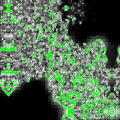}}
    \end{minipage} &
     \begin{minipage}[c]{0.1\textwidth} \centering
    {\includegraphics[height=\textwidth,width=\textwidth]{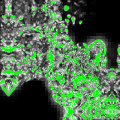}}
    \end{minipage} &
     \begin{minipage}[c]{0.1\textwidth} \centering
    {\includegraphics[height=\textwidth,width=\textwidth]{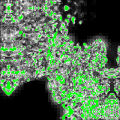}}
    \end{minipage} &
     \begin{minipage}[c]{0.1\textwidth} \centering
    {\includegraphics[height=\textwidth,width=\textwidth]{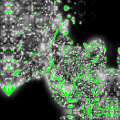}}
    \end{minipage} &
     \begin{minipage}[c]{0.1\textwidth} \centering
    {\includegraphics[height=\textwidth,width=\textwidth]{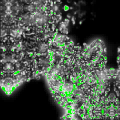}}
    \end{minipage} \\
    \end{tabular}
    \begin{minipage}[c]{0.21\textwidth} \centering
    {\includegraphics[height=\textwidth,width=\textwidth]{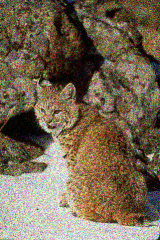}}
    \end{minipage}
\begin{tabular}{ccccccc}
     \begin{minipage}[c]{0.1\textwidth} \centering
    {\includegraphics[height=\textwidth,width=\textwidth]{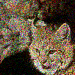}}
    \end{minipage} &
    \begin{minipage}[c]{0.1\textwidth} \centering
    {\includegraphics[height=\textwidth,width=\textwidth]{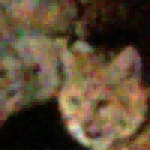}}
    \end{minipage} &
    \begin{minipage}[c]{0.1\textwidth} \centering
    {\includegraphics[height=\textwidth,width=\textwidth]{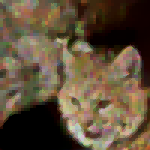}}
    \end{minipage} &
    \begin{minipage}[c]{0.1\textwidth} \centering
    {\includegraphics[height=\textwidth,width=\textwidth]{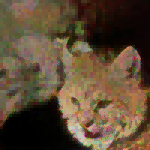}}
    \end{minipage} &
    \begin{minipage}[c]{0.1\textwidth} \centering
    {\includegraphics[height=\textwidth,width=\textwidth]{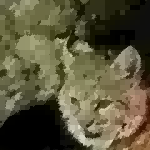}}
    \end{minipage} &
    \begin{minipage}[c]{0.1\textwidth} \centering
    {\includegraphics[height=\textwidth,width=\textwidth]{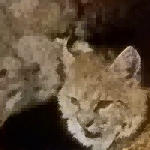}}
    \end{minipage} &
    \begin{minipage}[c]{0.1\textwidth} \centering
    {\includegraphics[height=\textwidth,width=\textwidth]{figs/326085,L1,denoise/326085_zoom.jpg}}
    \end{minipage}\\ \specialrule{0em}{1pt}{1pt}
     \begin{minipage}[c]{0.1\textwidth} \centering
    {\includegraphics[height=\textwidth,width=\textwidth]{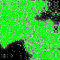}}
    \end{minipage} &
     \begin{minipage}[c]{0.1\textwidth} \centering
    {\includegraphics[height=\textwidth,width=\textwidth]{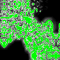}}
    \end{minipage} &
     \begin{minipage}[c]{0.1\textwidth} \centering
    {\includegraphics[height=\textwidth,width=\textwidth]{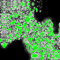}}
    \end{minipage} &
     \begin{minipage}[c]{0.1\textwidth} \centering
    {\includegraphics[height=\textwidth,width=\textwidth]{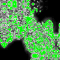}}
    \end{minipage} &
     \begin{minipage}[c]{0.1\textwidth} \centering
    {\includegraphics[height=\textwidth,width=\textwidth]{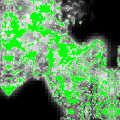}}
    \end{minipage} &
     \begin{minipage}[c]{0.1\textwidth} \centering
    {\includegraphics[height=\textwidth,width=\textwidth]{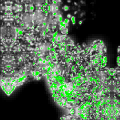}}
    \end{minipage}\\
    \end{tabular}
    \caption{Top to bottom: the corresponding results with noise level d = 0.2, 0.3, 0.4 respectively. The results include the noisy image (left large picture), the corresponding zoom-in parts of the noise image, the restored results by using CTV, HTV, NLTV, SVTV, SVS-NLTV, the ground-truth image respectively. The spatial distributions of S-CIELAB error (larger than 15 units) are also shown.}
    \label{326085L1zoom}
\end{figure}

\begin{table}[htbp]
\centering
\caption{Measure values of the restored results (Motion blur) by using different methods}
\begin{tabular}{|c|c|c|c|c|c|c|}
\hline
\multicolumn{1}{|l|}{} & \multicolumn{1}{l|}{} & CTV & GVTV  & NLTV & SV-TV & Proposed\\ \hline
\multirow{4}{*}{Fig \ref{figmotionblur} (a)}  & QSSIM  & 0.66495 & 0.6063 & 0.69742 & 0.6058 & \textbf{0.73872}\\ \cline{2-7} 
& SSIM     & 0.65032 & 0.59372 & 0.68193 & 0.59372 & \textbf{0.72605}\\ \cline{2-7} 
& PSNR  & 22.1474 & 21.0914  & 21.9889 & 21.1379 & \textbf{22.9145} \\ \cline{2-7} 
& S-CIELAB   & 9259  & 13567  & 10447  & 15106  & \textbf{7365}  \\ \hline
\multirow{4}{*}{Fig \ref{figmotionblur} (b)}  & QSSIM    & 0.58332 & 0.48217  & 0.60557 &0.50845 & \textbf{0.66137} \\ \cline{2-7} 
& SSIM    & 0.56739  & 0.46872  & 0.59644 &0.49868 & \textbf{0.65047}  \\ \cline{2-7} 
& PSNR   & 19.9519 & 18.8467  &19.6638 & 19.0823 & \textbf{20.5542}       \\ \cline{2-7} 
& S-CIELAB    & 15729  & 20770     & 18613  & 18045 & \textbf{10109}  \\ \hline
\multirow{4}{*}{Fig \ref{figmotionblur} (c)}  & QSSIM  & 0.58862 & 494  & 58819 & 0.51897 & \textbf{0.64619}   \\ \cline{2-7} 
& SSIM  & 0.57229 & 0.48287  & 0.56983 & 50617 & \textbf{0.6363}   \\ \cline{2-7} 
& PSNR   & 22.4961 & 21.4421  & 22.2852 &  21.4578 & \textbf{23.0126}     \\ \cline{2-7} 
& S-CIELAB   & 12453  & 18640  & 14127  & 19495   & \textbf{10847}  \\ \hline
\multirow{4}{*}{\thead{\scriptsize Average of\\ 60 testing images}} & QSSIM    & 0.73279 & 0.67031   & 0.73297 & 0.6715         & \textbf{0.76944} \\ \cline{2-7} 
& SSIM    & 0.72219 & 0.6598 & 0.7214 & 0.6588  & \textbf{0.7634} \\ \cline{2-7} 
& PSNR    & 25.3757 & 24.0101 & 25.1935 & 23.8565 & \textbf{26.1201}      \\ \cline{2-7} 
& S-CIELAB & 60409  & 52501           & 61159  & 35116   & \textbf{32239}  \\ \hline
\end{tabular}
\label{table5.2}
\end{table}

\subsection{Image deblurring: Gaussian blur and Motion blur}

\begin{figure}[htbp]
\centering
\tabcolsep=1pt
\begin{tabular}{cccccc}
    \begin{minipage}[c]{0.16\textwidth} \centering
   \subfigure{\includegraphics[width =\textwidth]{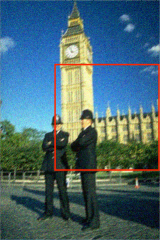}}
    \end{minipage} &
        \begin{minipage}[c]{0.16\textwidth} \centering
   \subfigure{\includegraphics[width =\textwidth]{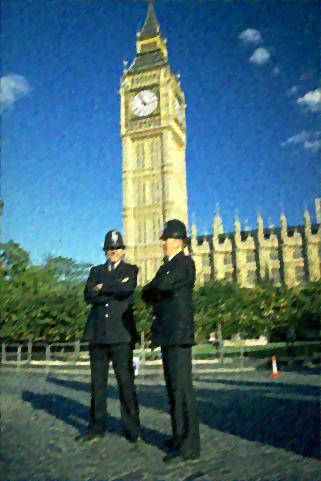}}
    \end{minipage} &
        \begin{minipage}[c]{0.16\textwidth} \centering
   \subfigure{\includegraphics[width =\textwidth]{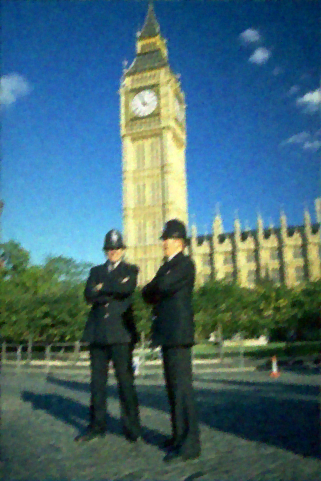}}
   \end{minipage}  &
    \begin{minipage}[c]{0.16\textwidth} \centering
   \subfigure{\includegraphics[width = \textwidth]{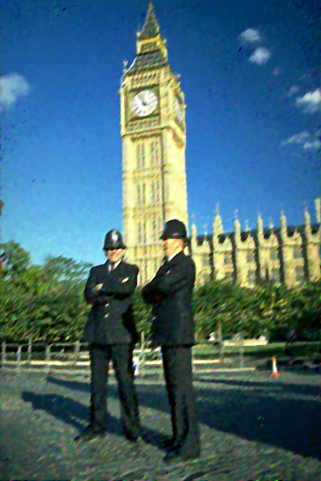}}
    \end{minipage} &
        \begin{minipage}[c]{0.16\textwidth} \centering
   \subfigure{\includegraphics[width = \textwidth]{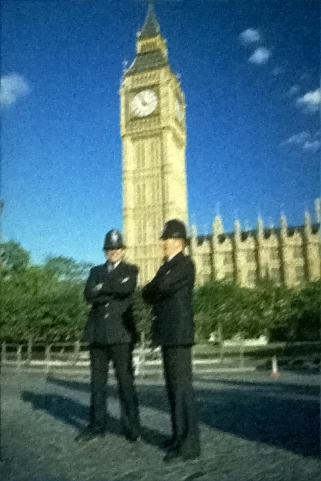}}
    \end{minipage} &
       \begin{minipage}[c]{0.16\textwidth} \centering
   \subfigure{\includegraphics[width =\textwidth]{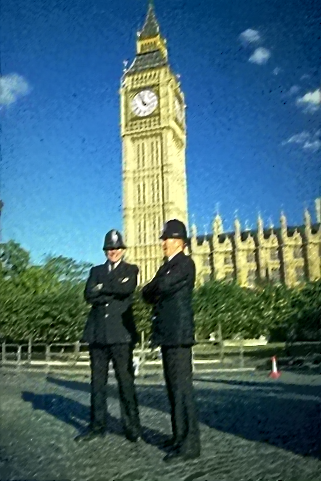}}
    \end{minipage} \\
        \begin{minipage}[c]{0.16\textwidth} \centering
   \subfigure{\includegraphics[width =\textwidth]{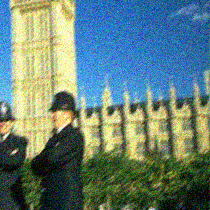}}
    \end{minipage} &
        \begin{minipage}[c]{0.16\textwidth} \centering
   \subfigure{\includegraphics[width =\textwidth]{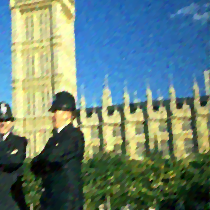}}
    \end{minipage} &
    \begin{minipage}[c]{0.17\textwidth} \centering
   \subfigure{\includegraphics[width = \textwidth]{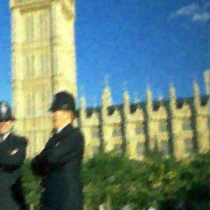}}
    \end{minipage} &
        \begin{minipage}[c]{0.16\textwidth} \centering
   \subfigure{\includegraphics[width = \textwidth]{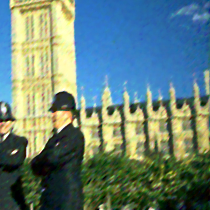}}
    \end{minipage} &
    \begin{minipage}[c]{0.16\textwidth} \centering
   \subfigure{\includegraphics[width =\textwidth]{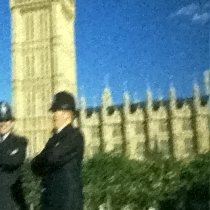}}
    \end{minipage} &
        \begin{minipage}[c]{0.16\textwidth} \centering
\subfigure{\includegraphics[width =\textwidth]{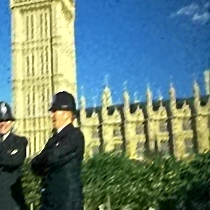}}
    \end{minipage} \\
        \begin{minipage}[c]{0.16\textwidth} \centering
\subfigure{\includegraphics[width=\textwidth]{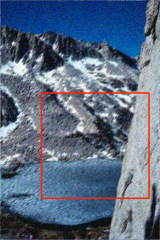}}
    \end{minipage} &
        \begin{minipage}[c]{0.16\textwidth} \centering
\subfigure{\includegraphics[width =\textwidth]{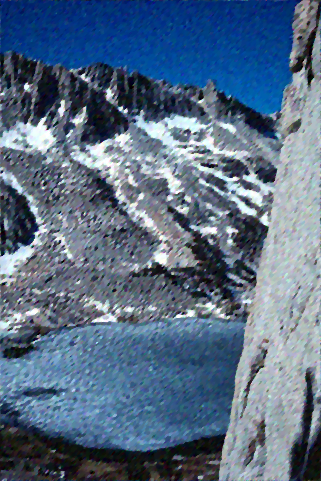}}
    \end{minipage} &
        \begin{minipage}[c]{0.16\textwidth} \centering
\subfigure{\includegraphics[width=\textwidth]{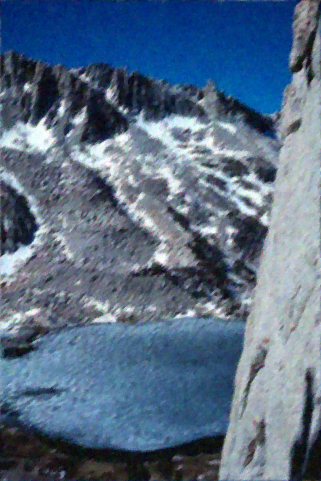}}
    \end{minipage} &
        \begin{minipage}[c]{0.16\textwidth} \centering
   \subfigure{\includegraphics[width =\textwidth]{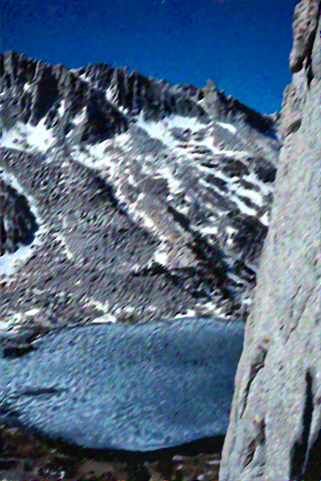}}
    \end{minipage} &
        \begin{minipage}[c]{0.16\textwidth} \centering
   \subfigure{\includegraphics[width =\textwidth]{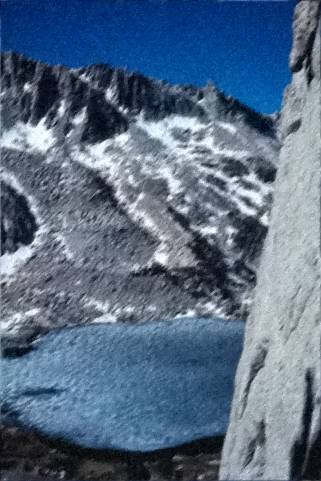}}
    \end{minipage} &
        \begin{minipage}[c]{0.16\textwidth} \centering
   \subfigure{\includegraphics[width =\textwidth]{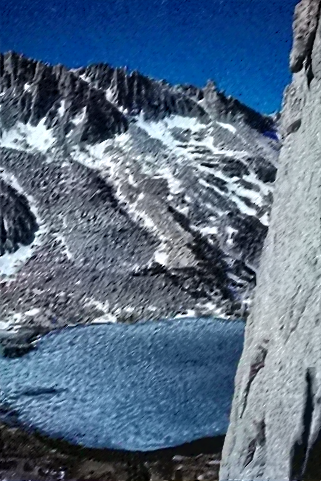}}
   \end{minipage} \\
           \begin{minipage}[c]{0.16\textwidth} \centering
   \subfigure{\includegraphics[width =\textwidth]{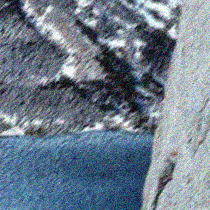}}
    \end{minipage} &
        \begin{minipage}[c]{0.16\textwidth} \centering
   \subfigure{\includegraphics[width =\textwidth]{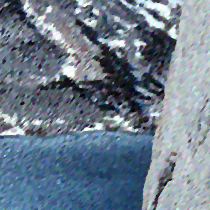}}
    \end{minipage} &
    \begin{minipage}[c]{0.16\textwidth} \centering
   \subfigure{\includegraphics[width = \textwidth]{figs/SCIELAB_ERROR/167083,motionblur/167083_CTV-L2_motion_zoom.png}}
    \end{minipage} &
        \begin{minipage}[c]{0.16\textwidth} \centering
   \subfigure{\includegraphics[width = \textwidth]{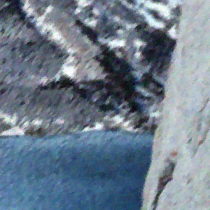}}
    \end{minipage} &
    \begin{minipage}[c]{0.16\textwidth} \centering
   \subfigure{\includegraphics[width =\textwidth]{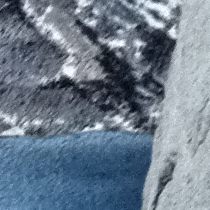}}
    \end{minipage} &
        \begin{minipage}[c]{0.16\textwidth} \centering
\subfigure{\includegraphics[width =\textwidth]{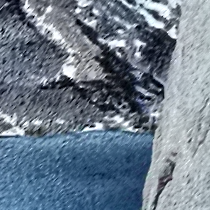}}
    \end{minipage} \\
           \begin{minipage}[c]{0.16\textwidth} \centering
\subfigure{\includegraphics[width=\textwidth]{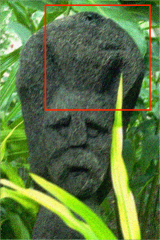}}
    \end{minipage} &
        \begin{minipage}[c]{0.16\textwidth} \centering
\subfigure{\includegraphics[width =\textwidth]{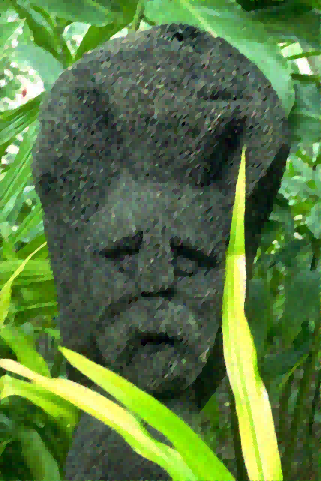}}
    \end{minipage} &
        \begin{minipage}[c]{0.16\textwidth} \centering
\subfigure{\includegraphics[width=\textwidth]{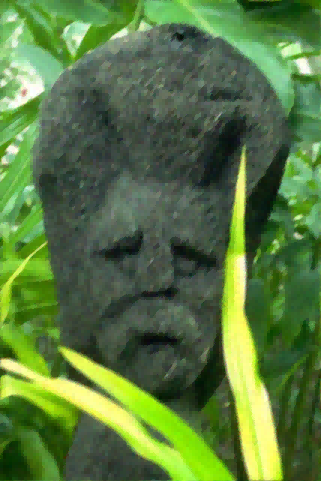}}
    \end{minipage} &
        \begin{minipage}[c]{0.16\textwidth} \centering
   \subfigure{\includegraphics[width =\textwidth]{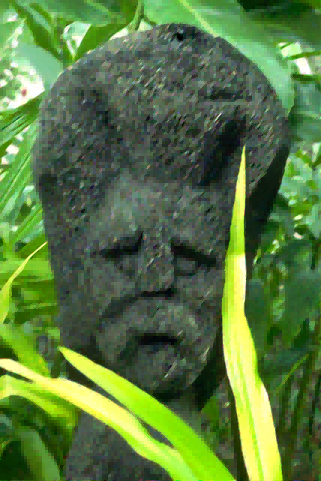}}
    \end{minipage} &
        \begin{minipage}[c]{0.16\textwidth} \centering
   \subfigure{\includegraphics[width =\textwidth]{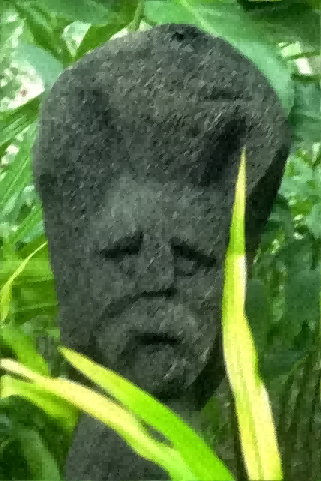}}
    \end{minipage} &
        \begin{minipage}[c]{0.16\textwidth} \centering
   \subfigure{\includegraphics[width =\textwidth]{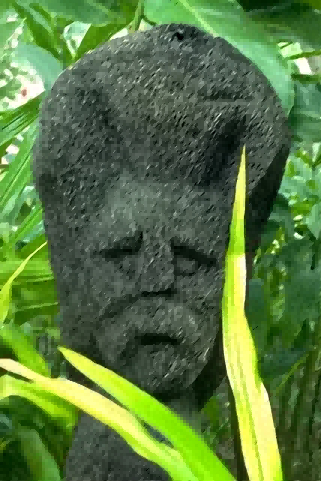}}
   \end{minipage} \\
           \begin{minipage}[c]{0.16\textwidth} \centering
   \subfigure{\includegraphics[width =\textwidth]{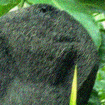}}
    \end{minipage} &
        \begin{minipage}[c]{0.16\textwidth} \centering
   \subfigure{\includegraphics[width =\textwidth]{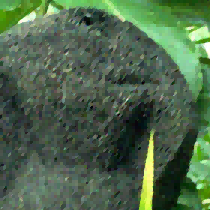}}
    \end{minipage} &
    \begin{minipage}[c]{0.16\textwidth} \centering
   \subfigure{\includegraphics[width = \textwidth]{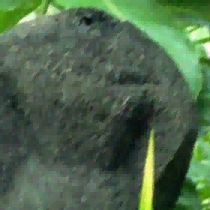}}
    \end{minipage} &
        \begin{minipage}[c]{0.16\textwidth} \centering
   \subfigure{\includegraphics[width = \textwidth]{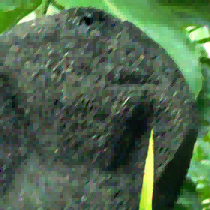}}
    \end{minipage} &
    \begin{minipage}[c]{0.16\textwidth} \centering
   \subfigure{\includegraphics[width =\textwidth]{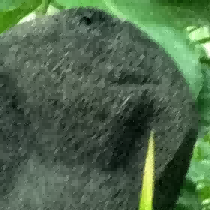}}
    \end{minipage} &
        \begin{minipage}[c]{0.16\textwidth} \centering
\subfigure{\includegraphics[width =\textwidth]{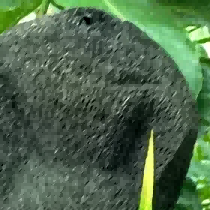}}
    \end{minipage}
\end{tabular}\\
\caption{From left to right: The degraded image, and the restored results by using CTV, GVTV, NLTV, SVTV, SVS-NLTV respectively. The corresponding zooming parts are also shown.}
\label{figmotionblur}
\end{figure}

\begin{figure}[htbp]
\centering
\begin{subfigure}
\centering
\includegraphics[width=0.3\textwidth]{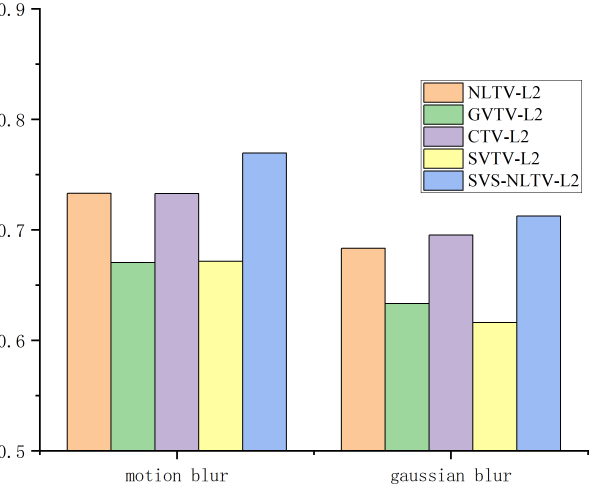}
\end{subfigure}
\begin{subfigure}
\centering
\includegraphics[width=0.3\textwidth]{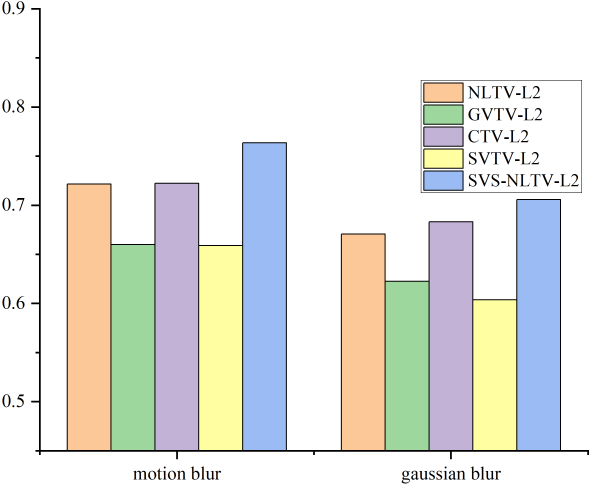}
\end{subfigure}
\begin{subfigure}
\centering
\includegraphics[width=0.3\textwidth]{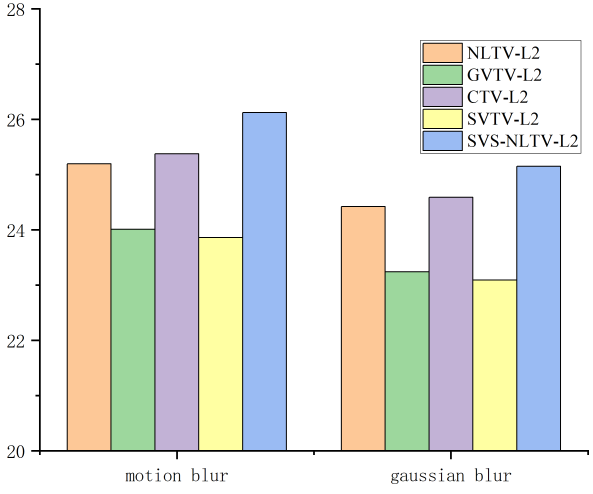}
\end{subfigure}
\caption{The histograms of QSSIM, SSIM, PSNR values of 60 images with Gaussian and Motion blur with Gaussian noise d = 20/255.}
\label{figaveragehistogramblur}
\end{figure}

In this section, we test the performance of the proposed $\stvct$ model for image deblurring problem. For the blurring kernel, we consider the Gaussian kernel of standard deviation 1.5 and the Motion kernel of motion length 3 and motion angle 45. We add Gaussian noise of standard deviation 20/255 in each channel to further degrade the blurred pictures and obtain the degraded testing images. For comparison, we consider CTV-L2\cite{bresson2008fast}, GVTV-L2\cite{rodriguez2009generalized}, SVTV-L2\cite{jia2019color}, NLTV-L2 \cite{Xiaoqun2010Bregmanized} and the proposed $\stvct$. For the proposed $\stvct$ model, we set the parameter of the value channel to be $\mu$ = 0.05, the parameters $\lambda$, $\delta$ in Bregman iteration to be $\lambda =1$, $\delta = 1$. For the regularization parameter $\alpha$, we set a range of [$\frac{\sqrt{N}}{1000}$, $\frac{\sqrt{N}}{10}$] with a step size of 0.01. For the parameter of the regularization parameter($\lambda$), we set a range of [$\frac{\sqrt{N}}{1000}$, $\frac{\sqrt{N}}{10}$] with a step size of 0.01 for both $\stvc$ and NLTV. The regularization parameter($\lambda$) range for CTV model is set to be [$\frac{\sqrt{N}}{1000}$, $\frac{\sqrt{N}}{10}$]. For SVTV and GVTV model, we set a range of [$\frac{\sqrt{N}}{1000}$, $\frac{\sqrt{N}}{10}$].

We compute the PSNR, SSIM, QSSIM values and the S-CIELAB error value (pixel number) for each restored result which is corresponding to the optimal regularization parameter in terms of PSNR value. In Figures \ref{figspatialgaussianblur} and \ref{figspatialmotionblur}, we display the spatial distributions of PSNR, SSIM, and QSSIM values of the restored results by using CTV, GVTV, NLTV, SVTV, SVS-NLTV for 60 testing images. We also show the histograms of the average PSNR, SSIM, and QSSIM values in \ref{figaveragehistogramblur}. We observe from the figures that the proposed SVS-NLTV-L2 model provides competitive individual values and best average values of PSNR, SSIM, and QSSIM.

As examples, we show 6 sets of restored results and the corresponding zoom-in parts in Figures \ref{figgaussianblur} and \ref{figmotionblur} respectively. We see from the results that some detailed geometries and textures can be well recovered, and the noise/color artifacts are effectively eliminated at the mean time by using the proposed $\stvc$ model. Combining the above mentioned results, we emphasize that $\stvc$ performs better than other testing methods. In Table \ref{table5.1} and \ref{table5.2}, we report the measure values of the restored results in Figures \ref{figgaussianblur} and \ref{figmotionblur}. The best values are marked in bold for presentation. We see that the proposed model has the best results among all the testing methods for the testing images.

\begin{figure}[htbp]
\centering
\tabcolsep=1pt
\begin{tabular}{cccccc}
    \begin{minipage}[c]{0.16\textwidth} \centering
   \subfigure{\includegraphics[width =\textwidth]{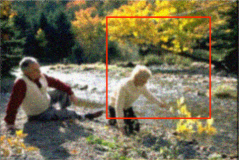}}
    \end{minipage} &
        \begin{minipage}[c]{0.16\textwidth} \centering
   \subfigure{\includegraphics[width =\textwidth]{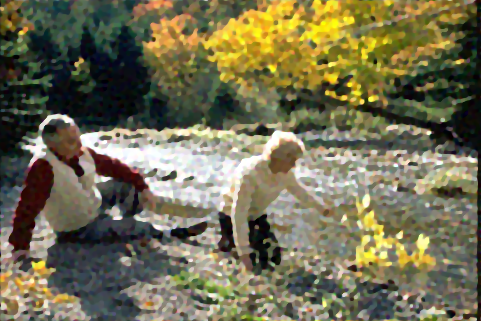}}
    \end{minipage} &
        \begin{minipage}[c]{0.16\textwidth} \centering
   \subfigure{\includegraphics[width =\textwidth]{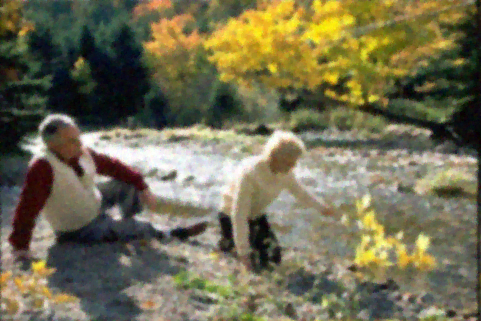}}
   \end{minipage}  &
    \begin{minipage}[c]{0.16\textwidth} \centering
   \subfigure{\includegraphics[width = \textwidth]{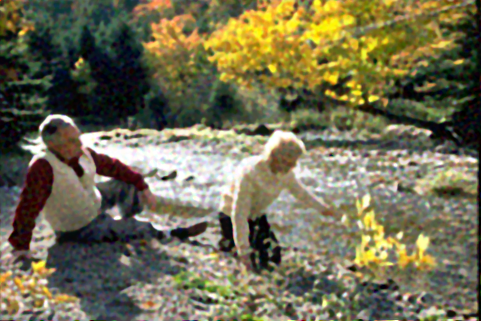}}
    \end{minipage} &
        \begin{minipage}[c]{0.16\textwidth} \centering
   \subfigure{\includegraphics[width = \textwidth]{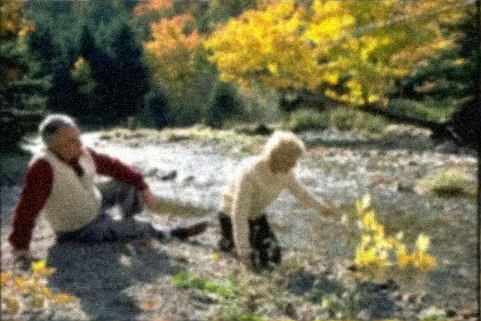}}
    \end{minipage} &
       \begin{minipage}[c]{0.16\textwidth} \centering
   \subfigure{\includegraphics[width =\textwidth]{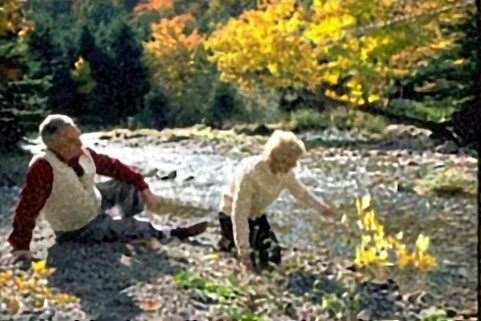}}
    \end{minipage} \\
        \begin{minipage}[c]{0.16\textwidth} \centering
   \subfigure{\includegraphics[width =\textwidth]{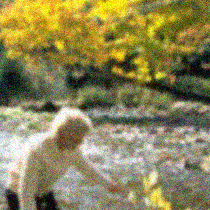}}
    \end{minipage} &
        \begin{minipage}[c]{0.16\textwidth} \centering
   \subfigure{\includegraphics[width =\textwidth]{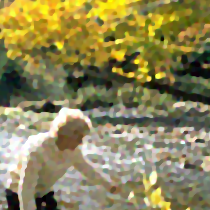}}
    \end{minipage} &
    \begin{minipage}[c]{0.16\textwidth} \centering
   \subfigure{\includegraphics[width = \textwidth]{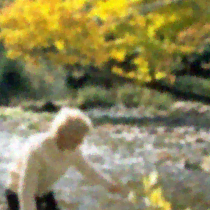}}
    \end{minipage} &
        \begin{minipage}[c]{0.16\textwidth} \centering
   \subfigure{\includegraphics[width = \textwidth]{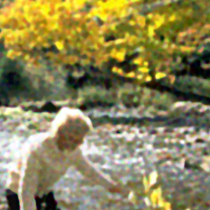}}
    \end{minipage} &
    \begin{minipage}[c]{0.16\textwidth} \centering
   \subfigure{\includegraphics[width =\textwidth]{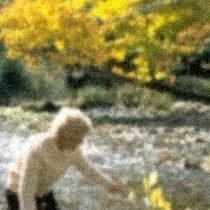}}
    \end{minipage} &
        \begin{minipage}[c]{0.16\textwidth} \centering
\subfigure{\includegraphics[width =\textwidth]{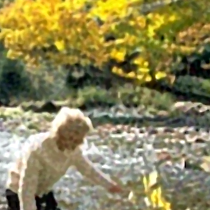}}
    \end{minipage} \\
        \begin{minipage}[c]{0.16\textwidth} \centering
\subfigure{\includegraphics[width=\textwidth]{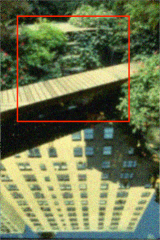}}
    \end{minipage} &
        \begin{minipage}[c]{0.16\textwidth} \centering
\subfigure{\includegraphics[width =\textwidth]{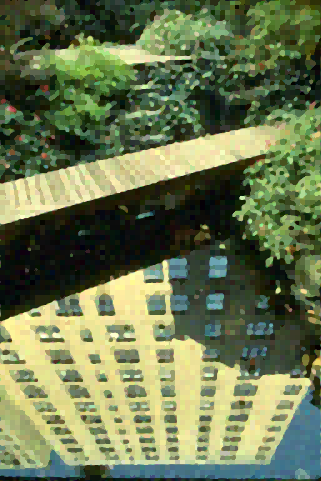}}
    \end{minipage} &
        \begin{minipage}[c]{0.16\textwidth} \centering
\subfigure{\includegraphics[width=\textwidth]{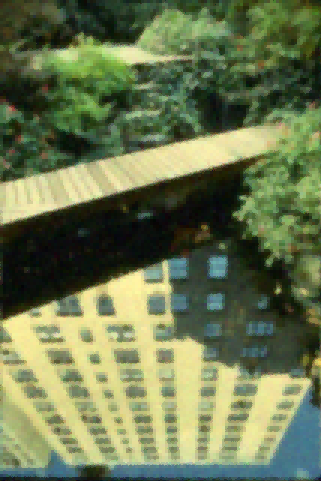}}
    \end{minipage} &
        \begin{minipage}[c]{0.16\textwidth} \centering
   \subfigure{\includegraphics[width =\textwidth]{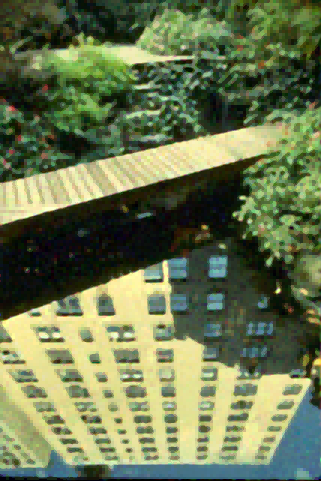}}
    \end{minipage} &
        \begin{minipage}[c]{0.16\textwidth} \centering
   \subfigure{\includegraphics[width =\textwidth]{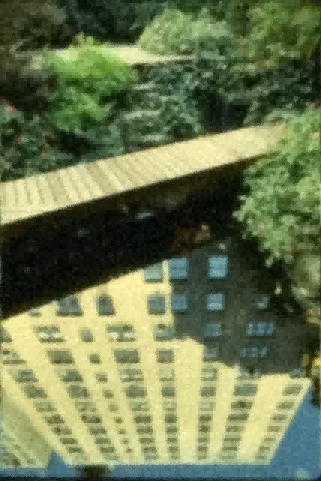}}
    \end{minipage} &
        \begin{minipage}[c]{0.16\textwidth} \centering
   \subfigure{\includegraphics[width =\textwidth]{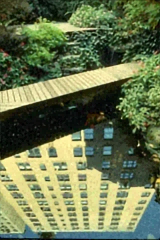}}
   \end{minipage} \\
           \begin{minipage}[c]{0.16\textwidth} \centering
   \subfigure{\includegraphics[width =\textwidth]{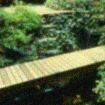}}
    \end{minipage} &
        \begin{minipage}[c]{0.16\textwidth} \centering
   \subfigure{\includegraphics[width =\textwidth]{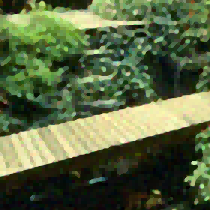}}
    \end{minipage} &
    \begin{minipage}[c]{0.16\textwidth} \centering
   \subfigure{\includegraphics[width = \textwidth]{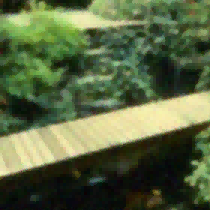}}
    \end{minipage} &
        \begin{minipage}[c]{0.16\textwidth} \centering
   \subfigure{\includegraphics[width = \textwidth]{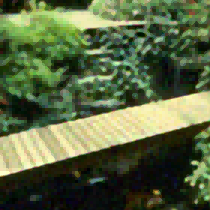}}
    \end{minipage} &
    \begin{minipage}[c]{0.16\textwidth} \centering
   \subfigure{\includegraphics[width =\textwidth]{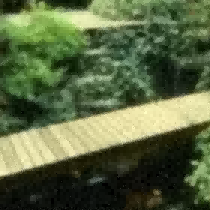}}
    \end{minipage} &
        \begin{minipage}[c]{0.16\textwidth} \centering
\subfigure{\includegraphics[width =\textwidth]{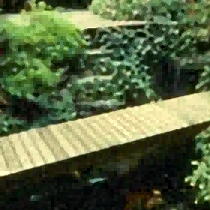}}
    \end{minipage} \\
           \begin{minipage}[c]{0.16\textwidth} \centering
\subfigure{\includegraphics[width=\textwidth]{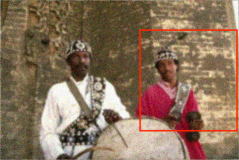}}
    \end{minipage} &
        \begin{minipage}[c]{0.16\textwidth} \centering
\subfigure{\includegraphics[width =\textwidth]{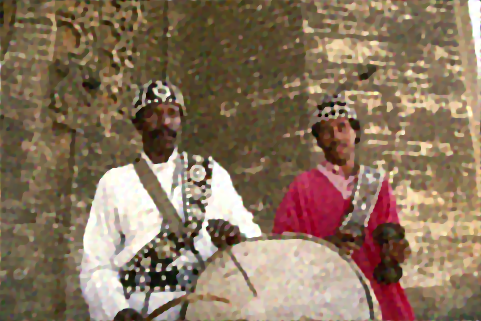}}
    \end{minipage} &
        \begin{minipage}[c]{0.16\textwidth} \centering
\subfigure{\includegraphics[width=\textwidth]{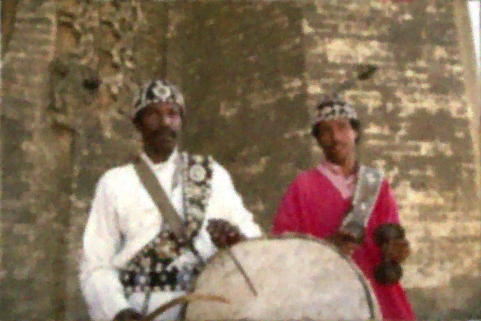}}
    \end{minipage} &
        \begin{minipage}[c]{0.16\textwidth} \centering
   \subfigure{\includegraphics[width =\textwidth]{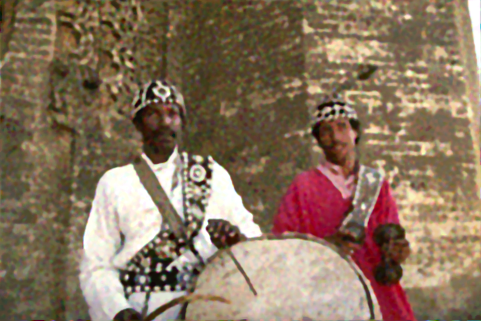}}
    \end{minipage} &
        \begin{minipage}[c]{0.16\textwidth} \centering
   \subfigure{\includegraphics[width =\textwidth]{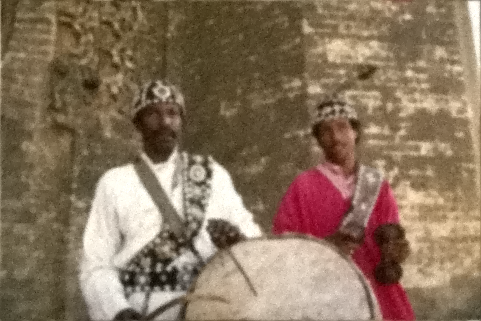}}
    \end{minipage} &
        \begin{minipage}[c]{0.16\textwidth} \centering
   \subfigure{\includegraphics[width =\textwidth]{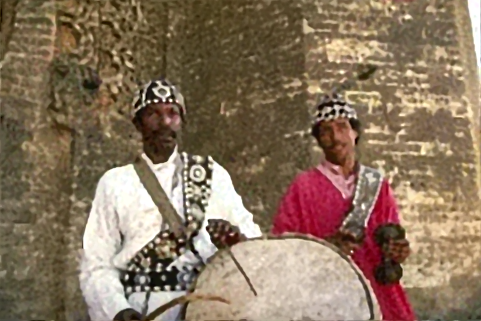}}
   \end{minipage} \\
           \begin{minipage}[c]{0.16\textwidth} \centering
   \subfigure{\includegraphics[width =\textwidth]{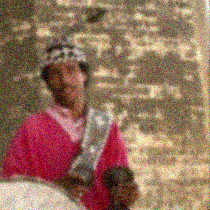}}
    \end{minipage} &
        \begin{minipage}[c]{0.16\textwidth} \centering
   \subfigure{\includegraphics[width =\textwidth]{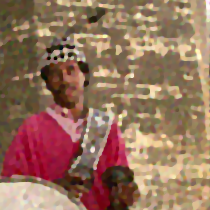}}
    \end{minipage} &
    \begin{minipage}[c]{0.16\textwidth} \centering
   \subfigure{\includegraphics[width = \textwidth]{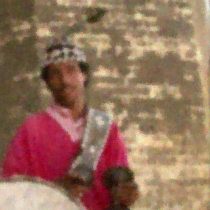}}
    \end{minipage} &
        \begin{minipage}[c]{0.16\textwidth} \centering
   \subfigure{\includegraphics[width = \textwidth]{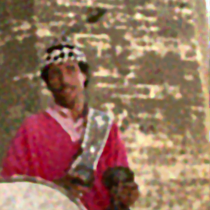}}
    \end{minipage} &
    \begin{minipage}[c]{0.16\textwidth} \centering
   \subfigure{\includegraphics[width =\textwidth]{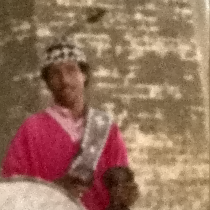}}
    \end{minipage} &
        \begin{minipage}[c]{0.16\textwidth} \centering
\subfigure{\includegraphics[width =\textwidth]{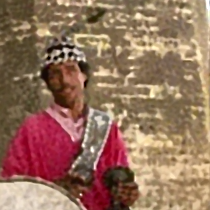}}
    \end{minipage}
\end{tabular}\\
\caption{From left to right: The degraded image, and the restored results by using CTV, GVTV, NLTV, SVTV, SVS-NLTV respectively. The corresponding zooming parts are also shown.}
\label{figgaussianblur}
\end{figure}

\section{Conclusion}
As a summary, we propose and develop a new nonlocal variational technique based on saturation-value similarity for color image restoration. By considering saturation-value similarity of color image patches, two types of total variation functions are studied based on the nonlocal gradient in saturation-value space. The contribution of this paper is twofold. First, we establish saturation-value similarity based nonlocal total variation by incorporating nonlocal method into saturation-value space of color images. We then formulate the proposed color image restoration models by considering L2 fidelity and L1 fidelity to handle different types of noise, e.g., Gaussian noise, Poisson noise, etc. Second, we design an effective and efficient algorithm to solve the proposed optimization problem numerically by employing bregmanized operator splitting method numerically. We also study the convergence of the proposed algorithm. Numerical examples are presented to demonstrate the effectiveness of the proposed models and the efficiency of the numerical scheme, and shows the performance of the proposed model is better than other testing methods.

\section{Appendix}\label{appendix}

\subsection{Proof of proposition \ref{prole}}\hspace{10cm}

\begin{proof}
We take $\mathbf{u}_v$ as example, for fixed \(\phi\in C^1_c(\Omega\times\Omega)^2\) satisfying \(\|\phi\|_\infty\le1\), by using \ref{pro2.1} we have
\begin{gather*}
\int_\Omega \mathbf{u}_v \,\mathrm{div}_{\omega}^v \phi\,\mathrm d x
= \lim_{n\to\infty}\int_\Omega \mathbf{u}_v^n \,\mathrm{div}_{\omega}^v \phi\,\mathrm d x
\;\le\;
\liminf_{n\to\infty}\int_\Omega |\nabla_{\omega}^v \mathbf{u}_v^n|\mathrm d x.
\end{gather*}
Taking the supremum over all such \(\phi\) yields the stated inequality.
\end{proof}

\subsection{Proof of proposition \ref{proapp}}\hspace{10cm}

\begin{proof}
We also take $\mathbf{u}_v$ as example. First, we remark that it's easy to deduce that we can choose \(\mathbf{u}^\varepsilon\in W^{1, 1}(\Omega)\cap C^\infty(\Omega)\) satisfying, 
\begin{gather*}
\int_\Omega|\mathbf{u}^\varepsilon - \mathbf{u}|\,\mathrm d x \;\to\;0
\quad(\varepsilon\to0).
\end{gather*}
First, $\omega_v$ clearly has the upper bounded M, with Jensen inequality and absolute value inequality we have
\begin{gather*}
\begin{aligned}
&\int_{\Omega} |\nabla^v_{\omega} (\mathbf{u}^{\epsilon}-\mathbf{u})| \mathrm d x\\
=&\int_{\Omega} \sqrt{ \int_{\Omega} \bigl((\mathbf{u}^{\epsilon}-\mathbf{u})(x)-(\mathbf{u}^{\epsilon}-\mathbf{u})(y)\bigr)^2 \omega_v  \mathrm d y  } \mathrm d x\\
\le & \sqrt{M} \int_{\Omega} \int_{\Omega} |\mathbf{u}^{\epsilon}-\mathbf{u})(x)|+|(\mathbf{u}^{\epsilon}-\mathbf{u})(y)| \mathrm d y  \mathrm d x\\
= & 2 \sqrt{M} C(\Omega) \int_{\Omega} |\mathbf{u}^{\epsilon}-\mathbf{u}| \mathrm d x,\\
\end{aligned}
\end{gather*}
since $\mathbf{u}^{\epsilon} \xrightarrow{L^1(\Omega)} \mathbf{u}$, any $\delta > 0$, we can find $\epsilon_0 > 0$, such that if $\epsilon < \epsilon_0$,
\begin{gather*}
\delta > \int_{\Omega} |\nabla^v_{\omega} (\mathbf{u}^{\epsilon}-\mathbf{u})| \mathrm d x \geq \int_{\Omega} | (\mathbf{u}^{\epsilon}_v-\mathbf{u}_v) div^v_{\omega} \phi | \mathrm d x,
\end{gather*}
thus we have
\begin{gather*}
\delta+\int_{\Omega} \mathbf{u}_v div^v_{\omega} \phi \mathrm d x \ge \int_{\Omega} \mathbf{u}^{\epsilon}_v) div^v_{\omega} \phi \mathrm d x.
\end{gather*}
Then take the supremum of both sides, by the arbitrariness of $\delta$ , for $\epsilon \le \epsilon_0$, we have 
\begin{equation}
\int_{\Omega} |\nabla^v_{\omega} \mathbf{u}| \mathrm d x\geq \int_{\Omega} |\nabla^v_{\omega} \mathbf{u}^{\epsilon}|\mathrm d x.  
\label{eq2.3.2}
\end{equation}
On the other hand, by using lower semi-continuity \ref{prole}, we have
\begin{equation}
\liminf_{\varepsilon\to0}\int_\Omega|\nabla_{\omega}^v \mathbf{u}^\varepsilon|\mathrm d x
\;\ge\;
\int_\Omega|\nabla_{\omega}^v \mathbf{u}|\mathrm d x. 
\label{eq2.3.1}
\end{equation}
Combining these two inequalities \ref{eq2.3.1} and \ref{eq2.3.2} yields
\begin{gather*}
\int_\Omega|\nabla_{\omega}^v \mathbf{u}|\mathrm d x
\;\le\;
\liminf_{\varepsilon\to0}\int_\Omega|\nabla_{\omega}^v \mathbf{u}^\varepsilon|\mathrm d x
\;\le\;
\limsup_{\varepsilon\to0}\int_\Omega|\nabla_{\omega}^v \mathbf{u}^\varepsilon|\mathrm d x
\;\le\;
\int_\Omega|\nabla_{\omega}^v \mathbf{u}|\mathrm d x.
\end{gather*}
Thus we have
\begin{gather*}
\lim_{\varepsilon\to0}\int_\Omega|\nabla_{\omega}^v \mathbf{u}^\varepsilon|\mathrm d x
=\int_\Omega|\nabla_{\omega}^v \mathbf{u}|\mathrm d x,
\end{gather*}
as claimed.
\end{proof}

\subsection{Proof of proposition {\ref{procom}}}\hspace{10cm}

\begin{lemma}\label{lemma2.5}
Let $\Omega\subset\mathbb R^n$ be bounded, if $\mathbf{u} \in W^{1,\infty}(\Omega)$, then for $ x\ a.e. \in\Omega$,  there exists $r_0>0$ and a constant $C(n) > 0$ such that 
\begin{gather*}
\int_{B_{r_0}(x)}\bigl|\mathbf{u}(y)-\mathbf{u}(x)\bigr|\,\mathrm d y
\;\ge\;
C(n)\,\bigl|\nabla \mathbf{u}(x)\bigr|.
\end{gather*}
\end{lemma}

\begin{proof}
Since $\mathbf{u}\in W^{1, 1} $, we take the first order Taylor expansion around $x$,
\begin{gather*}
\mathbf{u}(y)-\mathbf{u}(x)
=\nabla \mathbf{u}(x)\cdot(y-x)+R_x(y),
\end{gather*}
where the remainder $R_x(y)$ satisfies
\begin{gather*}
\bigl|R_x(y)\bigr|\le\varepsilon(r)\,|y-x|,
\quad
\varepsilon(r)\to0\quad(r\to0).
\end{gather*}
Hence 
\begin{gather*}
\int_{B_r(x)}\bigl|\mathbf{u}(y)-\mathbf{u}(x)\bigr|\,\mathrm d y
\;\ge\;
\int_{B_r(x)}\bigl|\nabla \mathbf{u}(x)\cdot(y-x)\bigr|\,\mathrm d y
\;-\;
\int_{B_r(x)}\bigl|R_x(y)\bigr|\,\mathrm d y.
\end{gather*}
\medskip
Let $\nabla \mathbf{u}(x)=\beta\,e$, $\beta =|\nabla \mathbf{u}(x)|$, $e\in S^{n-1}$.  Then
\begin{gather*}
\int_{B_r(x)}\bigl|\nabla \mathbf{u}(x)\cdot(y-x)\bigr|\,\mathrm d y
=\beta \int_{|h|\le r}|e\cdot h|\,\mathrm d h.
\end{gather*}
In polar coordinates $h=r'\theta$, $\theta\in S^{n-1}$, $\mathrm d h=(r')^{n-1}\mathrm d r'\,\mathrm d \theta$, and a standard identity
\(\int_{S^{n-1}}|\theta\cdot e|\,\mathrm d \theta
=\tfrac{2 \omega_{n-2}}{n-1}\)
gives ( Here $\omega_{n-1}=\mathrm{Area}(S^{n-1})$.) 
\begin{gather*}
\int_{|h|\le r}|e\cdot h|\,\mathrm d h
=\int_0^r(r')^n\,\mathrm d r'\,
\int_{S^{n-1}}|\theta\cdot e|\,\mathrm d\theta
=\frac{2\,\omega_{n-2}}{(n-1)(n+1)}\,r^{n+1}.
\end{gather*}
Set $C_0=\tfrac{2\,\omega_{n-2}}{(n-1)(n+1)}$, thus
\begin{gather*}
\int_{B_r(x)}\bigl|\nabla \mathbf{u}(x)\cdot(y-x)\bigr|\,\mathrm d y
=C_0 \beta r^{n+1}.
\end{gather*}

\begin{gather*}
\int_{B_r(x)}\bigl|R_x(y)\bigr|\,\mathrm d y
\le\varepsilon(r)\int_{B_r(x)}|y-x|\,\mathrm d y
=\varepsilon(r)\,\frac{\omega_{n-1}}{n+1}\,r^{n+1}.
\end{gather*}
Set $D_0=\tfrac{\omega_{n-1}}{n+1}$, we have 
\begin{gather*}
\int_{B_r(x)}\bigl|\mathbf{u}(y)-\mathbf{u}(x)\bigr|\,\mathrm d y
\ge r^{n+1}\bigl(C_0 - \frac{D_0}{\beta}\,\varepsilon(r)\bigr) \beta.
\end{gather*}
Since $\varepsilon(r)\to0$, choose $r_0>0$ so small that for all $r\le r_0$, $\varepsilon(r)\le \tfrac12\,\frac{C_0 }{D_0 \beta} $, then for $r\le r_0$,
\begin{gather*}
\int_{B_r(x)}\bigl|\mathbf{u}(y)-\mathbf{u}(x)\bigr|\,\mathrm d y
\ge \frac{C_0}{2}\,r^{n+1}\beta
=\frac{\omega_{n-2}}{(n-1)(n+1)}\,r^{n+1}\,\bigl|\nabla \mathbf{u}(x)\bigr|
\end{gather*}
For every $0<r\le r_0$ and every $x\in\Omega$, we have
\begin{gather*}
\int_{B_r(x)}\bigl|\mathbf{u}(y)-\mathbf{u}(x)\bigr|\,\mathrm d y
\;\ge\;
\tfrac12\,C_0\,r^{n+1}\,\bigl|\nabla \mathbf{u}(x)\bigr|.
\end{gather*}
In particular, noting that $\mathbf{u} \in W^{1, \infty}$, setting $\| \nabla u(x) \|_{\infty} =L$, by taking $r_0 = \frac{C_0}{D_0 L}$ 
\begin{gather*}
C(n) \;=\;\tfrac12\,C_0\,r_0^{\,n+1}
\;=\;\frac{\omega_{n-2}}{(n-1)\omega_{n-1}}\,r_0^{\,n+1},
\end{gather*}
yields
\begin{gather*}
\int_{B_{r_0}(x)}\bigl|\mathbf{u}(y)-\mathbf{u}(x)\bigr|\,\mathrm d y
\;\ge\;
C(n) \,\bigl|\nabla \mathbf{u}(x)\bigr|,
\quad\forall\ x\in\Omega.
\end{gather*}
as claimed.
\end{proof}

We finally give the proof of proposition \ref{procom}.

\begin{proof}
Assume that there exists $M>0$ such that
\begin{gather*}
\|\mathbf{u}^n \|_{L^1(\Omega)} + \int_\Omega |\nabla_{\omega}^s \mathbf{u}^n|\mathrm d x + \int_\Omega |\nabla_{\omega}^v \mathbf{u}^n|\mathrm d x\;\le\;M
\quad\forall\,n.
\end{gather*}
By using Proposition \ref{proapp}, for each $n\in\mathbb N$, we can choose a smooth approximation $\mathbf{p}^n\in  W^{1, \infty}(\Omega) \cap C^\infty(\Omega)$ such that
\begin{gather*}
\|\mathbf{u}^n - \mathbf{p}^n\|_{L^1(\Omega)} \le \tfrac1n,
\quad
\int_\Omega |\nabla_{\omega} \mathbf{p}^n|\mathrm d x \le \int_\Omega |\nabla_{\omega} \mathbf{u}^n|\mathrm d x + \tfrac1n.
\end{gather*}
Noting that $\{\mathbf{u}^n\}$ is uniformly bounded in $\sbvc (\Omega)$, we derive the following inequalities by using Jensen inequality,
\begin{gather*}
\begin{aligned}
&\int_{\Omega} |\nabla_{\omega}^s \mathbf{u}^n| \mathrm d x\,+\mu \int_{\Omega} |\nabla_{\omega}^v \mathbf{u}^n| \mathrm d x\, +\frac{1+\mu}{n}\\
\ge & \int_{\Omega} |\nabla_{\omega}^s \mathbf{p}^n| \mathrm d x\,+\mu \int_{\Omega} |\nabla_{\omega}^v \mathbf{p}^n| \mathrm d x\,\\
\ge &N \Bigl( \int_{\Omega} \sqrt{\int_{\Omega}  \bigl(\mathbf{p}^n_s(x)-\mathbf{p}^n_s(y) \bigr)^2  \mathrm d y\,} \mathrm d x\,+\mu \int_{\Omega} \sqrt{\int_{\Omega}  \bigl(\mathbf{p}^n_v(x)-\mathbf{p}^n_v(y) \bigr)^2  \mathrm d y\,} \mathrm d x\, \Bigr)\\
\end{aligned}
\end{gather*}
By using Lemma \ref{lemma2.5}, we know that for a zero measure set E in $\Omega$, the integral on E is 0, thus we can transform the above inequality into the following inequality
\begin{gather*}
\begin{aligned}
 &N  \big( \int_{\Omega/E}\int_{r_x}\bigl|\mathbf{p}^n_s(x)-\mathbf{p}^n_s(y)\bigr|\,\mathrm d y\,\mathrm d x\, +\mu \int_{\Omega/E}\int_{r_x}\bigl|\mathbf{p}^n_v(x)-\mathbf{p}^n_v(y)\bigr|\,\mathrm d y\,\mathrm d x \big)\,\\
\ge &\mu N C \int_{\Omega/E} \bigl|\nabla \mathbf{p}^n_s(x) \bigr| + \bigl|\nabla \mathbf{p}^n_v(x) \bigr| \mathrm d x\\
\ge &\mu N C \int_{\Omega/E} \bigl|\nabla \mathbf{p}^n(x) \bigr| \mathrm d x\, =\mu N C \int_{\Omega} \bigl|\nabla \mathbf{p}^n(x) \bigr| \mathrm d x \\
\end{aligned}
\end{gather*}
We then obtain
\begin{gather*}
\begin{aligned}
\int_\Omega |\nabla \mathbf{p}^n|\,dx  + \|\mathbf{p}^n\|_{L^1}
\;& \le\;
\mu N C \Bigl(\int_\Omega |\nabla_{\omega}^s \mathbf{p}^n| + \int_\Omega |\nabla_{\omega}^v \mathbf{p}^n| +  \|\mathbf{p}^n\|_{L^1(\Omega)}\Bigr) \\
&\le \mu M N C+\frac{1+\mu}{n}.
\end{aligned}
\end{gather*}
Hence $\{\mathbf{p}^n\}$ (and therefore $\{\mathbf{u}^n\}$) is uniformly bounded in $\bv (\Omega)$. By using the embedding property of $\bv (\Omega)$, we can extract a subsequence (still denoted as $\mathbf{u}^n$) and find $\mathbf{u} \in \bv (\Omega)$ such that
\begin{gather*}
\mathbf{u}^n \longrightarrow \mathbf{u}
\quad\text{in }L^1(\Omega).
\end{gather*}
Lower semicontinuity \ref{prole} then gives
\begin{gather*}
\int_\Omega |\nabla_{\omega}^s \mathbf{u}|\;\le\;\liminf_{n\to\infty}\int_\Omega |\nabla_{\omega}^s \mathbf{u}^n|, \quad \int_\Omega |\nabla_{\omega}^v \mathbf{u}|\;\le\;\liminf_{n\to\infty}\int_\Omega |\nabla_{\omega}^v \mathbf{u}^n|,
\end{gather*}
which shows $\mathbf{u}\in \sbvc (\Omega)$ and completes the proof.

\end{proof}

	\bibliographystyle{siam}
	\bibliography{ref.bib}

\end{document}